\documentclass[11pt]{article}
\usepackage[utf8]{inputenc}
\usepackage[T1]{fontenc}
\usepackage{amsthm, amsmath, mathrsfs, mathtools}
\usepackage{amsfonts}
\usepackage[margin=1in]{geometry}
\usepackage[algo2e,ruled,noend,resetcount,linesnumbered]{algorithm2e}
\usepackage{graphicx}
\usepackage{comment}
\usepackage{enumitem}
\usepackage{thmtools}
\usepackage{thm-restate}
\usepackage{algorithm, algorithmicx, algpseudocode} 
\usepackage[draft,margin=false,inline=true]{fixme}
\fxsetup{mode=multiuser,theme=color, layout=inline}

\setlist[enumerate]{nosep, topsep=1ex}
\setlist[itemize]{nosep, topsep=1ex}
\setlist[description]{nosep}
\allowdisplaybreaks
\usepackage[colorlinks, backref=page]{hyperref}
\usepackage{cleveref}
\usepackage{xspace}

\usepackage{caption}
\usepackage{subcaption}

\usepackage{algpseudocode}

\newcommand{\chris}[1]{\textcolor{blue}{[CY: #1]}}
\newcommand{\maxh}[1]{\textcolor{red}{[MH: #1]}}
\usepackage{xcolor}

\newcommand{\proofparagraph}[1]{\item \paragraph{#1}\mbox{}\\}

\newcommand{\alg}[1]{\textsc{\bfseries \footnotesize #1}}
\newcommand{\innerAlg}{\mathcal{A}}
\newcommand{\outerAlg}{\mathcal{A}'}

\newcommand{\bigO}[1]{O \left( #1 \right)}
\newcommand{\littleO}[1]{o \left( #1 \right)}
\newcommand{\bigTh}[1]{\Theta \left( #1 \right)}
\newcommand{\tO}[1]{\tilde{O} (#1)}
\newcommand{\bigtO}[1]{\tilde{O} \left( #1 \right)}
\newcommand{\bigOm}[1]{\Omega \left( #1 \right)}

\newcommand{\poly}[1]{\mathrm{poly} \left( #1 \right)}

\newcommand{\E}[1]{\mathbb{E}\left[#1\right]}
\newcommand{\Ep}{\mathbb{E}}

\newcommand{\BinomD}[2]{\mathrm{Binom}\left(#1, #2\right)}
\newcommand{\UnifD}[1]{\mathrm{Unif} \left(#1\right)}

\newcommand{\BernD}[1]{\mathrm{Bern} \left(#1\right)}
\newcommand{\iid}{{i.i.d.}\ }
\newcommand{\ind}[1]{\mathbf{1}[#1]}
\newcommand{\tvd}[1]{d_{TV} \left(#1\right)}
\newcommand{\supp}{\mathrm{supp}}

\newtheorem{theorem}{Theorem}[section]
\newtheorem{corollary}[theorem]{Corollary}
\newtheorem{claim}[theorem]{Claim}
\newtheorem{definition}[theorem]{Definition}
\newtheorem{lemma}[theorem]{Lemma}
\newtheorem{proposition}[theorem]{Proposition}
\newtheorem{remark}[theorem]{Remark}

\newcommand{\N}{\mathbb{N}}
\newcommand{\Z}{\mathbb{Z}}
\newcommand{\R}{\mathbb{R}}
\newcommand{\Q}{\mathbb{Q}}
\newcommand{\family}{\mathcal{F}}
\newcommand{\distribution}{\mathcal{D}}
\newcommand{\domain}{\mathcal{X}}
\newcommand{\range}{\mathcal{Y}}
\newcommand{\hypotheses}{\mathcal{H}}
\newcommand{\calO}{\mathcal{O}}
\newcommand{\partition}{\mathcal{P}}
\newcommand{\gaussian}{\mathcal{N}}

\newcommand{\sign}{\text{sgn}}
\newcommand{\snorm}[2]{||#2||_{#1}}
\newcommand{\eps}{\varepsilon}

\newcommand{\vol}{\mathrm{vol}_{N}}
\newcommand{\area}{\mathrm{vol}_{N - 1}}
\newcommand{\volF}[1]{\mathrm{vol}_{#1}}
\newcommand{\closure}{\mathrm{cl}}
\newcommand{\interior}{\mathrm{int}}
\newcommand{\exterior}{\mathrm{ext}}

\newcommand{\boundary}{\partial}

\newcommand{\given}{\textrm{\xspace s.t.\ \xspace}}

\newcommand{\orT}{\textrm{\xspace or \xspace}}

\newcommand{\otherwise}{\textrm{\xspace o/w \xspace}}
\newcommand{\ceil}[1]{\left\lceil #1 \right\rceil}

\newcommand{\set}[1]{\{#1\}}
\newcommand{\bigset}[1]{\left\{#1\right\}}

\newcommand{\norm}[2]{\left|\left|#1\right|\right|_{#2}}

\newcommand{\lp}{\left}
\newcommand{\rp}{\right}
\newcommand{\abs}[1]{\lp| #1 \rp|}

\newcommand{\snew}[1]{\textcolor{blue}{#1}}

\newcommand{\SO}{\mathbf{SO}}
\newcommand{\Uni}{\text{Uniform}}
\newcommand{\seg}{\text{seg}}
\newcommand{\algname}[1]{\textit{#1}}

\newcommand{\BD}{ K }
\newcommand{\BX}{ Q }
\newcommand{\iat}{isoperimetric approximate tiling}

\newcommand{\simpleHypothesisTester}{\alg{rSimpleHypothesisTesting}}

\newcommand{\rSimpleMultiCoinTester}{\alg{rSimpleMultiCoinTester}}
\newcommand{\rApproxMultiCoinTester}{\alg{rApproxMultiCoinTester}}
\newcommand{\rAdaptiveCoinTester}{\alg{rAdaptiveCoinTester}}
\newcommand{\rAdaptiveStatQ}{\alg{rAdaptiveStatQ}}
\newcommand{\rAdaptiveHeavyHitters}{\alg{rAdaptiveHeavyHitters}}

\newcommand{\rKIdentifier}{\alg{rKBiasIdentification}}

\newcommand{\rMultiStatQ}{\alg{rMultiStatQ}}

\newcommand{\rKPseudoMaximumIdentifier}{\alg{rKPseudoMaximumIdentification}}
\newcommand{\rPseudoMaximumIdentifier}{\alg{rPseudoMaximumIdentification}}
\newcommand{\rCVPP}{\alg{rCVPP}}
\newcommand{\rCVPPP}{\alg{rCVPP-PreProcessing}}
\newcommand{\rCVPPR}{\alg{rCVPP-RealTime}}

\newcommand{\multiCoinTester}{\alg{MultiCoinTester}}
\newcommand{\multiStatQ}{\alg{MultiStatQ}}
\newcommand{\findMaximum}{\alg{findMaximum}}

\newcommand{\rHeavyHitters}{\alg{rHeavyHitters}}

\newcommand{\rAdaptiveAmplify}{\alg{rAdaptiveAmplify}}
\newcommand{\rCoordinateRound}{\alg{rCoordinateRound}}

\newcommand{\cube}{\mathcal{C}}

\newcommand{\accept}{\alg{Accept}}
\newcommand{\reject}{\alg{Reject}}

\newcommand{\failtoreject}{\alg{FailToReject}}

\title{Replicability in High Dimensional Statistics}

\author{Max Hopkins\thanks{University of California, San Diego. nmhopkin@ucsd.edu. Supported by NSF Award CCF-1553288 (CAREER) and a Sloan Research Fellowship.}, Russell Impagliazzo\thanks{University of California, San Diego. rimpagliazzo@ucsd.edu. Supported by NSF Award AF: Medium 2212136}, Daniel Kane\thanks{University of California, San Diego. dakane@ucsd.edu. Supported by NSF Award CCF-1553288 (CAREER) and a Sloan Research Fellowship}, Sihan Liu\thanks{University of California, San Diego. sil046@ucsd.edu. Supported by NSF Award CCF-1553288 (CAREER) and a Sloan Research Fellowship.}, Christopher Ye\thanks{University of California, San Diego. czye@ucsd.edu. Supported by NSF Award AF: Medium 2212136, NSF grants 1652303, 1909046, 2112533, and HDR TRIPODS Phase II grant 2217058.}}

\begin{document}

\pagenumbering{gobble}
\maketitle
\begin{abstract}
    The replicability crisis is a major issue across nearly all areas of empirical science, calling for the formal study of replicability in statistics. Motivated in this context, [Impagliazzo, Lei, Pitassi, and Sorrell STOC 2022] introduced the notion of replicable learning algorithms, and gave basic procedures for $1$-dimensional tasks including statistical queries. In this work, we study the computational and statistical cost of replicability for several fundamental \emph{high dimensional} statistical tasks, including multi-hypothesis testing and mean estimation.

Our main contribution establishes a computational and statistical equivalence between optimal replicable algorithms and high dimensional isoperimetric tilings.
As a consequence, we obtain matching sample complexity upper and lower bounds for replicable mean estimation of distributions with bounded covariance, resolving an open problem of [Bun, Gaboardi, Hopkins, Impagliazzo, Lei, Pitassi, Sivakumar, and Sorrell, STOC2023] and for the $N$-Coin Problem, resolving a problem of [Karbasi, Velegkas, Yang, and Zhou, NeurIPS2023] up to log factors.

While our equivalence is computational, allowing us to shave log factors in sample complexity from the best known efficient algorithms, efficient isoperimetric tilings are not known. To circumvent this, we introduce several relaxed paradigms that do allow for sample and computationally efficient algorithms, including allowing pre-processing, adaptivity, and approximate replicability. In these cases we give efficient algorithms matching or beating the best known sample complexity for mean estimation and the coin problem, including a generic procedure that reduces the standard quadratic overhead of replicability to linear in expectation.

\end{abstract}
\newpage

\tableofcontents
\newpage
\interfootnotelinepenalty=10000
\pagenumbering{arabic}
\setcounter{page}{1}


\section{Introduction}
\label{sec:intro}

The \textit{replicability crisis} permeates almost all areas of science. Recent years have seen the repeated failure of influential work in oncology \cite{begley2012raise}, clinical research \cite{ioannidis2005contradicted}, and other high impact areas to replicate under scrutiny. Indeed the problem is so pervasive that in a survey of 1500 scientists, $70\%$ reported they had tried and failed to replicate another researcher's findings \cite{baker2016}. While many factors underlie the failure of replicability in science, a key component is the instability of underlying \textit{statistical methods}. Even techniques as basic as hypothesis testing suffer from these issues \cite{ioannidis2005most}, and combined with the explosion in \textit{number} of performed tests each year, it seems inevitable published false positives will skyrocket unless new methods are developed.

Motivated in this context, we study the cost of replicability in statistics in the recent algorithmic framework of Impagliazzo, Lei, Pitassi, and Sorrell \cite{impagliazzo2022reproducibility}. An algorithm $\innerAlg$ drawing samples from an (unknown) population $\distribution$ is called $\rho$-replicable if, run twice on \textit{independent} samples and \textit{the same} randomness, $\innerAlg$ produces exactly the same answer with probability $1-\rho$. We focus on characterizing the computational and statistical complexity of replicability for two core interrelated problems: \textit{multi-hypothesis testing} and \textit{high dimensional mean estimation}.

As a warm-up, consider the setting of a single hypothesis test. A typical procedure sets up a test statistic $Z$ to distinguish between a null $h_\emptyset$ and alternative hypothesis $h_1$ such that under $h_\emptyset$, $Z$ is uniform on $[0,1]$, while under $h_1$ there exists $q_0>p_0$ such that $\Pr[Z \leq p_0] \geq q_0$. Formalized in this way, hypothesis testing is equivalent to one of the earliest problems in replicability and distribution testing, the \textit{coin problem} (testing the bias of a weighted coin). Despite its central position, the complexity of the replicable coin problem is not fully understood. Worse, current methods have \textit{quadratic} overhead in $\rho$ which may be infeasible in practice. Our first contribution is a tight characterization of the coin problem, reducing this cost to just linear in expectation.

The coin problem is a fundamental example of \textit{1-dimensional problem} in statistics but, in practice, most problems are really \textit{high dimensional}. An epidemiologist, may, for instance, want to test the prevalence of a suite of $N$ diseases in some population. Or, even in a single hypothesis test, the test statistic itself may involve computing the mean of some $N$-dimensional data; if such pre-processing steps are non-replicable, the final test may be as well. This brings us to the main question addressed in this paper: \textit{how does the cost of replicability scale with dimension $N$?}

High dimensional replicability in this sense was first considered in \cite{DBLP:conf/stoc/BunGHILPSS23} and \cite{KVYZ23}. In \cite{KVYZ23}, the authors study the \textit{$N$-Coin Problem}, akin to the `multi-hypothesis' setup above. They argue that while independently estimating each coin replicably takes $N^3$ flips, by \textit{correlating} choices one can improve this cost to $N^2$, albeit in exponential time. Likewise, \cite{DBLP:conf/stoc/BunGHILPSS23} show a correlated strategy for replicably estimating an $N$-dimensional Gaussian in $N^2$ samples. At outset, it was unclear whether the proposed strategies were optimal: while \cite{KVYZ23} conjectured no better algorithm could exist, \cite{DBLP:conf/stoc/BunGHILPSS23} asked if the problem could be solved in $N$ samples. Is there a principled approach to understanding the cost of such problems?

We resolve this question by proving a tight connection between high dimensional replicability and a well-studied problem in high dimensional geometry: \emph{low surface area tilings of $\R^N$}. 
Low surface area tilings, closely related to optimal packings, are a classical problem dating back to Pappus of Alexandria in the 4th century,\footnote{Pappus claimed a solution for the 2-dimensional case, later proved by Hales \cite{hales2001honeycomb}.} with asymptotically optimal constructions known since the 1950s \cite{rogers50,butler1972simultaneous}. In computer science, such tilings have seen more recent study due to their close connections with lattice cryptography (see e.g.\ \cite{micciancio2004almost,mook2021lattice}) and hardness of approximation \cite{kindler2012spherical,naor2023integer}. 

We prove a computational and statistical equivalence between (efficient) replicable algorithms and (efficient) tilings. Given a replicable algorithm with low sample complexity, we give an oracle-efficient construction of an (approximate) tiling with low surface area. Conversely given an (approximate) tiling with low surface area, we give an oracle-efficient replicable algorithm with low sample complexity. Applying the classical isoperimetric theorem, we immediately get near-tight lower bounds for Gaussian mean estimation and the $N$-Coin Problem matching the algorithms of \cite{KVYZ23,DBLP:conf/stoc/BunGHILPSS23} up to log factors, resolving their corresponding open questions.\footnote{Formally, we resolve the sample complexity of the \textit{non-adaptive} $N$-Coin problem up to log factors. The authors of \cite{KVYZ23} do not consider the adaptive sample model. We discuss this subtlety later on.}

On the algorithmic side, while isoperimetric tilings exist, all known constructions take exponential time. Thus achieving true sample optimality via this approach, similar to \cite{KVYZ23,DBLP:conf/stoc/BunGHILPSS23}, currently requires exponential time. On the other hand, there are efficient tilings that (slightly) beat naive `independent estimation' \cite{micciancio2004almost}. Combined with our equivalence theorem, this gives the best known polynomial time algorithms for $N$-dimensional mean estimation and the coin problem. Further, even if no efficient isoperimetric tilings exist, we argue it is nevertheless possible to \textit{pre-process} an inefficient tiling in such a way that sample-optimal replicability can be achieved in polynomial time with query access to the pre-processing output. We leave the construction (or hardness of) truly efficient isoperimetric tilings as the main question (re)raised by this work.

Finally, in light of the lack of efficient isoperimetric tilings, we introduce two relaxed paradigms for replicability and multi-hypothesis testing that do allow for sample and computationally efficient algorithms. First, we consider \textit{adaptive} algorithms which may choose which of $N$ coins they flip during execution based on prior observations. We exhibit a polynomial time algorithm in this model matching the best-known sample complexity of prior (inefficient) non-adaptive methods. Second, we look at relaxations that only require \textit{approximate} replicability. In particular, we show if one only requires the outputs over two runs to agree on most coins, it is possible to build efficient algorithms \textit{beating} the sample complexity implied by isoperimetric tilings.


\subsection{Our Contributions}
Before stating our results, we briefly recall the formal notion of a replicable algorithm.

\begin{definition}[\cite{impagliazzo2022reproducibility}]
    \label{def:replicability}
    An algorithm $\innerAlg$ is $\rho$-replicable if for all distributions $\distribution$ and i.i.d. samples $S, S' \sim \distribution$
    \begin{equation*}
        \Pr_{r, S, S'} \left( \innerAlg(S; r) = \innerAlg(S'; r) \right) \geq 1 - \rho,
    \end{equation*}
    where $r$ denotes the internal randomness of the algorithm $\innerAlg$.
\end{definition}
Replicable algorithms are inherently randomized, and typically have a corresponding `failure probability' $\delta$. For simplicity, in this overview we will ignore sample dependence on $\delta$ which always scales \textit{logarithmically} in $\frac{1}{\delta}$. Formally, the below results can be thought of as in the regime where $\delta=\Theta(\rho)$. Formal dependencies on all parameters are given in the main body.

\subsubsection{On Replicability in \texorpdfstring{$1$}{1}-Dimension (\texorpdfstring{\Cref{sec:repro-hypothesis-testing}}{})}
While our eventual goal is to understand the price of replicability in high dimensions, it is of course natural to first ask for a tight understanding in $1$-dimension. With this in mind, we first consider the fundamental problems of \textit{single hypothesis testing} and \textit{bias estimation}.

Suppose we have some hypothesis $h_0$ and an experiment designed to test this hypothesis is repeated $m$ times, thus creating a sequence of $m$ $p$-values.
If $h_0$ is true, then the $p$-values should be uniformly distributed. 
On the other hand, if $h_0$ is false, we should gather small $p$-values with higher probability than normal. 
Quantitatively, there are constants $p_0, q_0$ such that $p$-values smaller than $p_0$ are observed with probability $q_0 > p_0$ (in statistics, $q_0$ is called the \textit{power} of the experiment). Given a sequence of $p$-values, we want to design an algorithm that \emph{replicably} determines whether to reject the null hypothesis $h_0$. We formalize this in the following definition.

\begin{restatable*}[Hypothesis Testing]{definition}{defhypothesistesting}
    \label{def:hypothesis-testing}
    Let $0 \leq p_0 < q_0 \leq 1$, $\delta < \frac{1}{2}$.
    A (randomized) algorithm $\innerAlg$ is a $(p_0, q_0)$-hypothesis tester if given sample access $S$ to some unknown $\distribution$ on $[0, 1]$:
    \begin{enumerate}
        \item Given $\distribution = \UnifD{[0, 1]}$, then $\Pr (\innerAlg(S) = \reject) < \delta$.
        \item Given $\Pr_{x \sim \distribution} (x < p_0) \geq q_0$, then $\Pr (\innerAlg(S) = \failtoreject) < \delta$.
    \end{enumerate}
\end{restatable*}

Single hypothesis testing is computationally and statistically equivalent to a well-studied problem in distribution testing, the \textit{coin problem} \cite{beyond_lc}. Given a coin with a hidden bias $p$, the $(p_0,q_0)$-coin problem asks the learner to determine whether the bias $p$ is at most $p_0$, or at least $q_0$. The coin problem was one of the first questions studied in algorithmic replicability and plays a critical role as a subroutine in later works. Nevertheless, there is a still gap in the best known bounds:
\begin{theorem}[\cite{impagliazzo2022reproducibility,KVYZ23}]\label{thm:old-coin}
    Let $p_0,q_0 \in (0,1/2)$ and $\rho \in (0,1)$, there is a computationally efficient $\rho$-replicable algorithm for the $(p_0,q_0)$-coin problem using
    \[
    \tilde{O}\left(\frac{1}{(q_0-p_0)^2\rho^2}\right)
    \]
    samples. Conversely, any algorithm for the $(p_0,q_0)$-coin problem uses at least
    \[
    \Omega\left(\frac{p_0}{(q_0-p_0)^2\rho^2}\right)
    \]
    samples in the worst-case.
\end{theorem}
We tighten \Cref{thm:old-coin} in two key aspects. First, 
we resolve the gap in sample dependence on $p_0$ and $q_0$ in the numerator. 
Second, we address a more subtle issue regarding \Cref{thm:old-coin}'s dependence on $\rho$. In particular, we argue that while quadratic dependence on $\rho$ is indeed necessary in the worst-case, \textit{in expectation} the dependence can actually be reduced to \textit{linear}.
\begin{theorem}[Informal \Cref{thm:r-adapt-coin-problem} and \Cref{thm:q0-num-lower-bound}]\label{thm:intro-1-coin}
     Let $p_0,q_0 \in (0,1/2)$ and $\rho \in (0,1)$. The $\rho$-replicable $(p_0, q_0)$-coin problem coin problem requires
    \[
    \tilde{\Theta}\left(\frac{q_0}{(q_0 - p_0)^2 \rho}\right)
    \]
    samples in expectation. 
    Moreover, the same bound holds in the worst-case with quadratic dependence on $\rho$ and the upper bound is computationally efficient.
\end{theorem}
A few remarks are in order. First, we note that the linear overhead of \Cref{thm:intro-1-coin} is not specific to the coin problem. In \Cref{sec:adaptive-replicability} we give a generic amplification lemma showing \textit{any} replicable procedure can be performed with linear overhead (in $\rho$) in expectation. Second, we remark that as an immediate consequence of \Cref{thm:intro-1-coin} we obtain a generic procedure to efficiently transform any \textit{non}-replicable distribution testing algorithm into a replicable one with linear expected overhead. In particular, let $\hypotheses_0, \hypotheses_1$ be two families of distributions and suppose some distribution testing algorithm $\innerAlg$ accepts samples from distributions $\distribution \in \hypotheses_0$ with probability at most $\frac{1}{3}$ and rejects samples from distributions 
$\distribution \in \hypotheses_1$ with probability at most $\frac{2}{3}$. We may view the output of $\innerAlg$ as a biased coin and apply \Cref{thm:intro-1-coin} to replicably determine membership in $\hypotheses_0$ or $\hypotheses_1$ with high probability. This gives replicable algorithms for a wide range of distribution testing problems including uniformity, closeness, independence, log-concavity, and monotonicity.

For simplicity of presentation, in the rest of the introduction we state only worst-case sample complexity with quadratic dependence on $\rho$. Up to polylog factors, all our bounds can equivalently be stated in terms of expected complexity with linear dependence.





\subsubsection{Replicability and Isoperimetry in High Dimensional Statistics (\texorpdfstring{\Cref{sec:isoperimetry-from-replicability}}{})}
\label{sec:intro:direct-sum-theorems}

In many applications, scientists may wish to perform multiple experiments simultaneously; an epidemiologist, for instance, may want to determine the prevalence of several diseases or conditions in a population at once. Consider a setting in which a scientist runs $N$ simultaneous hypothesis tests. In the context of replicability, we'd like to ensure that all $N$ findings are simultaneously replicable---how does the cost of this guarantee scale with $N$ and $\rho$?

Like single hypothesis testing, such a \textit{multi}-hypothesis test is equivalent to the problem of testing biases of multiple coins (typically called the $N$-Coin Problem). In this section, we study the more general problem of \textit{high dimensional mean estimation}. In particular, given sample access to a distribution $\distribution$ over $\R^N$, how many samples are required to $\rho$-replicably output an estimate $\hat{\mu}$ s.t.
\[
\Pr_{S \sim D}[||\hat{\mu}-\mu_{\distribution}||_p \geq \varepsilon] \leq \delta \, ?
\]
We say such an algorithm $(\varepsilon,\ell_p)$-learns the mean $\mu_D$ and refer to the problem of giving such an estimator as the \textit{$(\varepsilon,\ell_p)$-mean estimation problem}. We will always assume the distribution $\distribution$ has bounded covariance. Up to log factors, the $N$-Coin problem is the special case where $\distribution$ is the product of $N$ independent Bernoullis and $p=\infty$ (see \Cref{lemma:l-inf-tester-implies-learner}).

Our core contribution is that replicable mean estimation (and therefore multihypothesis testing) is \textit{computationally and statistically equivalent} to the construction of (approximate) low-surface area tilings of space. To state this more formally, first consider the notion of an approximate tiling:

\begin{definition}[Isoperimetric Approximate Tilings (Informal \Cref{def:approximate-tiling})]
\label{def:informal-tiling}
    A $(\gamma, A)$-\iat{} (IAT) of $\R^N$ is a collection of sets $\mathcal{P}=\{P\}$ such that for any cube $\cube \subset \R^N$
    \begin{enumerate}
        \item ($\gamma$-Approximate Volume): 
        $
        \vol(\mathcal{P} \cap \cube) \geq (1-\gamma)\vol(\cube).
        $
        \item ($A$-Approximate Isoperimetry): 
        $
        \area(\partial \mathcal{P} \cap \cube) \leq A \vol(\cube).
        $
        \item (Bounded Diameter): Each $P \in \mathcal P$ has diameter at most $1$.
    \end{enumerate}
    We call $\mathcal{P}$ \textbf{efficient} if there is an efficient 
    membership oracle $\mathcal{O}: \R^N \to \mathcal P$ such that for any $P \in \mathcal{P}$ and $w \in P$, $\mathcal{O}(w)=P$ with high probability.
\end{definition}
In other words, a good approximate tiling covers `most' of $\R^N$ with diameter $1$ bubbles with low surface-area to volume ratio. We prove the sample complexity of replicable mean estimation tightly corresponds to the surface area of an associated tiling, and moreover that there are oracle-efficient reductions between the two. We state the theorem below only for the case of $\ell_2$-estimation, but will discuss its implications and variants for any $p \in [2,\infty]$ shortly.
\begin{theorem}[Replicability $\iff$ Isoperimetry (Informal \Cref{thm:tiling-to-replicable} and \Cref{thm:replicable-non-uniform-tiling-formal})]
    \label{thm:replicable-non-uniform-tiling}~
    \begin{enumerate}
        \item (Replicability $\to$ Isoperimetry): Let $\mathcal{A}$ be a $\rho$-replicable algorithm on $m$ samples that $(\varepsilon,\ell_2)$-learns the mean of $N$ independent Bernoulli variables. Given oracle access to $\innerAlg$, there is an efficient algorithm generating an efficient $N$-dimensional $\left(\rho, O(\eps\rho\sqrt{m}) \right)$-IAT.
        \vspace{.2cm}
        \item (Isoperimetry $\to$ Replicability): Let $\mathcal{P}$ be an $N$-dimensional $(\rho,A)$-IAT. Given access to $\mathcal{P}$'s membership oracle and sample access to a bounded covariance ditribution $\mathcal{D}$ over $\R^N$, there is an efficient $\rho$-replicable algorithm that $(\varepsilon,\ell_2)$-learns $\mu_\mathcal{D}$ in $O(\frac{A^2}{\varepsilon^2\rho^2})$ samples.\footnote{This statement assumes $\delta \geq 2^{-N}$ for simplicity. The true bound is $O\lp(\frac{A^2}{\varepsilon^2\rho^2}+\frac{A^2\log\frac{1}{\delta}}{N\varepsilon^2\rho^2}\rp)$.}
    \end{enumerate}
\end{theorem}
A few remarks are in order. First, notice the surface area and sample complexity in \Cref{thm:replicable-non-uniform-tiling} `match' up to constant factors. That is starting with an $m$-sample algorithm we get an IAT with surface area $O(\varepsilon\rho\sqrt{m})$. Starting with a surface area $O(\varepsilon\rho\sqrt{m})$-IAT, we get an algorithm on $O(\frac{(\varepsilon\rho\sqrt{m})^2}{\varepsilon^2\rho^2})=O(m)$ samples. Second, we note the forward direction above really only relies on the family of input distributions satisfying certain mutual information bounds (see \Cref{lemma:gaussian-mutual-info-bound}), and therefore also holds e.g.\ for standard Gaussians.

By the isoperimetric inequality, the best possible surface area for an isoperimetric approximate tiling is $A=\Omega(N)$, while simply tiling space by cubes achieves $A=O(N^{3/2})$.\footnote{This comes from the diameter restriction. To have diameter $1$, the cubes must be of side-length $\frac{1}{\sqrt{N}}$.} Moreover, constructions of isoperimetric tilings, that is $(0,O(N))$-IATs, have existed since the $50$'s \cite{rogers50}.
Combined with \Cref{thm:replicable-non-uniform-tiling}, these facts lead to a tight statistical characterization of replicable mean estimation:
\begin{corollary}[Replicable $\ell_2$ Mean Estimation (Informal \Cref{thm:replicable-alg-partition-lb} and \Cref{cor:ell-2-mean-estimation}]\label{cor:intro-mean-est-2}
    Let $\varepsilon,\rho \in (0,1)$. The $\rho$-replicable $(\varepsilon,\ell_2)$-mean-estimation problem requires    
    \[
    \Theta\left(
    \frac{N^2}{\varepsilon^2\rho^2}\right)
    \]
    samples. Moreover, the lower bound holds even under Bernoulli or Gaussian distributions. 
\end{corollary}
\Cref{cor:intro-mean-est-2} resolves (in the negative) \cite[Open Question 4]{DBLP:conf/stoc/BunGHILPSS23} regarding whether estimation can be performed in $O(N)$ samples, as well as the $\ell_2$-variant of \cite{KVYZ23}'s question regarding the complexity of the $N$-Coin Problem. 

\paragraph{Computational Efficiency:} \Cref{thm:replicable-non-uniform-tiling} and \Cref{cor:intro-mean-est-2} leave two important questions: what can we say about computational efficiency, and to what extent does the above hold for norms beyond $\ell_2$? Toward the former, unfortunately all known isoperimetric tilings have membership oracles that run in (at best) \textit{exponential} time, so the above algorithms are not efficient. The best known tiling with an efficient membership oracle, a lattice-based construction of Micciancio \cite{micciancio2004almost}, only manages to shave a log factor. Nevertheless, this gives the first efficient algorithm for replicable mean estimation with (slightly) sub-cubic sample complexity. 
\begin{corollary}[Efficient Mean Estimation in Sub-Cubic Samples (Informal \Cref{cor:sub-cubic})]
    Let $\varepsilon, \rho \in (0,1)$. There is an efficient $\rho$-replicable algorithm for $(\varepsilon,\ell_2)$-mean-estimation using
    \[
    O\left(\frac{N^3}{\rho^2\varepsilon^2}\cdot \frac{\log\log(N)}{\log(N)}\right)
    \]
    samples.
\end{corollary}
By the reverse direction of our reduction, any algorithm beating the above must imply improved efficient IATs. In the lattice setting, this problem has remained open since it was proposed in Micciancio's work \cite{micciancio2004almost}. We leave the construction of tilings satisfying our relaxed approximate notion as the main open question from this work.

\paragraph{Replicability Beyond the $\ell_2$-Norm:}

Finally, recall in the context of hypothesis testing we are really interested in learning biases in $\ell_\infty$ rather than $\ell_2$-norm. A version of the equivalence theorem indeed holds for general $\ell_p$-norms
as a consequence of the forward direction of the $\ell_2$ equivalence (\Cref{thm:replicable-non-uniform-tiling}), an $\ell_\infty$ learner based on IATs (\Cref{thm:ell-infty-mean}), and H\"{o}lder's inequality (\Cref{lemma:l-2-lb-implies-l-c-lb}).
\begin{corollary}[$\ell_p$-norm Replicability $\iff$ Tilings]\label{cor:p-norm-eq} 
Fix $p \in [2,\infty], \rho \in (0,1), \varepsilon \in (0, 0.1)$. Then:
        \begin{enumerate}
        \item (Replicability $\to$ Isoperimetry): 
        Let $\mathcal{A}$ be a $(\rho/24)$-replicable algorithm on $m$ samples that $(\frac{\varepsilon}{N^{\frac{1}{2}-\frac{1}{p}}},\ell_p)$-learns biases of $N$ Bernoulli variables. Given oracle access to $\innerAlg$, there is an efficient algorithm generating an efficient $N$-dimensional $\left(\rho, O \lp( \eps\rho\sqrt{m} \rp) \right)$-IAT.
        \vspace{.2cm}
        \item (Isoperimetry $\to$ Replicability): Let $\mathcal{P}$ be an $N$-dimensional $(\rho,A)$-IAT. Given $\mathcal{P}$'s membership oracle and sample access to a bounded covariance ditribution $\mathcal{D}$ over $\R^N$, there is an efficient $O(\rho)$-replicable algorithm that $(\varepsilon,\ell_p)$-learns $\mu_D$ in $\tilde{O}\lp(\frac{A^2}{N^{1-\frac{2}{p}}\varepsilon^2\rho^2}\rp)$ samples.
    \end{enumerate}
\end{corollary}
\Cref{cor:p-norm-eq} is somewhat weaker than its $\ell_2$-analog in terms of the applicable range of $\eps$. Namely while it is possible to derive a lower bound for $\ell_p$-estimation of $\Omega(\frac{N^{1+\frac{2}{p}}}{\rho^2\varepsilon^2})$ via \Cref{cor:p-norm-eq}, the result only holds in the regime where $\varepsilon \leq \frac{1}{N^{\frac{1}{2}-\frac{1}{p}}}$. 
To circumvent this issue we prove a direct lower bound in the special case of the $\ell_\infty$-norm by an extra `reflection' trick in our IAT analysis. This results in a near-tight characterization of replicable $\ell_\infty$-mean estimation:
\begin{theorem}[Replicable $\ell_\infty$-Mean-Estimation (Informal \Cref{thm:ell-infty-mean} and \Cref{thm:n-coin-lower-const-delta})]\label{thm:intro-mean-infty}
 Let $\varepsilon,\rho \in (0,1)$. The $\rho$-replicable $(\varepsilon,\ell_\infty)$-mean-estimation problem requires
    \[
    \tilde{\Theta}\left(\frac{N}{\varepsilon^2\rho^2}\right)
    \]
    samples. Moreover, the lower bound holds even under Bernoulli or Gaussian distributions
\end{theorem}
\Cref{thm:intro-mean-infty} essentially resolves the complexity of the $N$-Coin Problem up to log factors, settling in the positive \cite[Conjecture D.8]{KVYZ23}. We remark that an $\tilde{\Omega}(N)$ lower bound for $N$-Coins was also given in \cite{DBLP:conf/stoc/BunGHILPSS23} under the moniker `One-Way-Marginals' using fingerprinting. It is not clear, however, how to get the appropriate dependence on $\rho$ and $\varepsilon$ using their method.

\subsubsection{Efficient Replicability from Relaxed Models (Sections \texorpdfstring{\ref{sec:cvpp}}{}-\texorpdfstring{\ref{sec:efficient-n-coin-problem}}{})}\label{sec:intro-eff}

In the previous section we saw in the standard model, any replicable algorithm improving over the trivial union bound strategy (beyond log factors) must make progress on the efficient construction of low surface area tilings. In this section, we argue this connection can be circumvented if one is willing to relax the model in question. We consider three relaxations that allow us to obtain efficient algorithms matching (in some cases even beating) the sample complexity implied by isoperimetric partitions: pre-processing, coordinate samples, and approximate replicability.
    
    


\paragraph{Pre-Processing:} While it is true all known constructions of isoperimetric tilings have exponential time membership oracles, instead of paying this cost every time we perform a replicable procedure, we might instead hope to pay this high cost \textit{just once} by constructing a large data structure after which membership queries can be performed in \textit{polynomial} time. In the world of lattices, this problem is actually well-studied; it is known as the \textit{Closest Vector Problem with Pre-processing} (CVPP). Unfortunately, existing algorithms for CVPP still run in exponential time. We show with sufficient pre-processing, it is in fact possible to solve CVPP on any lattice in \textit{polynomial} time. More formally, we show CVPP is solvable in the \textit{decision tree model}:
\begin{theorem}[CVPP (Informal \Cref{thm:CVPP})]\label{thm:intro-CVPP}
        Let $N \in \mathbb{N}$ and $\mathcal{L} \subset \R^N$. There is a depth $O(N^2\log(N))$ decision tree $\mathcal{T}$ satisfying
        \begin{enumerate}
            \item \textbf{Pre-processing}: $\mathcal{T}$ can be constructed in $2^{\text{poly}(N)}$ time and space.
            \item \textbf{Run-time}: Given $\mathcal{T}$, there is an algorithm solving CVP for all $t \in \R^N$ in $\text{poly}(N)$ time.
        \end{enumerate}
\end{theorem}
Since deterministic isoperimetric lattice tilings exist \cite{micciancio2004almost}, all statistical upper bounds in the previous sections relying on the existence of an isoperimetric partition can in fact be executed in polynomial time after a single pre-processing cost of $2^{\text{poly}(N)}$. We remark that \Cref{thm:intro-CVPP} may also be of independent interest. CVPP is an NP-hard problem, and prior results typically focus on improving the constants in the exponent. The decision tree model circumvents the classical hardness of CVPP by allowing access to an exponential size data structure, drawing inspiration from similar results for subset sum and other combinatorial NP hard problems \cite{meyer1984polynomial}.

    
    



\paragraph{Adaptivity and Coordinate Samples:} In \Cref{sec:intro:direct-sum-theorems} we assumed our algorithm draws \textit{vector} samples from an $N$-dimensional distribution over $\R^N$. In hypothesis testing (or indeed even mean estimation), sometimes the tester has more freedom and may instead choose to restrict their test to a particular \textit{subset} of coordinates, drawing from the relevant marginal distribution. Consider, for instance, our prior example of the epidemiologist testing disease prevalence. In this setting, each `vector sample' corresponds to a patient, and each coordinate a particular test or disease. The practitioner need not run every test on the patient (indeed this may not even be possible). Moreover, if during the procedure of the experiment some diseases are exceedingly common or rare, the practioner may wish to adaptively choose to avoid these tests and focus only on coordinates on which the result is less certain.


The equivalence of replicable mean estimation and tilings (and its corresponding lower bounds) actually holds in this coordinate sampling model as well, but only against \textit{non-adaptive} algorithms that must choose ahead of time how many samples they'll draw for each coordinate. In the \textit{adaptive} setting, we can actually give an efficient algorithm with coordinate sample complexity roughly $\tilde O(N^2)$, matching the number of coordinate samples implied by the isoperimetric lower bound for non-adaptive $\ell_\infty$ learning. 
Since the coordinate sampling model is most natural for hypothesis testing and the coin problem, we state the result in this regime:
\begin{restatable}[\Cref{thm:r-n-coin-problem-formal}, informal]{theorem}{rncoinproblem}
    \label{thm:r-n-coin-problem}
    Let $\frac{1}{2} \geq \rho > \delta > 0$.
    There is a $\rho$-replicable algorithm solving the $N$-coin problem using at most $\bigtO{\frac{N^2 q_0}{(q_0 - p_0)^2 \rho^2}}$ coordinate samples and runtime.
\end{restatable}
Our algorithm requires no assumption of independence between coins.
In particular, the estimates are correct and replicable even if certain diseases might be correlated.

In fact, \Cref{thm:r-n-coin-problem} is really a special case of a general \textit{adaptive composition theorem} (see \Cref{sec:adaptive-composition}), a computationally efficient procedure that can solve
any collection of $N$ statistical tasks replicably with $\bigO{\frac{N^2}{\rho}}$ expected samples. The basic procedure proceeds in two steps. First, using adaptive amplification, we can solve each individual task in only $\frac{1}{\rho}$ expected samples. We then compose $N$ such instances that are $\frac{\rho}{N}$-replicable into a $\rho$-replicable algorithm for the composed problem. Each of the $N$ individual procedures costs $\frac{N}{\rho}$ samples in expectation, so linearity of expectation gives $\bigO{\frac{N^2}{\rho}}$ total expected cost. Note that the use of average-case dependence on $\rho$ is critical in this procedure. Composing using worst-case bounds results in a blow-up of $N^3$, since running each individual procedure at $\rho/N$-replicability costs $\frac{N^2}{\rho^2}$ samples.

\paragraph{Relaxing Replicability and the Coin Problem:}


Despite the above improvements, in practice sample complexity quadratic in dimension may still be prohibitively expensive. Toward this end, we consider two final relaxations of the $N$-coin problem where we obtain efficient algorithms with subquadratic sample complexity.

First, we consider relaxing replicability itself by allowing the output sets of the algorithm $\innerAlg$ between two runs to differ in at most $R$ elements, rather than to be exactly identical.

\begin{restatable}[Approximate Replicability]{definition}{approxreplicability}
    \label{def:approx-replicable-n-coin}
    Let $1 \leq R \leq N$. 
    An algorithm $\innerAlg$ 
    that outputs a set 
    is {\bf $(\rho, R)$-replicable} if for all input distributions $\distribution$,
    \begin{equation*}
        \Pr_{r, S, S'} \left( |\innerAlg(S; r) \triangle \innerAlg(S'; r)| \geq R \right) \leq \rho.
    \end{equation*}
\end{restatable}
The output of the $N$-coin problem can be naturally viewed as a set (say the set of output large bias coins). We give an efficient adaptive $(\rho, R)$-replicable algorithm for the $N$-coin problem.
\begin{restatable}[\Cref{thm:r-approx-n-coin-alg-formal}, informal]{theorem}{rapproxncoinalg}
    \label{thm:r-approx-n-coin-alg}
    There exists an efficient, $(\rho, R)$-replicable algorithm solving the $N$-coin problem 
    using at most $\bigtO{\frac{q_0 N^2}{(q_0 - p_0)^2 R \rho^2}}$ coordinate samples.
\end{restatable}

Second, we study the cost of determining only the \textit{maximally} biased coins. Returning to our epidemiologist, while we may not have the resources to determine the prevalence of every disease, it may still be useful to determine say the $10$ most prevalent, identifying a subset for which to prioritize treatment.
We design an algorithm that replicably returns a set of $K$ coins within $\varepsilon$ of the maximum bias $p_{\max}$.

\begin{theorem}[\Cref{thm:r-pseudo-maximum-identifier-formal}, informal]
    \label{thm:r-pseudo-maximum-identifier}

    There is an efficient, $\rho$-replicable algorithm that outputs a set of at least $K$ coins $i$ such that $p_i \geq p_{\max} - \varepsilon$ using at most $\bigtO{\frac{N^{4/3} K^{2/3}}{\rho^2 \varepsilon^2}}$ coordinate samples. 
\end{theorem}

\subsection{Technical Overview}
We now give a high level overview of our core results and techniques, focusing on the equivalence theorem and replicability with linear overhead.


\subsubsection{Replicable Algorithms and Isoperimetry}


\paragraph{Replicable Algorithms to Isoperimetric Tilings:}

Suppose there is a $\rho$-replicable algorithm $\innerAlg$ on $m$ samples estimating the mean of $N$ Bernoulli variables up to $\ell_{2}$-error $\varepsilon$ with probability at least $1 - \delta$. We show $\innerAlg$ induces an approximate partition of the cube $\cube = [1/2 - 5\eps, 1/2 + 5\eps]^N$ whose sets \textbf{1) }cover at least $1 - O(\rho)$ fraction of points from $\cube$, \textbf{2)} have covering radius at most $O(\eps)$, and \textbf{3)} have surface area at most $O(\rho \sqrt{m})$ (excluding the cube boundary). After scaling and translation, we obtain an approximate tiling with constant covering radius and  $A \leq O(\rho \eps \sqrt{m})$ surface area. 

We appeal to a minimax-type argument. Consider an adversary that chooses a random mean vector $p \in \cube$. 
Because $\innerAlg$ is correct and replicable over all biases, it must be the case that for many random strings $r$ the deterministic procedure $\innerAlg(; r)$ is correct and replicable on most $p \in \cube$. Fix such an $r$. For each $p \in \cube$ on which $\innerAlg(; r)$ is replicable, there is some `canonical hypothesis' $\hat{p}$ such that $\innerAlg(S_p, r)=\hat{p}$ with high probability when $S_p$ is drawn from an $N$-Bernoulli distribution with mean $p$. Moreover, $\innerAlg(; r)$ should map any \textit{close} biases $p,p' \in \cube$ to the \textit{same canonical solution} since $S_p$ and $S_{p'}$ will be statistically indistinguishable, suggesting each $\hat{p}$ sits in a small `bubble' of biases mapping to it.
This suggests a natural candidate partitioning of the cube by these bubbles:
\[
F_{ \hat p } \coloneqq \left\{p \in \cube: \Pr[\innerAlg(; r) = \hat{p}] > \frac{3}{4}\right\}.
\]
We note a similar partitioning strategy is taken in \cite{dixon2024list} to lower bound the number of random strings needed by a replicable algorithm (an orthogonal consideration to our goal of characterizing sample complexity). We discuss connections with \cite{dixon2024list} and other geometric methods in algorithmic stability in \Cref{sec:related-work}.

Observe that by definition this partition already (nearly) satisfies Properties \textbf{(1)} and \textbf{(2)}. By replicability of $\innerAlg(;r)$, all but an $O(\rho)$ fraction of biases have some canonical $\hat{p}$, promising the $\{F_{ \hat p } \}$ cover a $1-O(\rho)$ fraction of $\cube$. On the other hand, by correctness at most an $O(\delta)$ fraction of biases $p$ have a canonical hypothesis $\hat{p}$ which is $\varepsilon$-far, meaning the sets $F_{ \hat p }$ almost have small diameter. To ensure the sets truly have bounded diameter, we slightly modify each $F_{ \hat p }$ by intersecting with the $\varepsilon$-ball $B_\varepsilon(\hat{p})$. This forces each set to have $2\varepsilon$-diameter while only removing an $O(\delta)$ fraction of points from the partition.\footnote{Formally this means we need to assume $\delta \leq O(\rho)$. We remark that this step is not really necessary, and one can instead define an equivalence with partitions that have a `$\delta$-approximate diameter' of this sort. However, since $\delta \leq \rho$ is really the main regime of interest anyway, we choose to make this simplifying assumption.} Denote the new partition by $G_{\hat{p}} \coloneqq F_{ \hat p } \cap B_\varepsilon(\hat{p})$.

Next we turn to surface area. Consider a point $p \in \partial G_{ \hat p } $. By construction, $p$ either lies on $\partial B_{\eps}(\hat p)$ or $\partial F_{\hat p}$.
If $p$ lies in the former, its associated canonical output $\hat p$ is $\varepsilon$-far from $p$, so $\innerAlg(;r)$ typically fails correctness on this bias. On the other hand, if $p$ lies in the latter, there is no `canonical hypothesis' and $\innerAlg(;r)$ fails replicability. The key is to observe that this is true not only of points in $ \partial G_{ \hat p } $, but for any point sufficiently nearby. Using tools from information theory, we show that for any $p,q \in \cube$ satisfying $\snorm{2}{p-q} \leq \frac{1}{\sqrt{m}}$, $\innerAlg(;r)$ has similar outputs on samples from $p$ and $q$. As a result, $\innerAlg(;r)$ fails either correctness or replicability on any point in the thickening $\partial G_{ \hat p } + B_{\frac{1}{\sqrt{m}}}$. The desired bound now follows from considering the volume of this set. On the one hand, the volume of this thickening is roughly $\frac{1}{\sqrt{m}}$ times the surface area of $G_{\hat{p}}$.\footnote{In reality, the volume is the integral over boundaries $\partial(G_{\hat{p}} + B_{r})$ for $r \leq \frac{1}{\sqrt{m}}$. We argue there exists some $r^*$ for which the surface area satisfies the desired bound, and take the true final partition to be $G_{\hat{p}}+B_{r^*}$, arguing this does not greatly effect the other desired properties.} On the other hand, by replicability and correctness of $\innerAlg(;r)$, the volume is at most $O(\rho+\delta) \leq O(\rho)$, giving the desired bound.

Finally, observe that $\innerAlg(; r)$ itself immediately gives a membership oracle for this approximate partition. In particular, given $p \in  G_{\hat p}$, the oracle simply runs $\innerAlg(S_p;r)$ on a simulated $p$-biased sample $S_p$ several times and outputs the majority. Since $\Pr[\innerAlg(S_p;r)=\hat{p}] \gg \frac{1}{2}$, the outcome should agree with $\hat p$ with high probability by Chernoff. All that is left to generate such a partition is to actually \textit{find} a good random string $r$. 
We show most strings are good, and one can be easily found by drawing a small number and efficiently testing them for replicability.


\paragraph{Isoperimetric Tilings to Replicable Mean Estimation:}
Suppose we are given a $(\rho,A)$-IAT $\mathcal{P}$ and its associated membership oracle $\mathcal{P}(\cdot)$. We outline an oracle-efficient algorithm for replicable mean estimation for bounded covariance distributions.
Our main technical contribution is an oracle efficient procedure turning any \iat{} into a randomized rounding scheme such that
\textbf{1)} the output after rounding is $\eps$-close to the input with high probability, and \textbf{2)} running the rounding scheme on two inputs within distance $\eps \rho \sqrt{N} / A $ leads to identical outputs with high probability when the two runs share randomness.

Given such a scheme, observe it suffices to estimate the mean non-replicably up to accuracy $\min(\eps/2, \eps \rho \sqrt{N} / A )$. Rounding the estimator then ensures the output is replicable and within $\eps$ distance of the true mean with high probability by the triangle inequality. For simplicity, we focus below on the regime where $\varepsilon$ is constant; general $\eps$ error can be achieved by scaling the tiling by $\eps$. 

Given $p \in \cube$, the most straightforward approach to rounding $p$ would simply be to apply the membership oracle $\mathcal{P}(p)$. This clearly fails Property (2) in the worst case: no matter how close two inputs $p$ and $p'$ may be, as long as the segment $p-p'$ connecting them crosses a boundary of our IAT we will round to different points. This is a standard issue in replicablility (even in $1$-dimension) \cite{impagliazzo2022reproducibility}; the typical trick is to first apply a \textit{random shift} before rounding. In our case, applying a random shift (and wrapping around $\cube$ when necessary) ensures rounding leads to consistent outputs with probability at least $1 - \rho$ whenever the two inputs have distance at most $\rho / A$. 


In high dimensions, however, a simple random shift is insufficient. Estimating the mean up to $\rho/A$ accuracy requires $N A^2 / \rho^2$ samples, so even using an isoperimetric partition ($A=\Theta(N)$) we'd require $N^3$ samples. The issue is we have not accounted for \textit{direction}. Consider inputs $u^{(1)}, u^{(2)}$ that are 
within distance $\eta$ both from each other and the boundary of the partition. Rounding $u^{(1)}$ and $u^{(2)}$ only leads to inconsistent outputs if $u^{(2)} - u^{(1)}$ points in the worst case direction, namely towards the boundary. We can avoid this by \textit{randomly rotating} our input before shifting it. The resulting difference vector $u^{(2)} - u^{(1)}$ then points in a random direction and a simple calculation shows the worst-case direction has size $\frac{1}{\sqrt{N}} \; \snorm{2}{ u^{(2)} - u^{(1)} }$ in expectation. This saves a $\sqrt{N}$ factor, meaning our original points only need to be within distance $O( \rho \sqrt{N} / A )$ as desired.



\subsubsection{Lower Bounds for the \texorpdfstring{$N$}{N}-Coin Problem}
\label{sub-sec:n-coin-lb-tech}

Recall for $\ell_\infty$-estimation and the $N$-Coin Problem, the procedure described above only gives a tight sample lower bound of $\Omega( N \eps^{-2} \rho^{-2} )$ vector samples when $\varepsilon \leq \frac{1}{\sqrt{N}}$. 
We now discuss how to modify the argument to give a tight bound in all regimes. For convenience we work directly with the $N$-Coin Problem, and assume that the algorithm invokes $m$ flips for each of the coin (alternatively, the algorithm takes $m$ vector samples).

Similar to the argument in the $\ell_2$-case, we look at the set of possible canonical outputs $\hat{o} \in \{\accept,\reject\}^N$,
and the approximate partition $\{ F_{\hat{o}} \}$ over $\cube = \left[ \frac{1}{4}, \frac{3}{4} \right]^{N}$ induced by the algorithm in the same manner.
If we could show the surface area of the boundaries of $\{F_{\hat o}\}$ (excluding the cell boundary) is at least $\Omega( \sqrt{N} \eps^{-1} )$, we would be able to use a similar argument to the $\ell_2$ case to show the fraction of non-replicable points is at least $\sqrt{N/m} \eps^{-1} \leq O(\rho)$, implying the desired sample complexity lower bound on $m$.

The main difficulty in the $\ell_\infty$ setting is an issue we brushed under the rug in the previous section: the cube boundary. In particular, the naive way of lower bounding the surface area of $\{F_{\hat o}\}$ is to apply the isoperimetric inequality to $\partial (\cup_{\hat o} F_{\hat o} \cap \cube)$, then subtract out the boundary of the cube. Since we now measure error in $\ell_\infty$-norm, however, we can only bound the radius of $F_{\hat{o}}$ by $\sqrt{N} \eps$ and the above method gives surface area $A \geq \sqrt{N} \eps^{-1} - O(N)$, useless when $\eps > \frac{1}{\sqrt{N}}$.

To circumvent this, we need to somehow apply the isoperimetric inequality to $\partial F_{\hat o} \backslash \partial \mathcal C$ directly. To this end, first observe that, by correctness, $F_{\hat{o}}$ can only intersect a $\delta$-fraction of faces of $\cube$ not incident to the corner $\hat{o}$.
Moreover, if $F_{\hat{o}}$ only intersects such faces, we can create a valid surface by \textit{reflecting} $F_{\hat o}$ across the cube boundary. This forces points on the cube boundary to become interior while otherwise `copying' the boundary of $F_{\hat o}$ itself $2^N$ times. Since reflecting only changes the $\ell_\infty$-radius by a constant factor, we can now apply the isoperimetric inequality to the reflected set with no asymptotic loss to get the desired lower bound on $\partial F_{\hat{o}}$.

\subsubsection{Adaptivity}

Prior replicable algorithms in the literature are \textit{non-adaptive}: they draw a fixed number of samples ahead of time, typically incurring a quadratic dependence on $\rho$ as a result. We show this strategy is wasteful. By instead allowing the algorithm to terminate early based on initial observations, we can reduce this cost to just \emph{linear} expected overhead. As discussed in \Cref{sec:intro-eff}, this also leads to an adaptive composition theorem with improved overhead.

Here we overview our most basic adaptive algorithm, testing the bias of a single coin (say between $1/3$ and $2/3$). Prior algorithms based on statistical queries compute the empirical bias of the coin using a \emph{fixed} number of samples and compare it with a random threshold, ensuring sufficient samples are drawn such that even if the threshold is within $O(\rho)$ of the bias $p$, our estimate still lands on the correct side.
Our adaptive algorithm samples a random threshold $r$ and draws samples adaptively until it determines whether the true bias $p$ lies above or below the threshold $r$.
The key observation is that when the true bias $p$ and the random threshold $r$ are far apart, we only need $\frac{1}{|r - p|^2}$ samples to determine with high confidence whether the true bias is above or below the threshold.
Since $r$ uniformly random, $|r - p|$ is (roughly) uniform over $(\rho, 1/3)$ and
the expected sample complexity is
\begin{equation*}
    \int_{\rho}^{1/3} \frac{1}{x^2} dx = \bigO{\frac{1}{\rho}}.
\end{equation*}
Using similar ideas, we also build an adaptive algorithm for the heavy hitters problem. This allows us to run an adaptive variant of replicability amplification (similar to \cite{impagliazzo2022reproducibility,DBLP:conf/stoc/BunGHILPSS23}) to show \textit{any} replicable algorithm can be run with only linear expected overhead.

\subsection{Further Related Work}\label{sec:related-work}

\paragraph{Replicability:}

Algorithmic replicability was independently introduced in \cite{impagliazzo2022reproducibility, DBLP:conf/nips/GhaziKM21}. Replicable algorithms have since been developed for PAC Learning \cite{DBLP:conf/stoc/BunGHILPSS23,DBLP:conf/icml/KalavasisKMV23}, reinforcement learning \cite{KVYZ23, eaton2023replicable}, bandits \cite{DBLP:conf/iclr/EsfandiariKK0MV23, komiyama2023improved}, clustering \cite{EKMVZ23}, and large-margin halfspaces \cite{impagliazzo2022reproducibility, kalavasis2024replicable}. Several works have shown tight statistical connections between replicability and other notions of algorithmic stability \cite{DBLP:conf/stoc/BunGHILPSS23, DBLP:conf/icml/KalavasisKMV23, CCMY23, CMY23, MSS23, dixon2024list}. Most closely related to our work are the discussed algorithms (and lower bounds) for $N$-Coins and mean estimation problems in \cite{KVYZ23, DBLP:conf/stoc/BunGHILPSS23} respectively, and the work of \cite{dixon2024list} studying `list' or `certificate' replicability for $N$-Coins. The latter in particular uses a similar partitioning strategy to our lower bound, but relies on totally different properties of the partition.

\paragraph{Geometry and Algorithmic Stability:} Our work adds to a growing line of connections between geometry, topology, and algorithmic stability. Such ideas were first introduced in the study of pure differential privacy in \cite{hardt2010geometry}, where packing lower bounds are now a standard tool \cite{bhaskara2012unconditional, nikolov2013geometry, beimel2014bounds, vadhan2017complexity}.
Impossibility results for related notions of replicability, specifically list replicability, certificate replicability, and global stability have been obtained via geometric and topological tools \cite{CMY23, CCMY23, dixon2024list}, in particular via the Sperner lemma and variants of Borsuk-Ulam.


\paragraph{Tilings and Rounding:}

The basic connection between replicability, tilings, and randomized rounding was first observed in \cite{impagliazzo2022reproducibility}. The authors used $1$-dimensional randomized rounding to give the first replicable algorithms for statistical queries and heavy hitters, and the high dimensional scheme of \cite{kindler2012spherical} to build a replicable PAC-learner for large margin halfspaces. The authors also analyzed rounding via cubical tiling, equivalent to independent handling of each coordinate. 


There are many known constructions of isoperimetric tilings \cite{rogers50,butler1972simultaneous,micciancio2004almost,kindler2012spherical,naor2023integer}. Our work is mostly closely related to \cite{kindler2012spherical}, who also observe their construction induces a `noise resistant' rounding scheme. Both our work and \cite{kindler2012spherical} critically rely on the Buffon needle theorem to analyze surface area and noise resistence. The main difference is that \cite{kindler2012spherical} study a specific randomized framework that in some sense `automatically' results in rounding, while we show how to take an arbitrary tiling and transform it into a rounding algorithm.

\section{Preliminaries}
\label{sec:prelims}

We begin with a few preliminaries, starting with basic probability and computational definitions, then necessary background in geometry, information theory, and finally discuss the formal models of sample access used in this work.

Let $\tvd{X, Y}$ denote the total variation distance between two distributions $X, Y$.
We denote a Bernoulli distribution with parameter $p$ as $\BernD{p}$, 
and a Binomial distribution with parameters $n, p$ as $\BinomD{n}{p}$.


Let $A,B \in \R^{N \times N}$ be two symmetric matrices. 
We say $A$ is bounded from above by $B$ if $A$ is at most $B$ in Loewner order, i.e., $B - A$ is a positive semi-definite matrix.

We denote by $\SO(N)$ the set of 
$\R^{N \times N}$
matrices representing the special orthogonal group, i.e., the set of all rotation matrices.

We say an algorithm $\innerAlg$ is \textit{efficient} if it runs in time polynomial in its input size.
An algorithm is \textit{$\outerAlg$-efficient} if, given oracle access to $\outerAlg$, it makes polynomially many queries to $\outerAlg$ and runs in time polynomial of its input size and the size of the query responses from $\outerAlg$. 

At many points throughout the paper, it will be convenient to set some variable to be an appropriate constant multiple of some expression.
Towards this, we use $x \ll y$ to denote that $x \leq c y$ for some sufficiently small constant $c$ and $x \gg y$ to denote that $x \geq c y$ for some sufficiently large constant $c$.


\subsection{Geometry}

For any set $S \subset \R^{n}$, let $\boundary S, \interior(S), \exterior(S), \closure(S)$ denote the boundary, interior, exterior, and closure of $S$ respectively.
Given two sets $S, T \subset \R^{n}$, denote their Minkowski sum $S + T = \set{s + t \given s \in S, t \in T}$.
Given $u =[u_1, \cdots, u_n] \in \R^n$ and $v = [v_1, \cdots, v_m] \in \R^m$, we denote by $u \oplus v$ their direct sum $[u_1, \cdots, u_n, v_1, \cdots v_m] \in \R^{n+m}$.
Similarly, given two sets $S \subset \R^n$ and $T \subset \R^m$, we denote $S \oplus T = \{u \oplus v \mid u \in S, v \in T  \}$.

We require the following notion of `niceness' for sets.
 \begin{definition}
    \label{def:semi-algebraic}
    A set $S \subset \R^{n}$ is \emph{semialgebraic} if it is a finite boolean combination of sets of the form $\set{x \given f(x) > 0}$ and $\set{x \given g(x) = 0}$ where $f, g$ are real polynomials in $x_1, \dotsc, x_{n}$.
\end{definition}
A consequence of a set being semialgebraic and compact
in $\R^N$ is that its boundary is a piece-wise smooth surface, i.e., the disjoint union of finitely many smooth surfaces of Hausdorff  dimension $N-1$, and some other sets of Hausdorff dimension at most $N-2$ \cite{kindler2012spherical}.

Finally we will use the standard isoperimetric inequality.
\begin{lemma}[Isoperimetric Inequality]
    \label{lemma:isoperimetric-inequality}
    Suppose $S \subset \R^{N}$ is a set with finite volume $\vol(S)$.
    Then, the surface area\footnote{Formally here `surface area' means Minkowski content, and the inequality holds for sets whose closure has finite measure. All sets we consider are bounded and have piece-wise smooth boundaries, in which case the Minkowski content is equivalent to the standard Haudorff measure.
    }
    $\area(\boundary S)$ is at least
    \begin{equation*}
        \area(\boundary S) \geq N \vol(S)^{(N-1)/N} \vol(B_1)^{1/N},
    \end{equation*}
    where $\vol(B_1) \geq \left( \frac{2}{\sqrt{N}} \right)^{N}$ is the volume of the unit ball in $N$ dimensions.
\end{lemma}
We note the following two immediate corollaries of the isoperimetric inequality:
    \begin{enumerate}
        \item For any set $S$:
        \[
        \area(\boundary S) \geq 2 \sqrt{N} \vol(S)^{(N-1)/N}.
        \]
        \item If $\boundary S$ has covering radius $\eps$, i.e. $\boundary S \subset B_\varepsilon(x)$ for some $x \in \R^N$, then:
        \[
                \area( \boundary S ) \geq \frac{N}{\varepsilon}\; \vol(S).
        \]
    \end{enumerate}


\subsection{Information Theory}

We need the following tools from information theory.
\begin{definition}
    \label{def:entropy}
    Let $X$ be a random variable over domain $\domain$.
    The \emph{entropy} of $X$ is
    \begin{equation*}
        H(X) = \sum_{x \in \domain} \Pr(X = x) \log \frac{1}{\Pr(X = x)}.
    \end{equation*}
\end{definition}

\begin{definition}
    \label{def:k-l-divergence}
    Let $X, Y$ be random variables over domain $\domain$.
    The \emph{KL-divergence} of $X, Y$ is
    \begin{equation*}
        D(X||Y) = \sum_{x \in \domain} \Pr(X = x) \log \frac{\Pr(X = x)}{\Pr(Y = x)}.
    \end{equation*}
\end{definition}



\begin{definition}
    \label{def:mutual-info}
    Let $X, Y$ be random variables over domain $\domain$.
    The \emph{mutual information} of $X, Y$ is
    \begin{equation*}
        I(X: Y) = \sum_{x, y \in \domain} \Pr((X, Y) = (x, y)) \log \frac{\Pr((X, Y) = (x, y))}{\Pr(X = x) \Pr(Y = y)}.
    \end{equation*}
\end{definition}

\begin{lemma}[Data Processing Inequality]
    \label{lemma:data-processing-inequality}
    Let $X, Y, Z$ be random variables over domain $\domain$ such that $Z$ is independent of $X$ conditioned on $Y$.
    Then
    \begin{equation*}
        I(X: Y) \geq I(X: Z).
    \end{equation*}
\end{lemma}

\subsection{Sampling Models and Adaptivity}


The problems that we consider in this paper have an underlying distribution on $\R^{N}$.
In the mean estimation problem, the distribution is a Gaussian on $\R^{N}$, while in the bias estimation problem, we have a binary product distribution on $\set{0, 1}^{N}$.

We consider two natural models of sample access.
A \emph{vector sample} consists of receiving a vector in $\R^{N}$. 
For example, in mean estimation a vector sample consists of a vector drawn from the Gaussian distribution, while in the $N$-Coin problem a vector sample consists of flipping each coin once, and thus receiving a vector in $\set{0, 1}^{N}$.
A \emph{coordinate sample} consists of receiving a real-valued sample from the marginal distribution restricted to a specific coordinate.
The algorithm can choose freely the coordinate it samples from.
In the $N$-coin problem, this corresponds to flipping a specific coin and observing its outcome.
The coordinate sample complexity is given by
the total number of coordinate samples consumed by the algorithm.
Note that an algorithm taking $m$ vector samples implies that there is an algorithm taking $Nm$ coordinate samples.

We also distinguish adaptive and non-adaptive algorithms.
Suppose the algorithm has sample access to $N$ distributions, $\distribution_1, \dotsc, \distribution_{N}$.

\begin{definition}
    \label{def:non-adaptive}
    An algorithm $\innerAlg$ is non-adaptive if there exist $\set{m_i}_{i \in [N]}$ such that $\innerAlg$ always draws $m_i$ samples from the $i$-th distribution.
\end{definition}

Otherwise, we say an algorithm is adaptive.
We remark the non-adaptive model is essentially equivalent to restricting the algorithm to drawing \emph{vector} samples, where each vector corresponds to a flip on each of the $N$ coins.
This is the more typical setup in mean estimation, discussed in \Cref{sec:isoperimetry-from-replicability} and \Cref{sec:replicability-from-isoperimetry}, where we are simply given sample access to a distribution over $\R^N$. 
In hypothesis testing, this naturally corresponds to scenarios in which a single physical test results in several outcomes, such as taking a blood sample for a suite of diseases. 
In this case, the blood sample (and all the tests the epidemiologist runs on it) corresponds to a vector sample.
On the other hand, it might be the case the epidemiologist only checks the blood sample for a single disease. This would correspond to a coordinate sample.

\section{Single Hypothesis Testing and the Coin Problem}
\label{sec:repro-hypothesis-testing}

In this section, we present our main results for single hypothesis testing. We first establish an equivalence between the replicable hypothesis testing problem and the standard replicable coin problem. We then introduce sample-adaptivity as a technique to obtain replicable algorithms with linear expected sample complexity (Theorem \ref{thm:r-adapt-coin-problem}). Finally we show a matching lower bound (for the constant success regime), introducing our general framework for replicability sample lower bounds via mutual information (Theorem \ref{thm:q0-num-lower-bound}).

\subsection{Hypothesis Testing and the Coin Problem}
For convenience, we recall our formalization of a $(p_0,q_0)$-hypothesis test.

\defhypothesistesting

At a glance, it is clear hypothesis testing is closely related to a standard paradigm in computer science, \textit{bias estimation}. For the moment, we focus on a classical promise variant called the coin problem: given a (possibly biased) coin, how many flips do we need to determine whether the coin's bias is (at most) $p_0$, or at least some $q_0 >p_0$?

\begin{restatable}[Coin Problem]{definition}{defonecoinproblem}
    \label{def:coin-problem}
    Let $0 \leq p_0 < q_0 \leq 1$ be probabilities in $[0, 1]$. 
    Let $\delta < \frac{1}{2}$.
    An algorithm $\innerAlg$ solves the $(p_0, q_0)$-coin problem if given sample access to a coin with bias $p \in [0, 1]$, $\innerAlg$ satisfies the following:
    \begin{enumerate}
        \item If $p = p_0$, then $\Pr_{S, r} (\innerAlg(S; r) = \accept) < \delta$.
        \item If $p \geq q_0$, then $\Pr_{S, r} (\innerAlg(S; r) = \reject) < \delta$.
    \end{enumerate}
    where $S \sim \distribution^{m}$ for $\innerAlg$ taking $m$ samples from $\distribution$ and $r$ is the randomness of algorithm $\innerAlg$.
\end{restatable}

The coin problem is computationally and statistically equivalent to hypothesis testing. With this in mind, beyond this point we will typically focus directly on the coin problem and its variants, giving hypothesis testing bounds as immediate corollaries.

\begin{lemma}
    \label{lem:hypo-test-to-coin}
    Any $\rho$-replicable algorithm $\innerAlg$ for the $(p_0, q_0)$-hypothesis testing problem induces a $\rho$-replicable algorithm $\outerAlg$ for the $(p_0, q_0)$-coin problem with the same sample complexity and asymptotic runtime. The converse also holds.
\end{lemma}
\noindent\textbf{Hypothesis Testing $\to$ Coin Problem:} 
\begin{proof} 
Given a $\rho$-replicable algorithm $\innerAlg$ for the $(p_0, q_0)$-hypothesis testing problem, our goal is to design a $\rho$-replicable algorithm $\outerAlg$ solving the $(p_0, q_0)$-coin problem.

    Recall in the latter problem we are given a series of $m$ coin tosses $\set{b_i}$ with $b_i \in \set{H, T}$ drawn from some $p$-biased coin. We show that over the randomness of these coins, it is possible to simulate drawing $m$ samples $\{p_i\}$ from a distribution $\distribution_p$ depending only on the bias parameter such that if $p =p_0$, $\distribution_p$ is uniform over $[0,1]$, and if $p \geq q_0$, then $\Pr_{x \sim \distribution_p} (x < p_0) \geq q_0$. Running $\innerAlg$ on the resulting samples then solves the original instance. Replicability is maintained since two independent executions over a $p$-biased coin result in two independent samples from the same distribution $\distribution_p$.

    We generate the distribution $\distribution_p$ via the following procedure. Given $b_i=H$, generate a random sample $p_i \sim \UnifD{[0, p_0]}$. Given $T$, generate a random sample from $p_i \sim \UnifD{[p_0, 1]}$. It is easy to check that if $p=p_0$, the resulting distribution over the randomness of $b_i$ is uniform, while if $p \geq q_0$, $\Pr_{x \sim \distribution} (x < p_0) \geq q_0$ as desired. Finally note that $\outerAlg$ uses the same number of samples as $\innerAlg$, and the only additional cost in runtime is generating the $p_i$ given $b_i$, which adds no asymptotic cost.
\end{proof}
\noindent\textbf{Coin Problem $\to$ Hypothesis Testing:}
\begin{proof}
    Given a $\rho$-replicable algorithm $\innerAlg$ for the $(p_0, q_0)$-coin problem, our goal is to design a $\rho$-replicable algorithm $\outerAlg$ solving the $(p_0, q_0)$-hypothesis testing problem.

    The strategy is similar. Given samples $\set{p_i}$ from the hypothesis testing distribution $\distribution$, if $p_i \leq p_0$ we let $b_i = H$, and otherwise set $b_i=T$. If $\distribution$ is uniform, then the resulting distribution over coin flips is $p_0$-biased. If $\Pr_{x \sim \distribution} (x \leq p_0) \geq q_0$, the resulting coin is at least $q_0$ biased. Replicability, sample complexity, and runtime follow as above.    
\end{proof}
While the coin problem defined above requires $p = p_0$ in the $\reject$ case (and this is necessary for the equivalence), we remark that all our algorithms hold even if we require correctness for $p \leq p_0$ and all our lower bounds hold when we only require correctness for $p = p_0$. For the rest of this work, we focus on the stricter variant that requires correctness for any $p \leq p_0$.
\subsection{The Coin Problem with Linear Overhead}

A natural strategy for coin problem (likewise hypothesis testing) is to use \textit{statistical queries}, one of the most powerful primitives in replicable algorithm design. In particular, for $x \sim \BernD{p}$ drawn from a $p$-biased coin, let $\phi(x)$ be the indicator of whether $x < p_0$. If we can estimate the expectation of $\phi(x)$ up to $\frac{q_0 - p_0}{2}$ error, we can determine whether $\E{\phi(x)} \leq p_0$ or $\E{\phi(x)} \geq q_0$. This can be done replicably in a blackbox fashion through \cite{impagliazzo2022reproducibility}'s replicable SQ-oracle, giving the following upper bound for hypothesis testing (see \Cref{app:missing-proofs-hypothesis-testing} for details).
\begin{restatable}[\cite{impagliazzo2022reproducibility, KVYZ23}]{theorem}{rstathypothesistest}
    \label{thm:r-stat-hypothesis-testing}
    Let $0 < \delta < \rho \leq 1$.
    There is an efficient $\rho$-replicable algorithm that solves $(p_0, q_0)$-coin problem with sample complexity $\bigO{\frac{\log(1/\delta)}{(q_0 - p_0)^2 \rho^2}}$.
\end{restatable}
We improve over the naive SQ approach in two senses. First, leveraging the structure of the coin problem, we improve the sample complexity when $(p_0,q_0)$ are small. Second, while worst-case quadratic dependence on $\rho$ is tight \cite{impagliazzo2022reproducibility}, we show \textit{typically} this is quite wasteful. We introduce a new algorithmic paradigm using only
\textit{linear} overhead with no loss in worst-case complexity.



\begin{restatable}{theorem}{radaptcoinproblem}
    \label{thm:r-adapt-coin-problem}
    Let $0 < \delta < \rho \leq 1$.
    Algorithm \ref{alg:r-adapt-coin-problem} is a $\rho$-replicable algorithm solving the $(p_0, q_0)$-coin problem with expected sample complexity $\bigO{\frac{q_0 \log(1/\delta)}{(q_0 - p_0)^2 \rho}}$ and worst-case complexity $\bigO{\frac{q_0 \log(1/\delta)}{(q_0 - p_0)^2 \rho^2}}$.
\end{restatable}

By \Cref{lem:hypo-test-to-coin} we immediately obtain a computationally and statistically equivalent algorithm for the $(p_0, q_0)$-hypothesis testing problem. We begin with a high level description of the algorithm.


\paragraph{Algorithm Overview}

Fix $q_0-p_0=\varepsilon$. Suppose we are given samples from distribution $\distribution = \BernD{p}$.
Inspired by the statistical query oracle of \cite{impagliazzo2022reproducibility}, consider an algorithm that takes $m$ samples, computes an empirical bias $\hat{p}$, and checks $\hat{p}$ against a random threshold $r \in (p_0, q_0)$. 
In the standard algorithm we'd take $m = \bigtO{\frac{1}{\varepsilon^2 \rho^2}}$ samples to ensure $|\hat{p} - p| \leq \rho\varepsilon$ with high probability. 
However, if $|r - p|$ is large, this is actually wasteful! In particular, a coarser estimate of $\hat{p}$ will still ensure we land on the correct side of $r$, guaranteeing replicability. Since $r$ is random and $p$ is fixed, this means the standard algorithm `typically' draws far too many samples.

With this in mind, our algorithm will instead proceed in iterations, corresponding to the coarseness of our estimate of the true bias $p$. At the $t$-th iteration, the algorithm has drawn $m_t$ samples and can guarantee $|\hat{p}-p| \leq \varepsilon_t$. If $|r - p| \geq 3 \varepsilon_t$, then with high probability any sample of size $m_t$ drawn from $\distribution$ will satisfy $|\hat{p} - r| \geq 2 \varepsilon_t$.
Detecting this, the algorithm terminates after taking $m_t$ samples and outputs $\accept$ if and only if $\hat{p} > r$.
In order to show this process is replicable we argue the algorithm outputs the same result regardless of which iteration it terminates. In particular, since we only terminate when $|\hat{p} - r| \geq 2 \varepsilon_t$, assuming $\varepsilon_t$-correctness of the estimate the algorithm only terminates when $\hat{p}, p$ are on the same side of the random threshold $r$.
We choose the number of iterations so that in the final $T$-th iteration, $\varepsilon_T < \rho \varepsilon$.
Therefore, unless $|r - p| \leq \rho \varepsilon$ which occurs with probability at most $\rho$, Algorithm \ref{alg:r-adapt-coin-problem} is replicable.

\paragraph{Analysis} 
We give a proof sketch in the regime where $p_0,q_0$ are bounded from below by some absolute constant.
After sampling a random threshold $r$, the above procedure terminates after roughly $\frac{1}{|r - p|^2}$ samples. Since $r$ is uniformly distributed in $(p_0, q_0)$, conditioned on $|r - p| \geq \rho \varepsilon$ we have that $|r - p|$ is (essentially) uniform in $(\rho \varepsilon, \varepsilon)$ (this conditioning captures the fact that our algorithm forcibly terminates when $|r - p| \geq \rho \varepsilon$).
A simple computation then yields expected sample complexity roughly
\begin{equation*}
    \frac{1}{\varepsilon} \int_{\rho \varepsilon}^{\varepsilon} \frac{1}{x^2} dx = \bigO{\frac{1}{\rho \varepsilon^2}}.
\end{equation*}


\IncMargin{1em}
\begin{algorithm}

\SetKwInOut{Input}{Input}\SetKwInOut{Output}{Output}\SetKwInOut{Parameters}{Parameters}
\Input{Sample access $S$ to distribution $\distribution=\text{Bern}(p)$. Bias thresholds $p_0 < q_0$ with $\varepsilon = q_0 - p_0$.}
\Parameters{$\rho$ replicability and $\delta$ accuracy}
\Output{$\accept$ if $p=p_0$ and $\reject$ if $p \geq q_0$}

$\delta \gets \min\left(\delta, \frac{\rho}{4}\right)$

$b \gets \frac{\rho (q_0 - p_0)}{16}$ 

$r \gets \UnifD{[p_0 + b, q_0 - b]}$

\For{$t = 1$ to $T = 4 + \log \frac{1}{\rho}$}{
    $\eps_t \gets \frac{(q_0 - p_0)}{2^{t + 2}}$ \Comment{Note: $\eps_{T} = \frac{\rho (q_0 - p_0)}{64}$ and $b = 4 \eps_T$}
    
    $m_t \gets \frac{3 q_0}{\eps_t^2} \log \frac{2 T}{\delta}$
    
    $S_t \gets (b_1, \dotsc, b_{m_t})$ is a fresh sample of size $m_t$ drawn from $\distribution$
    
    $\hat{p}_{t} \gets \frac{1}{m_t} \sum_{i = 1}^{m_t} b_i$
    
    $\hat{b}_{t}^{-} \gets \hat{p}_t - \eps_t$, $\hat{b}_{t}^{+} \gets \hat{p}_{t} + \eps_{t}$
    
    \If{$\min(|\hat{b}_{t}^{-} - r|, |\hat{b}_{t}^{+} - r|) > \eps_t$}{
        \If{$\hat{p}_t > r$}{
            \Return $\accept$
            \Comment{Note: $\hat{b}_{t}^{-} > r + \eps_t$}
        }
        \If {$\hat{p}_t \leq r$}{
            \Return $\reject$
            \Comment{Note: $\hat{b}_{t}^{+} < r - \eps_t$}
        }
    }
}

\Return $\reject$

\caption{$\rAdaptiveCoinTester(\distribution, p_0, q_0, \rho, \delta)$} 
\label{alg:r-adapt-coin-problem}

\end{algorithm}
\DecMargin{1em}
The proof of \Cref{thm:r-adapt-coin-problem} relies on the following elementary claim bounding the success probability of empirical estimation in each iteration.
\begin{claim}
    \label{claim:sample-error-bound}
    Let $\distribution = \BernD{p}$ with $p \in [p_0, q_0]$.
    For $1 \leq t \leq T$, define $E_t$ to be the event that $|\hat{p}_t - p| \geq \eps_t$.
    Define $E = \bigcup_{t = 1}^{T} E_t$ to be the event that any $E_t$ occurs.
    Then, $\Pr(E_t) < \frac{\delta}{2 T}$ and $\Pr(E) < \frac{\delta}{2}$.
\end{claim}
\Cref{claim:sample-error-bound} is by Chernoff, and proved at the end of the section. We now prove \Cref{thm:r-adapt-coin-problem}.

\begin{proof}[Proof of Correctness of Algorithm \ref{alg:r-adapt-coin-problem}] First consider the case where $\distribution = \BernD{p}$ for $p \leq p_0$.
    For any iteration $t$, Algorithm \ref{alg:r-adapt-coin-problem} will $\accept$ only if
    \begin{equation*}
        \hat{b}_{t}^{-} > r + \eps_{t} \implies \hat{p}_t > r + 2 \eps_t \geq p_0 + 2 \eps_t \geq p + 2 \eps_t.
    \end{equation*}
    By \Cref{claim:sample-error-bound}, this occurs in any iteration with probability at most $\frac{\delta}{2}$.
    
    Similarly if $\distribution = \BernD{p}$ for $p \geq q_0$, Algorithm \ref{alg:r-adapt-coin-problem} will $\reject$ only if either in some iteration
    \begin{equation*}
        \hat{b}_{t}^{+} < r - \eps_{t} \implies \hat{p}_t < r - 2 \eps_t \leq q_0 - 2 \eps_t \leq p - 2 \eps_t,
    \end{equation*}
    or if the algorithm fails to terminate before round $T$. By \Cref{claim:sample-error-bound}, the former occurs in any iteration with probability at most $\frac{\delta}{2}$. For the latter, conditioned on $\overline{E}$ we have $\hat{p}_{T} \geq p - \eps_T$ and $b = 4 \eps_T$, so  
    \begin{equation}
        \label{eq:adaptive-eps-gap}
        \hat{b}_{T}^{-} - r \geq (\hat{p}_T - \eps_T) - (q_0 - b) = \hat{p}_T - q_0 + 3 \eps_T \geq 2 \eps_T 
    \end{equation}
   and Algorithm \ref{alg:r-adapt-coin-problem} terminates (and in particular outputs $\accept$) by iteration $T$ as desired.
\end{proof}

\begin{proof}[Proof of Replicability of Algorithm \ref{alg:r-adapt-coin-problem}]
    Fix $p \in [0, 1]$ and consider $\distribution = \BernD{p}$. By our assumption $\delta \leq \rho$, we may actually assume without loss of generality $\rho \geq 4 \delta$ (losing only constant factors). By correctness and $\rho \geq 4 \delta$, we have that Algorithm \ref{alg:r-adapt-coin-problem} is $\rho$-replicable on distributions $\distribution = \BernD{p}$ for all $p \notin (p_0, q_0)$.
    Assume therefore that $p \in (p_0, q_0)$.
    Define $m = \sum_{t = 1}^{T} m_t$ the worst case sample complexity and let $S_0, S_1 \sim \distribution^{m}$ denote two independent sample sets of size $m$ drawn from $\distribution$.
    Note that Algorithm \ref{alg:r-adapt-coin-problem} may only look at some prefix of $S_0, S_1$.
    Since $p = r$ with probability zero, so we can disregard this case.
    Again, we condition on event $E$ not occurring, bounding this probability by $\frac{\delta}{2}$ with \Cref{claim:sample-error-bound}.
    
    First, assume $p < r$. We argue the algorithm outputs $\reject$ with high probability across both runs. Namely, for any iteration $t$, Algorithm \ref{alg:r-adapt-coin-problem} only outputs $\accept$ if
    \begin{equation*}
        \hat{b}_t^{-} = \hat{p}_t - \eps_t \leq p < r. 
    \end{equation*}
    In other words, Algorithm \ref{alg:r-adapt-coin-problem} only outputs $\accept$ if $E$ occurs. Thus by \Cref{claim:sample-error-bound}, both runs of Algorithm \ref{alg:r-adapt-coin-problem} output $\reject$ with probability at least $1-\delta \geq 1 - \frac{\rho}{4}$.

    Now, assume $p > r$. We argue Algorithm \ref{alg:r-adapt-coin-problem} outputs $\accept$ with high probability across both runs. Toward this end, observe that if $|p-r| > 3\varepsilon_T$, then    
    \begin{equation}
        \label{eq:round-t-terminate-gap}
        \hat{b}_T^{-} - r \geq (\hat{p}_t - \eps_T) - (p + 3 \eps_T) = \hat{p}_t - p + 2 \eps_T \geq \eps_T.
    \end{equation}
    In particular, whenever $r < p - 3 \eps_T$, the algorithm outputs $\accept$ if $E$ does not occur.
    Union bounding over both samples, Algorithm \ref{alg:r-adapt-coin-problem} outputs $\accept$ on both runs with probability at least $1 - \delta \geq 1 - \frac{\rho}{4}$.
    Finally, it remains to bound the probability that $|r - p| \leq 3 \eps_T$.
    This can be upper bounded by
    \begin{align*}
        \frac{6 \eps_{T}}{(q_0 - p_0) - 2b} &= \frac{6 \eps_{T}}{(q_0 - p_0) - 8 \eps_{T}} \\ 
        &= \frac{\frac{3}{32} \rho}{1 - \frac{\rho}{8}} \\
        &= \frac{3 \rho}{32 - 4 \rho} \leq \frac{\rho}{10}
    \end{align*}
    for all $\rho \leq \frac{1}{2}$, proving replicability.
\end{proof}

\begin{proof}[Proof of Sample Complexity of Algorithm \ref{alg:r-adapt-coin-problem}]
    Fix a distribution $\distribution$ with parameter $p$, and condition on the event $\overline{E}$ as in correctness and replicability. Notice that since the samples per iteration $m_t$ increases geometrically, the sample complexity of \Cref{alg:r-adapt-coin-problem} is $O(m_{t_{\text{term}}})$, where $t_{term} \leq T$ is the iteration in which the algorithm terminates.
    
    Fix an iteration $t$. Then conditioned on $\overline{E}$, as in Equation \ref{eq:round-t-terminate-gap} we have
    \[
    |p - r| \geq 3 \eps_t \quad \implies \quad \min(|\hat{b}_{t}^{-} - r|, |\hat{b}_{t}^{+} - r|) > \eps_t.
    \]
    With this in mind, let 
    $t^*$ denote the smallest $t$ such that $|p - r| \geq 3 \eps_t$ and observe that $t_{\text{term}} \leq t^*$. Since $t^*$ is the first such round, we can bound its value in terms of the distance $|r-p|$, namely:
    \begin{align*}
        |r - p| &\leq 3 \eps_{t^*-1} = 3 \cdot \frac{(q_0 - p_0)}{2^{t^*+1}} \quad \implies \quad t^* \leq \log \frac{3 (q_0 - p_0)}{|r - p|}.
    \end{align*}
Since $t^* \geq t_{\text{term}}$, our sample complexity is therefore bounded by
    \begin{equation*}
        O(m_{t^*}) = \bigO{4^{t^*} \cdot \frac{q_0}{(q_0 - p_0)^2} \log \frac{2T}{\delta}} = \bigO{\frac{q_0}{|r - p|^2} \log \frac{2T}{\delta}}.
    \end{equation*}
Note that in the above we have ignored the case where Algorithm \ref{alg:r-adapt-coin-problem} fails to terminate before round $T$, in particular when $|r-p| \leq 3 \varepsilon_{T}$. 
Conditioned on $\overline{E}$, this is a low probability event
    \begin{equation*}
        \Pr\left(\min\left(|\hat{b}_{T - 1}^{-} - r|, |\hat{b}_{T - 1}^{+} - r|\right) \leq \eps_{T - 1} \right) \leq \Pr\left( |r - p| \leq 3 \eps_{T} \right) \leq \frac{6 \eps_{T}}{(q_0 - p_0) - 2b} \leq \frac{3 \rho}{32 - 4 \rho} \leq \frac{\rho}{10}.
    \end{equation*}

    Assume therefore that $|r - p| \geq 3 \eps_{T}$.
    Conditioned on $|r - p| \geq 3 \eps_{T}$, $r$ is uniformly distributed on the remaining interval $I = (p_0 + b, p - 3 \eps_{T}) \cup (p + 3 \eps_{T}, q_0 - b)$.
    Denote $I_1 = (p_0 + b, p - 3 \eps_{T})$, $I_2 = (p + 3 \eps_{T}, q_0 - b)$, and let $L \geq (q_0 - p_0) - 2b - 6 \eps_{T}$ be the length of interval $I$.
    Then there is some universal constant $C$ such that the sample complexity $m$ is

    \begin{align*}
        \E{m} = \int_{I} \frac{C q_0}{|r - p|^2} \log \frac{2T}{\delta} \cdot \frac{1}{L} dr = \frac{C q_0}{L} \log \frac{2T}{\delta} \int_{I} \frac{1}{|r - p|^2} dr.
    \end{align*}

    Without loss of generality, assume $p < \frac{p_0 + q_0}{2}$ so that $|I_2| \geq |I_1|$ and apply the change of variable $s = r - p$ so that
    \begin{align*}
        \E{m} &\leq \frac{2 C q_0}{L} \log \frac{2T}{\delta} \int_{p + 3 \eps_{T}}^{q_0} \frac{1}{|r - p|^2} dr \\
        &= \frac{2 C q_0}{L} \log \frac{2T}{\delta} \int_{3 \eps_{T}}^{q_0 - p} \frac{1}{s^2} ds \\
        &= \frac{2 C q_0}{L} \log \frac{2T}{\delta} \left( \frac{1}{3 \eps_{T}} - \frac{1}{(q_0 - p)} \right) \\
        &\leq \frac{2 C q_0}{3 L \eps_{T}} \log \frac{2T}{\delta}.
    \end{align*}

    Finally we lower bound $L$ as
    \begin{equation*}
        L \geq (q_0 - p_0) - 20 \eps_T = (q_0 - p_0) \left( 1 - \frac{5}{8} \rho \right) \geq \frac{1}{2} (q_0 - p_0).
    \end{equation*}

    Combining, we have that whenever $E$ does not occur and $|r - p| \geq 3 \eps_{T}$, the expected sample complexity is at most
    \begin{equation*}
        \E{m} \leq \bigO{\frac{q_0}{(q_0 - p_0)^2 \rho} \log \frac{2T}{\delta}} = \bigO{\frac{q_0}{(q_0 - p_0)^2 \rho} \log \frac{1}{\delta}}.
    \end{equation*}

    The worst case sample complexity is
    \begin{equation*}
        O(m_T) = \bigO{\frac{q_0}{(q_0 - p_0)^2 \rho^2} \log \frac{1}{\delta}}.
    \end{equation*}

    But this occurs with probability at most $\frac{\delta}{10} + \frac{\rho}{5} < \frac{\rho}{2}$, so that the overall expected sample complexity is as claimed.
\end{proof}


\begin{proof}[Proof of \Cref{claim:sample-error-bound}]
    Fix an iteration $t$.
    Define
    \begin{equation*}
        \delta_t = \sqrt{\frac{3}{m_t p} \log \frac{2 T}{\delta}}.
    \end{equation*}
    By a simple application of the Chernoff Bound
    \begin{equation*}
        \Pr\left(|\hat{p}_t - p|> \delta_t p \right) < \frac{\delta}{2 T}.
    \end{equation*}
    By our choice of $m_t$
    \begin{equation*}
        p \delta_t = \sqrt{\frac{3 p}{m_t} \log \frac{2 T}{\delta}} \leq \eps_t,
    \end{equation*}
    proving $\Pr(E_t) < \frac{\delta}{2T}$. 
    The upper bound on $\Pr(E)$ follows from the union bound.
\end{proof}

\subsection{Lower Bound for Replicability via Mutual Information}

In this section, we prove a near-matching lower bound for \Cref{thm:r-adapt-coin-problem}. 
In the process, we introduce useful tools from information theory that are reused when we prove lower bounds against replicable algorithms for other problems. 
\begin{restatable}{theorem}{onecoinlowerbound}
    \label{thm:q0-num-lower-bound}
    Let $p_0 < q_0 < \frac{1}{2}$.
    Let $\rho, \delta < \frac{1}{16}$.
    Any $\rho$-replicable algorithm $\innerAlg$ solving the $(p_0, q_0)$-coin problem requires expected sample complexity at least  $\bigOm{\frac{q_0}{\rho (q_0 - p_0)^2}}$ and worst-case sample complexity at least $\bigOm{\frac{q_0}{\rho^2 (q_0 - p_0)^2}}$.
\end{restatable}
We remark it suffices to show the worst-case bound
as any algorithm with expected sample complexity 
$m$ can be easily transformed into one with worst case sample complexity $O(m / \rho)$.
In particular, during the runtime, we can forcibly terminate an algorithm if it consumes more than 
$\frac{1}{\rho}$ times its expected sample size, and then run a non-replicable coin tester.
By Markov's inequality, such forcible terminations happen with probability at most $\rho$, making the modified algorithm lose a constant factor in replicability while maintaining its correctness. 

We begin with a brief proof overview of Theorem \ref{thm:q0-num-lower-bound}.

\paragraph{Proof Overview}
Consider a replicable algorithm $\innerAlg$ on $m$ samples.
Let $\distribution$ be a distribution over coins, where the coin bias is sampled uniformly from the interval $(p_0, q_0)$.
To obtain a sample complexity lower bound, we 
argue that
a random coin from $\distribution$ forces $\innerAlg$ to be non-replicable with probability at least $\Omega(\rho)$
unless $m$ is sufficiently large.
Let $r$ be the random string representing the internal randomness of $\innerAlg$.
We first use a minimax style argument to fix a ``good'' random string $r$ such that
$\innerAlg(; r)$ becomes deterministic algorithm that is $O(\rho)$-replicable and correct with probability at least $1 - O(\delta)$
against a random coin instance from $\distribution$.
For such a good random string,
we identify a bias $p_r$ such that the probability $\innerAlg(; r)$ outputs $\accept$ given a $p_r$-biased coin is exactly $\frac{1}{2}$. $\innerAlg(; r)$ clearly fails replicability on this coin, but we need to show this is true for an $\Omega(\rho)$ fraction of biases from $\distribution$. We do this by extending 
the non-replicable region into an interval of length $\Omega(\rho(q_0-p_0))$. 
In particular, if $\innerAlg(;r)$ were replicable 
under a bias $p_r'$ within this range, then $\mathcal{A}(;r)$ would give a procedure to distinguish between $p_r$ and $p_r'$. Roughly speaking, we argue this is information theoretically impossible when $m \leq \frac{q_0}{\rho^2(q_0-p_0)}$, giving the desired bound.
\medskip
\\
\indent More formally, to prove the existence of our non-replicable interval we use mutual information to show that a small change in the bias of the input coin does not significantly affect the probability that $\innerAlg(; r)$ outputs $\accept$. To this end we rely on the following two lemmas. 
The first states that any function with constant advantage predicting a uniformly random bit $X$ from a correlated variable $A$ implies $X$ and $A$ have $\Omega(1)$ mutual information. 
In our context, this will imply that the sample set must contain $\Omega(1)$ information about the underlying distribution.
\begin{lemma}
    \label{lemma:alg-mutual-info-lower-bound}
    Let $X \sim \text{Bern}(\frac{1}{2})$ and $A$ be a random variable possibly correlated with $X$.
    If there exists a (randomized) function $f$ so that $f(A) = X$ with at least 51\% probability, then $I(X:A) \geq 2 \cdot 10^{-4}$. 
\end{lemma}

\begin{proof}
    We give the proof of this standard fact for completeness (see for example \cite{diakonikolas2016testing}).

    The conditional entropy $H(X|f(A))$ is the expectation over $f(A)$ of $h(q)$ where $h(q)$ is the binary entropy function and $q$ is the probability that $X = f(A)$ given $f(A)$.
    Since $\E{q} \geq 0.51$ and $h$ is concave, $H(X|f(A)) \leq h(0.51) < 1 - 2 \cdot 10^{-4}$.
    Then, by the data processing inequality
    \begin{equation*}
        I(X:A) \geq I(X:f(A)) \geq H(X) - H(X|f(A)) \geq 2 \cdot 10^{-4}.
    \end{equation*}\qedhere

\end{proof}

On the other hand, given $\eta$-close biases $p_1$ and $p_2$, we can upper bound the mutual information between the samples and the underlying distribution by $O(m \eta^2)$.
This ensures that the acceptance probabilities of the algorithm on two $\eta$-close coins cannot be too different unless
$m \gg \frac{1}{\eta^2}$.

\begin{lemma}
    \label{lemma:one-coin-mutual-info-bound}
    Let $m \geq 0$ be an integer and $0 \leq a < b \leq 1$.
    Let $X \sim \text{Bern}(\frac{1}{2})$ and $Y$ be distributed according to $\BinomD{m}{a}$ if $X = 0$ and $\BinomD{m}{b}$ if $X = 1$.
    Then
    \begin{equation*}
        I(X:Y) = \bigO{\frac{m (b - a)^2}{\min(a, b, (1 - a), (1 - b))}}.
    \end{equation*}
\end{lemma}
We defer the proof to the end of the section. We now give the formal argument.
\begin{proof}[Proof of Theorem \ref{thm:q0-num-lower-bound}]
    Let $\innerAlg$ be an algorithm on $m$ samples.
    For any $p \in [0, 1]$, 
    we denote by $S_p$ a dataset of $m$ \iid samples from $\BernD{p}$.
    Let $r$ be a random string representing the internal randomness of $\innerAlg$.
    Observe that by Markov's inequality and the correctness and replicability of $\innerAlg$, for any fixed distribution $\distribution$ over $[0,1]$, the following three properties hold for many random strings:
    \begin{enumerate}
        \item Correctness at $p=p_0$: 
        \[
        \Pr_r \left( \Pr_{S_{p_0}} \left( \innerAlg(S_{p_0}; r) = 1 \right) > 4 \delta \right) < \frac{1}{4}.
        \]
        \item Correctness at $p=q_0$:
        \[
        \Pr_r \left( \Pr_{S_{q_0}} \left( \innerAlg(S_{q_0}; r) = 0 \right) > 4 \delta \right) < \frac{1}{4}.
        \]
        \item Distributional replicability: 
        \[
        \Pr_r \left( \Pr_{p \sim \distribution,S_p,S_{p}'} \left( \innerAlg(S_{p}; r) \neq \innerAlg(S_{p}'; r) \right) > 4 \rho \right) < \frac{1}{4}.
        \]
    \end{enumerate}
    We refer to any string satisfying the negations of all three (inner) conditions
    {\bf $\distribution$-good}. 
    Note that 
    a simple union bound implies at least $\frac{1}{4}$ of random strings are $\distribution$-good for any $\distribution$.

    Our goal is to find a distribution $\distribution$ over coins such that for any $\distribution$-good string $r$, $\innerAlg(;r)$ fails to be replicable with probability $\Omega(\frac{\sqrt{q_0}}{\sqrt{m}(q_0-p_0)})$:
    \[
    \Pr_{p\sim \distribution, S_p,S'_p}\left( \innerAlg(S_{p}; r) \neq \innerAlg(S_{p}'; r) \right) > \Omega\left(\frac{\sqrt{q_0}}{\sqrt{m}(q_0-p_0)}\right).
    \]
    Since $\innerAlg(;r)$ is guaranteed to be $O(\rho)$-replicable over $p \sim \distribution$ by Property (3), this forces $m \geq \Omega(\frac{q_0}{(q_0-p_0)^2\rho^2})$.

    More specifically, based on whether $p_0 \geq q_0 / 100$, we actually provide different constructions of $\distribution$ such that 
    $\innerAlg(;r)$ fails to be replicable with this probability. We build these distributions leveraging the following two core claims.    

    Fix any distribution $\distribution$ over potential coin biases.
    Following \cite[Lemma 7.2]{impagliazzo2022reproducibility}, we first argue that for any $\distribution$-good random string $r$, there exists a bias $p_r$ such that $\innerAlg(S_{p_r};r)$ accepts with probability exactly $1/2$ where the randomness is over $S_{p_r}$.
    \begin{claim}
    \label{clm:p_r-exist}
    For any $\distribution$-good random string $r$,
     there exists some $p_r$ such that 
    $\Pr(\innerAlg(S_{p_r}; r) = \accept) = \frac{1}{2}$.        
    \end{claim}
    \begin{proof}
    Fix a $\distribution$-good string $r$. 
    The algorithm $\innerAlg(; r)$ is then a deterministic mapping from possible samples to $\set{\reject, \accept}$ whose acceptance probability given samples from $\BernD{p}$ is
    \begin{equation*}
        f(p) = \Pr(\innerAlg(S_p; r) = \accept) = \sum_{j = 0}^{m} a_j \binom{m}{j} p^j (1 - p)^{m - j}
    \end{equation*}
    where $a_j$ denotes the proportion of strings of hamming weight $j$ that $\innerAlg(; r)$ accepts.
    Since $f(p)$ is continuous in $p$ with $f(p_0) = \Pr(\innerAlg(S_{p_0}; r) = \accept) < 4 \delta < \frac{1}{4}$ and $f(q_0) = \Pr(\innerAlg(S_{q_0}; r) = \accept) > 1 - 4 \delta > \frac{3}{4}$, the intermediate value theorem implies the existence of some bias $p_r \in (p_0,q_0)$ for which $f(p_r) = \Pr(\innerAlg(S_{p_r}; r) = \accept) = \frac{1}{2}$ as desired.
    \end{proof}

    To build our hard distributions, we extend this single unreplicable bias to a full interval using our mutual information bounds.

    \begin{claim}
    \label{clm:I_r-exist}
    For any $\distribution$-good random string $r$,
     there exists some $p_r \in (p_0, q_0)$ and an interval $I_r:= [p_r, p_r + c \sqrt{p_r/m}]$ 
     for some constant $c$
     satisfying the following: for any $p \in I_r$ it holds that
    $\Pr(\innerAlg(S_{p}; r) = \accept)
    \in (1/3, 2/3)$.                    
    \end{claim}
    \begin{proof}
    Let $p_r$ be the non-replicable point guaranteed by \Cref{clm:p_r-exist}.
    Fix some $b > p_r$ and define $X$ to be a uniformly random bit and $Y$ distributed according to $\BinomD{m}{p_r}$ if $X = 0$ and $\BinomD{m}{b}$ if $X = 1$. By \Cref{lemma:one-coin-mutual-info-bound}, the mutual information between $X$ and $Y$ is at most
    \begin{equation*}
        I(X:Y) = \bigO{\frac{m (b - p_r)^2}{p_r}}.
    \end{equation*}
    Consider any $b \leq p_r + c \sqrt{\frac{p_r}{m}}$ for $c>0$ some sufficiently small constant. Then the mutual information satisfies
    $I(X:Y) < 2 \cdot 10^{-4}$ and by \Cref{lemma:alg-mutual-info-lower-bound} there is no function $f$ such that $f(Y)=X$ with probability at least $51\%$. 
    
    On the other hand, if $|g(b)-g(p_r)| > \frac{1}{10}$ there is an elementary distiguisher $f$, namely $\innerAlg$ itself. In particular, assume without loss of generality that $g(b) > g(p_r)$ (else take $1-f$ in what follows), and define $f(Y)=\innerAlg(Y;r)$. Then $\Pr[f(Y)=1 | X=1] = g(b) \geq \frac{3}{5}$ and $\Pr[f(Y)=0 | X=0] = g(p_r)=\frac{1}{2}$. Since $X$ is unbiased, we then have
    \[
    \Pr[f(Y)=X] \geq \frac{1}{2}\Pr[f(Y)=1 | X=1] + \frac{1}{2}\Pr[f(Y)=0 | X=0] \geq \frac{11}{20},
    \]
    violating \Cref{lemma:alg-mutual-info-lower-bound}. Note there is a subtlety here that $Y$ is a $p$-biased binomial variable, while $\mathcal{A}$ takes a $p$-biased $m$-wise Bernoulli. This is handled formally by passing a random $m$-bit string of hamming weight $Y$ into $\mathcal{A}$ instead of $Y$ itself, which is then distributed as the desired Bernoulli.
    \end{proof}


    We are finally ready to construct $\distribution$ for our two cases:
    \paragraph{Case 1: $\mathbf{p_0 \geq \frac{q_0}{100}}$.}
    Take $\distribution$ to be the uniform distribution over $(p_0, q_0)$.
    By \Cref{clm:p_r-exist},
    for every $\distribution$-good random string $r$,
    there exists some interval $I_r$ 
    of length $ \Omega \lp( \sqrt{ \frac{p_0}{m} } \rp) = \Omega \lp( \sqrt{ \frac{q_0}{m} } \rp)$
    such that $\innerAlg(;r)$ is constantly non-replicable under any coin with bias $p \in I_r$.    
    Since this non-replicable interval has mass 
    $\Omega\lp(\frac{\sqrt{ q_0 } }{(q_0 - p_0) \sqrt{m} }\rp)
    $ over $\distribution$,
    the non-replicable probability lower bound follows.



    
    \paragraph{Case 2: $\mathbf{p_0 < \frac{q_0}{100}}$.}
    We construct a new adversary distribution as follows.
    Consider the function $g(x) = x + c\sqrt{\frac{x}{m}}$ for some sufficiently small constant $c$, and the set $P = \set{p_0, g(p_0), g^{(2)}(p_0), \dotsc, g^{(T)}(p_0)}$ where $g^{(k)}$ denotes $g$ applied $k$ times and $T$ is the largest integer such that $g^{(T)}(p_0) \leq q_0$. 
    Our coin distribution $\distribution$ will be uniform over $P$. It is left to prove that for any $\distribution$-good random string $r$, 
    $\innerAlg(;r)$ is (constantly) non-replicable under
    an $\Omega(\frac{1}{\sqrt{m q_0}}) = \Omega(\frac{\sqrt{q_0}}{\sqrt{m}(q_0-p_0)})$ fraction of biases from $P$. 

    Let $p_r, I_r$ be the non-replicable point and non-replicable interval guaranteed by \Cref{clm:p_r-exist} and \Cref{clm:I_r-exist} respectively.
    We first claim that $|P \cap I_r| \geq 1$
    for any $\distribution$-good $r$.
    Fix such an $r$, and let $t_r$ be such that $p_r \in \left[ g^{(t_r)}(p_0), g^{(t_r + 1)}(p_0) \right]$.
    It suffices to show that $g^{(t_r + 1)}(p_0) \in I_r$.
    This follows as $g$ is an increasing function:
    \begin{equation}
    \label{eq:g-increase}
        p_r < g^{(t_r + 1)}(p_0) = g\left( g^{(t_r)}(p_0) \right) < g(p_r).
    \end{equation}
Recall that $I_r$ is an interval of length at least 
$ \Omega\lp( \sqrt{\frac{p_r}{m}} \rp) $ starting at $p_r$, and we define $g(x) = x + c \sqrt{x/m}$. 
We thus have $[p_r,g(p_r)] \subset I_r$ when $c$ is sufficiently small.
Combining this with \Cref{eq:g-increase} then shows that $g^{(t_r + 1)}(p_0) \in I_r$.

It is now sufficient to show that $|P| \leq O(\sqrt{m q_0})$, since $I_r \cap P$ then accounts for at least a $\Omega\lp( \frac{1}{ \sqrt{m q_0} } \rp)$ fraction of $P$ and the non-replicable probability lower bound follows. Let $K$ be the largest integer such that $2^K p_0 \leq q_0$.
    Thus, $K \leq \log \frac{q_0}{p_0}$.
    Let $J_k = [2^k p_0, 2^{k + 1} p_0]$ for $0 \leq k \leq K$ and define $T_k = | P \cap J_k|$. 
    We first upper bound $T_k$. 
    Since $g$ is increasing in $x$, then points in $P$ are at least $c \sqrt{\frac{2^k p_0}{m}}$ apart in the interval $J_{k}$.
    Thus, in an interval of length $2^{k} p_0$, there are at most
    \begin{equation*}
        T_k \leq \frac{2^{k} p_0 \sqrt{m}}{c \sqrt{2^{k} p_0}} = \frac{\sqrt{2^{k} p_0 m}}{c}
    \end{equation*}
    points.
    Then, summing over all $T_k$, we obtain the bound
    \begin{equation*}
        T = \sum_{k = 0}^{K} T_k \leq \frac{\sqrt{m p_0}}{c} \sum_{k = 0}^{K} 2^{k/2} = \frac{\sqrt{m p_0}}{c} \frac{2^{K/2} - 1}{\sqrt{2} - 1} = \bigO{\sqrt{m p_0} \sqrt{\frac{q_0}{p_0}}} = \bigO{\sqrt{m q_0}}
    \end{equation*}
as desired. This concludes the proof of Case 2 and \Cref{thm:q0-num-lower-bound}.
\end{proof}

It is left to prove the key lemma upper bounding mutual information between the underlying distribution and the observed samples as a function of the number of samples taken by the algorithm.
\begin{proof}[Proof of \Cref{lemma:one-coin-mutual-info-bound}]
    We proceed by the following identity:
    \begin{equation*}
        I(X:Y) = H(X) + H(Y) - H(X, Y).
    \end{equation*}
    Since $X$ is a uniformly random bit, $H(X) = 1$.
    Next,
    \begin{align*}
        H(X, Y) &= - \sum_{b} \sum_{j} \Pr(X = b, Y = j) \log \frac{1}{\Pr(X = b, Y = j)} \\
        &= \sum_{j} \frac{\Pr(Y = j | X = 0)}{2} \log \frac{2}{\Pr(Y = j | X = 0)} + \frac{\Pr(Y = j | X = 1)}{2} \log \frac{2}{\Pr(Y = j | X = 1)} \\
        &= 1 + \frac{1}{2} \left( H(Y|X = 0) + H(Y|X = 1) \right).
    \end{align*}
    Finally, we compute the entropy of $Y$ as
    \begin{align*}
        H(Y) &= \sum_{j} \Pr(Y = j) \log \frac{1}{\Pr(Y = j)} \\
        &= \sum_{j} \frac{\Pr(Y = j | X = 0) + \Pr(Y = j | X = 1)}{2} \log \frac{1}{\Pr(Y = j)}.
    \end{align*}
    Combining the three terms, we get
    \begin{align*}
        I(X:Y) &= \frac{1}{2} \left( H(Y|X = 0) + H(Y|X = 1) \right) - H(Y) \\
        &= \sum_{b} \sum_{j} \frac{\Pr(Y = j|X = b)}{2} \log \frac{\Pr(Y = j|X = b)}{\Pr(Y = j)} \\
        &= \frac{1}{2} \left( D(\distribution_{0} || \distribution_{1/2}) + D(\distribution_{1} || \distribution_{1/2}) \right),
    \end{align*}
    where $\distribution_{0} \sim \BinomD{m}{a}, \distribution_{1} \sim \BinomD{m}{b}$ and $\distribution_{1/2}$ is the mixture of $\distribution_{0}, \distribution_{1}$ with weight $\frac{1}{2}$ each.
    Since $\distribution_{1/2} = \frac{\distribution_{0} + \distribution_{1}}{2}$, we apply the convexity of KL divergence, so
    \begin{align*}
        I(X:Y) &\leq \frac{1}{2} \left( D(\distribution_{0} || \distribution_{1}) + D(\distribution_{1} || \distribution_{0}) \right) \\
        &= \frac{m}{2} \left( D(a || b) + D(b || a) \right) \\
        &= \bigO{\frac{m (b - a)^2}{\min(a, b, (1 - a), (1 - b))}},
    \end{align*}
    where $D(a||b)$ denotes the KL divergence of two Bernoulli random variables with parameters $a, b$, and in the last line we use the elementary inequality $\ln x \leq x - 1$ for all $x > 0$.
\end{proof}

\section{Isoperimetric Tiling and Replicable Learning}
\label{sec:isoperimetry-from-replicability}

In Section \ref{sec:repro-hypothesis-testing}, we characterized the sample complexity of testing a single coin/hypothesis. 
In many cases, however, we may want to simultaneously test many hypotheses, or simultaneously conduct a large number of statistical inference tasks.
The multiple hypothesis testing problem is a fundamental and well studied problem in the hypothesis testing literature, including error controlling procedures for family-wise error rate \cite{tukey1949, vsidak1967, holm1979, wilson2019hmp} and false discovery rate \cite{bonferroni1935calcolo, benjaminihochberg}.
See Shafer for a review of the classical results in the multiple hypothesis testing literature \cite{shaffer1995multiple}.

Returning to our running practical example, an epidemiologist hopes to determine the most prevalent diseases in a population in order to better understand public health/prioritize pharmaceutical development.
Suppose the scientist conducts the study applying a $0.05$-replicable hypothesis testing algorithm to analyze each treatment.
If the scientist tests the drug against $100$ diseases, we expect to see that the effectiveness of the drug fails to replicate for $5$ diseases through random chance alone, showing how replicability naturally degrades as the scale of the statistic tasks increase.

In this section, we study the computational and statistical complexity of algorithms that replicate with high probability across \emph{every} tested hypothesis. Similar to the single hypothesis setting, we start by equating basic multi-hypothesis testing with the classical problem of \textit{high dimensional mean estimation}. 
We then show sample and computationally efficient replicable mean estimation is itself equivalent to the efficient construction of \textit{low surface area tilings of $\R^N$}. 
This connection allows us to derive tight upper and lower bounds on the sample complexity of replicable mean estimation and multi-hypothesis testing via the isoperimetric inequality, as well as to tie the construction of efficient sample-optimal algorithms to longstanding open problems in TCS and high dimensional geometry.

In this and the following section, we focus on the vector sample and \textit{non-adaptive} coordinate sample models (\Cref{def:non-adaptive}). 
We remark that our worst-case lower bounds for vector-sample algorithms imply lower bounds for adaptive vector-sample algorithms by applying Markov's inequality as in \ref{thm:q0-num-lower-bound}.
We study the adaptive coordinate sample model further in \Cref{sec:adaptive-replicability} and \Cref{sec:efficient-n-coin-problem}, and also discuss how to reduce the expected sample complexity in the adaptive vector sample model.
\subsection{From Multi-hypothesis Testing to High Dimensional Mean Estimation}


We start by defining the formal notion of a multi-hypothesis test:
\begin{definition}[Multi-Hypothesis Test]
    \label{def:multi-hypothesis-testing}
    Let $0 \leq p_0 < q_0 \leq 1$,
    $\delta \in (0, 1/2)$,
    and $N  \in \mathbb Z_+$.
    An algorithm $\innerAlg$ solves the $N$-hypothesis testing problem if given sample access to (possibly correlated) distributions $\{\distribution_i\}_{i=1}^N$ on $[0, 1]$, $\innerAlg$ outputs a set $\calO \subset [N]$ that satisfies the following:
    \begin{equation*}
        \Pr(i \in \calO \textrm{\xspace for any uniform $\distribution_i$} \orT i \not\in \calO \textrm{\xspace for any $\Pr_{x \sim \distribution_i}(x < p_0) \geq q_0$}) < \delta.
    \end{equation*}
\end{definition}
Following the equivalence between single hypothesis testing and coin problem established in Section~\ref{sec:repro-hypothesis-testing}, 
we can study replicability of multi-hypothesis via
the $N$-coin problem.
\begin{restatable}[$N$-Coin Problem]{definition}{ncoinproblem}
    \label{def:n-coin-problem}
    Let $0 \leq p_0 < q_0 \leq 1$,
    $\delta \in (0, 1/2)$,
    and $N  \in \mathbb Z_+$. 
    An algorithm $\innerAlg$ solves the $N$-coin problem if given sample access\footnote{We will be clear about the sampling model when we state the sample complexity of the problem.} to $N$ (possibly correlated) coins with bias $\set{p_i}_{i = 1}^{N}$ with $p_i \in [0, 1]$, $\innerAlg$ outputs a set $\calO \subset [N]$ that satisfies the following:
    \begin{equation*}
        \Pr(i \in \calO \textrm{\xspace for any $p_i \leq p_0$} \orT i \not\in \calO \textrm{\xspace for any $p_i \geq q_0$}) < \delta.
    \end{equation*}
\end{restatable}
By exactly the same argument as \Cref{lem:hypo-test-to-coin}, the $N$-coin problem is computationally and statistically equivalent to multi-hypothesis testing.



The $N$-coin problem as we have defined it is a distributional \emph{testing} problem. An equally fundamental variant of the coin problem is distribution \emph{learning}. In other words, rather than test if the biases satisfy certain constraints, we want to learn the biases up to some small error.

\begin{restatable}[Learning $N$-Coin Problem]{definition}{linfncoinproblemlearning}
    \label{def:l-inf-learning-n-coin-problem}
    Let $\delta \in (0, 1/2)$, $\eps > 0$, $c \in [1,\infty]$, and $N  \in \mathbb Z_+$. 
    An algorithm $\innerAlg$ solves the $\ell_c$ learning $N$-coin problem if given sample access to $N$ (possibly correlated) coins with bias $\set{p_i}_{i = 1}^{N}$ where $p_i \in [0, 1]$, $\innerAlg$ outputs a vector $\hat{p} \in [0, 1]^{N}$ such that
    \begin{equation*}
        \Pr(\norm{\hat{p} - p}{c} \geq \varepsilon) < \delta.
    \end{equation*}
\end{restatable}
The $N$-Coin problem is essentially equivalent to learning biases in $\ell_\infty$ up to logarithmic factors. In particular, any replicable algorithm that learns biases up to error $\frac{q_0-p_0}{2}$ in $\ell_{\infty}$-norm clearly solves the $N$-coin problem, and conversely any replicable algorithm that solves the $N$-Coin Problem can be used to learn biases up to error $\varepsilon$ in $\ell_{\infty}$-norm via binary search.

\begin{restatable}{lemma}{testBinarySearchLearn}
    \label{lemma:l-inf-tester-implies-learner}
    Suppose there is a $\rho$-replicable algorithm $\innerAlg$ solving the $N$-Coin Problem with vector sample complexity $f(N, p_0, q_0, \rho, \delta)$.
    Then, there is a $\rho$-replicable algorithm solving the $\ell_{\infty}$-Learning $N$-Coin Problem with 
    vector sample complexity $$
    \bigO{\log \frac{1}{\varepsilon} \cdot f\left( N, \frac{1}{2} - \frac{\varepsilon}{8} , \frac{1}{2} + \frac{\varepsilon}{8}, \frac{\rho}{\log(1/\varepsilon)}, \frac{\delta}{\log(1/\varepsilon)} \right)}.$$
\end{restatable}

The proof is fairly standard and is deferred to Appendix \Cref{app:testing-learning-equivalence}. 
The only subtlety is that the N-coin problem tests the same $[p_0,q_0]$ interval over every coordinate, while simultaneous binary search requires testing different intervals across each coordinate. 
This can be simulated by by reflipping each coin with certain probability to shift each desired test interval into $[\frac{1}{2} - \frac{\varepsilon}{8} , \frac{1}{2} + \frac{\varepsilon}{8}]$.



Moreover, we note that the connection also extends to mean estimation under different norms.
In particular, an elementary application of H\"{o}lder's inequality shows that any $\ell_{c}$-learner implies an $\ell_2$-learner.

\begin{lemma}
    \label{lemma:l-2-lb-implies-l-c-lb}
    Suppose there is an algorithm $\innerAlg$ solving the $\ell_{c}$ Learning $N$-Coin Problem for $c \geq 2$ with $m = f(N, \varepsilon, \rho, \delta)$ vector samples.
    Then, there is an algorithm solving the $\ell_{2}$ Learning $N$-Coin Problem with $f\left( N, \frac{\varepsilon}{N^{\frac{1}{2} - \frac{1}{c}}}, \rho, \delta \right)$ vector samples.
\end{lemma}

\begin{proof}
    Suppose we have an algorithm for the $\ell_{c}$ Learning $N$-Coin Problem.
    In particular, with probability at least $1 - \delta$, $\innerAlg(S_p; r)$ outputs $\hat{p}$ such that
    \begin{equation*}
        \norm{\hat{p} - p}{c} < \frac{\varepsilon}{N^{\frac{1}{2} - \frac{1}{c}}}.
    \end{equation*}
    An elementary application of H\"{o}lder's inequality gives
    \[
    \norm{\hat{p} - p}{2} \leq N^{\frac{1}{2} - \frac{1}{c}}\norm{\hat{p} - p}{c} \leq \varepsilon.
    \]
    as desired.
\end{proof}



\subsection{Isoperimetric Approximate Tilings}

We now establish an equivalence between replicable mean estimation and isoperimetric approximate tilings of $\R^N$.
At a high level, an isoperimetric approximate tiling of $\R^N$ is a collection of disjoint, bounded radius sets that approximately cover $\R^N$ and have good surface area ``on average''. Formally, we also impose a few additional `niceness' conditions on the sets as below.
\begin{restatable}[Isoperimetric Approximate Tiling]{definition}{ApproximateTiling}
    \label{def:approximate-tiling}
    A countable collection of sets $\set{S_{v}}$ labeled by vectors $v \in \R^N$ is a \emph{$(\gamma, A)$-approximate tiling}
    of $\R^N$
    if the following are satisfied:
    \begin{enumerate}
        \item (Disjoint) The interiors $\interior(S_{v}) \subset \R^{N}$ are mutually disjoint.
        \label{item:approximate-tiling:disjoint}

        \item (Non-Zero Volume) $\vol(S_{v}) > 0$ for each $v$.
        \label{item:approximate-tiling:non-zero-volume}

        \item (Piecewise Smooth) The boundary of each set of $\boundary S_{v}$ is piece-wise smooth.
        \label{item:approximate-tiling:smooth-boundary}

        \item ($\gamma$-Approximate Volume) For all $u \in \Z^{N}$ and $\cube_{u} = u + [0, 1]^{N}$,
        \begin{equation*}
            \vol \left( \bigcup_{v} S_{v} \cap \cube_{u} \right) \geq \gamma \cdot \vol(\cube_{u}) = \gamma.
        \end{equation*}
        \label{item:approximate-tiling:approx-volume}
        
        \item (Bounded Radius) For all $S_{v}$ with label $v \in \R^N$ and $x \in S_{v}$,
        \begin{equation*}
            \snorm{2}{x - v} \leq 0.1.
        \end{equation*}
        \label{item:approximate-tiling:radius}
        
       \item (Normalized Surface Area) For all $u \in \Z^{N}$ and $\cube_{u} = u + [0, 1]^{N}$,
        \begin{equation*}
            \area \left( \bigcup_{v} \boundary S_{v} \cap \cube_{u} \right) \leq A \cdot \vol(\cube_{u}) = A.
        \end{equation*}
        \label{item:approximate-tiling:surface-area}  
    \end{enumerate}    
\end{restatable}
We briefly comment on the two parameters $( \gamma, A )$ of an \iat{}. 
The first parameter $\gamma$ characterizes the portion of points from $\R^N$ that are left uncovered by the tiling (hence the name ``approximate tiling''). 
We will see later that this parameter essentially corresponds to the replicability parameter $\rho$.
The second parameter $A$ intuitively measures the surface area of the sets within the tiling in a normalized manner.
More formally, we take an arbitrary unit integer cube in the space, and $A$ poses an upper bound on the $N-1$ dimensional volume of the boundaries of the sets inside the cube. 
By the isoperimetric inequality (\Cref{lemma:isoperimetric-inequality}), any set with small constant covering radius
should have surface area at least $\Omega(N)$ times its volume. With this in mind, we call a tiling `isoperimetric' if $A =\Theta(N)$.
As we will see shortly, the approximate tiling induced by a sample-optimal replicable algorithm indeed realizes this limit.

Crucially, our equivalence will be both statistical and computational.
Thus, we define a membership oracle for any isoperimetric approximate tiling.

\begin{definition}[Membership Oracle of Approximate Tiling] \label{def:tiling-membership-oracle-random}
A (possibly randomized) algorithm $\innerAlg$ is said to be a membership oracle of an approximate tiling $\{S_v\}$ of $\R^N$ 
if the following hold:
\begin{enumerate}
    \item Suppose $x \in \bigcup_{v} S_v$.
    With probability at least $\frac{2}{3}$, $\innerAlg(x) = v$ such that $x \in S_v$.
    \item Suppose $x \not \in \bigcup_{v} S_v$. $\innerAlg(x)$ could return an arbitrary vector $u$.
\end{enumerate}
The membership oracle is said to be efficient if the algorithm runs in polynomial time in the bit-complexity of the input $x$.
\end{definition}



\subsection{Replicable Learning via Isoperimetric Approximate Tiling}
\label{sec:replicability-from-isoperimetry}
In this subsection, we construct a replicable algorithm for mean estimation based on $(\gamma,A)$-isoperimetric approximate tilings whose sample complexity scales with the surface-to-volume ratio $A$. In particular, we will see in the following section that the algorithm achieves the optimal sample complexity when $A=\Theta(N)$.

Formally, our algorithm is based on access to the membership oracle of the \iat.
For simplicity, unlike in \Cref{def:informal-tiling}, we assume that the membership oracle deterministically returns the right answer on points within the partition.
This is without loss of generality: given a randomized membership oracle that succeeds with probability at least $2/3$, one can always query it $\Theta\lp(\log(1/\delta) \rp)$ times to boost its success probability to $1-\delta$. Since we will only call the oracle once in our algorithm, this incurs a cost of at most $O(\log(1/\delta))$ in the runtime of the membership query.

Given a $(\rho, A)$-\iat{} $\mathcal V$, we  design an $O(\rho)$-replicable algorithm that learns the mean of any distribution with bounded covariance in $\ell_2$ distance, and shows that the algorithm's sample complexity scales quadratically with the surface area parameter $A$ of $\mathcal V$.
At a high level, this is achieved by first computing an optimal, non-replicable $\ell_2$ mean estimator for bounded covariance distributions, and then use a randomized rounding scheme based on $\mathcal V$ (see \Cref{prop:replicable-rounding} and \Cref{alg:r-rounding}) to ensure replicability.
\begin{theorem}[From Tiling to $\ell_2$ Replicable Mean Estimation]
\label{thm:tiling-to-replicable}
Let  $\eps \in (0, \sqrt{N}]$,
$N \in \mathbb Z_+$,
$A > 0$, 
$\delta < \rho \in (0, 1)$, 
$\mathcal V$ be a $(\rho, A)$-\iat{} of $\R^N$,
and $\distribution$ be an $N$-dimensional distribution with covariance bounded from above by $I$.
Given access to a membership oracle $\mathcal R$ of $\mathcal V$, there exists a $O(\rho)$-replicable algorithm (\Cref{alg:r-mean-est}) that estimates the mean of $\distribution$ up to error $\eps$ in $\ell_2$ distance with success probability $1 - \delta$.
Moreover, the algorithm 
uses $m:=\bigO{\lp( N + \log(1/\delta) \rp) A^2 N^{-1} \eps^{-2} \rho^{-2}}$ vector samples,\footnote{Note that $A \geq \Omega(N)$ even for optimal tilings, so this does not violate the standard $\Omega(N \eps^{-2})$ lower bound for (non-replicable) mean estimation.}
uses $1$ deterministic membership query, and runs in time $\poly{N m}$.
\end{theorem}
We remark that there are known constructions, i.e., ``randomized foams'' from \cite{kindler2012spherical}, that give isoperimetric tilings of $\R^N$ with surface-area to volume ratio $O(N)$.
Combining the construction with \Cref{thm:tiling-to-replicable} then
immediately gives a replicable algorithm with an optimal {(vector)} sample complexity of $\Theta \lp( N^2 \rho^{-2} \eps^{-2}  \rp)$ (assuming $N \gg \log(1/\delta)$).
Unfortunately membership queries in such constructions take exponential time, so computational efficiency is lost in this process. 
In \Cref{sec:cvpp}, we will discuss other isoperimetrically optimal candidates for sample-optimal mean estimation based on lattices. 
Compared to randomized foams, these constructions have the advantage of having computationally efficient rounding schemes after an initial pre-processing stage, and seem to have greater potential to give a truly computationally efficient sample-optimal algorithm.

We also provide a similar algorithm (\Cref{alg:r-mean-est-infty}) that replicably learns the mean in $\ell_\infty$ distance.
Note that the $N$-coin distribution indeed has its covariance bounded from above by $I$.
Hence, \Cref{alg:r-mean-est-infty} can be used to solve the $N$-coin problem with the same sample complexity (see \Cref{cor:spherical-foam-l-inf-mean}).

Glossing over some detail, we show that one can learn the mean up to $\gamma$ error in $\ell_\infty$ distance if one simply uses (almost) the same algorithm 
from \Cref{thm:tiling-to-replicable} to learn the mean up to $\eps := \tilde \Theta \lp( \sqrt{N} \gamma \rp)$ error in $\ell_2$ distance. 
The intuition is that the rounding error will be roughly in a uniform direction, and therefore the $\ell_\infty$ error will be $N^{-1/2}$ times the $\ell_2$ error with high probability.
This gives the following upper bound on replicable mean estimation in $\ell_\infty$ distance.
\begin{theorem}[$\ell_\infty$ Replicable Mean Estimation]
\label{thm:ell-infty-mean}
Under the same setup as \Cref{thm:tiling-to-replicable},
if the objective is to estimate up to $\gamma$ error in $\ell_\infty$ distance for some $\gamma \in (0, 1)$, 
then there exists an $O(\rho)$-replicable mean estimation algorithm (\Cref{alg:r-mean-est-infty}) 
that succeeds with probability at least $1 - \delta$.
Moreover, the algorithm uses $m := \Theta \lp( A^2 N^{-1} \gamma^{-2} \rho^{-2} \log^3(N/\delta) \rp)$ vector samples, $1$ deterministic membership query, and runs in time $\poly{N m}$.
\end{theorem}
Similar to \Cref{thm:tiling-to-replicable}, \Cref{thm:ell-infty-mean} combined with an  isoperimetric tiling with optimal surface-area to volume ratio yields a replicable $\ell_\infty$ mean estimation algorithm that has a (nearly) optimal sample complexity of $ \Theta\lp(  N \gamma^{-2}  \rho^{-2} \log^3(N/\delta) \rp)$.

\subsubsection{Rounding Scheme from Isoperimetric Tiling}
The core of our mean estimation algorithms is a method of building a replicable rounding scheme from any isoperimetric approximate tiling, that is a scheme whose outputs on two nearby inputs with \emph{shared randomness} will with high probability 1) be the same, and 2) be close to their original inputs. Such schemes are also sometimes called ``noise resistant'' or ``coordinated discretization schemes''. \cite{kindler2012spherical}.

\begin{proposition}[Replicable Rounding Scheme]
\label{prop:replicable-rounding}
Let $\rho \in (0, 1)$ and $A >0$.
Let $\mathcal V$ be a $(\rho,A)$-isoperimetric approximate tiling of $\R^N$, and $\mathcal R: \R^N \mapsto \R^N$ its membership oracle.
Given access to $\mathcal R$, there exists an algorithm 
\alg{rPartialTilingRounding} that takes input 
$u \in [-N, N]^N$,
$\eps \in (0, \sqrt{N})$, and outputs a rounded vector $\bar u \in \R^N$
such that the following hold:
\begin{itemize}
\item \textbf{Bounded Rounding Error:} We have that $\snorm{2}{ \bar u - u } \leq \epsilon $.
\item \textbf{Replicability:}
Let $\bar u^{(i)} = \alg{rPartialTilingRounding}( u^{(i)}, \eps, \rho )$ for $i \in \{1, 2\}$ with shared internal randomness.
Assume that $\snorm{2}{ u^{(1)} - u^{(2)} } \leq c \frac{\sqrt{N} \eps \rho}{A}$ for some sufficiently small constant~$c$.  Then it holds
$ \bar u^{(1)} = \bar u^{(2)} $ with probability at least $1 - \rho$.
\end{itemize}
\end{proposition}
It is not hard to imagine that \Cref{thm:tiling-to-replicable} and \Cref{thm:ell-infty-mean} should follow from combining such a scheme with a non-replicable estimator, since rounding the output of the latter over two independent samples (with shared randomness) then results in the same solution with high probability. Of course some further detail is required, which we defer to the following subsections.


We first overview the steps of our rounding scheme. Naively, one might hope to simply round using the membership oracle of $\mathcal{V}$ itself, but as was observed by \cite{impagliazzo2022reproducibility} in the 1-D case, deterministic procedures of this sort always fail replicability in the worst-case. Namely no matter how close two inputs might be, if they cross the boundary of the partition, they will round to different outputs.
To avoid such worst case scenarios, we will apply a random shift $b \sim \Uni([-\BX, \BX]^N)$ to $u$ before we apply the rounding procedure for some $Q$ that is sufficiently large in $N, \epsilon, \rho^{-1}$.
For technical reasons, the proceeding analysis will be easier if the post-shift point has a uniformly random distribution over a fixed cube $\Uni([-\BX, \BX]^N)$ for any input $u \in [-N, N]$. 
Thus, we will also apply a ``wrap around'' after shifting:
\begin{align}
\label{eq:wrap-def}
    \text{Wrap}(x)_i =  \lp( (x_i + \BX)  \mod 2\BX \rp) - \BX. 
\end{align}
This ensures that the distribution over the post-shifted point will be uniform over the cube. As a result, no matter where the original input vector lies, after translating and wrapping the probability the point lies near the boundary of our tiling is proportional to the boundary's surface area.


Given only the above strategy, if two input vectors $u^{(1)}, u^{(2)}$ are $\eta$-close, to ensure replicability we need our wrapped translation procedure to send $u^{(1)}$ $\eta$-far away from the boundary since $u^{(2)}-u^{(1)}$ might point directly towards the boundary and cross otherwise. This turns out to be too expensive, and can be avoided by first \textit{rotating} the vectors, ensuring this difference lies in a random instead of worst-case direction.

Altogether, this leads to the following random transformation on the estimator $u$ before applying the membership query.
\begin{align*}
   v := \text{Wrap} \lp(  R u + b \rp) \, ,
\end{align*}
where $R \sim \Uni( \SO(N) )$, 
and $b \sim \Uni( [-\BX, \BX]^N )$.

The above addresses only replicability of the rounding scheme. Ensuring that we have small rounding error (i.e., the rounding output is close to the input) requires some extra care. First, since we are working with a tiling with constant radius, naively applying the membership oracle leads to constant error. To achieve any error $\varepsilon$, we may simply `scale' the tiling by outputting $\tilde v = \eps \mathcal R(\eps^{-1}v )$ on input $v$ to the membership oracle $\mathcal R$. Finally, to ensure our output is actually near the original point, we need to `invert' our rotation and translation procedure.
The final output will therefore be given by 
$ \bar u = R^{-1} \text{Wrap}\lp( \tilde v - b\rp) $.
This ensures the output $\bar u$ will be close to the original input $u$ rather than the transformed input $v$. 
The pseudocode is given in \Cref{alg:r-mean-est}.

In the remaining of the subsection, we provide the proof of \Cref{prop:replicable-rounding}.
We reiterate the notations that will be used throughout the analysis.
Let $u^{(1)}, u^{(2)}$ be the input vectors given in two different runs of the algorithm, $v^{(1)}, v^{(2)}$ be the randomly transformed points $v^{(i)} = \text{Wrap} \lp( R u^{(i)} + b \rp) $, $\tilde v^{(1)}, \tilde v^{(2)}$ be the rounded points, i.e., $\tilde v^{(i)} = \eps\mathcal R(\eps^{-1}v^{(i)} )$, and $\bar u^{(1)}, \bar u^{(2)}$ be the outputs after inverting the transformation, i.e.,
$\bar u^{(i)} = R^{-1}\text{Wrap} \lp( \tilde v^{(i)} - b \rp)$. 


Since the randomness of $R \in \SO(N)$ and $b\in [-\BX,\BX]^N$ is shared across the two runs of the algorithm, the outcomes are identical as long as $\eps^{-1} v^{(1)}, \eps^{-1} v^{(2)}$ lie in some common set $V \in \mathcal V$ within the tiling (implying that they will be rounded to the same point).
Our main task is to bound this probability by the surface-area parameter $A$ of the tiling.

\IncMargin{1em}
\begin{algorithm}

\SetKwInOut{Input}{Input}\SetKwInOut{Output}{Output}\SetKwInOut{Parameters}{Parameters}
\Input{Input vector $u \in [-N, N]^N$, membership oracle $\mathcal R: \R^N \mapsto \R^N$ of a $(\rho,A)$-approximate isoperimetric tiling $\mathcal V$.}
\Parameters{Rounding error $\eps \in (0, \sqrt{N})$, replicability parameter $\rho \in (0, 1)$.}
\Output{
$\bar \mu \in \R^N$.}
\begin{algorithmic}[1]


\State 
Set $\BX := C\lp( N^{3/2} + N \varepsilon \rho^{-1} \rp)$ for some sufficiently large constant $C$.

\State Sample $R \sim \Uni( \SO(N) )$, $b \sim\Uni\lp(   [-\BX, \BX]^N \rp)$.

\State Apply the random transformation: $v = \text{Wrap} \lp( R u + b \rp)$.

\State Apply the scaled membership query: $\tilde v = \eps \mathcal R \lp( 
 \eps^{-1} v \rp)$.

\State Invert the transformation: $\bar u = R^{-1} \text{Wrap} \lp( \tilde v - b \rp)$.

\State \Return $\bar u$ if $\snorm{2}{ \bar u - u } \leq \eps$ 
and $u$ otherwise.

\caption{Replicable Tiling Rounding} 
\label{alg:r-rounding}

\end{algorithmic}
\end{algorithm}
\DecMargin{1em}
If one thinks of the segments $\seg\lp( u^{(1)}, u^{(2)} \rp)$ as  a ``needle'',
at a high level, applying the random rotation $R$ ensures that the needle will point towards a random direction, and adding the random offset $b$ ensures that the needle will appear at any place of the cube $ [-\BX, \BX]^N$ uniformly at random.
A caveat is that the needle may cross the boundary of the cube after the random shift, and hence 
be  ``wrapped around'' the cube. 
A convenient way of thinking of the ``wrap around'' effect is to view the process as taking place on the torus $\mathbb T^N := \R^N / \mathbb Z^N$ after scaling and translating the space appropriately. 
This gives the following equivalent stochastic process (up to proper scaling and translation) described via a random Buffon needle in $\mathbb T^d$.
\begin{definition}[Random Buffon Needle in $\mathbb T^N$]
Let $\eps \in [0, 1/3)$.
Let $x$ be a uniformly random point from $\mathbb T^N$,  $u$ a uniformly random unit vector, and $y = x + \eps u   \in \mathbb T^N$.\footnote{We assume by default the addition and multiplication on elements from $\mathbb T^N$ are the ``wrapped around'' operations, i.e., $ a + b = (a+b) \mod 1$ and $a b = (a b) \mod 1$.}
Define the torus segment
$\ell_{\mathbb T}(x, y):= \{
x + t (y-x)  \in \mathbb T^N, \, \text{ where } 0 \leq t \leq 1
\}$.
We call the distribution of the torus segment $\ell_{\mathbb T}(x, y)$ the $\eps$-Buffon needle distribution.
\end{definition}
\begin{lemma}
\label{lem:Buffon-equivalence}
Let $u^{(1)}, u^{(2)} \in [-N, N]^N$,   
$R \sim \Uni(\SO(N))$, $b \in \Uni(  [-\BX, \BX]^N )$ for some $\BX \gg N^{3/2}$,  and $v^{(1)}, v^{(2)}$ be such that
$v^{(i)} = \text{Wrap}\lp( R u^{(i)} + b  \rp)$.
Then the distribution of
$$
\ell_{\mathbb T}\lp( 
\frac{v^{(1)}}{2\BX} + \mathbf 1_N/2
, \frac{v^{(2)}}{2\BX} + \mathbf 1_N/2 \rp)
$$ 
is identical to the $\eps$-Buffon needle distribution, where $\eps:=\snorm{2}{ u^{(2)} - u^{(1)} } / (2\BX)$.
\end{lemma}
\begin{proof}
It suffices to argue that (a) $R (u^{(1)} - u^{(2)})$ points towards a uniformly random direction, and
(b) $v^{(1)}$ is uniform over the cube $[-\BX, \BX]^{N}$ conditioned on any $R$. 

Since $u^{(i)} \in [-N, N]^N$, it follows that $R u^{(i)} \in [-N^{3/2}, N^{3/2}]^N$.
Claim (b) then follows from the observation that for any interval $[a,b]$ and $y \sim [0,b-a]$ uniformly, if $x \in [a,b]$ is any arbitrary point then $x + y$ wrapped around $[a,b]$ is uniform. Applying this to each coordinate of $v^{(i)}$ then implies the desired uniformity. 
Claim (a) follows from the fact that $R$ is sampled uniformly at random from~$\SO(N)$.
\end{proof}

With the equivalence of the two stochastic processes in mind, we can bound the probability that two close points cross the boundary of our partition (i.e.\ fail replicability) after the random transformation by a classical result in geometry known as Buffon's needle theorem \cite{santalo2004integral}.
\begin{lemma}[Buffon's Needle Theorem (as stated in \cite{kindler2012spherical})]
\label{lem:expected-torus-surface-intersection}
Let $S$ be a piecewise smooth surface in the torus $\mathbb T^N$.
Let $x,y$ be a pair of random points such that 
the torus segment $\ell_{\mathbb T}(x, y)$ follows the $\eps$-Buffon needle distribution. 
Let $\kappa$ denotes the number of times the torus segment $\ell_{\mathbb T}(x, y)$ intersects with the surface $S$.
Then it holds that
$$
\Ep_{ \ell_{\mathbb T}(x, y) }\lp[ \kappa \rp] \leq 
O \lp( \frac{1}{ \sqrt{N} } \rp) \; \eps \; \area(S).
$$
\end{lemma}

We are now ready to give the formal proof of \Cref{prop:replicable-rounding}.
\begin{proof}[Proof of \Cref{prop:replicable-rounding}]

The algorithm returns $\bar u$ if $\snorm{2}{ \bar u - u } \leq \eps$ and otherwise returns $u$.
Hence, the rounding error is always at most $\eps$.
To bound the replicability of the algorithm, we 
first argue that 
$\snorm{2}{ \bar u - u } > \eps$ with probability at most $O(\rho)$, and then argue that $\bar u$ is replicable with probability at least $1 - O(\rho)$.
It then follows that the rounding scheme is replicable with probability at least $1 - O( \rho )$ by the union bound.
\paragraph{Rounding Error of $\bar u$}
Recall our algorithm first applies a random transformation on the input and wraps it around the cube $ [-\BX, \BX]^N$, where $\BX := C \lp(  N^{3/2} + N \eps \rho^{-1} \rp)$ for some sufficiently large constant $C$.
In particular, let $v = \text{Wrap}\lp( R u + b \rp)$, where $R \sim \Uni( \SO(N) )$, $b \in \Uni\lp( [-\BX , \BX]^{N} \rp)$, and $\text{Wrap}: \R^N \mapsto  [-\BX, \BX]^N$ is as in \Cref{eq:wrap-def}.
The algorithm then applies the scaled membership oracle, which gives $\tilde v = \eps \mathcal R( \eps^{-1} v)$, where $\mathcal R$ is the rounding query with respect to the tiling $\mathcal V$.
$\bar u$ is defined as
$ \bar u = R^{-1} \text{Wrap}\lp(  \tilde v - b \rp) $.
Our goal is to show that
\begin{align}\label{eq:part-to-rep-correct}
|| \bar u - u ||_2 \leq \varepsilon
\end{align}
with probability at least $1 - O(\rho)$.

We will show this in two steps. 
First, we claim that the final transformation 
$f(x) = R^{-1} \text{Wrap}( x - b )$ inverts the random transformation $g(x) = \text{Wrap}(Rx +b)$.
\begin{claim}\label{claim:no-wrapping}
It holds that
$
R^{-1}\text{Wrap}(v-b) = u.
$
\end{claim}
Consider the following events.
\begin{itemize}
\item $\mathcal B: v$ is $\epsilon$-far to the boundary of the cube $[-\BX ,\BX]^N$.
\item $\mathcal E:$ the rounding error is $\eps$, i.e., $\snorm{2}{ v - \tilde v } \leq \eps$, which holds as long as $v$ lies in some set within the approximate tiling.
\end{itemize}
We next claim that, conditioned on $\mathcal B$ and $\mathcal E$,
inverting the transformation will not further enlarge the rounding error $\snorm{2}{ v - \tilde v }$.

\begin{claim}\label{claim:rounding-error}
If $\snorm{2}{ v - \tilde v } \leq \epsilon$ and $v$ is $\epsilon$-far from the boundary of the cube $[-\BX, \BX]^N$,
it holds that
\begin{align*}
\snorm{2}{ R^{-1}\text{Wrap}\lp( \tilde v - b \rp) 
-R^{-1}\text{Wrap}\lp( v - b \rp)
}\leq \varepsilon.
\end{align*}
\end{claim}
Note that the rounding error $\snorm{2}{ v - \tilde v }$ will indeed be at most $\eps$. 
In particular, as the approximate tiling has covering radius $0.1$, we must have that
$  \snorm{2}{\eps^{-1} \tilde v - \eps^{-1} v}= 
\snorm{2}{\mathcal R( \eps^{-1} v ) - \eps^{-1} v}  \leq 0.1$.
It then follows that $\snorm{2}{ \tilde v - v } \leq 0.1 \; \eps$.
Conditioned on events $\mathcal B$ and $\mathcal E$, we can apply \Cref{claim:rounding-error} and \Cref{claim:no-wrapping} to derive \Cref{eq:part-to-rep-correct}. 
In particular,
\begin{align*}
|| \bar u - u ||_2
&=
\| R^{-1} \text{Wrap}\lp( \tilde v - b \rp) - u \|_2 \\
&=
\| R^{-1} \text{Wrap}\lp( \tilde v - b \rp) -  R^{-1} \text{Wrap}\lp( v - b \rp) \|_2 \\
&\leq \eps,
\end{align*}
where the first equality is true by the definition of $\bar u = R^{-1} \text{Wrap}\lp( \tilde v - b \rp)$,
the second equality uses \Cref{claim:no-wrapping}, and the last inequality uses \Cref{claim:rounding-error}.
Using the fact that $v$ is uniform over $[-\BX, \BX]^N$ (\Cref{lem:Buffon-equivalence}), we have that $\mathcal B$ fails to hold with probability at most $ N \epsilon / Q \ll \rho $ by our choice of $Q$.  
Moreover, $\mathcal E$ fails to hold with probability at most $\rho$ since the approximate tiling is promised to cover at least $1 - \rho$ fraction of points within any cube.
By the union bound, we thus have that \Cref{eq:part-to-rep-correct} fails to hold with probability at most $O(\rho)$.


It is left to prove the claims. Toward \Cref{claim:no-wrapping}, recall $u \in [-N, N]^N$ by assumption, so $R u \in [-N^{3/2}, N^{3/2}]^N \subset [-\BX, \BX]^N$ and we have
 $$
 \text{Wrap} \lp( \text{Wrap} \lp( Ru + b \rp) - b \rp) = Ru.
 $$
 Combining this with the definition $v =  \text{Wrap}\lp( R u + b\rp)$ implies
\begin{align}
\label{eq:inverse-wrap}
u = R^{-1} \text{Wrap} \lp( \text{Wrap} \lp( Ru + b \rp) - b \rp)
=  R^{-1} \text{Wrap}\lp( v - b \rp).
\end{align}
as desired. 

Toward \Cref{claim:rounding-error}, 
we claim that  
\begin{align}
\label{eq:wrap-preserve-distance}    
\snorm{2}{ \text{Wrap}\lp( \tilde v - b \rp) 
-\text{Wrap}\lp( v - b \rp)
} = \snorm{2}{ \tilde v - v }
\end{align}
when $v$ is $\eps$-far from the boundary of the cube $[-\BX, \BX]^N$.
Since $R^{-1}$ is a rotation it preserves $\ell_2$-norm, so \Cref{claim:rounding-error} immediately follows from \Cref{eq:wrap-preserve-distance}.

The only obstruction to \Cref{eq:wrap-preserve-distance} is if exactly one of $\tilde v$ and $v$ `wraps' around the cube when shifted by $b$. Note this cannot occur if both $v$ and $\text{Wrap}(v-b)$ are at least $\eps$-far from the boundaries of the cube $ [-\BX, \BX]^N$. In this case the function $f(x) = \text{Wrap}(x - b)$ maps the $\eps$-ball around $v$ to the $\eps$-ball around $\text{Wrap}\lp( v - b \rp)$. 
Since $\tilde{v}$ is promised to lie in an $\eps$-ball around $v$, $\text{Wrap}\lp( \tilde v - b \rp)$ therefore must lie in an $\eps$-ball around $\text{Wrap}\lp(v - b \rp)$ as desired and in particular has the same $\ell_2$ distance to $\text{Wrap}\lp( v - b \rp)$ as that between $\tilde v$ and $v$.

It remains to show that $\text{Wrap}(v-b)$  and $v$ are both at least $\eps$-far from the boundaries with probability at least $1 - \rho$.
The latter is $\eps$-far from the boundaries by the assuption of \Cref{claim:rounding-error}.
By \Cref{claim:no-wrapping}, the former is exactly $R u$, and hence lies in the cube $[-N^{3/2}, N^{3/2}]^N$.
It follows that the former is also $\eps$-far from the boundaries of the cube $[-\BX, \BX]^N$ as $\BX \gg N^{3/2}$ and $\epsilon <\sqrt{N}$.
This completes the proof of \Cref{eq:wrap-preserve-distance} and therefore \Cref{claim:rounding-error} as desired.

\paragraph{Replicability of $\bar u$}
Next we argue that $\bar u$ is replicability with probability at least $1 - O(\rho)$ across two executions.
Let $u^{(1)}, u^{(2)}$ be any two inputs satisfying
\begin{align}
\label{eq:median-of-mean-concentration}
\snorm{2}{u^{(1)} - u^{(2)}}
\ll \frac{ \sqrt{N} \eps \rho  }{ A }.
\end{align}
Our goal is to show that $\bar u^{(1)}, \bar u^{(2)}$ computed from $u^{(1)}$ and $u^{(2)}$ are the same with high probability:
\[
\Pr[\bar{u}^{(1)} = \bar{u}^{(2)}] \geq 1-O(\rho)
\]
where the probability is over the (shared) internal randomness of the rounding procedure.

With this in mind, let $R \sim \Uni( \SO(N) )$, $b \in \Uni\lp( [-\BX,\BX]^N \rp)$, $v^{(i)} = \text{Wrap}\lp( R u^{(i)} + b\rp)$, and  $\tilde v^{(i)} = \eps \mathcal R\lp( \eps^{-1} v^{(i)} \rp)$.
Conditioned on the events
$L_1: \tilde v^{(1)} = \tilde v^{(2)}$,
and $L_2:\snorm{2}{ \tilde v^{(i)} - v^{(i)} } \leq \epsilon$ for $i \in \{1, 2\}$, 
the algorithm in both runs computes 
$\bar{u}^{(1)}=R^{-1} \text{Wrap}( \tilde v^{(1)} - b  ) = R^{-1} \text{Wrap}( \tilde v^{(2)} - b  )=\bar{u}^{(2)}$, so it suffices to show
\begin{align}
\label{eq:bad-rep-event}
\Pr(\neg L_1 \lor \neg  L_2) \leq O(\rho). 
\end{align}
Since we only have good control of the membership query when the input point lies in some $V \in \mathcal V$ within the approximate tiling, we will perform a case analysis depending on the ancillary event 
$J: \eps^{-1} v^{(1)}, \eps^{-1} v^{(2)} \in \bigcup_{V \in \mathcal V} V$. 
Note that
\begin{align}
\label{eq:marginalize-J}
\Pr(  \neg L_1 \lor \neg  L_2  )
&= 
\Pr(  \lp( \neg L_1 \lor \neg  L_2 \rp) \land J )
+ \Pr(  \lp( \neg L_1 \lor \neg  L_2 \rp) \land \neg J )
\nonumber \\
&\leq 
\Pr(  \lp( \neg L_1 \lor \neg  L_2 \rp) \land J )
+ 
\Pr(\neg J) \nonumber \\
&\leq
\Pr(  \neg L_1 \land J )
+ 
\Pr( \neg L_2 \land J )
+ 
\Pr(\neg J) \, ,
\end{align}
where in the second inequality we apply the union bound.

We will tackle the three terms separately.
By definition, $\bigcup_{V \in \mathcal V} V$ covers all but a $\rho$-fraction of points within every integer cube.
By \Cref{lem:Buffon-equivalence}, the distribution of $v^{(i)}$ is uniform over the cube $[-Q, Q]^N$. 
Then $\eps^{-1} v^{(i)}$ will be uniformly random over the cube $[-\eps^{-1} Q, \eps^{-1} Q]^N$, implying
\begin{align}
\label{eq:J-bound}
    \Pr(\neg J)
    \leq 2 \; \Pr_{ v \sim \Uni\lp( [-\eps^{-1} Q, \eps^{-1} Q]^N \rp) }( v^{(1)} \not \in \bigcup_{V \in \mathcal V} V \lor v^{(2)} \not \in \bigcup_{V \in \mathcal V} V )
    \leq 2 \rho.
\end{align}
By the definition of an approximate isoperimetric tiling (\Cref{def:approximate-tiling}), each set $V \in \mathcal V$ has diameter $0.1$. 
Thus, whenever $\eps^{-1} v^{(i)} \in V$ for some $V \in \mathcal V$, 
we will have that
$\snorm{2}{ \eps^{-1} \tilde v^{(i)} - \eps^{-1} v^{(i)} } \leq 1$, which implies that 
$\snorm{2}{  \tilde v^{(i)} -  v^{(i)} } \leq \eps$.
In other words, $J$ implies $L_2$, showing that
\begin{align}
\label{eq:J-imply-L2}
    \Pr(\neg L_2 \land J) = 0.
\end{align}

It remains to bound from above $\Pr(\neg L_1 \land J)$ over the randomness of $v^{(1)}, v^{(2)}$.
Toward this end, we need the following definitions regarding the boundary of our tiling $\mathcal{V}$:
\begin{itemize}
    \item Let $F$ be the union of all boundaries of $\mathcal V$, i.e., 
    $F = \bigcup_{V \in \mathcal V} \partial V$.
    \item Denote by $F_{\mathbb T}$ the torus restriction of $F$ after proper translation and scaling. Specifically, define $F_{\mathbb T}
= 
\lp( F \cap [- \eps^{-1} \BX,  \eps^{-1} \BX]^N \rp)
\; \frac{ \eps }{ 2 \BX } + \mathbf 1_N/2 \subset [0, 1]^N \cong \mathbb T^N$.
\end{itemize}
We note that $\neg L_1 \land J$ happens only if 
$\eps^{-1} v^{(1)}, \eps^{-1} v^{(2)}$ lie in different sets within $\mathcal V$, or equivalently, when the segment $\lp( \eps^{-1} v^{(1)}, \eps^{-1} v^{(2)}\rp)$ intersects with the union of the boundaries $F$.

To bound from above the probability this occurs, we view the segment as a random Buffon needle in $\mathbb T^N$ after proper translation and scaling.
In particular, we consider the torus segment  
$$
\ell_{\mathbb T}\lp( 
\frac{v^{(1)}}{2\BX} + \mathbf 1_N/2
, \frac{v^{(2)}}{2\BX} + \mathbf 1_N/2 \rp).
$$ 
One can verify that this is a well-defined segment in the torus $\mathbb T^N$, indeed $\frac{v^{(i)}}{2\BX} + \mathbf 1_N/2$ lies in $[0, 1]^N$ since $v^{(i)} \in  [-\BX, \BX]^N$.
There are only two cases when the segment $ \lp( \eps^{-1} v^{(1)}, \eps^{-1} v^{(2)} \rp)$ intersects with $F$: 
either the above torus segment intersects with $F_{\mathbb T}$   
or  $v^{(1)} + R ( u^{(2)} - u^{(1)} ) \not \in  [-\BX, \BX]^N$ 
(when the torus segment is wrapped around the torus). In other words, we have that 
\begin{align}
&\Pr(\neg L_1 \land J)  \nonumber \\
&\leq \Pr\lp( v^{(1)} + R ( u^{(2)} - u^{(1)} ) \not \in  [-\BX, \BX]^N \rp) + \Pr\lp(  
\ell_{\mathbb T}\lp( 
\frac{v^{(1)}}{2\BX} + \mathbf 1_N/2
, \frac{v^{(2)}}{2\BX} + \mathbf 1_N/2 \rp) \cap F_{\mathbb T} 
\neq \emptyset
\rp).
\label{eq:wrap-around-marginalize}
\end{align}

It is not hard to see that $ v^{(1)} + R ( u^{(2)} - u^{(1)} ) \not \in  [-\BX, \BX]^N $ only if $v^{(1)}$ is at least $\snorm{2}{ u^{(2)} - u^{(1)} }$-close to some boundary of the cube $ [-\BX, \BX]^N$. Since $v^{(1)}$ is uniform over the cube, we can bound this probability from above as
\begin{align}
\label{eq:close-to-boundary}
\Pr( v^{(1)} + R ( u^{(2)} - u^{(1)} ) \not \in  [-\BX, \BX]^N ) 
\leq 2 N \; \frac{ \snorm{2}{u^{(2)} - u^{(1)}} }{\BX} \ll 
\delta \leq \rho     \, ,
\end{align}
since $\snorm{2}{u^{(2)} - u^{(1)}} \ll \eps$ and 
$\BX \gg N \varepsilon \delta^{-1}$.

It remains to bound the probability that the torus segment $\ell_{\mathbb T}\lp( 
\frac{v^{(1)}}{2\BX} + \mathbf 1_N/2
, \frac{v^{(2)}}{2\BX} + \mathbf 1_N/2 \rp)$ intersects with 
$F_{\mathbb T}$.
To bound the probability of this event,
we consider the expected number of intersections between the torus segment and $F_{\mathbb T}$:
$$
\kappa := \Ep\lp[ 
\abs{
\ell_{\mathbb T}\lp( 
\frac{v^{(1)}}{2\BX} + \mathbf 1_N/2
, \frac{v^{(2)}}{2\BX} + \mathbf 1_N/2 \rp) \cap F_{\mathbb T} }
\rp].
$$
We claim that
\begin{align}
\label{eq:exp-intersection-bound}    
\kappa \leq \rho.
\end{align}
Given \Cref{eq:exp-intersection-bound}, we can apply Markov's inequality to derive that
\begin{align}
&\Pr\lp(  
\ell_{\mathbb T}\lp( 
\frac{v^{(1)}}{2\BX} + \mathbf 1_N/2
, \frac{v^{(2)}}{2\BX} + \mathbf 1_N/2 \rp) \cap F_{\mathbb T} 
\neq \emptyset
\rp) \nonumber \\
&= \Pr\lp(
\lp|
\ell_{\mathbb T}\lp( 
\frac{v^{(1)}}{2\BX} + \mathbf 1_N/2
, \frac{v^{(2)}}{2\BX} + \mathbf 1_N/2 \rp) \cap F_{\mathbb T} 
\rp|
\geq 1
\rp) \nonumber \\
&\leq \Pr\lp( 
\lp|
\ell_{\mathbb T}\lp( 
\frac{v^{(1)}}{2\BX} + \mathbf 1_N/2
, \frac{v^{(2)}}{2\BX} + \mathbf 1_N/2 \rp) \cap F_{\mathbb T} 
\rp|
\geq \rho^{-1} \kappa
\rp)  \leq \rho \, ,
\label{eq:prob-to-surface}
\end{align}
Combining \Cref{eq:prob-to-surface,eq:wrap-around-marginalize,eq:close-to-boundary} then shows that
\begin{align}
\label{eq:neg-L1-J-bound}
\Pr( \neg L_1 \land J )
\leq O(\rho).
\end{align}
Substituting \Cref{eq:J-bound,eq:J-imply-L2,eq:neg-L1-J-bound} into \Cref{eq:marginalize-J} then shows \Cref{eq:bad-rep-event},
implying that the algorithm is $O(\rho)$-replicable.


It remains to show \Cref{eq:exp-intersection-bound}.
Note that the distance between the two points
$ v^{(i)} / (2\BX) + \mathbf 1_{N} / 2 $ within the torus $\mathbb T^N$ is exactly $\snorm{2}{ u^{(1)} - u^{(2)} } / (2\BX)$
\footnote{We define the distance between two points within the torus to be the length of the shortest torus segment connecting the two points. Under such a notion, the torus distance between the two points $ v^{(i)} / (2\BX) + \mathbf 1_{N} / 2 $ is exactly $\snorm{2}{ u^{(1)} - u^{(2)} } / (2\BX)$ in all cases. }.
By \Cref{lem:Buffon-equivalence} and \Cref{lem:expected-torus-surface-intersection}, it holds
\begin{align*}
\kappa \leq O( N^{-1/2} )
\frac{\snorm{2}{ u^{(1)} - u^{(2)}  }}{2\BX}
\area( F_{\mathbb T} ).
\end{align*}
By our assumption \Cref{eq:median-of-mean-concentration}, we further have that
\begin{equation}
\label{eq:kappa-to-area}
\kappa \ll  \frac{ \eps \rho }{2 \; A \; Q} \; \area( F_{\mathbb T} ).    
\end{equation}

We then proceed to bound from above 
$\area( F_{\mathbb T}  )$.
Note that 
\begin{equation}
\label{eq:surface-scale-equality}
    \area( F_{\mathbb T}  ) = 
    \frac{ \area\lp( F \cap [-\eps^{-1} \BX , \eps^{-1} \BX ]^N \rp)  }{ (2 \BX \eps^{-1})^{N-1} }
\end{equation}
by definition.
{
Without loss of generality, we assume that $Q \eps^{-1}$ are integers. 
Then the cube $[- \eps^{-1} \BX, \eps^{-1} \BX]^N$ can be partitioned into $ \lp( 2 \eps^{-1} \BX \rp)^N $ many integer unit cubes.
}
Thus, by the property that an \iat{} has small normalized surface area (Property~\ref{item:approximate-tiling:surface-area} in \Cref{def:approximate-cube-isoperimetry}), 
we have that
\begin{equation}
\label{eq:extend-voronoi-boundary-surface}
\area\lp( F \cap [-\eps^{-1}\BX, \eps^{-1}\BX]^N \rp)
\leq 
A  \vol \lp( \bigcup_{V \in \mathcal V }
V \cap [-\eps^{-1}\BX, \eps^{-1}\BX]^N \rp)
\leq A \lp(2\BX \eps^{-1}\rp)^N.
\end{equation}
\begin{remark}
\label{rmk:huge-cube-relax}
We note that this is the only place where we invoke the normalized surface area of an \iat~in the argument. 
Since $\eps^{-1} \BX \gg N$, 
one can see that the argument still holds if we relax the definition to be that the surface area of the tiling within 
any cube of side length $C N$ is at most $O(A) \; (C N)^N$ for some appropriate constant $C$.
This will be important to the proof of \Cref{thm:rep-pre} in \Cref{sec:cvpp}, where we instantiate the tiling to be the Voronoi cells of lattices that don't necessarily give rise to partitions of integer cubes.
\end{remark}

Combining 
\Cref{eq:surface-scale-equality,eq:extend-voronoi-boundary-surface,}
then gives
\begin{align*}
\area( F_{\mathbb T}  )
\leq
2 A Q \eps^{-1}
\end{align*}
Substituting the above into \Cref{eq:kappa-to-area} then gives that
\begin{align*}
\kappa 
\ll \frac{\eps \rho}{A Q} \;  A Q  \; \eps^{-1}
\leq \rho \, ,
\end{align*}
which shows \Cref{eq:exp-intersection-bound},
and hence concludes the proof of \Cref{prop:replicable-rounding}.
\end{proof}

\subsubsection{Reduction to Bounded Domain}
A slight caveat of our rounding scheme (\Cref{prop:replicable-rounding}) is that it requires the input to be within a bounded cube $[-N, N]^N$. 
We show that this is a mild assumption when we want to apply the routine to replicable mean estimation: there is an easy replicable reduction to the case where the mean of the unknown distribution lies in a bounded domain.
This is done via a replicable procedure that coarsely learns the mean up to error $\BD < N$ in $\ell_\infty$ distance. Translating our coordinate system by the learned vector then places the true mean within a $\BD$-length cube with high probability.
\begin{lemma}
\label{lem:warmup-round}
Let $\delta < \rho \in (0, 1)$, and $\distribution$ be a distribution on $\R^N$ with covariance at most $I$. 
Then there exists a $\rho$-replicable algorithm \rCoordinateRound{} that 
draws $\Theta\lp(  
( K \rho / N )^{-2}
\log(N/\delta) \rp) $ many \iid 
vector samples, and efficiently estimates the mean of $\distribution$ up to $K$-accuracy in $\ell_\infty$-distance with probability at least $1 - \delta$.
\end{lemma}

\rCoordinateRound{} is based on the classical ``median-of-mean'' estimator. 
\begin{claim}[Median of Means]
\label{clm:median-of-mean}
Let $\delta \in (0, 1)$, and $\distribution$ be an $N$-dimensional distribution with 
mean $\mu$ and covariance at most $I$. 
Given $m$ \iid samples, there exists a \algname{Median-of-Mean} estimator 
$\bar \mu$ such that
\begin{align*}
\snorm{\infty}{ \bar \mu  - \mu}
\leq O \lp( \sqrt{\log (N/\delta) / m} \rp)
\end{align*}
with probability at least $1 - \delta$.
\end{claim}
\begin{proof}
After taking $m /  \log(N/\delta) $ many samples, 
the algorithm computes the empirical mean $\hat \mu$.
Let $\mu$ be the true mean of $\distribution$.
Fix some coordinate $i$.
By Chebyshev's inequality, we have that $\abs{\hat \mu_i - \mu_i} \leq O \lp( \sqrt{ \log (N/\delta) / m } \rp)$ with probability at least $2/3$.
If we repeat this process $T:=\log(N/\delta)$ many times, we obtain a set of empirical means $\hat \mu^{(1)}, \cdots, \hat \mu^{(T)}$.
Let $\bar \mu$ be the coordinate-wise median of these empirical means.
Fix some coordinate $i$.
Then, by the Chernoff Bound, it holds that
$\abs{\bar \mu_i - \mu_i} \leq O \lp( \sqrt{ \log (N/\delta) / m } \rp)$ with probability at least $1 - \delta / N$.
Finally union bounding over $N$ coordinates gives that
$$
\snorm{\infty}{ \bar \mu - \mu  } \leq O \lp( \sqrt{ \log (N/\delta) / m } \rp)
$$
with probability at least $1 - \delta$.
\end{proof}
We obtain the desired coarse replicable estimator simply by rounding each coordinate of Median-of-Means to the nearest integer multiple of $N$.
\begin{proof}[Proof of {\Cref{lem:warmup-round}}]
Let $\bar \mu$ be the \textit{Median-of-Mean} estimator from \Cref{clm:median-of-mean} computed from $m \geq C \; ( K \rho / N )^{-2}
\log(N/\delta)$ many samples for some sufficiently large constant $C$.
Then by \Cref{clm:median-of-mean}
\begin{align}
\label{eq:median-of-mean-err}
\snorm{\infty}{\bar \mu - \mu} \leq  K \rho  / N
\end{align}
with probability at least $1 - \delta$.
Consider the process in which we round each coordinate of $\bar \mu$ to the nearest integer multiples of $ K $ based on some random threshold chosen uniformly at random. 
Specifically, if $\bar \mu_i \in [a K, (a+1) K]$ for some integer $a$, we choose $\alpha_i \sim \Uni( [aK, (a+1)K] )$, and round $\bar \mu_i$ to $aK$ if $\bar \mu_i < \alpha_i$ (and to $(a+1) K$) otherwise.
Fix some coordinate $i$.
It holds that $\mu_i$ is $K \rho / N$ close to a rounding threshold with probability at most $O(\rho / N)$. 
By the union bound, the rounding threshold $\alpha_i$ is at least $K \rho / N$-far from the true mean $\mu_i$ for every coordinate $i$ with probability at least $1 - O(\rho)$.  
It then follows that the algorithm's output is replicable with probability at least $1 - O(\rho)$.
Finally, the resulting error of the rounded estimator follows from \Cref{eq:median-of-mean-err} and the fact that the rounding shifts each coordinate by at most $K$.
\end{proof}

\subsubsection{A Replicable \texorpdfstring{$\ell_2$}{2 norm} Estimator}
To prove \Cref{thm:tiling-to-replicable}, we combine \Cref{prop:replicable-rounding} with an $\ell_2$ estimator for bounded covariance distributions.
Since we care about the sample complexity of the problem in the high success probability regime, we will be using the estimator from \cite{cherapanamjeri2019fast}.
Importantly, the estimator achieves the optimal sub-gaussian convergence rate, and is efficiently computable.
\begin{theorem}[Sub-Gaussian mean estimator from \cite{cherapanamjeri2019fast}]
\label{thm:sub-gaussian-mean}
Let $\distribution$ be a distribution 
on $\R^N$
with covariance bounded from above by $I$, and $\snorm{2}{ \Ep[\distribution] } \leq \poly{N}$.
Given $m$ \iid samples from $\distribution$,
there exists an efficient algorithm that computes an estimate $\hat \mu$ such that 
\begin{align*}
\snorm{2}{\Ep[\distribution]  - \hat \mu}
\leq O\lp( 
\sqrt{ d / m }
+ \sqrt{ \log(1/\delta) / m }
\rp).
\end{align*}
\end{theorem}
\IncMargin{1em}
\begin{algorithm}

\SetKwInOut{Input}{Input}\SetKwInOut{Output}{Output}\SetKwInOut{Parameters}{Parameters}
\Input{Sample access to distribution $\distribution$ on $\R^N$ and rounding query $\mathcal R: \R^N \mapsto \R^N$.}
\Parameters{$\eps$ accuracy}
\Output{
$\hat \mu$: an estimate of $\Ep[\distribution]$ up to $\eps$-accuracy
in $\ell_2$ distance.}
\begin{algorithmic}[1]
\State 
Set $m_1:=C  \log(N/\delta) \rho^{-2}$ and $m_2 := 
C \; \lp( N + \log(1/\delta) \rp) A^2 N^{-1} \eps^{-2} \rho^{-2}$
for some sufficiently large constant $C$.
\State Use \Cref{lem:warmup-round} to replicably learn $\tilde u$
with $m_1$ samples such that $\snorm{\infty}{ \tilde u - \Ep[\distribution]} \leq N$ in $\ell_\infty$ distance.\label{line:warmup1}

\State Transform the coordinate system by $\tilde \mu$. \label{line:warmup2}
\State Compute the {estimator $u$ with the algorithm from \Cref{thm:sub-gaussian-mean}} with $m_2$ many samples.

\State Round $u$ into $\hat u$ with the randomized rounding scheme \Cref{alg:r-rounding} with accuracy $\epsilon$  and replicability parameter $\rho$.

\State Output $\hat u$.

\caption{Replicable $\ell_2$ mean estimator} 
\label{alg:r-mean-est}

\end{algorithmic}
\end{algorithm}
\DecMargin{1em}
We are now ready to conclude the proof of \Cref{thm:tiling-to-replicable}.
\begin{proof}[Proof of \Cref{thm:tiling-to-replicable}]
By \Cref{lem:warmup-round}, the warmup rounding process (Lines~\ref{line:warmup1} and~\ref{line:warmup2} from \Cref{alg:r-mean-est}) succeeds with probability at least $1 - \delta$, and is $\rho$-replicable.
Conditioned on that, we reduce to the case where the distribution mean $\Ep[\distribution]$ lies in the domain $[-N/2 , N/2]^{N}$.

Let $u$ be the estimator from \Cref{thm:sub-gaussian-mean} computed from $\Theta \lp(  \lp( N + \log(1/\delta) \rp)
A^2 N^{-1} \eps^{-2} \rho^{-2} \rp) \gg (N + \log(1/\delta)) \eps^{-2}$ \iid samples (Recall that we know the surface-to-volume ratio $A$ of any tiling is at least $N$ by the isoperimetric inequality).
By \Cref{thm:sub-gaussian-mean}, we have that
\begin{align}
\label{eq:sub-gaussian-mean}
\snorm{2}{u - \Ep[\distribution]} \ll  \frac{\sqrt{N} \eps \rho  }{A} \,
\end{align}
{with probability at least $1 - \delta$.}
Applying \Cref{prop:replicable-rounding} gives that
$\snorm{2}{ \hat u - u } \leq \eps$.
Thus, the output is close to $\Ep[\distribution]$ in $\ell_2$ distance with probability at least $1 - \delta$ by the triangle inequality.

Let $u^{(1)}, u^{(2)}$ be the estimator computed in two different runs of the algorithm. Still, by the triangle inequality and the union bound, it holds that
$$
\snorm{2}{ u^{(1)} - u^{(2)} } \ll  \frac{\sqrt{N} \eps \rho  }{A}.
$$
Conditioned on the above, 
applying \Cref{prop:replicable-rounding} gives that the algorithm is replicable with probability at least $1 - \rho$.
This concludes the proof of \Cref{thm:tiling-to-replicable}.
\end{proof}

\subsubsection{A Replicable \texorpdfstring{$\ell_\infty$}{infinity norm} Estimator}

Finally we cover the case of $\ell_\infty$-mean-estimation (\Cref{thm:ell-infty-mean}). The algorithm is identical to \Cref{alg:r-mean-est} except that we replace the non-replicable estimator from \Cref{thm:sub-gaussian-mean}, which is customized for $\ell_2$ estimation, with the Median-of-Means estimator from \Cref{clm:median-of-mean}. Roughly speaking replicability then follows the same argument, while for correctness we take advantage of the fact that
the rounding error vector $\xi$ points in a uniformly random direction.
This means that with high probability
$$
\snorm{\infty}{\xi} \leq O \lp( \log N \snorm{2}{\xi} / \sqrt{N} \rp),
$$
allowing us to save an extra factor of $\tilde \Theta\lp(N\rp)$ from the overall algorithm's sample complexity.
\IncMargin{1em}
\begin{algorithm}

\SetKwInOut{Input}{Input}\SetKwInOut{Output}{Output}\SetKwInOut{Parameters}{Parameters}
\Input{Sample access to distribution $\distribution$ on $\R^N$ and rounding query $\mathcal R: \R^N \mapsto \R^N$.}
\Parameters{$\gamma$ accuracy}
\Output{
$\hat \mu$: an estimate of $\Ep[\distribution]$ up to $\gamma$-accuracy
in $\ell_\infty$ distance.}
\begin{algorithmic}[1]
\State 
Set $m_1:=C  \log(N/\delta) \rho^{-2}$ and $m_2 := 
C \;    A^2 N^{-1} \; \log^3(N/\delta) \; \gamma^{-2} \rho^{-2}$
for some sufficiently large constant $C$.
\State Use \Cref{lem:warmup-round} to replicably learn $\tilde u$
with $m_1$ samples such that $\snorm{\infty}{ \tilde u - \Ep[\distribution]} \leq N$.

\State Transform the coordinate system by $\tilde \mu$.
\State Compute the median-of-mean estimator $u$ (\Cref{clm:median-of-mean}) with $m_2$ many samples.

\State Round $u$ into $\hat u$ with the randomized rounding scheme \Cref{alg:r-rounding} with accuracy $\epsilon:=\sqrt{N} \gamma / \log(N/\delta)$ and replicability parameter $\rho$.

\State Output $\hat u$.

\caption{Replicable $\ell_{\infty}$ mean estimator} 
\label{alg:r-mean-est-infty}

\end{algorithmic}
\end{algorithm}
\DecMargin{1em}
\begin{proof}[Proof of \Cref{thm:ell-infty-mean}]

As in the proof of \Cref{thm:tiling-to-replicable}, we will condition on the success of the warm-up rounding procedure and therefore assume that $\Ep[\distribution] \in [-N/2, N/2]^N$.

Let $u^{(1)}, u^{(2)}$ be the median-of-mean estimators computed 
from $m = C \; A^{2} N^{-1} \log^3(N/\delta) \gamma^{-2} \rho^{-2}$ many samples in two runs.
By \Cref{clm:median-of-mean}, we have that
$$
\snorm{\infty}{ u^{(i)} - \Ep[\distribution] } \ll 
\frac{\sqrt{N} \gamma \rho }{ \log(N/\delta) A} \, ,
$$
which implies that
$$
\snorm{2}{ u^{(i)} - \Ep[\distribution] } \ll 
\frac{N \gamma \rho }{\log(N/\delta) A}
$$
for $ i \in \{1, 2\}$.
By the union bound and the triangle inequality,
we have that
$$
\snorm{2}{ u^{(1)} - u^{(2)} } \ll \frac{N \gamma \rho }{\log(N/\delta) A}.
$$
Conditioned on the aboe inequality, applying \Cref{prop:replicable-rounding} 
with $\eps = \sqrt{N} \gamma / \log(N/\delta)$ gives that the algorithm is replicable
with probability at least $1 - \rho$.

It remains to argue correctness.
Let $u$ be the median-of-mean estimator computed in a single run of the algorithm.
Toward this end, we will condition on the following events: (i) the warmup rounding procedure succeeds such that we reduce to the case
$\Ep[\distribution] \in [-N/2, N/2]^{N}$, and (ii) 
the median-of-mean estimator $u$ is $\sqrt{N} \gamma / \log(N/\delta)$-close to $\Ep[\distribution]$ in $\ell_2$ distance.
By \Cref{lem:warmup-round}, (i) succeeds with probability at least $1 - \delta$.
(ii) holds with probability at least $1 - \delta$ as argued in the replicability proof.
By the union bound, the above events hold simultaneously with probability at least $1 - O(\delta)$.
We observe that $u \in [-N, N]^N$ conditioned on the above two events, hence meeting the input requirement of the rounding procedure.

Let $\hat u$ be the rounding result. 
To argue that $u$ and $\hat u$ are close in $\ell_\infty$ norm, we need to take a closer look at the randomized rounding scheme.
In particular, let $v = \text{Wrap}\lp( R u + b \rp)$, where $R \sim \Uni( \SO(N) )$, $b \in \Uni\lp( [-\BX , \BX]^{N} \rp)$, $\tilde v = \eps \mathcal R( \eps^{-1} v)$ where $\mathcal R$ is the rounding query with respect to the tiling $\mathcal V$, and $\eps = \sqrt{N} \gamma / \log (N/\delta)$, 
and $ \bar u = R^{-1} \text{Wrap}\lp(  \tilde v - b \rp)$.
The rounding scheme returns $\hat u = \bar u$ if 
$\snorm{2}{ \bar u - u } \leq \eps$. 
Otherwise, it directly returns $\tilde u = u$. 
In the latter case, the final error guarantee follows directly from the guarantee of the median-of-mean estimator.
Therefore, we focus on the case $\snorm{2}{ \bar u - u } \leq \eps = \sqrt{N} \gamma / \log(N/\delta)$.
This immediately gives us a bound on the rounding error in the $\ell_2$ norm. 
To get the bound on $\ell_\infty$, we take advantage of the random rotation $R$ applied to argue that the direction of the error vector is near-uniform. To do this, we need to first observe that $R$ remains uniform over $\SO(N)$ even under the condition $\snorm{2}{ \bar u - u } \leq \eps$. Note this condition is essentially with respect to the random variables $v$ and $u$ since $\bar u$ can be computed deterministically given $u, v$. $u$ is clearly independent from the randomness of $R$. Hence, it suffices to show that $R$ and $v$ are independent.
It is not hard to see that the distribution of $v$ is uniform conditioned on any $R$. Namely since $v=\text{Wrap}(Ru+b)$, which is uniform over the cube $[-Q,Q]^N$ no matter the `start point' $Ru$. This implies that $v, R$ are independent, so in particular $R$ is still uniform conditioned on any particular value of $v$.

We now bound the distance from $u$ to the output $\bar u = R^{-1}\text{Wrap}( \tilde v - b )$ in $\ell_\infty$ distance conditioned on $\snorm{2}{u - \bar u} \leq \eps$. 
By \Cref{claim:no-wrapping}, we always have that
\begin{align}\label{eq:no-wrapping-2}
R^{-1}\text{Wrap}(v-b) = u
\end{align}
This further implies that
\[
\snorm{\infty}{ \bar u  -  u
}
=
\snorm{\infty}{ R^{-1} 
\lp( 
\text{Wrap} \lp(\tilde v - b\rp)  - 
\text{Wrap} \lp(v - b\rp) 
\rp)}.
\]
Since any rotation $R$ preserves $\ell_2$ distance, 
we have that $\snorm{2}{ \bar u - u } \leq \eps  = \sqrt{N} \gamma / \log(N/\delta)$ if and only if
\begin{align}
\label{eq:before-rotation}
\snorm{2}{ 
\text{Wrap}\lp( v - b \rp) - 
\text{Wrap}\lp( \tilde v - b \rp) } \leq \sqrt{N} \gamma / \log(N/\delta).
\end{align}

\noindent 
Therefore, the condition on $\snorm{2}{ \bar u - u } \leq \eps$ is equivalent to the condition on \Cref{eq:before-rotation}.
Fix any $v$ such that \Cref{eq:before-rotation} is satisfied ($\tilde v$ is then also uniquely determined).
Consider the $i$-th coordinate of the vector $R^{-1} \lp( 
\text{Wrap} \lp(\tilde v - b\rp)  - 
\text{Wrap} \lp(v - b\rp) 
\rp)$. We thus have
\begin{align*}
&\Ep_{ R \sim \Uni( \SO(N) )  }
\lp[ 
\lp( R^{-1} 
\lp( 
\text{Wrap} \lp(\tilde v - b\rp)  - 
\text{Wrap} \lp(v - b\rp) 
\rp) \rp)_i
\rp] \\
 &\leq O\lp( N^{-1/2} \rp)  \snorm{2}{\lp( 
\text{Wrap} \lp(\tilde v - b\rp)  - 
\text{Wrap} \lp(v - b\rp) 
\rp)}\\
 &\leq O( \gamma / {\log (N/\delta)} ) \, ,
\end{align*}
where in the last inequality we use \Cref{eq:before-rotation}.
By the concentration of inner product of random unit vectors (see e.g.\ \cite[Theorem A.10]{diakonikolas2023algorithmic})\footnote{Theorem A.10 applies to the inner product of two random unit vectors. However, it is easy to see that the distribution of the inner product is equivalent to that between a random unit vector and a fixed unit vector. Here we apply the theorem to the inner product between the normalized version of $R^{-1} 
\lp( 
\text{Wrap} \lp(\tilde v - b\rp)  - 
\text{Wrap} \lp(v - b\rp) 
\rp)$, which is a random unit vector, and $e_i$, the $i$-th standard basis vector.}, we have that
$$
\Pr \lp(
 \lp| \lp( R^{-1} 
\lp( 
\text{Wrap} \lp(\tilde v - b\rp)  - 
\text{Wrap} \lp(v - b\rp) 
\rp) \rp)_i \rp|
 \geq  C' \gamma 
\rp)
\leq \delta / N \, ,
$$
where $C'$ is a sufficiently large constant.

By the union bound, it follows that 
$$
\snorm{\infty}{ 
R^{-1}  \text{Wrap} \lp(\tilde v - b\rp)  - 
u} =
\snorm{\infty}{ R^{-1} 
\lp( 
\text{Wrap} \lp(\tilde v - b\rp)  - 
\text{Wrap} \lp(v - b\rp) 
\rp)} \leq O(\gamma)
$$
with probability at least $1 - O(\delta)$, completing the proof.
\end{proof}

\subsection{Isoperimetric Approximate Tiling via Replicable Mean Estimation}

\newcommand{\boxwidth}{5}
\newcommand{\boxdiameter}{10}
\newcommand{\unscaledboxwidth}{20}
\newcommand{\unscaledboxdiameter}{40}

In this subsection, we show (efficient) replicable mean estimation for learning the biases of $N$-coins naturally induces an (efficient) \iat{} of $\R^N$.

\begin{restatable}[From $\ell_2$ Replicable Mean Estimation to Tiling ]{theorem}{ReplicableImpliesIsoperimetry}
    \label{thm:replicable-non-uniform-tiling-formal}
    Let $\delta \leq \rho < \frac{1}{2}$ and $\varepsilon < \frac{1}{\boxdiameter}$.
    Suppose there exists a non-adaptive $\frac{\rho}{24}$-replicable algorithm $\innerAlg$ using $m$ vector samples for the $\ell_{2}$ Learning $N$-Coin Problem that learns the mean up to error $\frac{\varepsilon}{4}$ with probability at least $1 - \delta^2 / \log(1/\delta)$.
    
    Given oracle access to $\innerAlg$, there is an $\innerAlg$-efficient algorithm $\outerAlg$ that with probability at least $1 - \delta$ generates an $\innerAlg$-efficient membership oracle for a $\left(\rho, \bigO{\rho \eps \sqrt{m}} \right)$-\iat.
\end{restatable}
If there exists a $\rho$-replicable mean estimation algorithm in $\ell_2$ norm with (vector) sample complexity $m = O( N^2 \eps^{-2} \rho^{-2} )$ (assuming $N \gg \log(1/\delta)$),
\Cref{thm:replicable-non-uniform-tiling} immediately yields an \iat~whose normalized surface area is asymptotically optimal, i.e., $\bigO{\rho \eps \sqrt{m}} = O(N)$. 
{Together with \Cref{thm:tiling-to-replicable}, this shows that our reductions are optimal up to constant factors
with respect to the bounds on the surface area and the sample complexity.}

\subsubsection{Approximate Partition of Cubes}
The core of \Cref{thm:replicable-non-uniform-tiling-formal} is really to show that any replicable algorithm for learning $N$-coins implies an \emph{approximate hypercube partition} whose boundary has low surface area. We will see how this can be extended to an \iat{} of $\R^N$ in \Cref{sec:extend}.
\begin{definition}[Approximate Partition of Cubes]
    \label{def:approximate-cube-isoperimetry}
    Let $\cube = [a_i, b_i]^{N} \subset \R^{N}$.
    A finite collection of sets $\set{S_{v}}$ labeled by $v \in \R^N$
    is a \emph{$(\gamma, \eps, A)$-approximate partition of $\cube$}
    if:
    \begin{enumerate}        
        \item (Disjoint) The interiors $\interior(S_{v}) \subset \R^{N}$ are mutually disjoint
        \label{item:approximate-cube-partition:disjoint}

        \item (Non-Zero Volume) $\vol(S_{v} \cap \cube) > 0$ for each $v$
        \label{item:approximate-cube-partition:non-zero-volume}

        \item (Semialgebraic) $S_{v}$ is semialgebraic
        \label{item:approximate-cube-partition:semialgebraic}

        \item ($\gamma$-Approximate Volume) 
        \begin{equation*}
            \vol \left( \bigcup_{v} S_{v} \cap \cube \right) \geq (1 - \gamma) \cdot \vol(\cube)
        \end{equation*}
        \label{item:approximate-cube-partition:approx-volume}
        
        \item ($\varepsilon$-Radius) For all $S_{v}$ with label $v$ and $x \in S_{v}$,
        \begin{equation*}
            \norm{x - v}{2} \leq \varepsilon
        \end{equation*}
        \label{item:approximate-cube-partition:radius}
        
        \item ($A$-Surface Area)
        \begin{equation*}
            \area \left( \bigcup_{v} \boundary S_{v} \cap \cube \right) \leq A \cdot \vol(\cube)
        \end{equation*}
        \label{item:approximate-cube-partition:surface-area}
    \end{enumerate}
    We call the partition \emph{efficient} if there is an algorithm $\innerAlg$ such that for all $v$ and $x \in S_{v}$, $\Pr(\innerAlg(x) = v) \geq \frac{2}{3}$.
    We call $\innerAlg$ the membership oracle.
\end{definition}
We first prove the following variant of \Cref{thm:replicable-non-uniform-tiling-formal} for approximate hypercube partitions.
The formal statement of the reduction is given below.
\begin{restatable}{proposition}{l2learnpartition}
    \label{thm:replicable-implies-isoperimetry}
    Let $\delta \leq \rho$ and $\varepsilon < \frac{1}{\unscaledboxdiameter}$.
    Suppose there exists a $\frac{\rho}{24}$-replicable algorithm $\innerAlg$ using $m$ vector samples for the $\ell_2$-Learning $N$-Coin Problem that learns the mean up to error $\frac{\varepsilon}{4}$ with probability at least $1 - \delta^2 / \log(1/\delta)$. 
    Given oracle access to $\mathcal{A}$, there is an $\mathcal{A}$-efficient algorithm $\outerAlg$ that with probability at least $1-\delta$ generates an $\innerAlg$-efficient membership oracle for a $\left(\rho, \frac{1}{\boxdiameter}, \bigO{\rho \varepsilon \sqrt{m}}\right)$-partition $\mathcal{P}$ of $\left[0, 1 \right]^{N}$ 
\end{restatable}

We first overview the proof. Assume for simplicity that $\delta = 0$ and $\varepsilon$ is a small constant.
The key idea behind our construction is to argue that a replicable algorithm, after fixing the random string $r$, defines a partition $\mathcal{P}_r$ of the input space whose sets correspond to the ``pre-image'' of certain high probability ``canonical'' outcomes of the algorithm, i.e., 
$F_{\hat p} = \{ p \in [0,1]^N \mid \innerAlg( S_p;r ) = \hat p \text{ with high probability}\}$ for a canonical output $\hat p$. 
To understand the properties of $\mathcal{P}_r$, we use a minimax type argument and consider an adversary drawing input biases uniformly from $\cube = \left[\frac{1}{4}, \frac{3}{4}\right]^{N}$. 
In general, our cube $\cube$ will have side length depending on $\varepsilon$, but we assume $\varepsilon$ is a constant for the overview.

We first argue that $\mathcal{P}_r$ has good volume and diameter. Similar to our lower bound for the single coin problem, we know that for any `good' random string $r$ most input biases $p$ will have some canonical (replicating) outcome $\hat p$. As a result, the corresponding partition $\mathcal{P}_r$ must cover at least a $1 - \rho$ fraction of the input space $\cube$. For diameter, consider one such set $F_{\hat{p}} \in \mathcal{P}_r$, corresponding to the set of biases $p \in \cube$ whose canonical output is $\hat{p}$. By correctness, we know most $p \in F_{\hat{p}}$ should satisfy $\norm{p-\hat{p}}{2} \leq \varepsilon$, giving the desired upper bound.

We now turn to the main challenge: surface area. The argument has two key components, similar to our lower bound for the 1-Coin problem. First, we argue that on the \textit{boundary} of each set $F_{\hat{p}}$, the probability of the canonical outcome is bounded away from $1$ and therefore $\mathcal{A}$ fails to replicate on such points (this takes the place of the `balanced' point $p_r$ for $1$-Coin). Using mutual information, we then argue that any algorithm on $m$ samples, and in particular $\mathcal{A}$,
cannot effectively distinguish input biases $\norm{p - q}{2} \leq \sqrt{1/m}$. This means the algorithm must fail to replicate not just on $\partial F_{\hat{p}}$, but on any `thickening' $\partial F_{\hat{p}} + B_{\ell}$ for $\ell \leq \sqrt{1/m}$. If the boundary of each thickening has surface area at least $A$, their combined volume is at least $A \sqrt{1/m} \leq \rho$. As a result, there is some radius $\ell^*$ such that the thickened sets $\{F_{\hat{p}} + B_{\ell^*}\}$ gives the desired partition.

Finally, to simulate the membership oracle of this partition,
given a point $p \in \cube$ we simply generate a sample from $p$ and run the algorithm with the fixed random string on the sample. Generating the membership oracle (and implicitly the partition) can be done by testing $\mathcal{A}$ on several random strings until a `good' (i.e.\ replicable and correct) string is found with high probability.

\begin{proof}[Proof of \Cref{thm:replicable-implies-isoperimetry}]
    We will in fact show that given a $\rho$-replicable algorithm $\innerAlg$ with $m$ {vector samples} that learns the mean up to error $\varepsilon$ in $\ell_{2}$ norm, we obtain an $\innerAlg$-efficient membership oracle for a $\left(24 \rho, 4 \varepsilon, \bigO{\rho \sqrt{m}}\right)$-partition of $\left[ \frac{1}{2} - \unscaledboxwidth \varepsilon, \frac{1}{2} + \unscaledboxwidth \varepsilon\right]^{N}$.
    Thus, by starting with a $\frac{\rho}{24}$-replicable algorithm with error $\frac{\varepsilon}{4}$, we obtain a $\left(\rho, \varepsilon, \bigO{\rho \sqrt{m}}\right)$-partition of $\left[ \frac{1}{2} - \boxwidth \varepsilon, \frac{1}{2} + \boxwidth \varepsilon\right]^{N}$.
    Then, by translating the sets of the partition we obtain a $\left(\rho, \varepsilon, \bigO{\rho \sqrt{m}}\right)$-partition of $\left[0, \boxdiameter \varepsilon \right]^{N}$.
    Scaling the partition by $\frac{10}{\eps}$ then gives the desired partition of the cube $[0, 1]^N$.
    \begin{restatable}[Partition Rescale]{lemma}{scalingPartition}
        \label{lemma:scaling-partition}
        Suppose $\partition$ is a $(\rho, \varepsilon, A)$-partition of $[0, \boxdiameter \varepsilon]^{N}$.
        Then, there is a $(\rho, 1/\boxdiameter, \boxdiameter \varepsilon A)$-partition of $[0, 1]^{N}$.
    \end{restatable}
    We defer the proof of \Cref{lemma:scaling-partition}, and first show how to construct our 
    $\left(24 \rho, 4 \varepsilon, \bigO{\rho \sqrt{m}} \right)$-partition of $\left[ \frac{1}{2} - \unscaledboxwidth \varepsilon, \frac{1}{2} + \unscaledboxwidth \varepsilon\right]^{N}$.
    Define $\cube = \left[ \frac{1}{2} - \unscaledboxwidth \varepsilon, \frac{1}{2} + \unscaledboxwidth \varepsilon\right]^{N}$, and consider an adversary who samples a bias vector $p \in \cube$ uniformly. 
    Our construction requires the following two facts:
    \begin{enumerate}
        \item Most random strings are good (Definition \ref{def:good-random-string}).

        \item On any good random string, $\innerAlg(; r)$ partitions $\cube$.
    \end{enumerate}
    Then, by sampling a random point $p$, whenever we sample $p \in \cube$ close to the boundary of the partition, $\innerAlg(; r)$ is not replicable on samples from $p$.
    Therefore, the surface area of the partition must be small for $\innerAlg(; r)$ to be replicable on a large fraction of biases.

    We begin by defining good random strings.
    \begin{definition}
        \label{def:good-random-string}
        A random string $r$ is good if it satisfies the following two conditions:
        \begin{enumerate}
            \item (Global Correctness)
            \begin{equation*}
                \Pr_{p, S_p} \left(\norm{p - \innerAlg(S_p; r)}{2} \geq \varepsilon \right) \leq 3 \delta .
            \end{equation*}
            \item (Global Replicability)
            \begin{equation*}
                \Pr_{p, S_p, S_p'} \left( \innerAlg(S_p; r) \neq \innerAlg(S_p'; r) \right) \leq 3 \rho ,
            \end{equation*}
            where $p$ is sampled uniformly from $\cube$.
    \end{enumerate}
    \end{definition}

    The following lemma states that most random strings are good.

    \begin{restatable}[Many Good Random Strings]{lemma}{manyGoodR}
        \label{lemma:non-adapt-learning-n-coin-good}
        If $r$ is selected uniformly at random, $r$ is good with probability at least $\frac{1}{3}$.
    \end{restatable}
    
    Fix a good random string $r$ satisfying correctness and replicability on a large fraction of inputs from $\cube$.
    Then, $\innerAlg(; r)$ induces an approximate partition of $[0, 1]^{N}$ as follows.
    
    Consider the following function of $p \in \cube$ and some outcome $\hat{p}$:
    \begin{align*}
        \Pr(\innerAlg(S_p; r) = \hat{p}) &= \sum_{S} \mathbf{1}[\innerAlg(S; r) = \hat{p}] \Pr(S | p),
    \end{align*}
    where $S$ ranges over all samples of size $m$ and $\Pr(S | p)$ is the probability of drawing sample $S$ from distribution with bias vector $p$. 
    Since every sample with $j_i$ heads on the $i$-th coin is equally likely, note that we may also write
    \begin{equation*}
        \Pr(\innerAlg(S_p; r) = \hat{p}) = \sum_{(j_1, j_2, \dotsc, j_{N})} f_{(j_1, j_2, \dotsc, j_{N})} \prod_{i = 1}^{N}  \binom{m}{j_i} p_i^{j_i} (1 - p_i)^{m - j_i},
    \end{equation*}
    where $f_{(j_1, j_2, \dotsc, j_{N})}$ is the proportion of samples with $j_i$ heads on the $i$-th coin on which $\innerAlg(; r)$ outputs $\hat{p}$.
    We define the following polynomial that agrees with $\Pr(\innerAlg(S_p; r) = \hat{p})$ on $\cube \subset [0, 1]^{N}$:
    \begin{equation}
        \label{eq:outcome-prob-polynomial}
        h_{\hat{p}}(p) \coloneqq \sum_{(j_1, j_2, \dotsc, j_{N})} f_{(j_1, j_2, \dotsc, j_{N})} \prod_{i = 1}^{N}  \binom{m}{j_i} p_i^{j_i} (1 - p_i)^{m - j_i}.
    \end{equation}

    We then partition $\cube$ into sets labeled by the canonical outcomes $\hat{p}$.
    \begin{definition}
        \label{def:f-p-hat}
        Fix random string $r$ and $\hat{p} \in \R^{N}$.
        Define
        \begin{equation*}
            F_{\hat{p}}(r) = \bigset{p \given h_{\hat{p}}(p) > \frac{3}{4}},
        \end{equation*}
        where $h_{\hat{p}}$ is defined as in \Cref{eq:outcome-prob-polynomial}.
        
        For any $\ell > 0$, define $F_{\hat{p}, \ell}(r) = F_{\hat{p}}(r) + B_{\ell}$.
        When the random string $r$ is clear, we omit $r$ and write $F_{\hat{p}}, F_{\hat{p}, \ell}$.
    \end{definition}

    Below we list some basic properties of the sets $\{F_{\hat p}\}$ that follow from their definitions.
    First, 
    for any $p \not\in F_{\hat{p}}$, we have that $\innerAlg(S_p; r)$ has no (strongly) canonical output, i.e., output that appears with probability strictly more than $3/4$, and therefore is not $\rho$-replicable.
    Second,
    we observe that each non-empty $F_{\hat{p}}$ is open, and therefore the collection of non-empty $F_{\hat{p}}$ is countable.
    Lastly, each set $F_{\hat{p}}$ is semialgebraic (as they are defined by polynomial threshold functions).
    \begin{restatable}[Basic Properties of Canonical Approximate Partition]{lemma}{complementNotReplicable}
        \label{lemma:complement-not-replicable}
        Suppose $p \in \cube \setminus \bigcup_{\hat{p}} F_{\hat{p}}$.        
        Let $S_{p}, S_{p}'$ denote two independent samples drawn from $\text{Bern}(m,p)$. Then
        \begin{equation*}
            \Pr(\innerAlg(S_p; r) = \innerAlg(S_p', r)) \leq \frac{5}{8}.
        \end{equation*}        
        Moreover, each set $F_{\hat{p}}$ is open and semialgebraic.
    \end{restatable}

    Recall from the proof overview that our end partition will actually correspond to thickenings of $F_{\hat{p}}$. In particular, we'd like to show both that such thickenings are disjoint, and further that the thickened boundaries $\partial F_{\hat{p}} + B_\ell$ are non-replicable. The key to both facts is the following lemma based on our mutual information framework showing nearby biases have similar canonical solutions:

    \begin{restatable}[Output Lipschitzness]
    {lemma}{replicableLipschitz}
        \label{lemma:replicable-lipschitz}
        There exists a universal constant $c$ such that for any $p, q \in \left[\frac{1}{4}, \frac{3}{4} \right]^{N}$ with $\norm{p - q}{2} \leq \frac{c}{\sqrt{m}}$ and any $\hat{p}$
        \begin{equation*}
            |\Pr_{S_p}(\innerAlg(S_p; r) = \hat{p}) - \Pr_{S_q}(\innerAlg(S_q; r) = \hat{p})| < \frac{1}{15}.
        \end{equation*}
    \end{restatable}
    Note that by assumption we have $\varepsilon < \frac{1}{\unscaledboxdiameter}$ so $\cube \subset \left[\frac{1}{4}, \frac{3}{4} \right]^{N}$.
    Let $R=\frac{c}{\sqrt{m}}$ where $c$ is as in the above lemma.
    Instead of working directly with the thickenings $F_{\hat{p},\ell}$, to ensure our sets have bounded diameter we instead work with thickenings of the thresholded sets $F_{\hat{p}} \cap B_\varepsilon(\hat{p})$. Note that this typically only removes a $\delta$ fraction of the set, since correctness promises most elements are indeed within $\varepsilon$ of $\hat{p}$. With this in mind, 
    for each $\ell \leq R$ we define the set of (thresholded) thickenings:
    \begin{definition}
        \label{def:g-p-hat}
        Given $F_{\hat{p}}$ and $0 \leq \ell \leq R$, define $G_{\hat{p}, \ell}$ as
        \begin{equation*}
            G_{\hat{p}, \ell} = (F_{\hat{p}} \cap B_{\varepsilon}(\hat{p})) + B_{\ell}.
        \end{equation*}
        Furthermore, let $G_{\hat{p}} = G_{\hat{p}, 0} = F_{\hat{p}} \cap B_{\varepsilon}(\hat{p})$.
    \end{definition}
    We choose our final partition to be given by the thickening $\ell$ with the lowest surface area. In particular, denote the infinum surface area over thickenings as
    \begin{equation*}
        A_0 = \inf_{\ell \in [0, R]} \area \left( \bigcup_{\hat{p}} \boundary G_{\hat{p}, \ell} \cap \cube \right)
    \end{equation*}
    Since the surface area $\area(\bigcup_{\hat{p}} \boundary G_{\hat{p}, \ell} \cap \cube)$ is continuous in $\ell$, the infimum is achieved on the compact interval $[0, R]$ so that
    there exists some $\ell^* \in [0, R]$ achieving surface area near the infinum:
    \begin{equation*}
        \area \left( \bigcup_{\hat{p}} \boundary G_{\hat{p}, \ell^*} \cap \cube \right) \leq 2 A_0.
    \end{equation*}
    We claim the collection of non-empty $G_{\hat{p}, \ell^*}$ 
    give the desired partition of $\cube$. We next show that this $G_{\hat{p}, \ell^*}$ satisfies Properties (\ref{item:approximate-cube-partition:disjoint}-\ref{item:approximate-cube-partition:radius}) of \Cref{def:approximate-cube-isoperimetry}.
    \paragraph{Properties (\ref{item:approximate-cube-partition:disjoint}-\ref{item:approximate-cube-partition:radius}).}

    We restate the properties here as a lemma for convenience.
    \begin{restatable}{lemma}{gplProperties}
        \label{lemma:learn-g-p-l-properties}
        Let $\delta \leq \rho$.
        Let $G_{\hat{p}, \ell}$ be a collection of subsets indexed by $\hat{p} \in \R^{N}$ specified in Definition \ref{def:g-p-hat} where $\ell \leq R$.
        Then, the following properties hold,
        \begin{enumerate}
            \item (Disjoint) $\closure \left(G_{\hat{p}_1, \ell} \cap \cube\right) \cap \closure \left(G_{\hat{p}_2, \ell} \cap \cube\right) = \emptyset$ for all $\hat{p}_1 \neq \hat{p}_2$.
            \label{item:learn-g-p-l-properties:disjoint}

            \item (Non-Zero Volume) If $G_{\hat{p}, \ell} \neq \emptyset$, then $\vol(G_{\hat{p}, \ell}) > 0$.
            \label{item:learn-g-p-l-properties:non-zero-volume}

            \item (Semialgebraic) Each $G_{\hat{p}, \ell}$ is semialgebraic.
            \label{item:learn-g-p-l-properties:semialgebraic}

            \item (Large Total Volume) 
            \begin{equation*}
                \vol \left( \bigcup_{\hat{p}} G_{\hat{p}, \ell} \cap \cube \right) = \sum_{\hat{p}} \vol \left( G_{\hat{p}, \ell} \cap \cube \right) \geq (1 - 12 \rho) \vol(\cube).
            \end{equation*}
            \label{item:learn-g-p-l-properties:total-volume}
            
            \item (Small Radius) For every $p \in G_{\hat{p}, \ell}$: $\norm{p-\hat{p}}{2} \leq 2\varepsilon$.
            \label{item:learn-g-p-l-properties:radius}
        \end{enumerate}
    \end{restatable}
\begin{proof}
    We proceed in order. Toward Property \ref{item:learn-g-p-l-properties:disjoint}, observe \Cref{lemma:replicable-lipschitz} implies that for every $p \in \closure(G_{\hat{p}, \ell})$ we have $\Pr(\innerAlg(S_{p}; r) = \hat{p}) \geq \frac{3}{4} - \frac{1}{15} = \frac{3}{5}$. 
    Thus $\{\closure(G_{\hat{p}, \ell})\}_{\hat p}$ are disjoint for distinct $\hat{p}$.
    

    Toward Property \ref{item:learn-g-p-l-properties:non-zero-volume}, observe any $G_{\hat{p}, \ell}$ is a Minkowski sum of a set with the open ball $B_{\ell}$ and is therefore open. Any non-empty open set has non-zero volume, so we are done.

    Toward Property \ref{item:learn-g-p-l-properties:semialgebraic}, observe that $F_{\hat{p}}$ is a semialgebraic set by \Cref{lemma:complement-not-replicable}.
    Furthermore, $B_{\varepsilon}(\hat{p})$ is a semialgebraic set as the set of $p$ satisfying $\norm{p - \hat{p}}{2}^2 = \sum_{i} (p_i - \hat{p}_i)^2 < \varepsilon^2$. Intersections and Minkowski sums of semialgebraic sets are semialgebraic so we are done (the former is standard, for the latter see \Cref{lemma:minkowski-sum-semialgebraic}).

    
    Toward Property \ref{item:learn-g-p-l-properties:total-volume}, observe that since $G_{\hat{p}} \subset G_{\hat{p}, \ell}$ it suffices to prove $\bigcup_{\hat{p}} G_{\hat{p}}$ satisfies the volume bound.
    Suppose $p \in \cube \setminus \bigcup_{\hat{p}} G_{\hat{p}}$.
    There are two cases:
    \begin{enumerate}
        \item $p \in \bigcup_{\hat{p}} F_{\hat{p}}$.
        Then $p \not\in B_{\varepsilon}(\hat{p})$ since otherwise $p \in G_{\hat{p}} \subset G_{\hat{p}, \ell}$ for the $\hat{p}$ such that $p \in F_{\hat{p}}$.
        Then, $\innerAlg(; r)$ is incorrect with probability at least $\frac{3}{4}$:
        \begin{equation*}
            \Pr(\norm{\innerAlg(S_p; r) - p}{2} \geq \varepsilon) \geq \Pr(\innerAlg(S_p; r) = \hat{p}) \geq \frac{3}{4}.
        \end{equation*}

        \item $p \not\in \bigcup_{\hat{p}} F_{\hat{p}}$.
        Then \Cref{lemma:complement-not-replicable} implies that $\innerAlg(; r)$ is not replicable with probability at least $\frac{3}{8}$.
    \end{enumerate}
    Since $p$ is sampled uniformly from $\cube$ we have by goodness of $r$:
    \begin{align*}
        3 \delta &\geq \Pr_{p, S_p} \left(\norm{\innerAlg(S_p; r) - p}{2} \geq \varepsilon \right) \geq \frac{3}{4} \frac{\vol\left( \bigcup_{\hat{p}} F_{\hat{p}} \setminus B_{\varepsilon}(\hat{p}) \right)}{\vol(\cube)} \\
        3 \rho &\geq \Pr_{p, S_p, S_p'} \left(\innerAlg(S_p; r) \neq \innerAlg(S_p'; r) \right) \geq \frac{3}{8} \frac{\vol\left( \cube \setminus \bigcup_{\hat{p}} F_{\hat{p}} \right)}{\vol(\cube)}.
    \end{align*}
    Finally, since $\cube \setminus \bigcup_{\hat{p}} G_{\hat{p}} \subset \left(\cube \setminus \bigcup_{\hat{p}} F_{\hat{p}} \right) \cup \left( \bigcup_{\hat{p}} F_{\hat{p}} \setminus B_{\varepsilon}(\hat{p}) \right)$, we upper bound,
    \begin{align*}
        \frac{\vol\left( \cube \setminus \bigcup_{\hat{p}} G_{\hat{p}} \right)}{\vol(\cube)} &\leq \frac{\vol\left( \cube \setminus \bigcup_{\hat{p}} F_{\hat{p}} \right)}{\vol(\cube)} + \frac{\vol\left( \bigcup_{\hat{p}} F_{\hat{p}} \setminus B_{\varepsilon}(\hat{p}) \right)}{\vol(\cube)} \\
        &\leq 4 \delta + 8 \rho \leq 12 \rho,
    \end{align*}
    where we use our assumption $\delta \leq \rho$.

    It is left to prove Property \ref{item:learn-g-p-l-properties:radius}. 
    Note that we can assume $m = \bigOm{1/\varepsilon^2}$ as such a lower bound holds even for (non-replicably) learning the bias of a single coin (see, e.g., \Cref{thm:q0-num-lower-bound}).
    Thus, for the appropriate setting of constants we have $R \leq \varepsilon$ and $G_{\hat{p}, \ell} \subset B_{2 \varepsilon}(\hat{p})$ and has radius at most $2 \varepsilon$ with respect to $\hat{p}$ as desired.
\end{proof}
    \paragraph{Property \ref{item:approximate-cube-partition:surface-area}.}
    To complete the proof that $\{G_{\hat{p},\ell^*}\}$ is a $\left(12\rho, 2\varepsilon, \bigO{\rho \sqrt{m}}\right)$-partition, we need to bound its surface area to volume ratio. We first argue that for every $0 \leq \ell \leq R$ and  $p \in \boundary G_{\hat{p}, \ell}$, $\innerAlg(\cdot;r)$ is either incorrect or non-replicable on $p$. Namely we break into the following two cases:    
    \begin{enumerate}
        \item \textbf{Case 1:} $p \not\in B_{\varepsilon}(\hat{p})$. 
        
        Since there exists by construction $x \in F_{\hat{p}}$ such that $\norm{p - x}{2} \leq \ell \leq R$, \Cref{lemma:replicable-lipschitz} implies
        \begin{equation*}
            \Pr_{S_p}(\norm{\innerAlg(S_p; r) - p}{2} \geq \varepsilon) \geq \Pr_{S_p}(\innerAlg(S_p; r) = \hat{p}) \geq \frac{3}{4} - \frac{1}{15} = \frac{3}{5}.
        \end{equation*}
        Since $\norm{p-\hat{p}}{2} \geq \varepsilon$, $\innerAlg(; r)$ is therefore incorrect with probability at least $\frac{3}{5}$.
        
        \item \textbf{Case 2:} $p \in B_{\varepsilon}(\hat{p}) = \interior(B_{\varepsilon}(\hat{p}))$.
        
        In this case, observe $p \not\in F_{\hat{p}}$. Otherwise $p \in \interior(F_{\hat{p}} \cap B_\varepsilon(\hat{p}))\subset \interior(G_{\hat{p}, \ell})$, and in particular is not in $\boundary G_{\hat{p}, \ell}$.
        Since $\boundary G_{\hat{p}, \ell}$ is disjoint from $\bigcup_{\hat{p}} F_{\hat{p}}$, \Cref{lemma:complement-not-replicable} then implies $\innerAlg(; r)$ is not replicable with probability at least $\frac{3}{8}$.
    \end{enumerate}
    Since the bias $p$ is drawn uniformly from $\cube$, goodness of $r$ then implies the following relation:    
    \begin{align}\label{eq:vol-bound-1}
        \frac{3}{8} \vol \left( \bigcup_{\ell} \bigcup_{\hat{p}} \boundary G_{\hat{p}, \ell} \cap \cube \right) &\leq 3 (\delta + \rho) \cdot \vol(\cube) \leq 6 \rho \cdot \vol(\cube).
    \end{align}
    Recall we showed for any $\hat{p}_1 \neq \hat{p}_2$, $G_{\hat{p}, R}$ are disjoint (in fact their closures are disjoint). Further it is elementary to see for fixed $\hat{p}$ and $\ell_1 \neq \ell_2$ the boundaries $\partial G_{\hat{p},\ell_1}$ and $\partial G_{\hat{p},\ell_2}$ are disjoint.
    
    Thus, we can lower bound the combined volume of all the boundaries $\boundary G_{\hat{p}, \ell}$ as
    
    \begin{align*}
        \vol \left( \bigcup_{\ell} \bigcup_{\hat{p}} \boundary G_{\hat{p}, \ell} \cap \cube \right) &= \int_{0}^{R} \area \left( \bigcup_{\hat{p}} \boundary G_{\hat{p}, \ell} \cap \cube \right) d\ell \\
        &\geq R A_0 \\
        &= \bigOm{\frac{A_0}{\sqrt{m}}}.
    \end{align*}
    Combined with \Cref{eq:vol-bound-1} we therefore have
    \begin{align*}
        A_0 &= \bigO{\rho \sqrt{m} \vol(\cube)},
    \end{align*}
    and re-arranging gives the desired bound on the surface to volume ratio of $\{G_{\hat{p},\ell^*}\}$:
    \begin{equation*}
        \area \left( \bigcup_{\hat{p}} \boundary G_{\hat{p}, \ell^*} \cap \cube \right) \leq 2 A_0 = \bigO{\rho \sqrt{m} \vol(\cube)}.
    \end{equation*}

    \paragraph{The Membership Oracle.} It is left to give the oracle-efficient algorithm $\outerAlg$ that `generates' some good $\mathcal{P}_r=\{G_{\hat{p},\ell^*}\}$ in the sense of producing an oracle-efficient membership oracle to $\mathcal{P}_r$. The main challenge in this process is simply finding a `good' random string, for which we've already shown $\mathcal{P}_r$ is a $(12\rho,2\varepsilon,O(\rho\sqrt{m}))$-partition.    
    

    \begin{restatable}
    [$\innerAlg$-Efficient Membership Oracle]
    {proposition}{generateP}
    \label{prop:find-good-random-string-eff}
        Let $\delta \leq \rho$.
        Suppose there is a $\rho$-replicable algorithm $\innerAlg$ with $f(N, \varepsilon, \rho, \delta)$ time complexity for the $\ell_{2}$ Learning $N$-Coin Problem.
        Then, there is a randomized algorithm in $\bigO{\frac{\log^2 1/\delta}{\rho^2} f \left(N, \varepsilon, \frac{\rho}{9}, \frac{\delta^2}{\log(3/\delta)} \right)}$ time that finds a good random string $r$ for $\innerAlg$ with probability at least $1 - \delta$.
    \end{restatable}
    We defer the proof, which simply amounts to directly testing replicability and bounding the correctness probability.
    Once we have the good string $r$ in hand, implementing the membership oracle is as simple as running $\mathcal{A}(;r)$. Namely for any $p \in G_{\hat{p}, \ell^*}$, run $\innerAlg(; r)$ on $m$ independently samples from $\text{Bern}(p)$. 
    By construction, $\innerAlg(; r)$ outputs $\hat{p}$ on samples under $p$ with probability at least $\frac{3}{5}$. We can repeat the procedure a constant number of times to boost the probability of the majority answer by an arbitrary constant factor, obtaining an algorithm outputting $\hat{p}$ with probability at least $\frac{2}{3}$ as desired.
    This concludes the proof of \Cref{thm:replicable-implies-isoperimetry}.
\end{proof}

It is left to prove the claimed lemmas. We proceed in order.
\paragraph{Many Random Strings are Good}

First, we prove that many random strings are good. 

\manyGoodR*

This follows from a similar application of Markov's inequality as in the single coin case (\Cref{thm:q0-num-lower-bound}), and we omit the proof.

\paragraph{Properties of $F_{\hat p}$} We next demonstrate that the canonical sets $\{F_{\hat p}\}$ constructed satisfy basic niceness properties.

\complementNotReplicable*

\begin{proof}
    Note that $F_{\hat{p}}$ is the pre-image of $(3/4, \infty)$ of $h_{\hat{p}}(p)$, which is a multivariate polynomial in $p$.
    This immediately implies that $F_{\hat{p}}$ is open and semialgebraic.
    
    For any $p \not\in \bigcup_{\hat{p}} F_{\hat{p}}$ we have
    \begin{align*}
        \Pr_{S_p, S_p'} \left( \innerAlg(S_p; r) = \innerAlg(S_p'; r) \right) &= \sum_{\hat{p}} \Pr_{S_p, S_p'} \left( \innerAlg(S_p; r) = \innerAlg(S_p'; r) = \hat{p} \right) \\
        &= \sum_{\hat{p}} \Pr_{S_p} \left( \innerAlg(S_p; r) = \hat{p} \right)^2 \\
        &\leq \frac{5}{8}
    \end{align*}
    since the sum of squares is maximized when the probabilities are concentrated on as few elements as possible (one with probability $3/4$ and the other with probability $1/4$).
\end{proof}

    \begin{restatable}{lemma}{minkowskiSumSemialgebraic}
        \label{lemma:minkowski-sum-semialgebraic}
        Suppose $X, Y \subset \R^{n}$ are semialgebraic sets.
        Then, the Minkowski sum $X + Y \subset \R^{n}$ is semialgebraic.
    \end{restatable}

The proof requires the Tarski-Seidelberg Theorem.

\begin{theorem}[Tarski-Seidelberg \cite{bierstone1988semianalytic}]
    \label{thm:tarski-seidelberg}
    Suppose $X \subset \R^{n + 1}$ is a semialgebraic set and let $\pi: \R^{n + 1} \mapsto \R^{n}$ be the projection map given by
    \begin{equation*}
        \pi(x_1, \dotsc, x_{n}, x_{n + 1}) = (x_1, \dotsc, x_{n}).
    \end{equation*}
    Then, $\pi(X) \subset \R^{n}$ is semialgebraic.
\end{theorem}

\begin{proof}[Proof of \Cref{lemma:minkowski-sum-semialgebraic}]
    Define the following set $\Gamma \subset \R^{n} \times \R^{n} \times \R^{n}$ as follows,
    \begin{equation*}
        \Gamma = \set{(z, x, y) \given z = x + y}.
    \end{equation*}
    Note that $\Gamma$ is semialgebraic as it is defined by the polynomial equalities $z_i = x_i + y_i$.
    Furthermore, note that $\R^{n} \times X \times \R^{n}$ and $\R^{n} \times \R^{n} \times Y$ are also semialgebraic, by considering the polynomial constraints imposed only on the relevant coordinates.
    Thus,
    \begin{equation*}
        \Gamma \times (\R^{n} \times X \times \R^{n}) \times (\R^{n} \times \R^{n} \times Y) = \set{(z, x, y) \given z = x + y, x \in X, y \in Y}
    \end{equation*}
    is semialgebraic as an intersection of semialgebraic sets.
    Finally, by repeatedly applying the Tarski-Seidelberg Theorem (\Cref{thm:tarski-seidelberg}), we can project onto the first $n$ coordinates, which is precisely the set $X + Y$.
\end{proof}

\paragraph{Lipschitz Conditions}
We now show that for two nearby points $p, q \in \cube$, the output distribution of $\innerAlg$ given sample sets $S_p, S_q$
generated from $N$ coins with mean $p$ and $q$ respectively
cannot be too different.

\replicableLipschitz*

\newcommand{\numAlgRepeats}{1000}
\begin{proof}
    The proof is similar to the 1-Coin variant. 
    Recall that the algorithm takes $m$ vector samples, and each vector sample consists of a flip of each coin.
    Given $p,q \in \cube$, let $X$ be a fair coin. 
    If $X = 0$, we generate samples from $N$ coins parametrized by $p$.
    Otherwise, we generate samples from $N$ coins parametrized by $q$.
    Let $Y = (Y_1, \ldots, Y_N)$ be a random integer vector correlated with $X$, where $Y_i$ denotes the number of 
    head counts of the $i$-th coin observed by the algorithm.
    Note that $Y_i$ are independet conditioned on $X$.
    Moreover, $Y_i$ is distributed as $\BinomD{m}{p_i}$ if $X = 0$ and $\BinomD{m}{q_i}$ if $X = 1$.
    Let $\overline{Y}$ denote \numAlgRepeats{} independent samples of $Y$.
    
    Thus, by \Cref{lemma:one-coin-mutual-info-bound}, the mutual information of $X$ and $Y$ is at most
    \begin{align*}
        I(X: Y) \leq \sum_{i} I(X: Y_i) = \bigO{\sum_{i} m |q_i - p_i|^2} = \bigO{m \norm{p-q}{2}^2}.
    \end{align*}
    Furthermore, since each sample of $Y$ in $\overline{Y}$ is independent, we also have
    \begin{equation*}
        I(X: \overline{Y}) = O(m \norm{p - q}{2}^2).
    \end{equation*}
    Thus for $\norm{p-q}{2} < \frac{c}{\sqrt{m}}$ for $c>0$ some sufficiently small constant, we have that the mutual information is at most $I(X: \overline{Y}) < 2 \cdot 10^{-4}$. In particular, by \Cref{lemma:alg-mutual-info-lower-bound}, there is no function $f(\overline{Y})=X$ with probability at least $51\%$.

    On the other hand, assume for the sake of contradiction there exists $\norm{p-q}{2} < \frac{c}{\sqrt{m}}$ such that    
    \begin{equation*}
        |\Pr_{S_p}(\innerAlg(S_p; r) = \hat{p}) - \Pr_{S_q}(\innerAlg(S_q; r) = \hat{p})| \geq \frac{1}{15}.
    \end{equation*}
    Assume without loss of generality that $\Pr_{S_q}(\innerAlg(S_q; r) = \hat{p}) \geq \Pr_{S_p}(\innerAlg(S_p; r) = \hat{p})$, and let $Z(\overline{Y}, \hat p)$ denote the fraction of datasets $S$ such that $\mathcal{A}(S;r)=\hat{p}$ among all datasets whose head counts agree with $Y$. Then define the function
    \[
    f(\overline{Y}) = \begin{cases}
        1 & Z(\overline{Y}, \hat p) > {\lp(\Pr_{S_p}(\innerAlg(S_p; r) = \hat{p})+\Pr_{S_q}(\innerAlg(S_q; r) = \hat{p})\rp)}/{2}\\
        0 & \otherwise
    \end{cases}.
    \]
    By standard concentration bounds, one can check that the above function guesses $X$ with probability at least $\Pr[f(\overline{Y})=X] \geq .51$, leading to a contradiction. 
    \qedhere
\end{proof}

\paragraph{Scaling of Partitions}
We show how scaling affects the parameters of an isoperimetric approximate hypercube partition.
\scalingPartition*
\begin{proof}
    Define the partition $\partition'$ as follows.
    For each $S \in \partition$, let $S' = \set{\frac{x}{\boxdiameter \varepsilon} \given x \in S}$ to be the set $S$ scaled by $\frac{1}{\boxdiameter \varepsilon}$.
    The partition $\partition'$ is made up of sets of $S'$.
    We equip set $S_v'$ with label $\frac{v}{\boxdiameter \varepsilon}$.

    Toward Property \ref{item:approximate-cube-partition:disjoint}, note that the scaling map is injective.
    
    Toward Property \ref{item:approximate-cube-partition:non-zero-volume}, note that the volume scales by $\frac{1}{\boxdiameter \varepsilon}$ so a set with non-zero volume continues to have non-zero volume.

    Toward Property \ref{item:approximate-cube-partition:semialgebraic}, note that scaling preserves semialgebraicity (by scaling the variables in the polynomial constraints by the appropriate constant).

    Toward Property \ref{item:approximate-cube-partition:approx-volume}, note that all volumes scale by $\frac{1}{(\boxdiameter \varepsilon)^{N}}$ so that the ratios of the volumes are preserved.

    Toward Property \ref{item:approximate-cube-partition:radius}, note that distances scale with $\frac{1}{\boxdiameter \varepsilon}$.

    Toward Property \ref{item:approximate-cube-partition:surface-area}, note that surfaces areas scale by $\frac{1}{(\boxdiameter \varepsilon)^{N - 1}}$ while volumes scale by $\frac{1}{(\boxdiameter \varepsilon)^{N}}$ so that surface area to volume ratios scale by $\boxdiameter \varepsilon$, as desired.
\end{proof}

\paragraph{Implement Membership Oracle}
We have shown that if there exists an efficient replicable algorithm with low sample complexity, then there exists a partition of the cube $\left[0, 1\right]^{N}$ with low surface area.
Furthermore, this algorithm (instantiated with the appropriate random string $r$) gives a membership oracle for points in the partition.
Thus, whenever $\innerAlg$ is efficient, so is the membership oracle $\innerAlg(; r)$.
However, it is not necessarily clear how to find such a good random string $r$.
We give such a procedure below.

\generateP*
    
\begin{proof}
    We initialize a $\frac{\rho}{9}$-replicable algorithm $\innerAlg$ such that
    \begin{equation*}
        \Pr_{p, S_p}(\norm{\innerAlg(S_p; r) - p}{2} \geq \varepsilon) \leq \frac{\delta^2}{\log(3/\delta)},
    \end{equation*}
    when $p \sim \cube$ is drawn from the adversarial distribution above.
    Applying a similar application of Markov's inequality as \Cref{lemma:non-adapt-learning-n-coin-good}, if $r$ is selected uniformly at random then with probability at least $\frac{8}{9}$,
    \begin{equation}
        \label{eq:prb-random-string-not-replicable}
        \Pr_{p, S_p, S_p'} \left( \innerAlg(S_p; r) \neq \innerAlg(S_p'; r) \right) \leq \rho.
    \end{equation}
    We design the following procedure to find a good string $r$ with probability at least $1 - \delta$:
    \begin{enumerate}
        \item Let $r_1, r_2, \dotsc, r_T$ be $T = \log \frac{3}{\delta}$ random strings chosen uniformly at random.
        \item Let $\phi_r(p, S_p, S_p')$ be the indicator for $\innerAlg(S_p; r) = \innerAlg(S_p'; r)$.
        \item For each $1 \leq t \leq T$, estimate $\E{\phi_{r_t}(p, S_p, S_p')}$ using $\bigO{\frac{\log(T/\delta)}{\rho^2}}$ samples drawn from distribution $\distribution$, where $\distribution$ generates tuples $(p, S_p, S_p')$ such that $p \sim \cube$ is drawn uniformly at random and $S_p, S_p'$ are independent samples drawn from the distribution parameterized by $p$.
        \item Return an arbitrary $r_t$ where $\hat{\phi_t} \geq 1 - 2 \rho$.
    \end{enumerate}
    We condition on the following events:
    \begin{enumerate}
        \item There exists some $r_t$ satisfying \Cref{eq:prb-random-string-not-replicable}.
        \item $|\hat{\phi}_t - \E{\phi_{r_t}}| < \rho$ for all $t \in [T]$.
        \item For all $t \in [T]$
        \begin{equation*}
            \Pr_{p, S_p}(\norm{\innerAlg(S_p; r_t) - p}{2} \geq \varepsilon) \leq 3 \delta.
        \end{equation*}
    \end{enumerate}
    By our choice of $T \geq \log \frac{3}{\delta}$, the first fails with probability at most $\frac{\delta}{3}$.
    By the union bound and standard concentration inequalities, the second fails with probability at most $\frac{\delta}{3}$.
    By Markov's inequality, each $r_t$ fails to satisfy the given equation with probability at most $\frac{\delta}{3 \log(3/\delta)}$.
    By the union bound, any single one fails with probability at most $\frac{\delta}{3}$.
    Therefore, we can condition on all of the events holding with probability at least $1 - \delta$.
    
    Under this condition, there exists some $r_t$ for which $\E{\phi_{r_t}} \geq 1 - \rho$, so that this $r_t$ has $\hat{\phi}_t \geq 1 - 2 \rho$, thus the procedure returns some $\hat{\phi}_t$.
    Furthermore, any random string returned must satisfy the global replicability condition of goodness.
    Finally, since every string satisfies the global correctness condition, the returned string is good.

    We now analyze the time complexity of the procedure.
    Each estimation of $\phi_{r}(p, S_p, S_p')$ requires time
    \begin{equation*}
        \bigO{\frac{1}{\rho^2} \log \frac{T}{\delta} f \left( N, \varepsilon, \frac{\rho}{9}, \frac{\delta^2}{\log(3/\delta)} \right)} = \bigO{\frac{\log 1/\delta}{\rho^2} f \left( N, \varepsilon, \frac{\rho}{9}, \frac{\delta^2}{\log(3/\delta)} \right)}.
    \end{equation*}
    Over $T$ iterations, this results in time
    \begin{equation*}
        \bigO{\frac{\log^2 1/\delta}{\rho^2} f \left( N, \varepsilon, \frac{\rho}{9}, \frac{\delta^2}{\log(3/\delta)} \right)}.
    \end{equation*}
\end{proof}

\paragraph{Partitions from Gaussians and Lipschitz Distributions}

The proof of \Cref{thm:replicable-implies-isoperimetry} does not make heavy use of the fact that the algorithm estimates the mean of a Bernoulli product distribution.
In fact, the only argument that specifically applies to the Bernoulli product distribution is \Cref{lemma:replicable-lipschitz}.
Otherwise, we only required that on the cube $\cube = \left[ \frac{1}{2} - \boxwidth \varepsilon, \frac{1}{2} + \boxwidth \varepsilon \right]^{N}$ the algorithm is correct and replicable with at least constant probability on more than a $(1 - O(\rho))$-fraction of input distributions. Our reduction goes through for any family of distributions parameterized by vectors in $\R^{N}$ that satisfy the Lipschitz property (\Cref{lemma:replicable-lipschitz}).

\begin{definition}
    \label{def:lipschitz-distribution-families}
    Let 
    $\eps > 0$, and $\cube \subset \R^{N}$ be a hypercube with side-length at least $10 \varepsilon$.
    Let $\family$ be a family of distributions $\set{\distribution_{p}}$ parameterized by $p \in \cube$, where $p$ is the mean of $\distribution_{p}$.
    
    We say $\family$ is \emph{$(R, \cube)$-lipschitz} if for any pair $p, q \in \cube$ with $\norm{p - q}{2} \leq R = R(m)$,
    \begin{equation*}
        |\Pr_{S_p}(\innerAlg(S_p) = y) - \Pr_{S_q}(\innerAlg(S_q) = y)| \leq \frac{1}{15},
    \end{equation*}
    where $\innerAlg$ is an arbitrary algorithm with $m$ samples and range $\range$ and $y \in \range$.
\end{definition}

In particular, following the argument of \Cref{thm:replicable-implies-isoperimetry}, any algorithm with sample complexity $m$ that replicably estimates the mean of a $(R, \cube)$-lipschitz family of distributions up to error $\varepsilon$ induces a $(\rho, \varepsilon, O(\rho R))$-partition of $\cube$.
As a useful example, we consider Gaussians $\gaussian(p, I)$ with mean $p$ and covariance matrix $I$.
We can prove the following analogue of \Cref{lemma:one-coin-mutual-info-bound}.

\begin{lemma}[Gaussian Mutual Information Bound]
    \label{lemma:gaussian-mutual-info-bound}
    Let $m \geq 0$ be an integer and $a < b$.
    Let $X$ be a uniformly random bit and $Y$ be $m$ \iid samples from $\gaussian(a, 1)$ if $X = 0$ and $m$ \iid from $\gaussian(b, 1)$ if $X = 1$.
    Then,
    \begin{equation*}
        I(X: Y) = \bigO{m (b - a)^2}.
    \end{equation*}
\end{lemma}
This gives an analogue of the Lipschitz condition for Gaussian distributions.
Moreover, unlike the case of Bernoulli product distributions, we require no assumption on $p, q \in \left[ \frac{1}{4}, \frac{3}{4} \right]$ and therefore do not require an assumption on $\varepsilon < \frac{1}{\unscaledboxdiameter}$ to ensure that the cube $\cube$ lies in $\left[ \frac{1}{4}, \frac{3}{4} \right]$.
Thus, \Cref{thm:replicable-implies-isoperimetry} in fact holds for algorithms that estimate the mean of Gaussians, even for large $\varepsilon = \Omega(1)$.
Let $\innerAlg$ be an algorithm that estimates the mean of Gaussians with covariance $I$.

\begin{lemma}
    \label{lemma:gaussian-lipschitz}
    Let $m \in \Z_+$, $p, q \in \R^N$, and $S_p, S_q$ be $m$ samples drawn from $\gaussian(p, I), \gaussian(q, I)$ respectively.
    There exists a universal constant $c$ such that for any $p, q$ with $\norm{p - q}{2} \leq \frac{c}{\sqrt{m}}$ and any~$\hat{p}$,
    \begin{equation*}
        |\Pr_{S_p}(\innerAlg(S_p; r) = \hat{p}) - \Pr_{S_q}(\innerAlg(S_q; r) = \hat{p})| < \frac{1}{15}.
    \end{equation*}
\end{lemma}
\begin{proof}
    Since the proof is identical to \Cref{lemma:replicable-lipschitz}, we only show the mutual information upper bound.
    We note that each sample consists of one sample each from each marginal distribution.
    Let $Y_i$ denote the samples observed from the $i$-th marginal distribution and $Y = (Y_1, \dotsc, Y_{N})$ denote the full observed samples.
    Since the $Y_i$ are independent conditioned on $X$, we may apply \Cref{lemma:gaussian-mutual-info-bound},
    \begin{equation*}
        I(X: Y) \leq \sum_{i} I(X:Y_i) = \bigO{\sum_{i} m |q_i - p_i|^2} = \bigO{m \norm{q - p}{2}^2}.
    \end{equation*}
\end{proof}

\begin{proof}[Proof of \Cref{lemma:gaussian-mutual-info-bound}]
    Since $Y = (Y_1, \dotsc, Y_{m})$ consists of \iid samples, the $Y_i$ are independent conditioned on $X$ and therefore
    \begin{equation*}
        I(X: Y) \leq m I(X: Y_1),
    \end{equation*}
    where $Y_1 \sim \gaussian(a, 1)$ if $X = 0$ and $Y_1 \sim \gaussian(b, 1)$ otherwise.
    As in \Cref{lemma:one-coin-mutual-info-bound}, we proceed by the following identity:
    \begin{equation*}
        I(X:Y_1) = H(X) + H(Y_1) - H(X, Y_1).
    \end{equation*}
    Since $X$ is a uniformly random bit, $H(X) = 1$.
    Next,
    \begin{align*}
        H(X, Y_1) &= - \sum_{b} \int_{- \infty}^{\infty} \Pr(X = b, Y_1 = y) \log \frac{1}{\Pr(X = b, Y_1 = y)} dy \\
        &= 1 + \frac{1}{2} \left( H(Y_1|X = 0) + H(Y_1|X = 1) \right).
    \end{align*}
    We compute the entropy of $Y_1$ as
    \begin{align*}
        H(Y_1) &= \int_{\infty}^{\infty} \frac{\Pr(Y_1 = y|X = 0) + \Pr(Y_1 = y| X = 1)}{2} \log \frac{1}{\Pr(Y_1 = y)} dy.
    \end{align*}
    Combining the three terms,
    \begin{align*}
        I(X:Y) &= \frac{1}{2} \left( H(Y|X = 0) + H(Y|X = 1) \right) - H(Y) \\
        &= \sum_{b} \int_{-\infty}^{\infty} \frac{\Pr(Y = y|X = b)}{2} \log \frac{\Pr(Y = y|X = b)}{\Pr(Y = j)} dy \\
        &= \frac{1}{2} \left( D(\distribution_{0} || \distribution_{1/2}) + D(\distribution_{1} || \distribution_{1/2}) \right),
    \end{align*}
    where $\distribution_{0} \sim \gaussian(a, 1), \distribution_{1} \sim \gaussian(b, 1)$ and $\distribution_{1/2}$ is the mixture of $\distribution_{0}, \distribution_{1}$ with weight $\frac{1}{2}$ each.
    Since $\distribution_{1/2} = \frac{\distribution_{0} + \distribution_{1}}{2}$, we apply the convexity of KL divergence so
    \begin{align*}
        I(X:Y_1) \leq \frac{1}{2} \left( D(\distribution_{0} || \distribution_{1}) + D(\distribution_{1} || \distribution_{0}) \right) = \bigO{(b - a)^2},
    \end{align*}
    where we used the KL divergence of $\gaussian(a, 1)$ and $\gaussian(b, 1)$ is $O((b - a)^2)$.
    Thus, $I(X: Y) = O(m(b - a)^2)$.
\end{proof}

\subsubsection{Extending Partitions to Tilings}
\label{sec:extend}

\Cref{thm:replicable-implies-isoperimetry} shows that a replicable algorithm induces an approximate partition of a bounded cube.
We now show one can extend this approximate partition to an \iat{} of $\R^{N}$ at the cost of slightly increasing the surface-area.

\begin{proposition}[From Hypercube Partition to Approximate Tiling]
    \label{prop:extend-partition-translation}
    Let $\innerAlg$ be a membership oracle for a {$(\gamma, 0.1, A)$}-isoperimetric approximate partition of $\left[0, 1\right]^{N}$.
    
    Then there is an $\innerAlg$-efficient membership oracle $\outerAlg$ for a $\left(\gamma, \frac{5 - 5 \gamma}{4 - 5 \gamma} A \right)$-\iat{} of $\R^N$.
\end{proposition}

\begin{proof}
    Let $\{S_v\}$ be the given isoperimetric approximate partition of $\left[0, 1\right]^{N}$. Let $T_{v} = S_{v} \cap \cube$ where $\cube = [0, 1]^{N}$
    We define the following \iat{}.
    For each $w \in \Z^{N}$ and $v$, define
    \begin{equation*}
        S_{w, v} = w + T_{v}
    \end{equation*}
    with the label $w + v$. Note that as a countable union of countable collections, the resulting collection is still countable.
    We verify the relevant properties.
    \begin{enumerate}
        \item (Disjoint) Suppose $\interior(S_{w_1, v_1}) \cap \interior(S_{w_2, v_2})$ is non-empty and contains some point $x$.
        Then, 
        \begin{equation*}
            x - w_i \in \interior(T_{v_i}).
        \end{equation*}
        Since the interiors $w + \cube$ are disjoint for distinct choices of $w$,
        we conclude $w_1 = w_2$ and $x - w_i \in \interior(\cube)$.
        However, since the interiors of $T_{v}$ are disjoint, $v_1 = v_2$.

        \item (Non-Zero Volume) Since each $T_{v}$ has non-zero volume, so does each $S_{w, v}$.

        \item (Piecewise Smooth) Note $\cube$ is semialgebraic so that $T_{v} = S_{v} \cap \cube$ is semialgebraic and therefore has a piecewise smooth boundary.
        Finally, observe that under translation by vector $w$ the boundary remains piecewise smooth.

        \item ($\gamma$-Approximate Volume) $\gamma$-approximate volume follows immediately as each cube $\cube_{w}$ is simply a translation of $\cube$, and the sets $S_{w, v}$ cover a $(1 - \gamma)$-fraction of its volume.

        \item ($\varepsilon$-Radius) Each set $S_{w, v}$ (up to translation) is a subset of $S_{v}$ with the same label, and therefore satisfies the radius constraint.

        \item ($A$-Normalized Surface Area) To show the surface area bound, we again note that $\{\cube_{w}\}$ are simply translations of $\cube$.
        Now, for a fixed $w$, without loss of generality assume $w = \mathbf{0}$.
        Using \Cref{lemma:boundary-w/o-cube} we argue that
        \begin{equation*}
            \boundary T_{v} \setminus (S_{v} \cap \boundary \cube) = \boundary(S_{v} \cap \cube) \setminus (S_{v} \cap \boundary \cube) \subset \boundary S_{v} \cap \cube,
        \end{equation*}
        so that summing over all $v$, using that the collections are countable and disjoint,
        \begin{align*}
            \area \left( \bigcup_{v} \boundary (S_{v} \cap \cube) \right) &\leq \area \left( \bigcup_{v} S_{v} \cap \boundary \cube \right) + \area \left( \bigcup_{v} \boundary S_{v} \cap \cube \right) \\
            &\leq \area(\boundary \cube) + A \cdot \vol(\cube) \\
            &\leq (2N + A) \cdot \vol(\cube).
        \end{align*}
    \end{enumerate}
    Finally, we claim that $2N + A \leq \frac{5 - 5 \gamma}{4 - 5 \gamma} A$.
    In particular, by the isoperimetric inequality (\Cref{lemma:isoperimetric-inequality}),
    \begin{equation*}
        \area (\boundary T_{v}) \geq N \cdot (\vol(T_{v}))^{(N - 1)/N} \cdot \vol(B_1)^{1/N} \geq \boxdiameter N \vol(T_{v}).
    \end{equation*}
    Summing over all the partitions $v$ and combining with \Cref{lemma:boundary-w/o-cube},
    \begin{align*}
        A \cdot \vol(\cube) &\geq \area \left( \bigcup_{v} \boundary S_{v} \cap \cube \right) \\
        &\geq \area \left(\bigcup_{v} \boundary T_{v}\right) - \area \left( \bigcup_{v} S_{v} \cap \boundary \cube \right) \\
        &\geq \boxdiameter N (1 - \gamma) \vol(\cube) - 2 N \vol(\cube) \\
        &= (8 - 10 \gamma) N,
    \end{align*}
    so that $A \geq (8 - 10 \gamma) N$, and therefore
    \begin{equation*}
        2 N + A \leq \left( 1 + \frac{1}{4 - 5 \gamma} \right) A = \frac{5 - 5 \gamma}{4 - 5 \gamma} A. 
    \end{equation*}

    To conclude, we describe the membership oracle.
    Given any $x \in \R^{N}$, let $w \in \Z^{N}, v \in [0, 1)^{N}$ be given by
    \begin{equation*}
        x_i = w_i + v_i,
    \end{equation*}
    so that $\outerAlg$ outputs $w + \innerAlg(x')$.
    Note that this is exactly a membership oracle for the \iat{} consisting of $S_{w, v}$ described above.
    Clearly, $\outerAlg$ is efficient whenever $\innerAlg$ is.
\end{proof}

We have constructed an \iat{} from our cube partition with a constant factor loss in surface area.
Altogether, \Cref{thm:replicable-implies-isoperimetry} and \Cref{prop:extend-partition-translation} gives the construction of an \iat{} from a replicable algorithm.


\begin{proof}[Proof of \Cref{thm:replicable-non-uniform-tiling}]
    From \Cref{thm:replicable-implies-isoperimetry} there is an algorithm that with probability at least $1 - \delta$ outputs an $\innerAlg$-efficient membership oracle $\innerAlg$ for a $\left( \rho, \frac{1}{\boxdiameter}, \bigO{\rho \varepsilon \sqrt{m}} \right)$-partition $\partition$ of $[0, 1]^{N}$.
    We condition on the success of this event.
    Using \Cref{prop:extend-partition-translation}, we obtain an $\innerAlg$-efficient membership oracle for a $\left(\rho, \bigO{\frac{5 - 5 \rho}{4 - 5 \rho} \rho \varepsilon \sqrt{m}}\right)$-\iat.
    We conclude the proof by noting that $\rho < \frac{1}{2}$ by assumption.
\end{proof}

\paragraph{Tiling from Algorithms with Coordinate Samples} Finally, we note that \Cref{thm:replicable-non-uniform-tiling-formal} (tilings from replicable eman estimation) also holds in the \textit{non-adaptive coordinate sample model}.
Let $\innerAlg$ be an algorithm with coordinate sample access, drawing $m_i$ samples from the $i$-th coin so that $m = \sum_{i} m_i$ is the total sample complexity.
Note that at most $\frac{N}{2}$ coins can have sample complexity $m_i > \frac{2m}{N}$ so we may simply restrict our adversary to give a fixed bias (say $p_i = 0$) to coins with $m_i > \frac{2m}{N}$ and construct a hard instance on the remaining at least $\frac{N}{2}$ coins. 
{This shows that we can assume without loss of generality that the algorithm takes $\Theta(m/N)$ vector samples from an at least $N/2$-dimension Bernoulli distribution.
Applying \Cref{lemma:replicable-lipschitz} then gives that}
there is a constant $c$ such that for any $\norm{p - q}{2} \leq c \sqrt{\frac{N}{m}}$ and outcome $\hat{p}$,
\begin{equation*}
    |\Pr_{S_p}(\innerAlg(S_p; r) = \hat{p}) - \Pr_{S_q}(\innerAlg(S_q; r) = \hat{p})| < \frac{1}{15}.
\end{equation*}

Then, our discussion above in fact yields a $(\rho, 0.1, \rho \varepsilon \sqrt{m/N})$-partition of the $N/2$ dimensional hypercube $\left[ \frac{1}{2} - \boxwidth \varepsilon, \frac{1}{2} + \boxwidth \varepsilon\right]^{N/2}$.
We show how to extend this to a partition of the $N$-dimensional hypercube, losing only a factor of $2$ in each parameter.

In particular, given an approximate partition 
$\partition$ of an $N/2$-dimensional
hypercube  $\left[ \frac{1}{2} - \boxwidth \varepsilon, \frac{1}{2} + \boxwidth \varepsilon\right]^{N/2}$,
we can easily construct a partition $\partition'$ 
of the $N$-dimensional hypercube $\left[ \frac{1}{2} - \boxwidth \varepsilon, \frac{1}{2} + \boxwidth \varepsilon\right]^{N}$ using the direct sum operation:
$\partition' := \set{S \oplus T \mid S,T \in \partition}$.
\begin{restatable}{proposition}{productPartition}
    \label{prop:product-partition}
    Suppose $\partition$ is a $(\rho, \varepsilon, A)$-partition of $\left[ \frac{1}{2} - \boxwidth \varepsilon, \frac{1}{2} + \boxwidth \varepsilon\right]^{N/2}$.
    Then, $\partition' = \set{S \oplus T \mid S, T \in \partition}$ is a $(2 \rho, 2 \varepsilon, 2 A)$-partition of $\left[ \frac{1}{2} - \boxwidth \varepsilon, \frac{1}{2} + \boxwidth \varepsilon\right]^{N}$.
\end{restatable}
The proof is deferred to \Cref{sub-sec:vector-to-non-adaptive}.

\subsection{Implications of the Replicability/Isoperimetry Equivalence}


We now show the implications of our equivalence theorem to replicable mean estimation and the $N$-Coin problem. First, note that for any $\rho > 0$, \cite{kindler2012spherical} give a $(\rho, \Theta(N))$-\iat{} (in fact they provide a $(0, \Theta(N))$-\iat{}).
Thus, applying \Cref{thm:tiling-to-replicable} we obtain a replicable algorithm for estimating the mean of distributions with bounded covariance.

\begin{corollary}[$\ell_2$ Mean Estimation Upper Bound]
    \label{cor:ell-2-mean-estimation}
    Let $\distribution$ be an $N$-dimensional distribution with covariance bounded from above by $I$.
    Then, there exists a $\rho$-replicable algorithm that estimates the mean of $\distribution$ up to error $\varepsilon$ in $\ell_{2}$ distance with success probability $1 - \delta$.
    Moreover, the algorithm uses $m := \bigO{\frac{(N + \log(1/\delta)) N}{\varepsilon^2 \rho^2}}$ vector samples, $1$ deterministic membership query to \cite{kindler2012spherical}'s tiling, and runs in time $\poly{mN}$.
\end{corollary}

Unfortunately, the membership oracle of \cite{kindler2012spherical}'s tiling requires exponential time, so the above algorithm is not computationally efficient. Using the best known lattice construction with an efficient decoding \cite{micciancio2004almost}, we also obtain the best known sample complexity for an \textit{efficient} non-adaptive algorithm.

\begin{corollary}[Efficient $\ell_{2}$ Mean Estimation]
\label{cor:sub-cubic}
    Let $\distribution$ be an $N$-dimensional distribution with covariance bounded from above by $I$. Then, there exists a $\rho$-replicable algorithm that estimates the mean of $\distribution$ up to error $\varepsilon$ in $\ell_{2}$ distance with success probability $1 - \delta$.
    Moreover, the algorithm uses $m := 
    \bigO{\frac{(N + \log(1/\delta)) N^2}{\varepsilon^2 \rho^2}\cdot \frac{\log\log(N)}{\log(N)}}$ vector samples and runs in time $\poly{mN}$.
\end{corollary}
We defer the proof of \Cref{cor:sub-cubic} to the next section in which we discuss lattice-based tilings.

Whenever $N \gg \log(1/\delta)$, the sample complexity upper bound of \Cref{cor:ell-2-mean-estimation}
is tight. 
In particular, applying the isoperimetric inequality leads to the following matching lower bound even for learning the biases of independent coins. 

\begin{theorem}[$\ell_2$ Learning $N$-coin Lower Bound]
    \label{thm:replicable-alg-partition-lb}
    Let $\delta \leq \rho < \frac{1}{16}$ and $\varepsilon < \frac{1}{\boxdiameter}$.
    Any $\rho$-replicable algorithm that succeeds in learning the mean of $N$ independent Bernouli random variables up to $\eps$-accuracy in $\ell_2$ distance with probability $1 - \delta$
    requires $\bigOm{\frac{N^2}{\rho^2 \varepsilon^2}}$ vector samples (or $\bigOm{\frac{N^3}{\rho^2 \varepsilon^2}}$ coordinate samples).
\end{theorem}
\begin{proof}
    From \Cref{thm:replicable-implies-isoperimetry}, we obtain a partition $G_{\hat{p}, \ell^*}$.
    For each $\hat{p}$, observe that $G_{\hat{p}, \ell^*}$ has radius $\frac{1}{\boxdiameter}$ so that $\vol(G_{\hat{p}, \ell^*}) \leq \vol(B_{1/\boxdiameter})$.
    Then, the isoperimetric inequality (\Cref{lemma:isoperimetric-inequality}) implies that,
    \begin{align*}
        \area(\boundary (G_{\hat{p}, \ell^*} \cap \cube)) &\geq N \vol(G_{\hat{p}, \ell^*} \cap \cube)^{(N - 1)/N} \vol(B_{1})^{1/N} \\
        &= N \vol(G_{\hat{p}, \ell^*})^{(N - 1)/N} \left( \boxdiameter^{N} \vol(B_{1/\boxdiameter}) \right)^{1/N} \\
        &\geq \boxdiameter N \vol(G_{\hat{p}, \ell^*} \cap \cube).
    \end{align*}

    The following lemma allows us to relate the perimeters of $\boundary(F_{\hat{p}} \cap \cube)$ and $\boundary F_{\hat{p}} \cap \cube$.
    
    \begin{lemma}
        \label{lemma:boundary-w/o-cube}
        For all $S \subset \R^{N}$, 
        \begin{equation*}
            \boundary(S \cap \cube) \setminus (S \cap \boundary \cube) \subset \boundary S \cap \cube.
        \end{equation*}
    \end{lemma}
    
    \begin{proof}
        Let $x \in \boundary(S \cap \cube) \setminus (S \cap \boundary \cube)$. 
        First, we claim $x \in \cube$ since $\cube$ is closed. 
        Otherwise, there exists a neighborhood around $x$ disjoint from $\cube$ so this same neighborhood must be disjoint from $S \cap \cube$ and therefore $x$ cannot be in $\boundary(S \cap \cube)$.
        By a similar argument, we can argue $x \in \closure(S)$.
        
        Now, suppose $x \not\in \boundary S$ or equivalently $x \in \interior(S) \subset S$. 
        We know $x \not\in S \cap \boundary \cube$ and therefore $x \not\in \boundary \cube$.
        Thus, we may assume $x \in \interior(\cube) \cap \interior(S) = \interior(S \cap \cube)$ or equivalently $x \not\in \boundary(S \cap \cube)$.
        Thus, $x \in \boundary S \cap \cube$, proving the claim.
    \end{proof}

    Now, we apply Lemma \ref{lemma:boundary-w/o-cube} and use disjointness to observe
    \begin{align*}
        \sum_{\hat{p}} \area(\boundary G_{\hat{p}, \ell^*} \cap \cube) &\geq \sum_{\hat{p}} \area(\boundary (G_{\hat{p}, \ell^*} \cap \cube)) - \area( G_{\hat{p}, \ell^*} \cap \boundary \cube) \\
        &\geq \boxdiameter N \sum_{\hat{p}} \vol(G_{\hat{p}, \ell^*} \cap \cube) - \area(\boundary \cube) \\
        &\geq 3 N \vol(\cube) - 2 N \vol(\cube) \\
        &= N \vol(\cube),
    \end{align*}
    where in the third inequality we used Property \ref{item:approximate-cube-partition:approx-volume}, $\delta \leq \rho < \frac{1}{16}$, and that the surface area of a unit hypercube is $2N$ times its volume.
    Finally, since the surface area of the partition is at most $\bigO{\rho \varepsilon \sqrt{m} \vol(\cube)}$, this gives the desired sample complexity lower bound.
    This concludes the proof of \Cref{thm:replicable-alg-partition-lb}.
\end{proof}

Using slightly more involved techniques, we can also remove the assumption $\delta \leq \rho$. The proof is deferred to \Cref{app:replicable-l-2-lb}.
\begin{restatable}{theorem}{lTwoLearnNCoinLB}
    \label{thm:l-2-learning-n-coin-lb}
    Let $\rho, \delta < \frac{1}{16}$, and $\varepsilon < \frac{1}{17}$.
    Any non-adaptive $\rho$-replicable algorithm $\innerAlg$ solving the $\ell_2$ Learning $N$-Coin Problem requires at least $\bigOm{\frac{N^3}{\rho^2 \varepsilon^2}}$ coordinate samples.
\end{restatable}

By plugging in \Cref{lemma:gaussian-lipschitz} instead of \Cref{lemma:replicable-lipschitz} in the proof of \Cref{thm:replicable-implies-isoperimetry}, an identical argument gives a near-tight lower bound for Gaussian mean estimation, resolving \cite[Open Question 4]{DBLP:conf/stoc/BunGHILPSS23}.
\begin{corollary}[$\ell_2$ Gaussian Mean Estimation Lower Bound]
    \label{cor:gaussian-partition-lb}
    Let $\delta \leq \rho < \frac{1}{16}$.
    Any $\rho$-replicable algorithm estimating the mean of an $N$-variate Gaussian with covariance $I$ up to error $\varepsilon$ requires at least $\bigOm{\frac{N^2}{\rho^2 \varepsilon^2}}$ vector samples.
\end{corollary}

\subsubsection{Implications for the \texorpdfstring{$N$}{N}-Coin Problem}

So far we have focused mostly on the case of $\ell_2$-mean estimation and vector samples. We now turn our attention to the testing (equivalently, $\ell_\infty$) counterpart to these questions --- the $N$-coin problem and (non-adaptive) coordinate samples.


First, we show that our techniques immediately yield a matching upper bound.
We combine the $(0, \Theta(N))$-\iat{} of \cite{kindler2012spherical} with \Cref{thm:ell-infty-mean}, and note that any algorithm learning means with $\ell_{\infty}$ error solves the $N$-Coin problem. 

\begin{corollary}[Upper bound on the $N$-coin problem]
    \label{cor:spherical-foam-l-inf-mean}
    There is a $\rho$-replicable algorithm for the $N$-Coin problem with
    {$\bigO{\frac{N}{(q_0 - p_0)^2 \rho^2} \log^3 \frac{N}{\delta}}$ vector samples
    (or $\bigO{\frac{N^2}{(q_0 - p_0)^2 \rho^2} \log^3 \frac{N}{\delta}}$ coordinate samples).}
\end{corollary}


While it is possible to combine our lower bound for $\ell_{2}$-mean estimation (\Cref{lemma:l-2-lb-implies-l-c-lb}) with Holder's inquality (\Cref{lemma:l-inf-tester-implies-learner}) to prove a lower bound on the coin problem, this only works in the regime where $\varepsilon \ll \frac{1}{\sqrt{N}}$.

Instead, we use an $\ell_\infty$-variant of isoperimetric approximate tilings requiring somewhat more careful handling of the surface area analysis to obtain the following near-tight lower bound, resolving \cite[Conjecture D.8]{KVYZ23} up to log factors.
\begin{restatable}[Lower bound on the $N$-coin problem]{theorem}{NCoinLBLogN}
    \label{thm:n-coin-lower-const-delta}
    Let $p_0 < q_0$ with $p_0, q_0 \in \left( \frac{1}{4}, \frac{3}{4} \right)$.
    Let $\rho < \frac{1}{10}$ and $\delta < \frac{1}{3}$.
    Any non-adaptive $\rho$-replicable algorithm $\innerAlg$ solving the $N$-Coin Problem requires at least $\bigOm{\frac{N^2}{\rho^2 (q_0 - p_0)^2 \log^3 N}}$ coordinate samples (or $\bigOm{\frac{N}{\rho^2 (q_0 - p_0)^2 \log^3 N}}$ vector samples).
\end{restatable}
We defer the formal proof to \Cref{app:replicable-l-inf-lb-reflection}, but recall the main idea here. In particular, recall in the $\ell_{2}$-learning lower bound, given a partition $G_{\hat{p}}$ where each set has small constant radius, we can apply the isoperimetric inequality, and even after removing the cube boundary $\boundary \cube$, we still have that the surface are to volume $A \geq N$.
However, in the $N$-Coin problem, the partitions $F_{\hat{o}}$ can have radius $\sqrt{N}$ so that $F_{\hat{o}} \cap \cube$ only has surface area to volume ratio $\sqrt{N}$, and we can no longer afford to remove the entire cube boundary $\boundary \cube$.
For each fixed outcome $\hat{o}$, there is one corner of the hypercube $\cube$ that agrees with $\hat{o}$, that is, the unique corner for which $\hat{o}$ is the correct output.
There are $N$ faces incident to this corner, and $N$ opposite faces.
Using correctness, we argue that $F_{\hat{o}}$ contains at most a $\delta$-fraction of the opposite faces.
Then, we use a reflection trick (see \Cref{sub-sec:n-coin-lb-tech}) to remove the contribution of the cube boundary on the incident faces.
\paragraph{Discussion on Adaptivity}
While \Cref{thm:n-coin-lower-const-delta} applies to any algorithm for the $N$-Coin problem with deterministic sample complexity (or non-adaptive sample access), it does not fully resolve the sample complexity of the $N$-Coin problem.
In particular, an algorithm could have adaptive sample access.
We study this model in depth in \Cref{sec:adaptive-replicability}, and show that given adaptive sample access, there is a computationally efficient algorithm for the $N$-Coin problem with sample complexity matching the best known algorithm of \cite{KVYZ23}.

\section{Lattices, Pre-Processing, and the Closest Vector Problem}
\label{sec:cvpp}
In the prior section we saw asymptotically beating the union bound in mean estimation is \textit{equivalent} to constructing approximate tilings of $\R^N$ with normalized surface area better than $\Theta(N^{3/2})$ (that is beating the trivial cube tiling), and moreover that this equivalence is \textit{algorithmic} in the sense that computational efficiency of the replicable algorithm corresponds to an efficient `decoding scheme' for the tiling. While good (indeed even isoperimetric) tilings of space exist \cite{rogers50,micciancio2004almost,kindler2012spherical}, all known constructions are random and very unlikely to have efficient (or even sub-exponential time) decoders. In this section, we discuss the first of three methods circumventing this issue to build sample and computationally efficient replicable procedures: \textit{pre-processing}. 

Consider the following motivating scenario. Over the course of a century, scientists perform millions of statistical tests. Instead of paying an exponential cost every time such a test is run (say to decode \cite{kindler2012spherical}'s foam), we may hope instead to pay this cost \textit{just once} by `pre-computing' an exponential size data structure $\mathcal{T}$ such that, given access to $\mathcal{T}$, replicable testing and estimation can be solved optimally in \textit{polynomial} time. In this section we formalize this intuition by giving a sample-optimal and polynomial time algorithm for replicable mean estimation in the \textit{decision tree model}. In other words, we pre-compute a polynomial depth tree $\mathcal{T}$ based on a `spherical lattice' which, once constructed, can be queried in polynomial time to simulate a rounding oracle. This leads to the following efficient replicable learning procedure.
\begin{theorem}[Efficient Replicability from Pre-Processing]\label{thm:rep-pre}
    Let $n \in \mathbb{N}$. There exists a weakly explicit,\footnote{That is to say $\mathcal{T}$ is computable in time polynomial in its size.} degree $N^{O(N)}$ and depth $O(N^2\log(N))$ decision tree $\mathcal{T}$ such that, given query access to $\mathcal{T}$, bounded covariance $(\varepsilon,\delta,\rho)_p$-replicable mean estimation can be solved in $O\left(\frac{N^{1+\frac{2}{p}}}{\rho^2\varepsilon^2}\right)$ vector samples and $\text{poly}(N)$ time.
\end{theorem}
We remark that the decision tree $\mathcal{T}$, while not strongly explicit (meaning neighboring nodes could be computed, and therefore traversed, in polynomial time), can at least be computed locally in $N^{O(N)}$ time. This means if one wishes to avoid pre-processing this procedure nearly recovers the singly-exponential strategy of rounding via \cite{kindler2012spherical} or using correlated sampling as in \cite{KVYZ23}. We note as well that for the sake of simplicity we have only stated the result for mean estimation, but it of course holds for multihypothesis testing as well.
\subsection{Spherical Lattices}
The data structure underlying \Cref{thm:rep-pre} is based on tilings stemming from \textit{lattices}, discrete additive subgroups $\Lambda \subset \R^N$. Lattices are typically described by a \textit{basis} $\{b_1,\ldots,b_n\}$, and we write 
\[
\Lambda(B) \coloneqq \left\{\sum\limits_{i=1}^n a_ib_i~\Bigg|~ a_i \in \mathbb{Z}\right\}.
\]

Lattices give rise to a natural tiling of space corresponding to the set of closest vectors to each lattice point, called the \textit{Voronoi cells}.
\begin{definition}[Voronoi Cell]
Let $\mathcal L$ be a lattice and $u \in \mathcal L$ be a lattice point. 
The Voronoi cell of $u$ is the set of points in $\R^N$ such that $u$ is their closest lattice point:
$$
V(u) = \lp \{ 
x \in \R^N : \snorm{2}{x - u} \leq \snorm{2}{x - w} \,\text{ for all }\, w \in \mathcal L \setminus \set{u}
\rp\}.
$$
\end{definition}
Since lattices are periodic, the Voronoi cells of an $N$-dimenional lattice give a tiling of $\R^N$ by identical polytopes. Ideally, we'd like these polytopes to be isoperimetric. In the case of lattices, this actually follows from a simple condition: the \emph{packing} and \emph{covering} radius of the lattice, i.e.\
\[
\lambda(\Lambda) \coloneqq \frac{1}{2}\inf_{x \in \Lambda \setminus \{0\}}\{ \norm{x}{} \}, \quad \quad \text{and} \quad \quad \mu(\Lambda) \coloneqq \inf\{\mu: \bigcup_{x \in \Lambda}B_\mu(x) = \R^N\}
\]
respectively, should be within a constant factor. We call such lattices \textit{spherical}.
\begin{definition}[Spherical Lattice]
    A lattice $\Lambda \subset \R^N$ is called $\alpha$-spherical if $\frac{\mu}{\lambda} \leq \alpha$.
\end{definition}
Spherical lattices, which are closely related to near-optimal packings, have seen substantial study in the both combinatorics and cryptography. Random constructions of spherical lattices have been known since the 50's \cite{rogers50,butler1972simultaneous}. These are not sufficient for our purposes since we need $\mathcal{T}$ to be weakly explicit (and efficiently accessible), but a weak de-randomization of these results was later shown by Micciancio \cite{micciancio2004almost}.
\begin{lemma}[Spherical Lattice {\cite[Theorem 2]{micciancio2004almost}}]\label{lem:sphere-lattice}
For every $N \in \mathbb{N}$, there exists an $N^{O(N)}$ time algorithm $\mathcal{A}$ which generates a full-rank basis $B$ of $\R^N$ such that
\begin{enumerate}
    \item \textbf{Isoperimetry:} $\Lambda(B)$ is $3$-spherical
    \item \textbf{Bit Complexity:} $B$ has $\text{poly}(N)$ bit complexity.
\end{enumerate}
\end{lemma}
It is well known that spherical lattices lead to isoperimetric tilings. This follows from a simple lemma used in the design of LDPC codes \cite{gallagher2003low}, stated more explicitly in this context as \cite[Lemma 3]{naor2023integer}.
\begin{lemma}
\label{lem:convex-surface-to-volume}
Fix $N \in \N$ and $r > 0$. Suppose $K \subset \R^N$ is a convex body such that $K \supseteq B_\lambda$. Then
$$
\frac{ \area{ \partial K} }{ \vol{K} }
\leq \frac{N}{\lambda}.
$$
\end{lemma}

Combined with Micciancio's spherical lattices, this implies the following (weakly) explicit isoperimetric tiling by its Voronoi cells.
\begin{corollary}[Isoperimetric Tilings from Lattices]
    Fix $N \in \mathbb{N}$. There exists a lattice $\Lambda_N$, constructable in $N^{O(N)}$-time, whose Voronoi cells give a tiling of $\R^N$ by semi-algebraic convex sets $\{V_{w}\}_{w \in \Lambda_N}$ such that for every $V_w$ and $v \in V_w$:
    \begin{enumerate}
        \item (Radius)
        \[
        \norm{v-w}{2} \leq 1.
        \]
        \item (Surface-to-Volume) 
        \[
        \area(\partial V_w) \leq N \vol(V_w).
        \]
        In particular, for any hypercube $\mathcal C$ of side length $C N$ for some sufficiently large constant $C$, it holds that
        $$
        \area \lp( \bigcup_{ w \in \Lambda_N} \partial V_w \cup  \mathcal C \rp) 
        \leq O(N) (C N)^N.
        $$
    \end{enumerate}
\end{corollary}
\begin{proof}
    Without loss of generality, we may scale the lattice such that its packing radius is $1$. Since $\Lambda$ is $3$-Spherical, its covering radius is at most $3$. Consider the partition given by the Voronoi cells, labeled by their respective lattice vector. This is a true tiling whose sets are disjoint, have non-zero volume, and are semi-algebraic (namely their boundaries are an intersection of hyperplanes). 
    The tiling radius is at most $3$ from the covering guarantee, and by \Cref{lem:convex-surface-to-volume} every cell $V$ satisfies $\area(\partial V) \leq N \vol(V)$ as desired.

    Lastly, let $\mathcal C$ be a hypercube of side length $ C N$ for some sufficiently large constant $C$, and $\mathcal V$ be the set of Voronoi cells that have non-trivial intersection with the cube $\mathcal C$.
    Since each Voronoi cell has covering radius $3$, one can see that $\bigcup_{V \in \mathcal V} V$ is included in a larger hypercube of side length $C N + 12$. 
    Thus, it follows that the total surface area of the union of the boundaries of these cells is at most
    $
    N \; (C N + 12)^N
    \leq O(N) (C N)^N
    $
    as long as $C$ is sufficiently large.
\end{proof}
The proof of \Cref{cor:sub-cubic} follows exactly the same argument.
\begin{proof}[Proof of \Cref{cor:sub-cubic}]
    For any $N \in \mathbb{N}$, \cite{micciancio2004almost} gives a full rank $N$-dimensional lattice $\Lambda$, constructed in polynomial time, such that $\frac{\mu}{\lambda} \leq O(\sqrt{\frac{N\log\log(N)}{\log(N)}})$. Normalizing to have covering radius $1$, $\Lambda$ has packing radius $\Omega(\sqrt{\frac{\log(N)}{N\log\log(N)}})$, so by the same argument as above results in the desired $(0,\sqrt{\frac{N\log\log(N)}{\log(N)}})$-\iat, and the desired sample complexity follows from \Cref{thm:tiling-to-replicable}.
\end{proof}
\subsection{The Closest Vector Problem}
Given a lattice $\Lambda$, the closest vector problem (CVP) on $\Lambda$ is the `decoding' problem for the Voronoi tiling. In other words, given $\mathcal{L} \subset \R^N$ (specified by basis $B$) and a target vector $t \in \R^N$, find $x \in \mathcal{L}$ minimizing $\norm{t-x}{2}$. Unsurprisingly, CVP is in general a challenging problem. It is NP-hard \cite{van1981another} (even to approximate \cite{dinur1998approximating}), and its many variants form the core of lattice-based cryptography \cite{regev2009lattices}. Allowing pre-processing in CVP (called CVPP) is a popular relaxation of the problem in lattice cryptography since, akin to our setting, the same lattice may be re-used many times (e.g. as a public key in various schemes). In this section we give a \textit{polynomial time} algorithm for CVPP.


More formally, in this section we solve CVPP in the \textit{decision tree model}, which pre-processes the lattice $\mathcal{L}$ by constructing a polynomial depth tree whose nodes correspond to efficiently computable queries on the target and whose leaves give the corresponding closest vector.
\begin{theorem}[CVPP]\label{thm:CVPP}
        Let $\ell, N \in \mathbb{N}$ and $\mathcal{L} \subset \R^N$ be a lattice given by a basis $B$ with bit complexity $\ell$. There is a depth $O(N^2(\ell+\log(N))$ decision tree $\mathcal{T}$ satisfying
        \begin{enumerate}
            \item \textbf{Pre-processing}: $\mathcal{T}$ can be constructed in $N^{O(N^3(\ell+\log(N)))}$ time and space
            \item \textbf{Run-time}: Given $\mathcal{T}$, there is an algorithm solving CVP for all $t \in \R^N$ in $\text{poly}(N,\ell,\ell_t)$ time\footnote{Formally this assumes access to $\mathcal{T}$ in a standard `pointer machine' model where given the memory block corresponding to any node in the tree, the algorithm may access the memory storing the nodes' children in $O(1)$ time.}
        \end{enumerate}
        where $\ell_t$ is the bit complexity of the target.
\end{theorem}
\Cref{thm:CVPP} is in contrast to typical results on CVPP which usually focus on giving algorithms with both exponential pre-processing and run-time, but with improved constants (see e.g.\ \cite{laarhoven2016sieving}). In fact, \Cref{thm:CVPP} may come as somewhat of a surprise to those familiar with the literature since, like CVP, CVPP is actually still NP-hard \cite{micciancio2001hardness,alekhnovich2005hardness}. The contrast stems from the fact that such hardness results typically assume that while pre-processing is computationally unbounded, its output should be polynomial size. In contrast we output an exponential size data structure that can be \textit{queried} in polynomial time, analogous to NP-hard problems such as subset-sum which can also be solved efficiently via polynomial depth decision trees \cite{meyer1984polynomial} (indeed our techniques draw heavily from this literature). 

For the purposes of replicability, we will really only need to solve CVPP for targets within some bounded radius. We say a decision tree $R$-solves CVPP on $\mathcal{L}$ if for every $t \in \R^N$ with norm $\norm{t}{2} \leq R$, its corresponding path ends at a leaf labeled by the closest vector in $\mathcal{L}$ to $t$. We first give an algorithm for this bounded setting, then prove \Cref{thm:CVPP} as a corollary by reducing to this regime via known algorithms for approximate CVPP.
\begin{proposition}[Bounded CVPP]\label{prop:CVPP}
    Let $R>0$, $\ell, N \in \mathbb{N}$, and $\mathcal{L} \subset \R^N$ be a lattice given by a basis $B$ whose vectors have bit complexity at most $\ell$. There is a depth $O(N^2\log\frac{R+\mu}{\lambda})$ decision tree $\mathcal{T}$ satisfying
        \begin{enumerate}
            \item \textbf{Pre-processing}: $\mathcal{T}$ can be constructed in $N^{O(N^3\log\frac{R+\mu}{\lambda})}$ time and space
            \item \textbf{Run-time}: Given $\mathcal{T}$, CVP is solvable for all $\norm{t}{2} \leq R$ in $\text{poly}(N,\ell,\ell_t,\log R)$ time,
        \end{enumerate}
where we recall $\lambda$ is the length of the shortest non-zero vector in $\mathcal{L}$ and $\mu$ is the covering radius.
\end{proposition}
The proof of \Cref{prop:CVPP} is by reduction to \emph{point location}, a classical problem in computational geometry which asks ``given a hyperplane arrangement $H$ and a target point $t$, find the cell in the arrangement containing $t$.'' Point location has long been known to be efficiently solvable in the decision tree model via (randomized) $\text{poly}(N)\log(|H|)$ expected depth linear decision trees \cite{meyer1984polynomial}, with a long line of work eventually bringing the depth to its information theoretical optimum $O(N\log(|H|))$ \cite{meiser1993point,cardinal2015solving,kane2017active,kane2018generalized,ezra2019nearly,hopkins2020power,hopkins2020point}.

The closest vector problem may equivalently be rephrased as determining the Voronoi cell in which $t$ lies. In the bounded radius setting, this leads to an obvious reduction to point location with respect to the Voronoi faces that intersect the radius $R$ ball. The proof of \Cref{thm:CVPP} follows from building a deterministic decision tree for point location on this family based on a simplified deterministic variant of \cite{hopkins2020point}'s randomized linear decision trees for point location.

Concretely, our construction is based on the two following lemmas. The first is a simple method for learning `large margin' hyperplanes via rounding. A (homogeneous) hyperplane $\langle h, \cdot \rangle =0$ is said to have margin $\gamma$ with respect to $t$ if
\[
m_M(h) \coloneqq |\langle h, t \rangle| \geq \gamma,
\]
or in other words, if $h$ is $\gamma$-far from $t$. Even though the Voronoi hyperplanes are affine, that is of the form $\langle h, \cdot \rangle = b$, we can always embed any affine hyperplane $h$ and target $t$ into one higher dimension via the transform
\[
\langle h, \cdot \rangle = b \to (h,b) \quad \quad \text{and} \quad \quad t \to (t,-1).
\]
This preserves the shifted inner product
\[
\langle h, t \rangle - b = \langle (h,b),(t,-1) \rangle,
\]
where $\sign(\langle h, t \rangle - b) = \sign(\langle (h,b),(t,-1) \rangle)$ determines on which side of the (affine) hyperplane the target $t$ lies. As such, determining the Voronoi cell over the embedded hyperplanes is sufficient to solve CVP, and we will assume homogeneity for the rest of the section.

Finally, for any $\varepsilon>0$, we write $\text{Round}^\varepsilon(t)$ to be the rounding of $t$ to the closest factor of $\varepsilon$ in every component. We can now state our learner for large margin hyperplanes.
\begin{lemma}[Large Margin Inference]\label{lem:margin-learner}
    Given a family of $n$-dimensional homogeneous hyperplanes $H$ with unit normals and a target $t \in \R^N$, there is a deterministic algorithm $\mathcal{A}: \R \times H \to \{\pm 1,\bot\}$ such that given access to $r(t)=\text{Round}^{\frac{1}{4N}}(t)$, $\mathcal{A}$ has the following properties
    \begin{enumerate}
        \item \textbf{Zero-Error}: if $\mathcal{A}(r(t),h) \in \{\pm 1\}$, then $\mathcal{A}(r(t),\cdot) = \text{sgn}(\langle h, t\rangle)$
        \item \textbf{Large Margin Coverage:} $\forall h \in H$ with margin at least $\frac{1}{2\sqrt{N}}$, $\mathcal{A}(h) \in \{\pm 1\}$
        \item \textbf{Computational Efficiency:} $\mathcal{A}(r(t),h)$ can be computed in $O(N)$ time.
    \end{enumerate}
\end{lemma}
\begin{proof}
By definition we can expand the inner product in the standard basis as
\[
\langle h, t \rangle = \sum\limits_{i=1}^n h_it_i.
\]
As such, if we consider $r(t)=\text{Round}^\varepsilon(t)$ instead of $t$, we induce error at most
\begin{equation}\label{eq:margin-error}
|\langle h, t \rangle - \langle h, r(t) \rangle| \leq \varepsilon \norm{h}{1}
\end{equation}
by Holder's inequality. In particular if $|(\langle h, r(t) \rangle)| > \varepsilon \norm{h}{1}$ it must be the case that
\[
\text{sgn}(\langle h, t \rangle) = \text{sgn}(\langle h, r(t) \rangle).
\]
Moreover, any $h$ with true margin $m_M(h)$ at least $2\varepsilon\norm{h}{1}$ the rounded margin meets this threshold:
\[
|(\langle h, r(t) \rangle)| > \varepsilon \norm{h}{1}.
\]
By assumption, we have $\norm{h}{1} \leq \sqrt{N}$. Thus setting $\varepsilon=\frac{1}{4N}$, the algorithm $\mathcal{A}$ simply computes $\langle h, r(t) \rangle$ and outputs the corresponding sign if the absolute value is at least $\frac{1}{4\sqrt{N}}$, and $\bot$ otherwise. By the first observation above this procedure never errs, and by the second observation it always labels any hyperplane of margin at least $\frac{1}{2\sqrt{N}}$.
\end{proof}
We remark that the above argument actually works over any subspace $V \subset \R^N$ where the standard basis is replaced with some orthonormal basis of $V$.

In the setting of CVP, if $t$ is close to the boundary of the cell it is quite likely that there are exponentially many hyperplanes with low margin, and the above method will fail. This leads us to the key idea of Kane, Lovett, and Moran \cite{kane2018generalized}: the relative sign of the target and cell boundaries is \textit{invariant} under linear transformation and scaling. If, given $H$, we could find a generic transform for the remaining small margin hyperplanes that forces a good fraction of them to have large margin in the transformed space, we could again apply the above lemma and repeat until we find $t$. 

This may sound far fetched, since it has to be done with no (or at least minimal) knowledge of the target to ensure the size of the decision tree remains bounded. However, such transforms indeed exist, and can even be found in strongly polynomial time! The method is based on the so-called (approximate) `Forster transform', originally introduced by Barthe \cite{barthe1998reverse}, Forster \cite{forster2002linear}, and Dvir, Saraf, and Wigderson \cite{dvir2014breaking}, and further developed in the setting of hyperplane learning by Hopkins, Kane, Lovett, and Mahajan \cite{hopkins2020point}, Diakonikolas, Kane, and Tzamos \cite{diakonikolas2023strongly}, and Dadush and Ramachandran \cite{dadush2024strongly}. We use the lattermost's method which gives a deterministic, strongly polynomial\footnote{Note here we mean polynomial with respect to $\max\{|H|,n\}$, so this is really exponential time for our setting. We only use this algorithm in pre-processing so this is not a problem.} time algorithm for finding sufficiently strong approximate scalings.\footnote{We remark the below is not quite as stated in \cite{dadush2024strongly}. Namely the authors allow for approximate `right scalings' $r \in \R^N$. However, for small enough approximation error (say $\text{poly}(d^{-1})$) in our setting it is always possible to take the scaling vector as $r_j=\frac{1}{\norm{T h_j}{2}}$ which is implicit in our normalization of $M(h)$ below in Property (2).}
\begin{lemma}[Vector Scaling {\cite[Theorem 1.3]{dadush2024strongly}}]\label{lem:scaling}
    There is a deterministic, strongly polynomial time algorithm which given a family of homogeneous hyperplanes $H$ in $\R^N$, finds a $k$-dimensional subspace $V$\footnote{Formally, their algorithm outputs the set of $h \in H$ lying in this subspace.} and a linear isomorphism $M: V \to V$ satisfying
    \begin{enumerate}
        \item \textbf{Density:} $|V \cap H| \geq \frac{k}{N}|H|$
        \item \textbf{Margin:} For any unit vector $v \in V$, many transformed elements $M(h)$ have large margin:
        \[
        \Pr_{h \in H}\left[\left\langle \frac{M(h)}{\norm{M(h)}{2}},v \right\rangle \geq \frac{1}{2\sqrt{k}}\right] \geq \frac{1}{2k}
        \]
    \end{enumerate}
\end{lemma}
More accurately, Dadush and Ramachandran give an algorithm that either scales $H$ into `approximate unit isotropic position', or finds a subspace $V$ with many vectors (in which case such a transform on $\R^N$ is impossible). It is observed in \cite{kane2018generalized,hopkins2020point} that this implies the above reformulation.

Given these two lemmas, our CVPP algorithm is roughly given by the following two procedures. First is the pre-processing algorithm, which is a modified point location procedure. In particular, given a lattice $\mathcal{L}$ and radius $R$, we assume $\rCVPPP$ is fed the family of hyperplanes $H$ which correspond to the faces of any Voronoi cell intersecting $B_R$. In the formal argument we will just intake a lattice basis and compute the above, but it is convenient to express the algorithm in terms of $H$ to set up its recursive structure.
\IncMargin{1em}
\begin{algorithm}[h!]

\SetKwInOut{Input}{Input}\SetKwInOut{Output}{Output}\SetKwInOut{Parameters}{Parameters}
\Input{Voronoi Relevant Hyperplanes $H$}
\Output{Decision Tree $\mathcal{T}$}



$\mathcal{T} \gets \emptyset$

\While{$H \neq \emptyset$}{
    
    $M,V \gets \text{Scale}(H)$ (\Cref{lem:scaling})

    Add node $(M,V)$ to $\mathcal{T}$

    \For{$v \in \{-1,-1 +\frac{1}{4N},\ldots,1-\frac{1}{4N},1\}$}{

    $m_M \gets \left\{h \in H \cap V: \left|\left\langle \frac{M(h)}{\norm{M(h)}{2}}, v \right\rangle_V\right| \geq \frac{1}{2\sqrt{\text{dim}(V)}}\right\}$

    Add sub-tree \rCVPP$(H \setminus m_M)$ as `$v$'-Child of $\mathcal{T}(M,V)$
}
}
\Return leaf labeled with $w \in \mathcal{L}_H$ matching $\sign\left(\left\langle \frac{M(h)}{\norm{M(h)}{2}}, v \right\rangle_V\right)$ from root-to-leaf path.

\caption{$\rCVPPP(H)$} 
\label{alg:CVPP}

\end{algorithm}
\DecMargin{1em}

\Cref{alg:CVPP} corresponds to building a decision tree whose internal nodes are labeled by linear transformations $M: V \to V$ (and a corresponding basis for $V$) computed via \Cref{lem:scaling}, and whose edges are labeled by $\{-1,-1+\frac{1}{4N},\ldots,1-\frac{1}{4N},1\}^{\text{dim}(V)}$ for some sufficiently small $\varepsilon$ (potential roundings of $\text{Proj}_V(t)$). Given access to this tree, the real-time algorithm traverses it starting at the root by querying for the transformation and basis, then following the edge corresponding to rounding the scaled target:
\IncMargin{1em}
\begin{algorithm}

\SetKwInOut{Input}{Input}\SetKwInOut{Output}{Output}\SetKwInOut{Parameters}{Parameters}
\Input{Decision Tree $\mathcal{T}$, target vector $t$}
\Output{Minimizer $w \in \mathcal{L}$ of $\norm{x-t}{2}$}



Node $\gets \mathcal{T}$(root)

\While{$\mathcal{T}(\text{\emph{Node}})$ is internal}{

    Query $\mathcal{T}(\text{Node})$ for scaling $(M,V)$

    $t_{M,V} \gets (M^{-1})^T\left(\text{Proj}_V(t)\right)$

    $r \gets \text{Round}^{\frac{1}{4N}}\left(\frac{t_{M,V}}{\norm{t_{M,V}}{2}}\right)$

    Node $\gets$ $r$-child of $\mathcal{T}(\text{Node})$
    


}

\Return vector $w \in \mathcal{L}$ labeling leaf $\mathcal{T}(\text{Node})$

\caption{$\rCVPPR(\mathcal{T},t)$} 
\label{alg:CVPPR}

\end{algorithm}
\DecMargin{1em}

We now give the formal analysis.
\begin{proof}[Proof of \Cref{prop:CVPP}]~
\vspace{.2cm}
\\
\textbf{Construction:} We first describe the construction of our decision tree, then prove its correctness and compute the pre-processing and runtime costs. Our first step in is to enumerate the hyperplanes defining the faces of any Voronoi cell non-trivially intersecting the ball $B_R$. We call such hyperplanes `relevant'. The lattice points defining these faces must sit within $B_{R+\mu}$, so to enumerate all relevant hyperplanes it is enough to enumerate all lattice points of norm $O(R+\mu)$ and take their bisectors. The total number of relevant hyperplanes is at most $\left(\frac{R+\mu}{\lambda}\right)^{O(N)}$ by elementary volume arguments, and enumeration can be done by the classical Kannan-Fincke-Pohst method \cite{kannan1983improved,fincke1985improved} in time $(\frac{R+\mu}{\lambda})^{O(N\log(N))}$. 

After completing enumeration, embed the relevant hyperplanes $H$ as homogeneous in one higher dimension as described above. We now recursively construct the described decision tree. We briefly introduce some relevant notation allowing us to freely move between $\R^N$ and subspaces. For a $k$-dimensional subspace $V \subset \R^N$, let $\{v_i\}_{i \in [k]}$ denote an orthonormal basis\footnote{Here we mean orthornormal with respect to the original inner product on $\R^N$.} for $V$, and define its corresponding $k$-dimensional inner product between vectors $v=\sum\limits_{i=1}^k a_i v_i$ and $w = \sum\limits_{i=1}^k b_i v_i$ as
\[
\langle v,w \rangle_V = \sum\limits_{i=1}^k a_ib_i = \langle v, w \rangle
\]
where the last equality views $v$ and $w$ as elements of $\R^N$ and is by orthonormality.

To construct the root node of our tree, \Cref{lem:scaling} states we can find a $k$-dimensional subspace $V$ containing an $\Omega(\frac{k}{N+1})$ fraction of $|H|$ and an invertible linear transform $M:V \to V$ such that for any unit vector $v \in V$:
\[
\Pr_{h \in H \cap V}\left[ \left\langle \frac{M(h)}{\norm{M(h)}{2}},v \right\rangle_V > \frac{1}{2\sqrt{k}} \right] \geq \frac{1}{2k}.
\]
$V$ is given as a list of vectors $h \in H$, that lie in $V$, which we use to generate an orthonormal basis in $\text{poly}(N,|H|)$ time.

We now introduce a child for every $v \in \{-1,-1+\frac{1}{4N},\ldots,1-\frac{1}{4N},1\}^{k}$, corresponding to possible roundings of (scaled and projected) target vectors for CVP in the $\{v_i\}$ basis. The main idea is that after scaling, projecting, and rounding, it is still possible to correctly learn the sign of any `large margin' hyperplanes:
\[
m_{M,v} \coloneqq \left\{h \in H \cap V: \left|\left\langle \frac{M(h)}{\norm{M(h)}{2}},v \right\rangle_V\right| \geq \frac{1}{2\sqrt{k}}\right\}.
\]
In particular, we prove in the correctness section that if $$
v=\text{Round}^{\frac{1}{4N}}\left(\frac{1}{\norm{(M^{-1})^T(\text{Proj}_V(t))}{2}}(M^{-1})^T(\text{Proj}_V(t))\right) \, ,
$$
that is the rounded transformation of the target $t$, then all hyperplanes $h \in m_{M,V}$ satisfy
\[
\text{sgn}\left( \langle h, t \rangle\right) = \text{sgn}\left(\left\langle \frac{M(h)}{\norm{M(h)}{2}},v \right\rangle_V\right)
\]

With this in mind, the child node corresponding to $v$ then falls into one of two cases. Either $H_v=H \setminus m_{M,v}$ is non-empty, in which case we recursively repeat the above process starting from $H_v$, or $H_v = \emptyset$, in which case the child becomes a leaf. In the latter case, we now need to compute the lattice vector corresponding to this leaf $L_v$.

To do this, consider the unique path from the root to $L_v$. Each hyperplane $h \in H$ lies in the set $m_{M,v}$ for exactly one node in the path. Assign $h$ the label
\[
L(h) \coloneqq \text{sgn}\left(\left\langle \frac{M(h)}{\norm{M(h)}{2}},v \right\rangle_V \right)
\]
of its corresponding $(M,v)$ pair. As discussed above, these signs correctly label every hyperplane for any target following this root to leaf path. To find the corresponding lattice vector, we then enumerate through $\mathcal{L}_H$ until we find a lattice vector $w$ such that 
\[
\forall h \in H: \text{sgn}(\langle w, h \rangle) = L(h)
\]
In other words, if all labels are correct, this means that $w$ and any target following the root to leaf paths lie in the same Voronoi cell as desired. Note that each lattice vector lies in a different cell of the arrangement $H$, so there is at most one possible matching vector. It is also possible no such vector is found. This is because we have been somewhat lax and allowed paths that don't correspond to any valid target. Such paths are irrelevant and can either be pruned or simply left unlabeled.

\paragraph{Correctness:} Fix any $\norm{t}{} \leq R$. It is enough to argue that following the path generated by edge
\[
v = \text{Round}^{\frac{1}{4N}}\left(\frac{1}{\norm{(M^{-1})^T(\text{Proj}_V(t))}{2}}(M^{-1})^T(\text{Proj}_V(t))\right)
\]
in every step ends at a leaf labeled by $t$'s closest vector in $v_t \in \mathcal{L}$, where rounding is performed with respect to a fixed orthonormal basis for $V$ as above, not with respect to the standard basis.

By construction of the leaf node labelings, it is sufficient to argue that the labelings $L(H)$ corresponding $t$'s root-to-leaf path match the labelings of $v_t$. To see this, denote the transformed target \textit{before} rounding as
\[
v' = \frac{1}{\norm{(M^{-1})^T(\text{Proj}_V(t))}{2}}(M^{-1})^T(\text{Proj}_V(t))
\]
Observe that in this case the sign of the transformed inner product with any $h \in H \cap V$ indeed matches the sign of the original inner product:
\begin{align*}
    \sign\left( \left\langle \frac{M(h)}{\norm{M(h)}{2}}, v' \right\rangle_V \right) &= \sign\left( \left\langle \frac{M(h)}{\norm{M(h)}{2}}, \frac{(M^{-1})^T(\text{Proj}_V(t))}{\norm{(M^{-1})^T(\text{Proj}_V(t))}{2}} \right\rangle_V \right)\\
    &= \sign\left( \left\langle \frac{h}{\norm{M(h)}{2}}, \frac{\text{Proj}_V(t)}{\norm{(M^{-1})^T(\text{Proj}_V(t))}{2}} \right\rangle_V \right) \\
    &= \sign\left( \left\langle h, \text{Proj}_V(t) \right\rangle_V \right) \\
    &= \sign(\langle h, t \rangle)
\end{align*}
On the other hand, \Cref{lem:margin-learner} shows that for any $h \in m_{T,v}$ we have
\[
\sign\left( \left\langle \frac{M(h)}{\norm{M(h)}{}}, v' \right\rangle \right) = \sign\left( \left\langle \frac{M(h)}{\norm{M(h)}{}}, \text{Round}^{\frac{1}{4N}}(v') \right\rangle \right) = \sign\left( \left\langle \frac{M(h)}{\norm{M(h)}{}}, v \right\rangle \right)
\]
Combining these facts, we get that the root-to-leaf labeling $L(H)$ for $t$ exactly records its cell in the Voronoi diagram. Since by construction the leaf of $t$ is labeled by the lattice vector matching these signs (i.e.\ the corresponding vector of the cell) we are done.

\paragraph{Pre-Processing, Depth, and Runtime}

By \Cref{lem:scaling}, each level of the decision tree removes at least a $\Omega(\frac{1}{N})$-fraction of the remaining hyperplanes. Thus after $O(N\log(|H|)) \leq O(N^2\log(\frac{R+\mu}{\lambda}))$ rounds, every branch of the construction halts giving the desired depth bound. Each node has at most $N^{O(N)}$ children, so the size bound follows. Computationally, the pre-processing is dominated by the time of constructing the Forster transform $T$ for every node, which takes \[
\text{poly}(\ell) \left(\frac{R+\mu}{\lambda}\right)^{O(N)}N^{O(N^3\log(\frac{R+\mu}{\lambda}))} \leq N^{O(N^3\log(\frac{R+\mu}{\lambda}))}
\]
time altogether \cite{dadush2024strongly}. The orthonormal basis of each $V$ can also be computed in time polynomial in the bit complexity and number of relevant hyperplanes $|H|=\left(\frac{R+\mu}{\lambda}\right)^{O(N)}$. For each of the $N^{O(N)}$ possible roundings, computing the set $m_{T,v}$ simply amounts to computing at most $|H|$ inner products.

Given access to the constructed tree, recall the real-time algorithm $\innerAlg$ queries the tree in each round for the desired transform $M: V \to V$ and an othornormal basis $\{v_i\}$ for $V$ of polynomial bit complexity. $\innerAlg$ then computes the corresponding rounding $\text{Round}^{\frac{1}{4N}}\left(\frac{1}{\norm{(M^{-1})^T(\text{Proj}_V(t))}{2}}(M^{-1})^T(\text{Proj}_V(t))\right)$ in $\text{poly}(N,\ell,\ell_t)$ time. $\innerAlg$ then traverses the edge corresponding to this rounding and repeats until it hits a leaf. Each node has $N^{O(N)}$ children so traversal can be implemented in $O(N \log N)$-time in a standard pointer model, and the number of steps is at most the depth of the tree. This gives a runtime in terms of the covering and packing radius. To relate to bit complexity, we appeal to the following elementary lemma
\begin{lemma}[{\cite[Lemma 2.1]{dadush2014closest}}]
    Let $\mathcal{L} \subset \Q^N$ be a lattice and $B$ a basis represented by $\ell$ bits. Then $\mu \leq 2^{O(\ell)}$ and $\lambda \geq 2^{-O(\ell)}$.
\end{lemma}
Thus $\log \frac{R+\mu}{\lambda} \leq O(\log(R)+\ell)$, giving the desired result.
\end{proof}
The proof of the main theorem is now essentially immediate from combining the above with the following efficient algorithm for approximate CVPP, which we use to find a lattice vector to shift the target into a bounded ball.
\begin{theorem}[Approximate CVPP {\cite{dadush2014closest}}]\label{thm:apx-CVPP}
    Let $\ell,N \in \mathbb{N}$ and $\mathcal{L} \subset \R^N$ be a lattice given by a basis $B$ with bit complexity $\ell$. There is an algorithm which after a $\text{poly}(\ell)2^{O(N)}$-time pre-processing procedure on $B$ outputs an advice string of size $\text{poly}(N,\ell)$, given any $t \in \R^N$ computes $y \in \mathcal{L}$ satisfying
    \[
    \norm{y - t}{2} \leq O\left(\frac{N}{\log(N)}\min_{x \in \mathcal{L}}\{\norm{x-t}{2}\}\right)
    \]
    in $\text{poly}(N,\ell,\ell_t)$ time.
\end{theorem}
We are finally ready to prove \Cref{thm:CVPP}.
\begin{proof}[Proof of \Cref{thm:CVPP}]
    Our pre-processing of the basis $B$ will simply consist of the decision tree from \Cref{prop:CVPP} with radius $R=N^22^{O(\ell)}$ and the advice string of \Cref{thm:apx-CVPP}.

    Given this structure, to compute the closest vector to a target $t \in \R^N$, we first run approximate-CVPP to find a lattice vector $y$ such that
    \[
    \norm{t-y}{2} \leq O\left(\frac{N}{\log(N)}\mu\right) \leq n2^{O(\ell)}
    \]
    where we've again used the fact that $\mu \leq 2^{\ell}$. Since $\norm{t-y}{2} \leq R$, we can now run our CVPP algorithm for bounded targets to find the closest $x \in \mathcal{L}$ to $t-y$. It is an elementary observation that $x+y$ is then the closest vector to $t$. Correctness is therefore immediate from the correctness of \Cref{prop:CVPP} and \Cref{thm:apx-CVPP}.

    The pre-processing cost of this algorithm is dominated by the construction of the decision tree for $B$, which takes $2^{O(N^3\log \frac{R+\mu}{\lambda})}\leq 2^{O(N^3(\ell+\log(N))}$ time and space. The decision tree is of depth $O(N^2(\ell+\log(N)))$, so the running time of given the target $t$ is $\text{poly}(N,\ell,\ell_t)$. Approximate CVPP also runs in time $\text{poly}(N,\ell,\ell_t)$, which gives the final desired runtime.
\end{proof}

For completeness, we now give the proof of \Cref{thm:rep-pre}.
\begin{proof}[Proof of \Cref{thm:rep-pre}]
    Recall \Cref{thm:ell-infty-mean} gives an efficient, sample-optimal algorithm for mean estimation assuming access to an rounding oracle of an (approximate) isoperimetric tiling. \Cref{prop:CVPP} shows that given access to the tree $\mathcal{T}$ correpsonding to a $O(1)$-spherical lattice, this oracle can be simulated in $\text{poly}(N)$ time. The result is then immediate from the existence of spherical lattices (\Cref{lem:sphere-lattice}).
\end{proof}

\section{Adaptivity and Replicability}
\label{sec:adaptive-replicability}

In this section, we discuss \textit{adaptivity} as a technique for building sample and computationally efficient algorithms for the $N$-Coin problem.
The main result of the section gives an adaptive algorithm matching the worst-case sample complexity of the best known non-adaptive algorithms for $N$-Coins.
Furthermore, we obtain an algorithm with improved expected sample complexity and time complexity that is linear in sample complexity, thus obtaining a truly efficient algorithm (without requiring preprocessing). 

\begin{theorem}[\Cref{thm:r-n-coin-problem}, formal]
    \label{thm:r-n-coin-problem-formal}
    Let $4 \delta < \rho < \frac{1}{2}$.
    There is a $\rho$-replicable algorithm solving the $N$-coin problem with expected coordinate sample complexity (and running time) $\bigO{\frac{N^2 q_0 \log(N/\delta)}{(q_0 - p_0)^2 \rho}}$ and worst case coordinate sample complexity (and running time) $\bigO{\frac{N^2 q_0 \log(N/\delta)}{(q_0 - p_0)^2 \rho^2}}$.
\end{theorem}

The algorithm is obtained essentially by composing $N$ instances of our linear overhead expected sample complexity algorithm for the single coin.
In fact, we give a general framework for adaptively composing many replicable algorithms (\Cref{lemma:adaptive-composition}), which in conjunction with \Cref{thm:r-adapt-coin-problem} will prove the above theorem.

\begin{remark}[Subset Coordinate Sampling]
Our algorithm also works in a more general sampling model we call {\bf subset coordinate sampling.} Consider again an epidemiologist testing patients for a suite of diseases, with each disease requiring a different test.
Beyond simply minimizing the number of tests (coordinate samples), the epidemiologist may also wish to minimize the number of patients (vector samples).
As a result, instead of testing a different patient for every single test, the epidemiologist may wish to perform a subset of tests on any given patient.
If we think of patients as independently distributed vectors drawn from some population, this naturally leads to the subset coordinate sampling model, where the algorithm draws a vector sample and can then choose to examine an arbitrary subset of the coordinates of the vector. Note that in this case the resulting coordinate samples may be correlated. The coordinate sample complexity of an algorithm in this model is the total number of coordinates viewed.

The algorithms presented in this section may be naturally adapted to the subset model to reduce the total number of vectors drawn (roughly speaking, our algorithms maintain a subset of the coins for which they continue to draw samples, and may draw a subset sample over these coins at each step). Unfortunately, the reduction in vector samples is not by more than constant factors, so we will ignore this detail and just work in the coordinate model for simplicity.

\end{remark}

\subsection{Statistical Queries}

Before introducing our general amplification and composition framework (which in particular will require a good adaptive algorithm for heavy hitters), we start with the classical setting of statistical queries as a warmup. We first recall the definition of a statistical query oracle from \cite{impagliazzo2022reproducibility}.
\begin{restatable}[Statistical Query Oracle]{definition}{statqoracle}
    Let $\tau, \delta \in [0, 1]$.
    An algorithm $\innerAlg$ is a statistical query oracle if given sample access to some unknown distribution $\distribution$ over domain $\domain$ and query $\phi: \domain \mapsto [0, 1]$,
    $\innerAlg$ outputs a value $v$ such that $|v - \mu| \leq \tau$ with probability at least $1 - \delta$, where $\mu = \E{\phi(x)}$ is the expectation of $\phi(x)$ when $x$ is drawn from $\distribution$.
    \label{def:stat-q-oracle}
\end{restatable}
We give an algorithm for statistical queries with improved expected sample complexity.
\begin{restatable}{theorem}{radaptstatq}
    \label{thm:r-adaptive-stat-q}
    Let $4 \delta < \rho < \frac{1}{2}$ and $\tau \in [0, 1]$.
    Algorithm \ref{alg:r-adapt-stat-q} is a $\rho$-replicable statistical query oracle with expected sample complexity $\bigO{\frac{\log(1/\delta)}{\tau^2 \rho}}$.
\end{restatable}

We observe that this is optimal with respect to $\rho$.
If there is any algorithm with expected sample complexity $\littleO{\frac{1}{\tau^2 \rho}}$, then by Markov's inequality, it has worst case sample complexity at most $\littleO{\frac{1}{\tau^2 \rho^{2}}}$, contradicting the lower bound of \cite{impagliazzo2022reproducibility}, as well as Theorem \ref{thm:q0-num-lower-bound} since determining the bias of a coin is a special case of a statistical query.

\paragraph{Algorithm Overview}

Recall the statistical query oracle of \cite{impagliazzo2022reproducibility}.
The oracle draws a random offset $\alpha_{\mathrm{off}}$ and partitions $[0, 1]$ into regions $[\alpha_{\mathrm{off}} + i \alpha, \alpha_{\mathrm{off}} + (i + 1) \alpha]$ for all $i \geq 0$. Finally, the algorithm computes the empirical mean $\hat{\mu}$ on $\frac{1}{\tau^2 \rho^2}$ and checks its corresponding region, outputting the center point of that sub-interval.
Whenever the true mean is more than $\rho \tau$ away from the nearest boundary $\alpha_{\mathrm{off}} + i \alpha$, concentration bounds imply that two samples drawn from the same distribution must have empirical means within the same region, so that the algorithm outputs the midpoint of the same region given either sample.

Following the adaptive algorithm for the $(p_0, q_0)$-coin problem, we design an adaptive statistical query oracle that recognizes when the current empirical mean $\hat{\mu}$ is sufficiently far from any randomized threshold. 
Let $\mu$ be the true mean of the statistical query, $\alpha_{\mathrm{off}}$ the random offset, and $b_{\mu}$ the boundary $\alpha_{\mathrm{off}} + i \alpha$ that is closest to $\mu$.
Suppose at the $t$-th iteration we have taken $m_t$ samples so that $|\hat{\mu} - \mu| \leq \tau_t$.
If $|\mu - b_{\mu}| \geq 3 \tau_t$ then for any sample, $|\hat{\mu} - b_{\mu}| \geq 2 \tau_t$.
When this occurs, the algorithm terminates in the $t$-th iteration and outputs the midpoint of the region containing $\hat{\mu}$.

To show that this algorithm is replicable, we argue that regardless of the termination iteration, the algorithm chooses the same region to output the midpoint of, in particular the region containing $\mu$.
Indeed, if $|\hat{\mu} - b_{\mu}| \geq 2 \tau_t$, then $\mu$ must be in the same region as $\hat{\mu}$.
We again choose the number of iterations so that the algorithm successfully terminates unless $|\mu - b_{\mu}| \leq \rho \tau$, which occurs with probability at most $\rho$.

\paragraph{Analysis}

To bound the expected sample complexity, we observe that the sample complexity given a random offset is $\frac{1}{|\mu - b_{\mu}|^2}$.
Conditioned on $|\mu - b_{\mu}| \geq \rho \tau$, this is uniformly distributed in $(\rho \tau, \tau)$.
A similar computation of the expected sample complexity yields
\begin{equation*}
    \frac{1}{\tau} \int_{\rho \tau}^{\tau} \frac{1}{x^2} dx = \bigO{\frac{1}{\rho \tau^2}}.
\end{equation*}

We now present the algorithm and proof of Theorem \ref{thm:r-adaptive-stat-q}.

\IncMargin{1em}
\begin{algorithm}

\SetKwInOut{Input}{Input}\SetKwInOut{Output}{Output}\SetKwInOut{Parameters}{Parameters}
\Input{Sample access $S$ to distribution $\distribution$ on $\domain$ and query $\phi: \domain \mapsto [0, 1]$.}
\Parameters{$\rho$ replicability, $\tau$ tolerance, and $\delta$ accuracy}
\Output{$v \in [0, 1]$ such that $|v - \mu| \leq \tau$ with probability at least $1 - \delta$}

$\delta \gets \min\left(\delta, \frac{\rho}{4}\right)$

$\alpha \gets \frac{\tau}{8}$

$\alpha_{\mathrm{off}} \gets \UnifD{[0, \alpha]}$

Split $[0, 1]$ into regions: $R = \set{[0, \alpha_{\mathrm{off}}], [\alpha_{\mathrm{off}}, \alpha_{\mathrm{off}} + \alpha], \dotsc, [\alpha_{\mathrm{off}} + k \alpha, 1]}$

\For{$t = 1$ to $T = 4 + \log \frac{1}{\rho}$}{
    $\tau_t \gets \frac{\tau}{2^{t + 2}}$ \Comment{Note: $\tau_{T} = \frac{\rho \tau}{64}$}
    
    $m_t \gets \frac{3}{\tau_t^2} \log \frac{2 T}{\delta}$
    
    $S_t \gets (x_1, \dotsc, x_{m_t})$ is a fresh sample of size $m_t$ drawn from $\distribution$
    
    $\hat{\mu}_{t} \gets \frac{1}{m_t} \sum_{i = 1}^{m_t} \phi(x_i)$
    
    $\hat{b}_{t}^{-} \gets \hat{\mu}_t - 2 \tau_t$, $\hat{b}_{t}^{+} \gets \hat{\mu}_{t} + 2 \tau_{t}$
    
    \If{$(\hat{b}_{t}^{-}, \hat{b}_{t}^{+}) \subset [\alpha_{\mathrm{off}} + (i^* - 1) \alpha, \alpha_{\mathrm{off}} + i^* \alpha]$ for some  $i^* \geq 0$}{
        \Return $\alpha_{\mathrm{off}} + \frac{2i^* - 1}{2} \alpha$
    }
}

\Return $\hat{\mu}_T$

\caption{$\rAdaptiveStatQ(\distribution, \phi, \rho, \tau, \delta)$} 
\label{alg:r-adapt-stat-q}

\end{algorithm}
\DecMargin{1em}

\begin{lemma}
    \label{lemma:stat-q-sample-error-bound}
    Let $\distribution$ be a distribution over $\domain$.
    For $1 \leq t \leq T$, define $E_t$ to be the event that $|\hat{\mu}_t - \mu| \geq \tau_t$.
    Define $E = \bigcup_{t = 1}^{T} E_t$ to be the event that any $E_t$ occurs.
    Then, $\Pr(E_t) < \frac{\delta}{2 T}$ and $\Pr(E) < \frac{\delta}{2}$.
\end{lemma}

\begin{proof}
    Fix an iteration $t$.
    By a simple application of the Chernoff Bound,
    \begin{equation*}
        \Pr(E_t) = \Pr\left(|\hat{p}_t - \mu|> \tau_t \right) < \frac{\delta}{2 T}.
    \end{equation*}
    The upper bound on $\Pr(E)$ follows from the union bound.
\end{proof}
Throughout the proof of Theorem \ref{thm:r-adaptive-stat-q}, we assume that error event $E$ of Lemma \ref{lemma:stat-q-sample-error-bound} does not occur, noting that this error occurs with probability at most $\frac{\delta}{2}$.
\begin{proof}[Proof of Correctness of Algorithm \ref{alg:r-adapt-stat-q}]

    For any iteration, we have $|\hat{\mu}_t - \mu| \leq \tau_t$ so that $\mu \in (\hat{b}_{t}^{-}, \hat{b}_{t}^{+})$ for all $t \in [T]$.
    In particular, if $\rAdaptiveStatQ$ returns $\alpha_{\mathrm{off}} + \frac{2i^* - 1}{2} \alpha$ in the $t$-th iteration, then
    \begin{equation*}
        \mu \in (\hat{b}_{t}^{-}, \hat{b}_{t}^{+}) \subset [\alpha_{\mathrm{off}} + (i^* - 1) \alpha, \alpha_{\mathrm{off}} + i^* \alpha],
    \end{equation*}
    and the mid-point of the interval $\alpha_{\mathrm{off}} + \frac{2i^* - 1}{2} \alpha$ is at most $\alpha < \tau$ from any point in the interval, including $\mu$.
    Otherwise, $\rAdaptiveStatQ$ returns $\hat{\mu}_T$, which is at most $\tau_T < \tau$ from the true mean $\mu$.
\end{proof}

\begin{proof}[Proof of Replicability of Algorithm \ref{alg:r-adapt-stat-q}]

    Since $\mu = \alpha_{\mathrm{off}} + i \alpha$ for any $i$ with probability $0$, we may disregard this case.
    Then, suppose $\mu \in (\alpha_{\mathrm{off}} + (i_{\mu} - 1) \alpha, \alpha_{\mathrm{off}} + i_{\mu} \alpha)$ for some $i_{\mu} \geq 0$.

    First, suppose $\mu \in (\alpha_{\mathrm{off}} + (i_{\mu} - 1) \alpha + 3 \tau_T, \alpha_{\mathrm{off}} + i_{\mu} \alpha - 3 \tau_T)$.
    From the correctness argument above, when the error event $E$ does not occur, $\rAdaptiveStatQ$ can only return $\alpha_{\mathrm{off}} + \frac{2i_{\mu} - 1}{2} \alpha$ in the first $T$ iterations, as $\mu$ must be contained within the interval of the output midpoint. 
    In the $T$-th iteration, $|\hat{\mu}_T - \mu| \leq \tau_T$ so that $(\hat{b}_{t}^{-}, \hat{b}_{t}^{+}) \subset [\alpha_{\mathrm{off}} + (i_{\mu} - 1) \alpha, \alpha_{\mathrm{off}} + i_{\mu} \alpha]$ so that we output the midpoint of this region except with probability $\delta$, by union bounding over the error event $E$ in both invocations of $\rAdaptiveStatQ$.

    Otherwise, suppose $|\mu - (\alpha_{\mathrm{off}} + (i_{\mu} - 1) \alpha)| \leq 3 \tau_T$ or $|\mu - (\alpha_{\mathrm{off}} + (i_{\mu}) \alpha)| \leq 3 \tau_T$. 
    In this case, we do not guarantee any replicability.
    However, this event is extremely unlikely.
    Since the offset $\alpha_{\mathrm{off}}$ is chosen uniformly at random, this event occurs with probability at most $\frac{6 \tau_T}{\alpha} = \frac{3}{4} \rho$.
    Therefore, by the union bound, the probability that the outcome of the two invocations do not agree is at most $\frac{3}{4} \rho + \delta = \rho$.
\end{proof}

\begin{proof}[Proof of Sample Complexity of Algorithm \ref{alg:r-adapt-stat-q}]

    Again, we assume that $E$ does not occur and furthermore we assume that $\mu \in (\alpha_{\mathrm{off}} + (i_{\mu} - 1) \alpha + 3 \tau_T, \alpha_{\mathrm{off}} + i_{\mu} \alpha - 3 \tau_T)$. 
    Following above arguments, these events occur with probability at least $1 - \rho$.

    Then, we follow a similar argument to the sample complexity bound of Theorem \ref{thm:r-adapt-coin-problem}. 
    Let $t^*$ be the smallest value of $t$ such that
    \begin{equation*}
        \mu \in (\alpha_{\mathrm{off}} + (i_{\mu} - 1) \alpha + 3 \tau_t, \alpha_{\mathrm{off}} + i_{\mu} \alpha - 3 \tau_t).
    \end{equation*}
    By our choice of conditioning, $t^* \leq T$ and
    \begin{equation*}
         (\hat{b}_{t^*}^{-}, \hat{b}_{t^*}^{+}) \subset [\alpha_{\mathrm{off}} + (i_{\mu} - 1) \alpha, \alpha_{\mathrm{off}} + i_{\mu} \alpha],
    \end{equation*}
    so that the algorithm terminates in the $t^*$-th iteration. 
    Since the sample sizes $m_t$ increase geometrically, the algorithm requires $O(m_{t^*})$ samples.

    Let $b_{\mu}$ be the boundary $\alpha_{\mathrm{off}} + (i_{\mu} - 1) \alpha$ or $\alpha_{\mathrm{off}} + i_{\mu} \alpha$ closest to $\mu$.
    Then we can bound
    \begin{align*}
        \frac{3 \tau}{2^{t^* + 2}} = 3 \tau_{t^*} \leq |\mu - b_{\mu}| &\leq 6 \tau_{t^*} = \frac{3 \tau}{2^{t^* + 1}} \\
        \log \frac{3 \tau}{4 |\mu - b_{\mu}|} \leq t^* &\leq \log \frac{3 \tau}{2 |\mu - b_{\mu}|},
    \end{align*}
    so that the sample complexity is
    \begin{equation*}
        O(m_{t^*}) = \bigO{\frac{1}{\tau_{t^*}^2} \log \frac{2T}{\delta}} = \bigO{\frac{4^{t^*}}{\tau^2} \log \frac{2T}{\delta}} = \bigO{\frac{1}{|\mu - b_{\mu}|^2} \log \frac{T}{\delta}}.
    \end{equation*}

    Conditioned on $|\mu - b_{\mu}| \geq 3 \tau_T$, observe that $\alpha_{\mathrm{off}} + i_{\mu} \alpha$ is uniformly drawn from $(\mu + 3 \tau_T, \mu + \alpha - 3 \tau_T)$.
    With probability $\frac{1}{2}$, $\alpha_{\mathrm{off}} + i_{\mu} \alpha \leq \mu + \frac{\alpha}{2}$ and $b_{\mu} - \mu$ is uniformly drawn from $\left(3 \tau_T, \frac{\alpha}{2} \right)$.
    Otherwise, $\alpha_{\mathrm{off}} + i_{\mu} \alpha \geq \mu + \frac{\alpha}{2}$ and $\mu - b_{\mu}$ is uniformly drawn from $\left(3 \tau_T, \frac{\alpha}{2} \right)$.
    Then, if $r = |\mu - b_{\mu}|$, there is a universal constant $C$ such that the sample complexity is
    \begin{equation*}
        \E{M} \leq \int_{3 \tau_T}^{\frac{\alpha}{2}} \frac{C}{r^2} \frac{2}{\frac{\alpha}{2} - 3 \tau_T} \log \frac{2T}{\delta} dr = \frac{4 C}{\alpha - 6 \tau_T} \log \frac{2T}{\delta} \int_{3 \tau_T}^{\frac{\alpha}{2}} \frac{1}{r^2} dr = \frac{4 C}{\alpha - 6 \tau_T} \log \frac{2T}{\delta} \left( \frac{1}{3 \tau_T} - \frac{2}{\alpha} \right).
    \end{equation*}
    Plugging in $\alpha = \frac{\tau}{8}$ and $\tau_T = \frac{\rho \tau}{64} \leq \frac{\tau}{128}$,
    \begin{equation*}
        \E{M} = \bigO{ \frac{1}{\tau^2 \rho} \log \frac{2T}{\delta}} = \bigO{ \frac{1}{\rho \tau^2} \log \frac{1}{\delta}}.
    \end{equation*}

    Now, suppose $E$ occurs or $|\mu - b_{\mu}| \leq 3 \tau_T$.
    In this case, the worst case sample complexity is $O(m_T) = \bigO{\frac{1}{\rho^2 \tau^2} \log \frac{1}{\delta}}$, but this occurs with probability at most $\rho$, so that the expected sample complexity is as claimed.
\end{proof}

A natural generalization of our multiple hypothesis testing framework is the problem of replicably answering many statistical queries at once. 
Our adaptive single statistical query algorithm above gives a simple computationally efficient algorithm for replicably answering $N$ statistical queries using only $N^2$ samples.

\begin{definition}[Adaptive Statistical Query Algorithm]
Let $\distribution$ be a distribution over domain $\domain$.
Let $\tau, \delta \in [0, 1]$ and $\phi_i: \domain \mapsto [0, 1]$ be queries.
Let $\mu_i = \E{\phi_i(x)}$ where $x$ is drawn from $\distribution$.
Then, a $N$-query adaptive statistical query oracle algorithm outputs $v$ such that $|v_i - \mu_i| \leq \tau$ for all $i$ with probability at least $1 - \delta$.
In the adaptive setting, the query function $\phi_i$ is allowed to depend on 
the results of the previous statistical query oracles, i.e., $\{ v_j \}_{j=1}^i$.
\end{definition}


We give the theorem below. 
Again, this will be a corollary of our general adaptive composition framework given by \Cref{lemma:adaptive-composition}.

\begin{restatable}{theorem}{rnstatqoracle}
    \label{thm:r-n-stat-q-adaptive}
    Let $4 \delta < \rho < \frac{1}{2}$.
    There is a $\rho$-replicable $N$-query adaptive statistical query oracle algorithm with expected sample complexity $\bigO{\frac{N^2 \log(N/\delta)}{\tau^2 \rho}}$ and worst case sample complexity $\bigO{\frac{N^2 \log(N/\delta)}{\tau^2 \rho^2}}$.
\end{restatable}

\subsection{Heavy Hitters}

So far, we have used adaptivity to improve the expected sample complexity of two fundamental statistical tasks: bias estimation and statistical queries.
In fact, we will show that we can adaptively amplify any replicable algorithm from a $0.01$-replicable algorithm to a $\rho$-replicable algorithm with linear overhead in expected sample complexity.

The critical subroutine of replicable amplification is the \textit{heavy hitters} problem. We begin by reviewing the best known algorithm for finding replicable heavy hitters.

\begin{theorem}[Lemma 5 of \cite{DBLP:conf/icml/KalavasisKMV23}]
    \label{thm:r-heavy-hitters}
    Let $\distribution$ be a distribution on domain $\domain$.
    For any $x \in \domain$, let $\distribution(x)$ denote the probability mass of $x$. 
    For any $v, \delta, \rho, \varepsilon \in (0, 1)$, there is an algorithm $\rHeavyHitters(\distribution, v, \varepsilon, \delta, \rho)$ that is $\rho$-replicable and outputs $S \subset \domain$ satisfying the following with probability at least $1 - \delta$:
    \begin{enumerate}
        \item If $\distribution(x) > v'$, then $x \in S$.
        \item If $\distribution(x) < v'$, then $x \notin S$.
    \end{enumerate}
    where $v' \in [v - \varepsilon, v + \varepsilon]$.
    Moreover, $\rHeavyHitters$ has sample complexity $\bigO{\frac{\log(1/\min(\delta, \rho) (v - \varepsilon))}{(v - \varepsilon) \rho^2 \varepsilon^2}}$
\end{theorem}

We note that the algorithm chooses $v'$ randomly from $[v - \varepsilon, v + \varepsilon]$ and could also return $v'$ along with the set of heavy hitters.
This result relies on the Bretagnolle-Huber-Carol inequality.

\begin{lemma}[\cite{wellner2013weak}]
    \label{lemma:bhc-inequality}
    Let $\varepsilon > 0$.
    Let $\distribution$ be a multinomial distribution supported on $k$ elements.
    Then, given access to $m$ i.i.d.\ samples from $\distribution$,
    \begin{equation*}
        \Pr \left( \sum_{i = 1}^{k} |\hat{p}_i - \distribution(i)| \geq \varepsilon \right) \leq 2^{k} e^{- m \varepsilon^2/2}
    \end{equation*}
    where $\hat{p}_i$ is the empirical frequency of $i$ in the sample $S$.
\end{lemma}

Our adaptive algorithm will rely on a variation of this inequality, obtaining concentration for each element in the multinomial distribution.

\begin{lemma}
    \label{lemma:max-multinomial-error}
    Let $\varepsilon > 0$.
    Let $\distribution$ be a multinomial distribution supported on $k$ elements.
    Then, given access to $m$ i.i.d.\ samples from $\distribution$,
    \begin{equation*}
        \Pr \left( |\hat{p}_i - \distribution(i)| \geq \sqrt{\frac{3 \distribution(i)}{m} \log \frac{k}{\delta}} \textrm{\xspace for any $i \in [k]$} \right) \leq \delta,
    \end{equation*}
    where $\hat{p}_i$ is the empirical frequency of $i$ in the sample $S$.
\end{lemma}

\begin{proof}
    Fix an element $i$ and an iteration $t$.
    Let $\eta_{i} = \sqrt{\frac{3}{m \distribution(i)} \log \frac{k}{\delta}}$.
    By a Chernoff bound,
    \begin{equation*}
        \Pr\left( |\hat{p}_i - \distribution(i)| > \eta_{i} \distribution(i) \right) < \frac{\delta}{k}
    \end{equation*}
    so that by the union bound, the following holds for all $i$ except with probability $\delta$,
    \begin{equation*}
        |\hat{p}_i - \distribution(i)| \leq \eta_{i} \distribution(i) = \sqrt{\frac{3 \distribution(i)}{m} \log \frac{k}{\delta}}
    \end{equation*}
    Note that we recover (up to polylogarithmic factors) the BHC inequality by summing  up all terms and applying Cauchy-Schwartz
    \begin{equation*}
        \sum_{i = 1}^{k} |\hat{p}_i - \distribution(i)| \leq \sqrt{\frac{3}{m} \log \frac{k}{\delta}} \sum_{i = 1}^{k} \sqrt{\distribution(i)} \leq \sqrt{\frac{3 k}{m} \log \frac{k}{\delta}}
    \end{equation*}
\end{proof}

We now give an adaptive algorithm with optimal expected sample complexity with respect to $\rho$.

\begin{restatable}{theorem}{radaptheavyhitters}
    \label{thm:r-adaptive-heavy-hitters}
    Let $\distribution$ be a distribution on domain $\domain$.
    For any $x \in \domain$, let $\distribution(x)$ denote the probability mass of $x$. 
    For any $v, \delta, \rho, \varepsilon \in (0, 1)$ with $4 \delta < \rho$ and $4 \varepsilon < v$, Algorithm \ref{alg:r-adapt-heavy-hitters} is $\rho$-replicable and outputs $S \subset \domain$ satisfying the following with probability at least $1 - \delta$:
    \begin{enumerate}
        \item If $\distribution(x) > v'$, then $x \in S$.
        \item If $\distribution(x) < v'$, then $x \notin S$.
    \end{enumerate}
    where $v' \in [v - \varepsilon, v + \varepsilon]$.
    Moreover, Algorithm \ref{alg:r-adapt-heavy-hitters} has expected sample complexity 
    \begin{equation*}
        \bigO{\frac{1}{(v - \varepsilon) \varepsilon^2 \rho} \log \frac{1}{\delta(v - \varepsilon)}}.
    \end{equation*}
\end{restatable}

First, we describe our algorithm and techniques at a high level.

\paragraph{Algorithm Overview}
Our algorithm applies adaptivity to reduce the sample complexity of the heavy hitters algorithm of \cite{DBLP:conf/icml/KalavasisKMV23}.
We begin by recapping the algorithm of \cite{DBLP:conf/icml/KalavasisKMV23}. 
As a first step, the algorithm takes $\frac{1}{v - \varepsilon}$ samples from $\distribution$ to collect a set of candidate heavy hitters $\domain'$.
With high probability, any $(v - \varepsilon)$ heavy hitter will be included in the candidate set.

Then, we take $\frac{1}{(v - \varepsilon) \varepsilon^2 \rho^2}$ samples so that the total estimation error over all candidates is $\rho \varepsilon$, i.e. $\sum_{x \in \domain'} |\hat{p}_x - \distribution(x)| \leq \rho \varepsilon$.
Next, we draw a random threshold $v_r \in (v - \varepsilon, v + \varepsilon)$ uniformly, and return the set of elements whose empirical frequencies exceed $v_r$.
In order to guarantee replicability, we have to ensure that over two samples, the same set of elements is returned. 
Since this occurs only when $\hat{p}_{x, 1} < v_r < \hat{p}_{x, 2}$ for some $x$, we can bound this probability by $\rho$, where $\hat{p}_{x, i}$ is the empirical frequency of $x$ in the $i$-th sample.

In the adaptive setting, the algorithm instead terminates whenever the random threshold is sufficiently far from the empirical frequency of any element.
First, we use a variation of the Bretagnolle-Huber-Carol inequality (Lemma \ref{lemma:max-multinomial-error}) which additionally specifies the error $|\hat{p}_{x} - \distribution(x)|$ on each candidate heavy hitter, rather than a bound only on the total error.
We obtain this by observing that for elements with smaller mass $\distribution(x)$, we can obtain tighter concentration in its empirical frequency.

Consider the $t$-th iteration.
For each iteration, after taking $m_t$ samples (where $m_t = 4 m_{t - 1}$ and $m_0 = \frac{1}{\varepsilon^2}$), the empirical frequencies satisfy $|\hat{p}_x - \distribution(x)| \leq \eta_{t, x} = \varepsilon_t \sqrt{\distribution(x)}$ for all $x$.
Suppose $|v_r - \distribution(x)| \geq 3 \eta_{t, x}$ for all $x$.
Then, $|v_r - \hat{p}_x| \geq 2 \eta_{t, x}$ for all $x$ for any sample drawn from $\distribution$.
Whenever this occurs, the algorithm guarantees that $\hat{p}_x > v_r$ if and only if $\distribution(x) > v_r$, so that the algorithm safely terminates with $O(m_t)$ sample complexity.

In order to show that the algorithm is replicable, we argue that the same set is returned regardless of which iteration the algorithm returns.
Following our arguments above, regardless of which iteration the algorithm terminates, $x$ is included in the output set if and only if $\distribution(x) > v_r$. 
We choose the number of iterations to guarantee that the algorithm only fails to terminate before the final iteration if $|v_r - \distribution(x)| < \rho \varepsilon \sqrt{\distribution(x) (v - \varepsilon)}$ which occurs with probability at most $\rho$.

\paragraph{Analysis Overview}
As in the adaptive statistical queries algorithm, we proceed in iterations, terminating whenever we can conclude with high confidence that we have found our set of heavy hitters.
We proceed for $T$ iterations, where $T$ is a parameter to be defined later.
Let $y$ denote the element whose true mean is closest to the random threshold. 
Then, the sample complexity of our algorithm is $\frac{1}{|v_r - \distribution(y)|}$.
Since $|v_r - \distribution(y)| \leq \rho \varepsilon \sqrt{\distribution(y) (v - \varepsilon)}$ with probability at most $\rho$, we can condition on $|v_r - \distribution(x)| \geq \rho \sqrt{\distribution(y) (v - \varepsilon)} = \eta_{T, x}$ for all $x \in \domain'$.
Conditioned on this event, let $v_r$ be drawn uniformly from the set $I$ where

\begin{equation*}
    I = \left[ v - \frac{\varepsilon}{2}, v + \frac{\varepsilon}{2} \right] \setminus \left( \bigcup_{x \in \domain'} \left( \distribution(x) - 3 \eta_{T, x}, \distribution(x) + 3 \eta_{T, x} \right) \right)
\end{equation*}

is the union of at most $\frac{1}{v - \varepsilon}$ disjoint intervals where each $v_r$ is in each sub-interval with probability $\frac{|I_i|}{|I|}$.
Conditioned on $v_r \in I_i$, $|v_r - \distribution(y)|$ is uniformly distributed on $(\rho \varepsilon \sqrt{\distribution(y) (v - \varepsilon)}, |I_i|)$, so that the expected sample complexity is

\begin{equation*}
    \E{M | v_r \in I_i} = \frac{\sqrt{\distribution(y)}}{\rho \varepsilon |I_i| \sqrt{v - \varepsilon}}.
\end{equation*}

Then, the total expectation is
\begin{equation*}
    \E{M} = \sum_{i} \E{M | v_r \in I_i} \frac{|I_i|}{|I|} = \frac{1}{\rho \varepsilon^2 (v - \varepsilon)}.
\end{equation*}

We now give the formal algorithm and proof of Theorem \ref{thm:r-adaptive-heavy-hitters}.

\IncMargin{1em}
\begin{algorithm}

\SetKwInOut{Input}{Input}\SetKwInOut{Output}{Output}\SetKwInOut{Parameters}{Parameters}
\Input{Sample access $S$ to distribution $\distribution$ on $\domain$ and threshold $v$.}
\Parameters{$\varepsilon$ tolerance, $\rho$ replicability, and $\delta$ accuracy}
\Output{$S \subset \domain$ containing all elements with mass $\distribution(x) \geq v + \varepsilon$ and no elements with mass $\distribution(x) \leq v - \varepsilon$}

$\delta \gets \min\left(\delta, \frac{\rho}{4}\right)$

$S_1 \gets m_0$ i.i.d. samples from $\distribution$ where $m_0 \gets \frac{\log (8/\delta(v - \varepsilon))}{(v - \varepsilon)}$

$\domain' \gets$ unique elements of $S_1$ 
\label{line:heavy-hitters:domain'}
\Comment{Observe that $|\domain'| \leq m_0$}

$v_r \gets \UnifD{[v - \varepsilon/2, v + \varepsilon/2]}$

\For{$t = 1$ to $T = 7 + \log \frac{\sqrt{m_0}}{\rho}$}{
    $\varepsilon_t \gets \frac{\varepsilon}{2^{t + 2}}$
    \Comment{$\varepsilon_T = \frac{\rho \varepsilon}{2^{9} \sqrt{m_0}}$}

    $S_t \gets m_t$ i.i.d. samples from $\distribution$ where $m_t \gets \frac{3}{\varepsilon_t^2} \log \frac{4 (m_0 + 1) T}{\delta}$ 
    \Comment{$m_T = \frac{3 \cdot 2^{18} m_0}{\rho^2 \varepsilon^2} \log \frac{4 (m_0 + 1) T}{\delta}$}
    \label{line:heavy-hitters:m-sample}

    $\hat{p}_{t, x} \gets \frac{1}{m_t} \sum_{i = 1}^{m_t} \ind{s_{t, i} = x}$ for all $x \in \domain'$
    \Comment{$\hat{p}_{t, x}$ is the empirical frequency of $x$ in $S_t$}

    $\eta_{t, x} \gets \min\left(\varepsilon_t, 2 \sqrt{\frac{\hat{p}_{t, x}}{m_t} \log \frac{4 T (m_0 + 1)}{\delta}} \right)$
    \Comment{$\eta_{T, x} \leq  \frac{\sqrt{5}}{3 \cdot 2^{9}} \rho \varepsilon \sqrt{\frac{\distribution(x)}{m_0}}$ if $\distribution(x) \geq v - 2 \varepsilon$}

    \If{$v_r \not\in \bigcup_{x \in \domain'} (\hat{p}_{t, x} - 2 \eta_{t, x}, \hat{p}_{t, x} + 2 \eta_{t, x})$}
    {
        \label{line:r-heavy-hitters:termination-condition}
        
        \Return $\domain^* \gets \set{x \in \domain' \given \hat{p}_{t, x} > v_r}$
        \label{line:r-heavy-hitters:output-set}
    }
}

\Return $\rHeavyHitters(\distribution, v, \varepsilon, \frac{\rho}{2}, \frac{\delta}{8})$
\label{line:r-heavy-hitters:failure-return}

\caption{$\rAdaptiveHeavyHitters(\distribution, v, \varepsilon, \rho, \delta)$} 
\label{alg:r-adapt-heavy-hitters}

\end{algorithm}
\DecMargin{1em}

As before we'll condition on the absence of sampling errors. 
\begin{lemma}
    \label{lemma:heavy-hitters-sample-error-bound}
    Let $\distribution$ be a distribution over domain $\domain$.
    For $\domain' \subset \domain$ defined in \Cref{line:heavy-hitters:domain'}, let $\distribution_{\domain'}$ be the distribution over $\domain' \cup \set{\bot}$ where $x \in \domain'$ has probability mass $\distribution(x)$ and $\bot$ has mass $\distribution(\domain \setminus \domain')$.
    For $1 \leq t \leq T$, define $E_t$ to be the event that $|\hat{p}_{t, x} - \distribution(x)| \geq \eta_{t, x}$ for any $\distribution(x) \geq v - 2 \varepsilon$ or $|\hat{p}_{t, x} - \distribution(x)| \geq \varepsilon_t$ for any $x \in \domain'$.
    Define $E = \bigcup_{t = 1}^{T} E_t$ to be the event that any $E_t$ occurs.
    Then, $\Pr(E_t) < \frac{3 \delta}{4 T}$ and $\Pr(E) < \frac{3 \delta}{4}$.
\end{lemma}

\begin{proof}[Proof of Lemma \ref{lemma:heavy-hitters-sample-error-bound}]
    First, denote an  element \emph{relevant} if $\distribution(x) \geq v - 2 \varepsilon$.
    Fix an iteration $t$.
    For all $x \in \domain$ (and therefore all $x \in \domain'$), a Chernoff bound implies that
    \begin{equation*}
        \Pr \left( |\hat{p}_{t, x} - \distribution(x)| > \varepsilon_t \right) < \frac{\delta}{4 (m_0 + 1) T}.
    \end{equation*}
    Note that for all irrelevant $x$ (i.e. $\distribution(x) < v - 2 \varepsilon$, $\hat{p}_{t, x} \leq \distribution(x) + \varepsilon_t < v - \varepsilon$.
    
    Next, we consider relevant $x$ (i.e. $\distribution(x) \geq v - 2 \varepsilon$).
    Following \Cref{lemma:max-multinomial-error}, we have with probability at most $\frac{\delta}{4 (m_0 + 1) T}$,
    \begin{equation}
        \label{eq:heavy-hitter-estimate-concentrate}
        |\hat{p}_{t, x} - \distribution(x)| \geq \sqrt{\frac{3 \distribution(x)}{m_t} \log \frac{4 (m_0 + 1) T}{\delta}}.
    \end{equation}

    To show that $|\hat{p}_{t, x} - \distribution(x)| \geq \eta_{t, x}$ with small probability, we require that $\distribution(x), \hat{p}_{t, x}$ are close.
    We show that that these values are not too far off for relevant $x$.
    Since \snew{$v > 3 \varepsilon$}, we have $\distribution(x) > \varepsilon$ so
    \begin{equation*}
        \E{m_t \hat{p}_{t, x}} = m_t \distribution(x) = \frac{3 \distribution(x)}{\varepsilon_t^2} \log \frac{4 (m_0 + 1) T}{\delta} \geq \frac{48 \cdot 4^{t}}{\varepsilon} \log \frac{4 (m_0 + 1) T}{\delta}.
    \end{equation*}
    Then, by a Chernoff bound,
    \begin{equation*}
        \Pr \left( |\hat{p}_{t, x} - \distribution(x)| \geq \frac{1}{4} \distribution(x) \right) < \frac{\delta}{4 (m_0 + 1) T}.
    \end{equation*}
    
    In particular, we have $3 \distribution(x) \leq 4 \hat{p}_{t, x}$.
    Then, combining with \Cref{eq:heavy-hitter-estimate-concentrate}, we obtain by the union bound that with probability at most $\frac{\delta}{2(m_0 + 1)T}$,
    \begin{equation*}
        |\hat{p}_{t, x} - \distribution(x)| \leq \sqrt{\frac{3 \distribution(x)}{m_t} \log \frac{3 T m_0}{\delta}} \leq 2 \sqrt{\frac{\hat{p}_{t, x}}{m_t} \log \frac{3 T m_0}{\delta}} < \eta_{t, x}
    \end{equation*}

    Since $\distribution_{\domain'}$ is supported on at most $m_0 + 1$ elements, we upper bound $\Pr(E_t) \leq \frac{3\delta}{4T}$ by a union bound.
    Finally, the upper bound on $\Pr(E)$ follows from a union bound.
\end{proof}
\begin{proof}[Proof of Correctness of Algorithm \ref{alg:r-adapt-heavy-hitters}]
    Let $v_r \in [v - \varepsilon, v + \varepsilon]$ be chosen as in the algorithm.

    First, we show that all $v_r$-heavy hitters are included in the candidate set $\domain'$.
    Let $x$ be a heavy hitter with $\distribution(x) \geq v_r \geq v - \varepsilon$.
    Note that there are at most $\frac{1}{v - \varepsilon}$ such elements.
    For a fixed $x$ with mass at least $v - \varepsilon$, $x$ does not occur in $S_1$ with probability at most,
    \begin{equation*}
        \left(1 - (v - \varepsilon) \right)^{m_0} < \frac{\delta (v - \varepsilon)}{8}
    \end{equation*}
    which is at most $\frac{\delta}{8}$ when union bounding over the at most $\frac{1}{v - \varepsilon}$ heavy hitters.

    Thus, in the following, assume $\domain'$ includes all $v_r$-heavy hitters.
    We also assume the error event $E$ does not occur since $\Pr[E] \leq \frac{3 \delta}{4}$ by \Cref{lemma:heavy-hitters-sample-error-bound}.
    The following then holds for all iterations $t \in [T]$.
    For all irrelevant $x$ such that $\distribution(x) \leq v - 2 \varepsilon$, we have $\hat{p}_{t, x} < \distribution(x) + \varepsilon_t \leq v - \varepsilon \leq v_r$ so that $x \notin S$.
    Similarly, for all relevant $\distribution(x) \geq v - 2 \varepsilon$, we have $|\hat{p}_{t, x} - \distribution(x)| \leq \eta_{t, x}$.
    
    Suppose for some $t \in [T]$ the condition in \Cref{line:r-heavy-hitters:termination-condition} is satisfied.
    Then, we have $|\hat{p}_{t, x} - v_r| \geq 2 \eta_{t, x}$ for all $x \in \domain'$.
    Then for every $x$ such that $\distribution(x) \geq v_r \geq v - \varepsilon$ (and therefore $x \in \domain'$),
    \begin{equation*}
        \hat{p}_{t, x} \geq \distribution(x) - \eta_{t, x} \geq v_r - \eta_{t, x},
    \end{equation*}
    which implies that $\hat{p}_{t, x} \geq v_r$, since otherwise $\hat{p}_{t, x} < v_r - 2 \eta_{t, x}$.
    Thus, $x \in S$ is returned in \Cref{line:r-heavy-hitters:output-set}.
    On the other hand, if $\distribution(x) \leq v_r \leq v + \varepsilon$, then
    \begin{equation*}
        \hat{p}_{t, x} \leq \distribution(x) + \eta_{t, x} \leq v_r + \eta_{t, x}.
    \end{equation*}
    Then, $\hat{p}_{t, x} \leq v_r$ since otherwise $\hat{p}_{t, x} > v_r + 2 \eta_{t, x}$, so $x \not\in S$ is not returned in \Cref{line:r-heavy-hitters:output-set}.
    Otherwise, if the condition of \Cref{line:r-heavy-hitters:termination-condition} is never satisfied, we output a set of heavy hitters in Line \ref{line:r-heavy-hitters:failure-return} with probability at least $1 - \frac{\delta}{8}$. 
    The overall failure condition follows from a union bound.
\end{proof}

\begin{proof}[Proof of Replicability of Algorithm \ref{alg:r-adapt-heavy-hitters}]
    Again, we assume that $E$ does not occur and all $(v - \varepsilon)$ heavy hitters are included in $\domain'$ except with probability at most $\frac{\delta}{8}$.

    Since $\eta_{t, x} \leq \varepsilon_t$, we have that $\hat{p}_{t, x} + 2 \eta_{t, x} \leq v - \frac{3 \varepsilon}{2}$ for all $\distribution(x) \leq v - 2 \varepsilon$.
    In particular, not only will these elements never be included in the output set, the interval $(\hat{p}_{t, x} - 2 \eta_{t, x}, \hat{p}_{t, x} + 2 \eta_{t, x})$ will never contain the threshold $v_r$.
    In the following, we consider only the elements $\distribution(x) \geq v - 2 \varepsilon$.
    We call these \emph{relevant} elements on $\domain'$.

    First, suppose $|v_r - \distribution(x)| \geq 3 \eta_{T, x}$ for all relevant $\distribution(x) \geq v - 2 \varepsilon$.
    Then, since $E$ does not occur, $\rAdaptiveHeavyHitters$ returns exactly the set $\set{x \in \domain' \given \distribution(x) > v_r}$, as $|\hat{p}_{t, x} - \distribution(x)| \leq \eta_{t, x}$ and if $\rAdaptiveHeavyHitters$ returns a set in the iteration $t$, every element in the output has $\hat{p}_{t, x} > v_r + 2 \eta_{t, x}$ so that $\distribution(x) > v_r + \eta_{t, x}$.
    Similarly, every element not in the output has $\distribution(x) < v_r - \eta_{t, x}$.
    
    Now we claim that the algorithm returns by the $T$-th iteration.
    Since $|v_r - \distribution(x)| \geq 3 \eta_{T, x}$ for all $x \in \domain'$, then in the $T$-th iteration if $E$ does not occur, $|\hat{p}_{T, x} - v_r| \geq 2 \eta_{T, x}$ for all $x \in \domain'$, so that the algorithm terminates in the $T$-th iteration and outputs the set of elements $\hat{p}_{T, x} > v_r$ (equivalently the set $\distribution(x) > v_r$).

    Finally, suppose $|v_r - \distribution(x)| \leq 3 \eta_{T, x}$ for some relevant $x$.
    We show that this case occurs with low probability.
    In fact, we show that $|v_r - \distribution(x)| \leq 50 \eta_{T, x}$ occurs with low probability.
    Observe that
    \begin{align*}
        \left| \bigcup_{\distribution(x) \geq v - 2 \varepsilon} \left( \distribution(x) - 50 \eta_{T, x}, \distribution(x) + 50 \eta_{T, x} \right) \right| &\leq \sum_{\distribution(x) \geq v - 2 \varepsilon} 100 \eta_{T, x} \\
        &\leq \sum_{\distribution(x) \geq v - 2 \varepsilon} 200 \sqrt{\frac{\hat{p}_{t, x}}{m_T} \log \frac{4 T (m_0 + 1)}{\delta}} \\
        &\leq \frac{200}{2^{9}} \frac{\rho \varepsilon}{\sqrt{3 m_0}} \sum_{\distribution(x) \geq v - 2 \varepsilon} \sqrt{\hat{p}_{t, x}} \\
        &\leq \frac{\rho \varepsilon}{4},
    \end{align*}
    where in the first inequality we apply the definition of $\eta_{T, x}$,
    in the second inequality we apply the definition of $m_T$ (\Cref{line:heavy-hitters:m-sample}),
    and in the third inequality, we apply Cauchy-Schwartz noting there are at most $m_0$ relevant elements.
    Then, since $v_r$ is drawn uniformly from an interval of length $\varepsilon$, $v_r$ falls in this set with probability at most $\frac{\rho}{4}$.
    Therefore, over both runs of the algorithm $\rAdaptiveHeavyHitters$, the probability that the algorithm is not replicable is at most $\delta + \frac{\rho}{2} < \rho$.
\end{proof}

\begin{proof}[Proof of Sample Complexity of Algorithm \ref{alg:r-adapt-heavy-hitters}]
    As before, we begin by assuming that the error event $E$ does not occur.
    Let $y = \arg \min_{\distribution(x) \geq v - 2 \varepsilon} |v_r - \distribution(y)|$ be the relevant element whose probability is closest to $r$.
    
    Let $t^*$ be the smallest value of $t$ such that
    \begin{equation*}
        |v_r - \distribution(y)| \geq 3 \eta_{t, y},
    \end{equation*}
    so that $|v_r - \hat{p}_{t, y}| \geq 2 \eta_{t, y}$.
    We claim that the condition of \Cref{line:r-heavy-hitters:termination-condition} is satisfied at the $t^*$-th iteration and the algorithm will terminate.
    Consider a relevant $\distribution(x) \geq v - 2 \varepsilon$.
    If $\distribution(x) < \distribution(y)$, then
    \begin{equation*}
        |v_r - \hat{p}_{t, x}| \geq |v_r - \distribution(x)| - |\distribution(x) - \hat{p}_{t, x}| \geq 3 \eta_{t, y} - \eta_{t, x} \geq 2 \eta_{t, x}.
    \end{equation*}
    On the other hand, if $\distribution(x) > \distribution(y)$, note that $\hat{p}_{t, x} - 2 \eta_{t, x} \geq \hat{p}_{t, y} - 2 \eta_{t, y}$ by considering the function
    \begin{align*}
        h(\hat{p}_{t, x}) = \hat{p}_{t, x} - 4 \sqrt{\frac{\hat{p}_{t, x}}{m_t} \log \frac{4 T (m_0 + 1)}{\delta}} = \hat{p}_{t, x} - \frac{\varepsilon}{2^{t} \cdot 3} \sqrt{\hat{p}_{t, x}}.
    \end{align*}
    observing that this function has positive derivative whenever $\hat{p}_{t, x} \geq \frac{\varepsilon^2}{9 \cdot 4^{t + 1}}$ which holds for $\hat{p}_{t, x} \geq v - 3 \varepsilon \geq \varepsilon$.
    In particular, we have shown that $|v_r - \hat{p}_{t, x}| \geq 2 \eta_{t, x}$ for all relevant $x$.
    For all irrelevant $\distribution(x) \leq v - 2 \varepsilon$, observe
    \begin{equation*}
        \hat{p}_{t, x} \leq \distribution(x) + \varepsilon_t \leq \distribution(x) + \varepsilon_t \leq v_r - 2 \varepsilon_t \leq v_r - 2 \eta_{t, x},
    \end{equation*}
    since $\varepsilon_t \leq \frac{\varepsilon}{4}$.

    Now we consider the sample complexity if the algorithm terminates in the $t^*$-th iteration.
    Since $m_t$ increases geometrically, the algorithm requires $O(m_{t^*})$ samples. 
    We can bound $t^*$ as follows. 
    Since $t^*$ is minimal,
    \begin{align*}
          |v_r - \distribution(y)| &\leq 6 \eta_{t^*, y} = \frac{3 \varepsilon \sqrt{\hat{p}_{t^*, y}}}{2^{t^*}} \leq \frac{7 \varepsilon \sqrt{\distribution(y)}}{2 \cdot 2^{t^*}} \implies t^* \leq \log \frac{7 \varepsilon \sqrt{\distribution(y)}}{2 |v_r - \distribution(y)|},
    \end{align*}
    where we hvae used $|\hat{p}_{t, x} - \distribution(x)| \leq \frac{1}{4} \distribution(x)$ for all relevant $x$.
    Then, we can bound the sample complexity
    \begin{equation*}
        O(m_{t^*}) = \bigO{\frac{1}{\varepsilon_{t^*}^2} \log \frac{T m_0}{\delta}} = \bigO{\frac{4^{t^*}}{\varepsilon^2} \log \frac{T m_0}{\delta}} = \bigO{\frac{\distribution(y)}{|v_r - \distribution(y)|^2} \log \frac{T}{\delta}}.
    \end{equation*}

    We further condition on $|v_r - \distribution(x)| \geq 3 \varepsilon_T$ for all relevant $x \in \domain'$.
    We note that $v_r$ is uniformly distributed on the union of intervals
    \begin{equation*}
        I = \left[ v - \frac{\varepsilon}{2}, v + \frac{\varepsilon}{2} \right] \setminus \left( \bigcup_{\distribution(x) \geq v - 2 \varepsilon} \left( \distribution(x) - 3 \eta_{T, x}, \distribution(x) + 3 \eta_{T, x} \right) \right).
    \end{equation*}
    Since there are at most $m_0$ relevant $x$, $I$ is composed of at most $m_0 + 1$ distinct intervals, $I_1, \dotsc, I_{s}$ for $s \leq m_0 + 1$, where each interval is of the form $(\distribution(x) + 3 \eta_{T, x}, \distribution(x') + 3 \eta_{T, x'})$ for relevant $x, x'$. 
    We can also assume that each interval $|I_i| \geq 47 \eta_{T, x}$ since $|v_r - \distribution(x)| \leq 50 \eta_{T, x}$ holds for all relevant $x$ with probability at most $\frac{\rho}{4}$, where $x$ is the element of smallest probability measure $\distribution(x)$ above the interval $I_i$.
    Note that $v_r \in I_i$ with probability $\frac{|I_i|}{\sum_{i = 1}^{s} |I_i|}$.
    To compute the overall expected sample complexity, we begin with the expected sample complexity conditioned on $v_r \in I_i$ for some fixed $i$.
    Since $v_r$ is uniformly distributed in $I_i$ and we have lower bounded the size of the interval $I_i$, $|v_r - \distribution(y)|$ stochastically dominates a uniformly distributed in the interval $\left(3 \eta_{T, x}, 20 \eta_{T, x} \right)$ where $x$ is the smallest $\distribution(x)$ above the interval $I_i$.
    Then, there is some universal constant $C$ such that the expected conditional sample complexity is,
    \begin{equation*}
        \E{M | I_i} \leq \int_{3 \eta_{T, x}}^{|I_i|/2} \frac{2}{|I_i|} \frac{C \distribution(x)}{r^2} \log \frac{m_0 T}{\delta} dr \leq \frac{2 C \distribution(x)}{|I_i|} \frac{1}{3 \eta_{T, x}} \log \frac{m_0 T}{\delta} = \bigO{\frac{\sqrt{m_0 \distribution(x)}}{|I_i| \rho \varepsilon} \log \frac{m_0 T}{\delta}}
    \end{equation*}
    where we have applied $\frac{1}{\eta_{T, x}} = \bigO{\sqrt{\frac{m_0}{\distribution(x)}} \frac{1}{\rho \varepsilon}}$.
    Now, marginalizing over the choice of $I_i$, we can bound the sample complexity using an appropriately large constant $C$,
    \begin{equation*}
        \E{M} \leq \sum_{i = 1}^{s} \frac{|I_i|}{|I|} \frac{C \sqrt{m_0 \distribution(x_s)}}{|I_i| \rho \varepsilon} \log \frac{m_0 T}{\delta} \leq \frac{C \sqrt{m_0}}{|I| \rho \varepsilon} \log \frac{m_0 T}{\delta} \sum_{x \in \domain'} \sqrt{\distribution(x)}
    \end{equation*}
    since each $x \in \domain'$ appears at most once in the summation over intervals.
    Applying the Cauchy-Schwartz inequality to $\sum \sqrt{\distribution(x)}$ and using $|I| \geq \varepsilon - 6 m_0 \varepsilon_T$, and $\varepsilon_T = \frac{\rho \varepsilon}{64 m_0}$, we obtain,
    \begin{equation*}
        \E{M} = \bigO{\frac{m_0}{\rho \varepsilon^2} \log \frac{m_0 T}{\delta}} = \bigO{\frac{1}{(v - \varepsilon) \varepsilon^2 \rho} \log \frac{1}{\delta (v - \varepsilon)}}
    \end{equation*}
    expected sample complexity.
    Now, the events that we conditioned on do not occur with probability at most $O(\rho)$.
    The worst case sample complexity of our algorithm is,
    \begin{equation*}
        m_T + \bigO{\frac{1}{(v - \varepsilon) \rho^2 \varepsilon^2} \log \frac{1}{\delta (v - \varepsilon)}} = \bigO{\frac{1}{(v - \varepsilon) \rho^2 \varepsilon^2} \log \frac{1}{\delta (v - \varepsilon)}}
    \end{equation*}
\end{proof}

\subsection{Adaptive Amplification}
Leveraging adaptive heavy hitters, we finally show that \textit{any} constantly replicable procedure can be amplified to a $\rho$-replicable algorithm with linear expected sample complexity. Toward this end, we will work in the slightly more general setup of `statistical problems' which capture all settings considered in this work. In particular a statistical problem consists of a domain $\domain$, a family of underlying distributions $\mathscr{P}=\{\distribution\}$ over $X$, a solution space $Y$, and a correctness relation $C: \mathscr{P} \to Y$. An algorithm $\innerAlg: \domain^* \to Y$ is said to be $\delta$-correct if for every $\distribution \in P$, it outputs $y \in C(\distribution)$ with probability at least $1-\delta$.

We now move to our adaptive amplification lemma, which simply modifies an amplification procedure in \cite{DBLP:conf/stoc/BunGHILPSS23} by replacing standard heavy hitters with adaptive heavy hitters.

\IncMargin{1em}
\begin{algorithm}

\SetKwInOut{Input}{Input}\SetKwInOut{Output}{Output}\SetKwInOut{Parameters}{Parameters}
\Input{Sample access $S$ to distribution $\distribution$ on $\domain$ and algorithm $\innerAlg: (\domain^n,R) \to Y$}
\Parameters{$\rho$ replicability, $\delta$ accuracy}
\Output{Correct $y \in Y$ except with probability $\delta$}

$k \gets O(\log\frac{1}{\rho})$

$\rho' \gets \frac{\rho}{2k}$

$\delta' \gets O(\frac{\rho^2}{\log^3\frac{1}{\rho}}\delta)$

$\{r_1,\ldots, r_k\} \gets R$

\For{$i = 1$ to $k$}{
    $y_i \gets \rAdaptiveHeavyHitters(\innerAlg(\distribution^n;r_i),0.8,0.1,\rho',\delta')$
}

\If{$\{y_i\}=\emptyset$}{
\Return any $y \sim \innerAlg(\distribution^n,r_i)$ sampled during adaptive heavy hitters.
}

\Return random $y \in \{y_i\}$

\caption{$\rAdaptiveAmplify (\distribution, \innerAlg, \rho, \delta)$} 
\label{alg:r-adapt-amp-q}

\end{algorithm}
\DecMargin{1em}

\begin{lemma}[Adaptive Amplification]
    \label{lem:adaptive-amplify}
    For any $\delta,\rho>0$, if $\innerAlg$ is $\bigtO{\rho^2 \delta}$-correct, $.01$-replicable, and uses at most $n$ samples, then $\rAdaptiveAmplify$ is $\delta$-correct, $\rho$-replicable, has expected sample complexity $\bigO{\frac{n}{\rho}  \log^2 \frac{1}{\rho} \log \frac{1}{\min(\delta, \rho)}}$, and worst-case complexity $\bigO{\frac{n}{\rho^2} \log^2 \frac{1}{\rho} \log \frac{1}{\min(\delta, \rho)}}$.
\end{lemma}

\begin{proof}
    Correctness and replicability follow exactly as in \cite{DBLP:conf/stoc/BunGHILPSS23}. Since each application of adaptive heavy hitters is $\frac{\rho}{2k}$-replicable, the list $\{h_1,\ldots,h_k\}$ is $\frac{\rho}{2}$-replicable by a union bound. If $\{y_i\}$ is non-empty, outputting a random element of the list is a fully replicable procedure. If $\{y_i\}$ is empty replicability may fail, but this occurs with probability at most $\rho/2$ by our choice of $k$ and the fact that at least $90\%$ of random strings have a $0.9$-heavy-hitter. For correctness, observe since each individual output $\innerAlg(S,r)$ is correct with probability at least $1/\delta$, union bounding over the at most $O(\frac{k^2\log\frac{1}{\rho}}{\rho^2})$ calls to $\innerAlg$ the output is correct except with $\delta$ probability by the correctness assumption on $\innerAlg$.

    Finally, by linearity of expectation the expected sample complexity of $\rAdaptiveAmplify$ is $k$ times the expected sample complexity of \rAdaptiveHeavyHitters. Likewise, the worst-case sample complexity is $k$ times its worst-case complexity. \Cref{thm:r-adaptive-heavy-hitters} then gives the desired bounds.
\end{proof}
We remark it is also possible to run a similar procedure with correctness $O(\delta + \rho)$ as in \cite{impagliazzo2022reproducibility}. However, we typically think of $\delta \ll \rho$ in learning and testing, so the above typically results in better parameters in our setting.

\subsection{Adaptive Composition}\label{sec:adaptive-composition}
We close the section by showing adaptive replicability implies a generic composition theorem beating the typical union bound approach. We'll consider the following general setup. Let $\{\innerAlg_i: (\domain_i^*,R) \to Y_i\}$ be a family of randomized sub-routines. We say a randomized algorithm $\innerAlg: (\domain^*,R) \to Y$ is an \textit{adaptive $n$-composition} if it conforms to the following procedure. For every $i \in [n]$:
\begin{enumerate}
    \item $\innerAlg$ draws a sample-string pair $(S_i,r_i)$.
    \item $\innerAlg$ runs the sub-routine $\innerAlg_{j_i}(S_i,r_i)$
    \item $\innerAlg$ chooses the subroutine $\innerAlg_{j_{i+1}}$ for the following round.\footnote{$\innerAlg_{j_1}$ may be any sub-routine.}
\end{enumerate}
After the sub-routines $\innerAlg$ applies an arbitrary aggregation function to the output transcript:
\[
\innerAlg(S_1,\ldots,S_n;r_1,\ldots,r_n) = f_{ag}(\innerAlg_1(S_1,r_1),\ldots,\innerAlg_n(S_n,r_n)).
\]
Finally, we allow $\innerAlg$ to have a `sample cap' $m$, beyond which it may ignore the composition and run a separate procedure. One should think of the cap as typically being a low probability event. If $\innerAlg$ and each of its sub-routines are solving a statistical task, we say $\innerAlg$ is \textit{consistent} if $f_{ag}(\cdot)$ is correct whenever all its inputs are correct.

Before moving to the composition lemma itself, we look at a concrete instantiation of this framework for $\ell_\infty$-mean estimation over $[0,1]^n$. In this case, one might reasonably take each $\innerAlg_i: [0,1] \to [0,1]$ to be an estimator for the $i$th coordinate and take $\innerAlg$ to be the trivial $n$-composition:
\[
\innerAlg(S_1,\ldots,S_n;r_1,\ldots,r_n) = (\innerAlg_1(S_1;r_1),\ldots,\innerAlg_n(S_n;r_n)),
\]
where each $S_i$ is drawn from the $i$th marginal. That is to say $\innerAlg$ runs each $\innerAlg_i$ independently on marginal samples from the $i$th coordinate. This is clearly a `consistent' procedure since the righthand tuple is $\varepsilon$-close to the mean in $\ell_\infty$ exactly when every coordinate is $\varepsilon$-close to its corresponding value.

We now formalize this example as a general lemma for adaptive composition. Since any replicable procedure may be amplified to have expected sample complexity near-linear in $\frac{1}{\rho}$ (\Cref{lem:adaptive-amplify}), we will assume this is the case for our sub-routines without loss of generality.

\begin{lemma}[Adaptive Composition]
    \label{lemma:adaptive-composition}
    Let $c>0$ be any constant, $N \in \N$, and $\{\innerAlg_i\}$ a family of $\frac{\rho}{3}$-replicable, $\frac{\delta}{2}$-correct sub-routines using $C \frac{\log^c\frac{1}{\rho}}{\rho}$ expected samples and $V=V(\rho)$ runtime. 
    Let $\innerAlg$ be any consistent $N$-composition
    \[
    \innerAlg(S_1,\ldots,S_{N};r_1,\ldots,r_{N}) = f_{ag}(\innerAlg_1(S_1,r_1),\ldots,\innerAlg_n(S_{N},r_{N}))
    \]
    with sample cutoff $m=\bigO{C \frac{\log^c\frac{1}{\rho}}{\rho^2}}$ which runs an $m$-sample $\frac{\delta}{2}$-correct procedure post-cutoff. Then:
    \begin{enumerate}
        \item $\innerAlg$ uses at most $\bigO{C \frac{N \log^{c} \frac{1}{\rho}}{\rho}}$ samples in expectation
        \item $\innerAlg$ uses at most $\bigO{C \frac{\log^{c} \frac{1}{\rho}}{\rho^2}}$ samples in the worst case
        \item $\innerAlg$ runs in time at most $\bigO{N V(\rho)}$
        \item $\innerAlg$ is $(N \delta)$-correct and $(N \rho)$-replicable
    \end{enumerate}
\end{lemma}
\begin{proof}
    Properties (1), (2), and (3) are essentially immediate from construction. 
    Namely, by assumption the algorithm uses at most $2m$ samples. By linearity of expectation, the expected sample complexity is at most the sum of the expected samples of each $\innerAlg_i$ (the cut-off only improves expected complexity). 
    Similarly, the run time is at most the sum of the sub-routine runtimes (we assume for simplicity the aggregation function compute time is negligible in comparison).

    We now move to property (4), correctness and replicability. Let $E_{cut}$ denote the event that $\mathcal{A}$ triggers the sample cutoff. 
    Since the expected sample complexity of $\mathcal{A}$ with no cutoff is at most $C \frac{N \log^{c} \frac{1}{\rho}}{\rho}$ by the same argument as above, we can choose $m$ large enough so that $E_{cut}$ occurs with probability at most $\frac{N \rho}{6}$ by Markov's inequality. 
    The correctness of the algorithm may be written as:
    \begin{align*}
        \Pr[\mathcal{A}~\text{is correct}] &= (1-\Pr[E_{cut}])\Pr[\mathcal{A}~\text{is correct}~|~\overline{E_{cut}}] + \Pr[E_{cut}]\Pr[\mathcal{A}~\text{is correct}~|~E_{cut}]\\
        &\geq 1-\frac{N \delta}{2}-\Pr[E_{cut}] + \Pr[E_{cut}]\left(1-\frac{N \delta}{2}\right)\\
        &\geq 1 - N \delta
    \end{align*}
    Replicability follows similarly. Consider two runs of $\mathcal{A}$ on independent samples. 
    The probability at least one run hits the sample cap is at most $\frac{N \rho}{3}$. 
    On the other hand, by a union bound with no sample cap the two runs are the same except with probability $\frac{N \rho}{3}$. 
    A similar argument to the above then gives $\rho$-replicability for the overall procedure.
\end{proof}
We remark one can of course give more fine-grain guarantees on the complexity and running time by considering the individual complexity of each sub-routine.

As an immediate corollary of adaptive composition, we obtain an efficient $N$-Coin algorithm (\Cref{thm:r-n-coin-problem-formal}) by combining \Cref{lemma:adaptive-composition} with \Cref{thm:r-adapt-coin-problem} by beginning with $\frac{\rho}{N}$-replicable and $\frac{\delta}{N}$-correct algorithms for a single coin.
Note that the cut-off algorithm can simply be obtained by taking $N$ non-replicable instances of \Cref{thm:r-adapt-coin-problem}, each $\frac{\delta}{N}$-correct.
Similarly, we prove \Cref{thm:r-n-stat-q-adaptive} can be shown by combining \Cref{lemma:adaptive-composition} with the single statistical query algorithm of \Cref{thm:r-adaptive-stat-q}.


\section{Efficient Replicability via Relaxations of the \texorpdfstring{$N$}{N}-Coin Problem}
\label{sec:efficient-n-coin-problem}

We have shown that any non-adaptive $\rho$-replicable algorithm solving the $N$-coin problem requires $\frac{N^2}{(q_0 - p_0)^2 \rho^2}$ samples.
For large values of $N$, this results in a prohibitively expensive experimental process, as the number of samples required scales \emph{quadratically}, not linearly as expected, with the number of experiments.
In this section, we discuss relaxations of the $N$-Coin Problem that allow us to circumvent the quadratic lower bound imposed by \Cref{thm:n-coin-lower-const-delta}.
Concretely, we relax either the correctness or replicability constraints of the problem, and obtain significantly more efficient algorithms with total sample complexity linear in the number of experiments (note that this is optimal even for non-replicable algorithms).

\subsection{Finding a Small Set of Pseudo-Maximum Coins}

We begin with an algorithm that replicably identifies a small set of maximally biased coins.
Consider the example of the epidemiologist.
While we may not have sufficient data to determine the prevalence of every disease, we might just hope to determine the most prevalent diseases. We give an algorithm whose sample and computational efficiency scales with the desired number of high bias coins.

To state the result, we first define the notion of a pseudo-maximal coin.

\begin{definition}
    \label{def:eps-pseudo-maximum}
    Let $N, \varepsilon > 0$.
    There are $N$ coins with bias $p_i \in [0, 1]$ for $i \in [N]$.
    Let $p_{\max} = \max_i p_i$ be the maximum bias.
    A coin $j$ is an $\varepsilon$ pseudo-maximum if  $p_j \geq p_{\max} - \varepsilon$.
\end{definition}

\begin{restatable}[\Cref{thm:r-pseudo-maximum-identifier}, formal]{theorem}{ThmPseudoMaximumIdentifier}
    \label{thm:r-pseudo-maximum-identifier-formal}
    Let $\varepsilon, \delta > 0$ and $0 \leq K \leq N$.
    \newline
    Algorithm $\rPseudoMaximumIdentifier$ is a $\rho$-replicable algorithm that with probability at least $1 - \delta$, outputs $S$ satisfying,
    \begin{enumerate}
        \item (Soundness) If $i \in S$, then $i$ is a $6 \varepsilon$ pseudo-maximum.
        \item (Completeness) $S$ satisfies the following:
        \begin{enumerate}
            \item If there are $C \leq K \left( \frac{N}{K} \right)^{1/3}$ $\varepsilon$ pseudo-maximum coins, then $S$ includes all $i$ such that,
            \begin{equation*}
                p_i \geq p_{\max} - \bigTh{\varepsilon \sqrt{\frac{\min(p_{\min}, 1 - p_{\max})}{\log(N/\delta) \log(K/\delta)}}}
            \end{equation*}
            and $|S| \geq \Omega(\frac{C}{\log(K/\delta)})$.
            \item If there are $C \geq K \left( \frac{N}{K} \right)^{1/3}$ $\varepsilon$ pseudo-maximum coins, then $|S| \geq K$.
        \end{enumerate}
    \end{enumerate}
    Moreover, $\rPseudoMaximumIdentifier$ has expected sample complexity
    \begin{equation*}
        \bigO{\frac{N^{4/3} K^{2/3}}{\rho \varepsilon^2} \log^{4} \frac{N}{\delta}}.
    \end{equation*}
\end{restatable}


First, we provide a key subroutine that replicably identifies pseudo-maximum coins when there are not too many.

\begin{lemma}
    \label{lemma:r-k-pseudo-maximum-identifier}
    Let $\rho, \delta < \frac{1}{6}$ and $K \leq \frac{N}{4}$.
    Let $p_i, \dotsc, p_{N} \in [0, 1]$ with $p_{\min} = \min p_i$ the minimum bias and $p_{\max} = \max p_i$ the maximum bias.
    Suppose there are $C \leq K$ coins that are $\varepsilon$ pseudo-maximum.
    
    Algorithm $\rKPseudoMaximumIdentifier$ is a $\rho$-replicable algorithm that with probability at least $1 - \delta$ outputs $S, \hat{p}_{\max}$ satisfying,
    \begin{enumerate}
        \item If $i \in S$, then Coin $i$ is a $10 \varepsilon$ pseudo-maximum.
        \item If $p_i \geq p_{\max} - \bigTh{\varepsilon \sqrt{\frac{\min(p_{\min}, 1 - p_{\max})}{\log(N/\delta) \log(K/\delta)}}}$, then $i \in S$
        \item $|S| \geq \frac{C}{18 \log(6K/\delta)} = \bigOm{\frac{C}{\log(K/\delta)}}$
        \item $|\hat{p}_{\max} - p_i| \leq 7 \varepsilon$ for all $i \in S$ (this holds even if $C > K$)
    \end{enumerate}
    Moreover, $\rKPseudoMaximumIdentifier$ has sample complexity $\bigO{\frac{K N}{\rho \varepsilon^2} \log^4 \frac{N}{\delta}}$.
\end{lemma}

We begin with a high level overview of our algorithm.

\paragraph{Algorithm Overview: $K = 1$}

First, consider the simple case of $K = 1$.
Consider an input distribution with $p_{i^*} > p_i + \varepsilon$ for all $i \neq i^*$ where $p_{i^*} = p_{\max}$ is the unique $\varepsilon$ pseudo-maximum coin.
Our algorithm proceeds as follows.
First, the subroutine $\findMaximum$ takes $\frac{1}{(q_0 - p_0)^2}$ samples from each coin, returning the coin with the highest empirical bias.
By taking sufficiently many samples, we have amplified the bias of the $N$ coins so that with high probability, $\findMaximum$ returns the coin with bias $p_{i^*}$. 
Let $\distribution_{\max}$ denote the output distribution of $\findMaximum$.
When $K = 1$, $i^*$ is a $\frac{1}{2}$-heavy hitter of the distribution $\distribution_{\max}$, while all other coins have probability mass at most $\frac{1}{N}$ in the distribution $\distribution_{\max}$.
Then, running our $\rAdaptiveHeavyHitters$ algorithm, we require only $\frac{1}{\rho}$ samples to replicably find a $\frac{1}{2}$-heavy hitter of $\distribution_{\max}$.
Overall, the sample complexity of this algorithm is,
\begin{equation*}
    \frac{N}{(q_0 - p_0)^2 \rho}
\end{equation*}
Furthermore, since $\rAdaptiveHeavyHitters$ is replicable, our algorithm for identifying $K = 1$ bias coin is also replicable.

\paragraph{Algorithm Overview: $B$ Buckets to Detect $K$ Biased Coins}

For larger values of $K$, the bias coins are no longer $\frac{1}{2}$-heavy hitters, but instead $\frac{1}{K}$ heavy hitters.
Since $\rAdaptiveHeavyHitters$ has sample complexity $\frac{1}{\varepsilon^2 (v - \varepsilon) \rho}$, we require $\frac{K^3}{\rho}$ samples from the distribution $\distribution_{\max}$, which becomes prohibitively expensive.
Indeed, for $K > N^{1/3}$, the total sample complexity exceeds $N^2$, with which we can already solve the general $N$-coin problem.
To address this issue, we randomly distribute the $N$ coins into $B$ buckets.
In expectation (indeed with high probability), each bucket has $\frac{K}{B}$ bias coins (up to polylogarithmic factors), so the pseudo-maximum coins are $\frac{B}{K}$-heavy hitters in output distribution $\findMaximum$ restricted to the $b$-th bucket.
Within each bucket, we replicably find $\frac{B}{K}$-heavy hitters.
Since we want to be replicable over all buckets, we require our algorithms to be $\frac{\rho}{B}$-replicable.
In particular, we draw $\frac{K^3}{B^3} \frac{B}{\rho} = \frac{K^3}{B^2 \rho}$ samples from $\distribution_{\max}$ so that the sample complexity for each bucket is,
\begin{equation*}
    \frac{K^3}{B^2 \rho} \frac{N}{B (q_0 - p_0)^2} = \frac{N K^3}{B^3 (q_0 - p_0)^2 \rho}
\end{equation*}
Over all $B$ buckets, the sample complexity is,
\begin{equation*}
    \frac{N K^3}{B^2 (q_0 - p_0)^2 \rho}
\end{equation*}
Therefore, we hope to maximize $B$ subject to the constraint $B \leq K$.
In particular, setting $B = K$ (up to constants), we obtain the desired sample complexity,
\begin{equation*}
    \frac{N K}{(q_0 - p_0)^2 \rho}
\end{equation*}

We formally present our algorithms below.

\IncMargin{1em}
\begin{algorithm}

\SetKwInOut{Input}{Input}\SetKwInOut{Output}{Output}\SetKwInOut{Parameters}{Parameters}
\Input{Sample access to $N$ Bernoulli Distributions $\distribution_i$ with parameter $p_i \in [0, 1]$.}
\Parameters{$\rho$ replicability and $\delta$ accuracy}
\Output{}

$B \gets \frac{K}{12 \log \frac{6 K}{\delta}}$

Uniformly at random distribute $N$ coins into $B$ buckets
\label{line:r-k-pseudo-maximum-identifier:bucket}

$S \gets \emptyset$

\For{$b = 1$ to $B$}{
    Let $I_b \subset [N]$ denote the set of coins in bucket $b$

    $\delta_{\max} \gets \frac{B}{18K}$
    
    Let $\distribution_{\max, b}$ supported on $I_b$ denote the output distribution of $\findMaximum(I_b, \delta_{\max}, \varepsilon)$
    \label{line:r-k-pseudo-maximum-identifier:bucket-distribution}
    
    $S_{b} \gets \rAdaptiveHeavyHitters\left(\distribution_{\max, b}, \frac{B}{9K} , \frac{B}{18K}, \frac{\delta}{3B}, \frac{\rho}{2B} \right)$ 
    \label{line:r-k-pseudo-maximum-identifier:heavy-hitters}

    \If{$S_b \neq \emptyset$}{
        Let $i \in S_{b}$ be chosen arbitrarily and $\distribution_{i}$ denote samples drawn from coin $i$

        Let $\phi: \set{T, H} \mapsto \set{0, 1}$
        
        $\hat{p}_{\max, b} \gets \rAdaptiveStatQ\left(\distribution_i, \phi, \frac{\rho}{2B}, \varepsilon, \frac{\delta}{3B} \right)$ is a replicable estimate of the bias of coin $i$
    }
}

$\hat{p}_{\max} \gets \max_{b} \hat{p}_{\max, b}$

$S \gets \bigcup S_b$ for $b$ where $\hat{p}_{\max, b} \geq \hat{p}_{\max} - 5 \varepsilon$
\label{line:r-k-pseudo-maximum-identifier:trim-sets}

\Return $S, \hat{p}_{\max}$

\caption{$\rKPseudoMaximumIdentifier(N, K, \varepsilon, \rho, \delta)$}
\label{alg:k-pseudo-maximum-identifier}

\end{algorithm}
\DecMargin{1em}

\IncMargin{1em}
\begin{algorithm}

\SetKwInOut{Input}{Input}\SetKwInOut{Parameters}{Parameters}
\Input{Sample access to $N$ Bernoulli Distributions $\distribution_i$ with parameter $p_i \in [0, 1]$.}
\Parameters{$\delta$ accuracy}

$m \gets \frac{27}{\varepsilon^2} \log \frac{2N}{\delta}$

\For{$i \in [N]$}{
    Draw $m$ samples $S_i = \set{b_j}_{j = 1}^{m}$ from $\distribution_i$.

    $\hat{p}_i \gets \frac{1}{m} \sum_{j = 1}^{m} b_j$
}

\Return $i \gets \arg \max_{i} \hat{p}_i$

\caption{$\findMaximum(N, \delta, \varepsilon)$}
\label{alg:find-maximum}

\end{algorithm}
\DecMargin{1em}

Below, we provide the auxiliary lemmas required by Lemma \ref{lemma:r-k-pseudo-maximum-identifier}.
The first Lemma shows that with high probability, the size of each bucket is roughly balanced around $\frac{N}{B}$ and does not contain more than $\frac{K}{B}$ pseudo-maximum coins (as long as there are at most $C \leq K$ such coins).

\begin{lemma}
    \label{lemma:num-coins-in-bucket}
    Let $C \leq K$.
    Suppose there is a subset $S \subset [N]$ of $|S| = C$ coins
    For any bucket $b \in [B]$, let $N_b$ denote the number of coins in bucket $b$ and $M_b$ denote the number of coins in $S$ in bucket $b$. 
    Then, with probability at least $1 - \frac{\delta}{2}$, $\frac{N}{2B} \leq N_b \leq \frac{3N}{2B}$ and $M_b \leq \frac{3K}{2B}$ for all $b \in [B]$.
\end{lemma}

\begin{proof}[Proof of Lemma \ref{lemma:num-coins-in-bucket}]
    For each coin let $X_{i, b}$ be the binary random variable with $X_{i, b} = 1$ if coin $i$ is placed into bucket $b$.
    Then, the number of coins in bucket $b$ is $N_b = \sum_{i} X_{i, b}$.
    Since $\Pr(X_{i, b} = 1) = \frac{1}{B}$ we have the expectation $\E{N_b} = \E{\sum_{i} X_{i, b}} = \frac{N}{B}$.
    As the sum of $N$ binary random variables, we can use a multiplicative Chernoff bound to bound the number coins in bucket $b$,
    \begin{equation*}
        \Pr \left( \left|N_b - \frac{N}{B}\right| \geq \frac{N}{2B} \right) < 2 \exp \left( - \frac{N}{12B} \right) \leq 2 \exp \left( - \frac{K}{12B} \right) < \frac{\delta}{2K}
    \end{equation*}
    Since we only require the upper bound $M_b \leq \frac{3K}{2B}$, we can assume without loss of generality that $C = K$ by augmenting the set of coins $S$ with an arbitrary set of $K - C \geq 0$ coins.
    If there are at most $\frac{3K}{2B}$ coins from this augmented set in each bucket, there are at most $\frac{3K}{2B}$ coins with bias $p_i = q_0$ in each bucket.
    By a similar argument, we have $M_b = \sum_{p_i = q_0} X_{i, b}$ and $\E{M_b} = \frac{K}{B}$. 
    Therefore,
    \begin{equation*}
        \Pr \left( M_b - \frac{K}{B} \geq \frac{K}{2B} \right) < \exp \left( - \frac{K}{12B} \right) < \frac{\delta}{4K}
    \end{equation*}
    Since $B \leq K$, we can union bound over $B$ buckets to conclude the proof.
\end{proof}

Next, the following lemma shows that $\findMaximum$ with high probability will not output any coin that is not an $\varepsilon$ pseudo-maximum in its bucket.

\begin{lemma}
    \label{lemma:find-maximum-low-bias}
    For any bucket $b \in [B]$ let $p_{\max, b}$ be the maximum bias of coins in bucket $b$.
    Let $E$ denote the event that any coin $i$ in the bucket $b$ with $p_i < p_{\max, b} - \varepsilon$ is output by $\findMaximum$.
    Then,
    \begin{equation*}
        \Pr\left(E\right) < \delta_{\max}
    \end{equation*}
\end{lemma}

\begin{proof}[Proof of Lemma \ref{lemma:find-maximum-low-bias}]
    The result follows from a simple application of the Chernoff bound. 
    In particular, for all $i$,
    \begin{equation*}
        \Pr\left( |\hat{p}_i - p_i| \geq \frac{q_0 - p_0}{3} \right) = \Pr\left( |\hat{p}_i - p_i| \geq \frac{q_0 - p_0}{3 p_i} p_i \right) < 2 \exp \left( - \frac{(q_0 - p_0)^2}{27 p_i} m \right) < \frac{\delta_{\max}}{2N}
    \end{equation*}
    by our choice of $m = \frac{27 q_0}{(q_0 - p_0)^2} \log \frac{2N}{\delta}$.
    
    Therefore, with probability at least $1 - \delta_{\max}$, $\findMaximum$ outputs $i$ such that $p_i \geq p_{\max, b} - \varepsilon$.
\end{proof}

The next lemma shows that coins with bias sufficiently close to the maximum or are not dominated by too many other coins are returned by $\findMaximum$ with reasonable probability.

\begin{lemma}
    \label{lemma:prob-lower-bias-wins}
    There exists a universal constant $c_0$ satisfying the following.
    Let $m, n > 0$ be integers and $0 < a < b$ satisfying,
    \begin{equation*}
        b - a \leq c_0 \frac{\min(a, 1 - b)}{\sqrt{n m}}
    \end{equation*}
    
    Suppose $Y_0 \sim \BinomD{m}{a}$ and $Y_1, \dotsc, Y_n \sim \BinomD{m}{b}$. 
    Then,
    \begin{equation*}
        \Pr\left(Y_0 \geq \max_{i = 0}^{n} Y_i \right) \geq \frac{1}{4 (n + 1)}
    \end{equation*}
\end{lemma}

\begin{proof}[Proof of Lemma \ref{lemma:prob-lower-bias-wins}]
    As before, let $Y_i$ denote the observed number of heads given $m$ samples from the $i$-th coin. 
    Let $X$ be a uniform random bit and we let $Y_0 \sim \BinomD{m}{a}$ if $X = 0$ and $Y_0 \sim \BinomD{m}{b}$ if $X = 1$.
    Let $Y = (Y_0, \dotsc, Y_n)$ be the collection of independent binomial random variables.
    Given $Y$, let $g(Y) = \arg \max Y_i$ with ties broken arbitrarily (say uniformly at random). 
    Then, by symmetry,
    \begin{equation*}
        \Pr(g(Y) = 0 \mid X = 1) = \frac{1}{n + 1}
    \end{equation*}
    Note that we aim to lower bound $\Pr(g(Y) = 0 \mid X = 0)$. 
    Suppose for contradiction,
    \begin{equation*}
        \Pr(g(Y) = 0 \mid X = 0) < \frac{1}{4} \Pr(g(Y) = 0 \mid X = 1) = \frac{1}{4 (n + 1)}
    \end{equation*}.
    Then, consider the following algorithm that distinguishes $X$ with $51\%$ probability using $O(n m)$ samples.
    Choose some constant $C_0$ such that,
    \begin{equation*}
        \Pr\left( \BinomD{C_0 \cdot n}{\frac{1}{n + 1}} < \frac{1}{2(n + 1)} \right) + \Pr\left( \BinomD{C_0 \cdot n}{\frac{1}{4(n + 1)}} > \frac{1}{2(n + 1)} \right) < \frac{1}{10}
    \end{equation*}
    Then, we can compute $g(Y)$ at most $C_0 \cdot n$ times and decide $X = 0$ if and only if the $g(Y) = 0$ at least $\frac{1}{2(n + 1)}$ times.
    By our above arguments, this decides $X$ with at least $51\%$ probability.

    Consider now $Z = (Y^{(1)}, \dotsc, Y^{(C_0 n)})$ to be the samples used by our algorithm. 
    Again, by independence and Lemma \ref{lemma:one-coin-mutual-info-bound},
    \begin{equation*}
        I(X: Z) \leq C_0 n I(X: Y) \leq C_0 n I(X: Y_0) = \bigO{\frac{n m (b - a)^2}{\min(a, 1 - b)}}
    \end{equation*}

    By our above arguments, our algorithm computes $f(Z) = X$ with probability at least $51\%$.
    There is some constant $c_0$ such that if $b - a \leq c_0 \frac{\min(a, 1 - b)}{\sqrt{nm}}$, then $I(X: Z) < 2 \cdot 10^{-4}$.
    Then, by Lemma \ref{lemma:alg-mutual-info-lower-bound} no algorithm given $Z$ can guess $X$ with probability more than $51\%$ of the time.
    In particular, we have,
    \begin{equation*}
        b - a \geq c_0 \frac{\min(a, 1 - b)}{\sqrt{n m}}
    \end{equation*}
\end{proof}

We now prove Lemma \ref{lemma:r-k-pseudo-maximum-identifier}.

\begin{proof}[Proof of Lemma \ref{lemma:r-k-pseudo-maximum-identifier}]
    Let $N_b$ be the number of coins in bucket $b$ and $M_b$ be the number of $\varepsilon$ pseudo-maximum coins in bucket $b$.
    We condition on the following three events:
    \begin{enumerate}
        \item For each bucket $b \in [B]$, $\frac{N}{2B} \leq N_b \leq \frac{3N}{2B}$ and $M_b \leq \frac{3K}{2B}$.
        \item For each bucket $b \in [B]$, $\rAdaptiveHeavyHitters(\distribution_{\max, b})$ returns a set $S$ containing all elements with mass at least $\frac{B}{6K}$ and no elements with mass at most $\frac{B}{18K}$.
        \item For each bucket $b \in [B]$, $\rAdaptiveStatQ(\distribution_{i})$ returns an estimate $\hat{p}_{\max, b}$ such that $|p_i - \hat{p}_{\max, b}| \leq \frac{\varepsilon}{4}$
    \end{enumerate}
    Note that only the first condition requires an assumption $C \leq K$.
    By Lemma \ref{lemma:num-coins-in-bucket} and the correctness of $\rAdaptiveHeavyHitters$ (Theorem \ref{thm:r-adaptive-heavy-hitters}) and $\rAdaptiveStatQ$ (Theorem \ref{thm:r-adaptive-stat-q}), we can union bound the probability that any condition does not hold by at most $\delta$.
    Thus, in the following assume all conditions hold.

    We first argue that any coin chosen in $S_b$ must have bias close to the maximum in its bucket.
    For any bucket $b$, let $p_i$ denote the bias of the chosen coin in $S_b$.
    Then by \Cref{lemma:find-maximum-low-bias}, we claim
    \begin{equation}
        \label{eq:s-b-max-bias}
        |p_i - p_{\max, b}| \leq \varepsilon
    \end{equation}
    Otherwise, by \Cref{lemma:find-maximum-low-bias}, the set of coins with bias $p_i < p_{\max, b} - \varepsilon$ are chosen with probability at most $\delta_{\max} \leq \frac{B}{18 K}$.
    Thus, any coin in $S_b$ must have bias at least $p_{\max, b} - \varepsilon$.
    This holds without any assumption on $C \leq K$.

    Now, we prove the correctness of our algorithm.
    Fix a bucket $b$ and consider the distribution $\distribution_{\max, b}$.
    Suppose there is no $9 \varepsilon$ pseudo-maximum coin in bucket $b$.
    We hope to prove that no coins from this bucket are returned in the final output.
    First, if $S_{b}$ is empty we do not include any coins from the bucket $b$ as desired. 
    Otherwise, if $S_b$ is non-empty, we show that $S_b$ will not be included in the output set $S$ in \Cref{line:r-k-pseudo-maximum-identifier:trim-sets}. 
    Thus, assume $S_{b}$ is non-empty.
    We compute $\hat{p}_{\max, b}$ such that
    \begin{align*}
        |p_{\max} - \hat{p}_{\max, b}| &\geq |p_{\max} - p_{\max, b}| - |p_{\max, b} - p_i| - |p_i - \hat{p}_{\max, b}| \\
        &> 9 \varepsilon - \varepsilon - \varepsilon \\
        &> 7 \varepsilon,
    \end{align*}
    where $|p_{\max, b} - p_i| \leq \varepsilon$ by \Cref{eq:s-b-max-bias} and $\hat{p}_{\max, b}$ is an $\varepsilon$-accurate estimate of $p_i$.
    Consider now the coin with maximum bias $p_{\max}$ say in bucket $b^*$.
    This coin has mass in $\distribution_{\max, b^*}$ at least
    \begin{equation*}
        \frac{(1 - \delta)}{M_b} \geq (1 - \delta) \frac{2B}{3K} \geq \frac{B}{2K},
    \end{equation*}
    so we have $S_{b^*} \neq \emptyset$.
    Again, let $i$ denote the randomly sampled coin from $S_{b^*}$, so
    \begin{equation}
        \label{eq:max-est-bias-c-leq-k}
        |\hat{p}_{\max, b^*} - p_{\max}| \leq |\hat{p}_{\max, b^*} - p_i| + |p_i - p_{\max}| < 2 \varepsilon.
    \end{equation}
    Combining the above equations, we have
    \begin{equation*}
        \hat{p}_{\max, b} < \hat{p}_{\max, b^*} - 5 \varepsilon \leq \hat{p}_{\max} - 5 \varepsilon,
    \end{equation*}
    so no coin from $S_{b}$ is included in the output set, as desired.

    Otherwise, suppose there is a $9 \varepsilon$ pseudo-maximum coin in bucket $b$.
    By \Cref{eq:s-b-max-bias}, no coin with bias $p_i < p_{\max, b} - \varepsilon$ will be in $S_{b}$.
    In particular, any coin in $S_b$ is a $10 \varepsilon$ pseudo-maximum.
    Thus, even if $S_b \subset S$, we have every coin in $S$ is a $10 \varepsilon$ pseudo-maximum.
    
    It remains to show that if there is a $\frac{c_0 \varepsilon}{9} \sqrt{\frac{2 B \min(p_{\min}, 1 - p_{\max})}{K \log(2N/\delta)}}$ pseudo-maximum, this coin is included in $S$.
    First, we claim that it is included in $S_b$.
    Applying Lemma \ref{lemma:prob-lower-bias-wins} with $m = \frac{27}{\varepsilon^2} \log \frac{2N}{\delta}$ and $M_b \leq \frac{3K}{2B}$, we see that any coin satisfying
    \begin{equation*}
        p_i \geq p_{\max} - \frac{c_0 \varepsilon}{9} \sqrt{\frac{2 B \min(p_{\min}, 1 - p_{\max})}{K \log(2N/\delta)}} \geq p_{\max, b} - c_0 \frac{\min(p_i, 1 - p_{\max, b})}{\sqrt{M_b m}}
    \end{equation*}
    is an $\varepsilon$ pseudo-maximum for large enough $N$ and has mass at least $\frac{1}{4 M_b} \geq \frac{B}{6 K}$.
    Therefore, it is included in $S_{b}$ by $\rAdaptiveHeavyHitters$.
    Furthermore, if $p_i$ is a $\frac{c_0 \varepsilon}{9} \sqrt{\frac{2 B \min(p_{\min}, 1 - p_{\max})}{K \log(2N/\delta)}}$ pseudo-maximum among all $[N]$ coins, it must be such a pseudo-maximum in its bucket, and therefore is included in $S_b$.
    Finally, we show that $S_b$ should be included in the final set $S$.
    
    To see this, let $i$ be the randomly sampled coin from $S_{b}$ and $i^*$ be the randomly samped coin from $S_{b^*}$ containing $p_{\max}$. Then
    \begin{align*}
        |\hat{p}_{\max} - \hat{p}_{\max, b}| &\leq |\hat{p}_{\max} - p_{i^*}| + |p_{i^*} - p_{\max}| + |p_{\max} - p_{\max, b}| + |p_{\max, b} - p_i| + |p_i - \hat{p}_{\max, b}| \\
        &\leq \varepsilon + \varepsilon + \frac{c_0 \varepsilon}{9} \sqrt{\frac{2 B \min(p_{\min}, 1 - p_{\max})}{K \log(2N/\delta)}} + \varepsilon + \varepsilon \\
        &\leq 4 \varepsilon + \frac{c_0 \varepsilon}{9} \sqrt{\frac{2 B \min(p_{\min}, 1 - p_{\max})}{K \log(2N/\delta)}} \\
        &\leq 5 \varepsilon,
    \end{align*}
    for appropriately large $N$, and therefore $S_b \subset S$, as desired.
    By our choice of $B, K$, we have that $S$ includes all coins that are $\frac{c_0 \varepsilon}{9} \sqrt{\frac{\min(p_{\min}, 1 - p_{\max})}{6 \log(2N/\delta) \log(6K/\delta)}} = \bigTh{\varepsilon \sqrt{\frac{\min(p_{\min}, 1 - p_{\max})}{\log(N/\delta) \log(K/\delta)}}}$ pseudo-maxima.

    Next, we argue that $|S|$ is not too small. 
    Let $i$ be a pseudo-maximum coin.
    If $p_i = p_{\max, b}$ is the coin with maximum bias in its bucket, then $\distribution_{\max, b}$ samples $i$ with probability at least $\frac{(1 - \delta)}{M_b} \geq \frac{B}{2K}$ so that $i \in S_b$.
    Furthermore, if $i_b$ denotes the randomly sampled coin in $S_b$
    \begin{equation*}
        \hat{p}_{\max, b} \geq p_{i_b} - \varepsilon \geq p_{i} - 2 \varepsilon \geq p_{\max} - 3 \varepsilon \geq \hat{p}_{\max} - 5 \varepsilon,
    \end{equation*}
    so $i \in S_b \subset S$.
    In particular, $S$ includes at least every $\varepsilon$ pseudo-maximum coin such that $p_i = p_{\max, b}$.
    In other words, $|S|$ is at least the number of buckets containing an $\varepsilon$ pseudo-maximum coin.
    Recall that we conditioned on
    \begin{equation*}
        M_b \leq \frac{3K}{2B} \leq 18 \log \frac{6K}{\delta},
    \end{equation*}
    so applying the pigeon-hole principle, there are at least $\frac{C}{18 \log(6K/\delta)}$ buckets with a $\varepsilon$ pseudo-maximum coin.

    Finally, let $b$ be a bucket such that $S_b \neq \emptyset$ and $S_b \subset S$.
    Let $i \in S_b$ and $p_{i, b}$ be the representative chosen from $S_b$.
    Then, by \Cref{eq:s-b-max-bias},
    \begin{equation*}
        |p_i - \hat{p}_{\max}| \leq |p_i - p_{i, b}| + |p_{i, b} - \hat{p}_{\max, b}| + |\hat{p}_{\max, b} - \hat{p}_{\max}| \leq 7 \varepsilon.
    \end{equation*}
    Note that this holds with no assumption on $C$, the number of $\varepsilon$ pseudo-maximum coins.

    \paragraph{Replicability}
    Next we prove that $\rKPseudoMaximumIdentifier$ is replicable.
    Consider any two executions of $\rKPseudoMaximumIdentifier$.
    Due to shared randomness, in Line \ref{line:r-k-pseudo-maximum-identifier:bucket} the indices $[N]$ are split into $B$ buckets in the same way.
    Therefore, for all $b \in [B]$, the distribution $\distribution_{\max, b}$ is identical across both executions of $\rKIdentifier$.
    Since $\rAdaptiveStatQ$ and $\rAdaptiveHeavyHitters$ are $\frac{\rho}{2B}$ replicable, we conclude that $\rKPseudoMaximumIdentifier$ is replicable by the union bound.
    Note that the choice of $i \in S_b$ is identical across both executions.

    \paragraph{Sample Complexity}
    Finally, we bound the sample complexity.
    Fix a given bucket $b \in [B]$.
    Theorem \ref{thm:r-adaptive-heavy-hitters} states that $\rAdaptiveHeavyHitters$ has expected sample complexity 
    \begin{equation*}
        \bigO{\frac{K^3}{B^3} \frac{B}{\rho} \log \frac{B K}{\delta B}} = \bigO{\frac{K^3}{B^2 \rho} \log \frac{K}{\delta}}.
    \end{equation*} 
    Then, since the sample complexity of $\findMaximum$ is
    \begin{equation*}
        \bigO{\frac{N}{B \varepsilon^2} \log \frac{K N}{B}},
    \end{equation*}
    we conclude that the overall expected sample complexity (summed over $B$ buckets) is
    \begin{equation*}
        \bigO{B \cdot \frac{K^3}{B^2 \rho} \log \frac{K}{\delta} \cdot \frac{N}{B \varepsilon^2} \log \frac{K N}{B}} = \bigO{\frac{K N}{\rho \varepsilon^2} \log^3 \frac{K}{\delta} \log \frac{K N}{\delta}} = \bigO{\frac{K N}{\rho \varepsilon^2} \log^{4} \frac{N}{\delta}},
    \end{equation*}
    using that $\frac{K}{B} = \bigO{\log \frac{K}{\delta}}$.
    Theorem \ref{thm:r-adaptive-stat-q} states that $\rAdaptiveStatQ$ has expected sample complexity $\bigO{\frac{1}{\rho \varepsilon^2} \log \frac{B}{\delta}}$ but this does not affect the overall sample complexity.
\end{proof}

Above, we presented an algorithm that efficiently recovers pseudo-maximum coins if there are few of them.
Below, we use random sampling to efficiently recover pseudo-maximum coins if there are many.
Combining these ideas we obtain an algorithm obtaining pseudo-maximum coins.
We are now ready to prove \Cref{thm:r-pseudo-maximum-identifier-formal}.

\ThmPseudoMaximumIdentifier*

Again, we begin with a high level overview of the algorithm.
Let $T = N^{1/3} K^{2/3} = K \left( \frac{N}{K} \right)^{1/3}$.
Our algorithm begins by calling $\rKPseudoMaximumIdentifier\left( N, T, \frac{\varepsilon}{10} \right)$, obtaining a set $S_0$ and an estimate $\hat{p}_{\max, 0}$ that with high probability is an estimate of the maximum bias of any coin in $S_0$.
Next, we sample a set $I \subset [N]$ by including each element independently with probability $\frac{K}{T}$ thus obtaining a set of size roughly $\frac{NK}{T}$.
Finally, we run $\rAdaptiveStatQ$ to estimate the bias of each coin in $I$ (denoted $\hat{p}_i$) and return a set containing:
\begin{enumerate}
    \item $S_0$ if $\hat{p}_{\max, 0} \geq \hat{p}_{\max, I} - \varepsilon$ where $\hat{p}_{\max, I}$ is the maximum empirical estimate from $I$.
    \item $I^* \subset I$ consisting of coins whose bias is $\varepsilon$-close to $\max(\hat{p}_{\max, 0}, \hat{p}_{\max, I})$.
\end{enumerate}

By the size of $|I|$, the expected sample complexity is
\begin{equation*}
    \frac{N^2 K^2}{T^2 \rho \varepsilon^2} + \frac{T N}{\rho \varepsilon^2} = \frac{N^{4/3} K^{2/3}}{\rho \varepsilon^2}
\end{equation*}
by our choice of $T$.
In particular, we can identify a single pseudo-maximum coin in $\frac{N^{4/3}}{\rho \varepsilon^2}$ expected sample complexity.

\paragraph{Few Pseudo-Maximum Coins: $C \leq T$}
We consider the case where $C$ is small.
Then, completeness follows immediately from Lemma \ref{lemma:r-k-pseudo-maximum-identifier} if $S_0 \subset S$.
To see this, we argue that when $C \leq T$, the maximum bias in $S_0$ is a good estimate of the maximum bias overall, i.e. $|\hat{p}_{\max, 0} - p_{\max}| < \frac{\varepsilon}{2}$.
Therefore, no coin in $I$ can have $\hat{p}_i > \hat{p}_{\max, 0} + \varepsilon$ and therefore $S_0 \subset S$.

To argue soundness, we simply observe that $\hat{p}_i$ is a good estimate of $p_i$ for all $i \in I$.
Since $p_i \geq \hat{p}_i - \frac{\varepsilon}{2}$, we can conclude $p_i \geq \hat{p}_{\max, 0} - \varepsilon$ implies $p_i$ is an $\varepsilon$ pseudo-maximum.
The soundness of every coin in $S_0$ is guaranteed by Lemma \ref{lemma:r-k-pseudo-maximum-identifier}.

\paragraph{Many Pseudo-Maximum Coins: $C \geq T$}

When $C \geq T$, we can argue that with high probability the expected number of pseudo-maximum coins in $I$ is large (at least $K$).
Let $F$ denote the set of $\varepsilon$ pseudo-maximum coins.
In this case, a randomly chosen subset $I$ will include many pseudo-maximum coins.
Completeness then follows as every coin in $|I \cap F|$ will satisfy $|\hat{p}_i - p_i|$ is small enough such that $\hat{p}_i$ is large enough for $i$ to be included in $S$.
Soundness follows from above since we only include coins with high empirical biases, and since many pseudo-maximum coins are included in $I$, any coin included in the final set has high true bias.

We present the algorithm and proof of Theorem \ref{thm:r-pseudo-maximum-identifier} below.

\IncMargin{1em}
\begin{algorithm}

\SetKwInOut{Input}{Input}\SetKwInOut{Output}{Output}\SetKwInOut{Parameters}{Parameters}
\Input{Sample access to $N$ Bernoulli Distributions $\distribution_i$ with parameter $p_i \in [0, 1]$.}
\Parameters{$\rho$ replicability and $\delta$ accuracy}
\Output{}

$K \gets \max\left(K, 6 \log \frac{3}{\delta} \right)$
\label{line:r-pseudo-maximum-identifier:set-k}

$T \gets N^{1/3} K^{2/3}$

$S_0, \hat{p}_{\max, 0} \gets \rKPseudoMaximumIdentifier\left(N, T, \frac{\varepsilon}{10}, \frac{\rho}{2}, \frac{\delta}{4}\right)$

$I \subset [N]$ is sampled randomly by including $i \in I$ with probability $\frac{2K}{T}$.

\lIf{$|I| > \frac{2KN}{T}$}{\Return $\emptyset$}

$\hat{p}_{\max, I} \gets 0$

\For{$i \in I$}{
    $\hat{p}_i \gets \rAdaptiveStatQ(\distribution_i, \phi, \frac{\rho}{2I}, \varepsilon, \frac{\delta}{4 I})$

    $\hat{p}_{\max, I} \gets \max(\hat{p}_i, \hat{p}_{\max, I})$
}

$S \gets \emptyset$

\If{$\hat{p}_{\max, 0} \geq \hat{p}_{\max, I} - 3 \varepsilon$}{
    $S \gets S_0$
}

\For{$i \in I$}{
    $S \gets S \cup \set{i}$ if $\hat{p}_i \geq \max(\hat{p}_{\max, 0}, \hat{p}_{\max, I}) - \varepsilon$
}

\Return $S$

\caption{$\rPseudoMaximumIdentifier(N, K, \varepsilon, \rho, \delta)$}
\label{alg:pseudo-maximum-identifier}

\end{algorithm}
\DecMargin{1em}

\begin{proof}[Proof of Theorem \ref{thm:r-pseudo-maximum-identifier}]
    We condition on the following events:
    \begin{enumerate}
        \item $\rKPseudoMaximumIdentifier$ satisfies the output constraints of Theorem \ref{lemma:r-k-pseudo-maximum-identifier}.
        \item $\rAdaptiveStatQ$ outputs $|\hat{p}_i - p_i| \leq \varepsilon$ for all $i \in I$.
        \item $|I| \leq \frac{3KN}{T}$.
    \end{enumerate}
     Note that $\E{|I|} = \frac{2 K N}{T}$ and $\Pr(|I| > \frac{3 K N}{T}) < \exp \left(- \frac{K N }{6 T}\right) < \exp \left( - \frac{K}{6}\right) < \frac{\delta}{4}$.
     Then, by the union bound all events occur with probability at least $1 - \frac{3 \delta}{4}$. 

    We begin with correctness.
    
    Let $C$ denote the number of $\varepsilon$ pseudo-maximum coins.
    We consider the following two cases:
    \begin{enumerate}
        \item {\bf Case 1: $C \leq T$}.
        Note Lemma \ref{lemma:r-k-pseudo-maximum-identifier} holds.
        Then, $\hat{p}_{\max, 0} \geq p_{\max} - \frac{\varepsilon}{5}$ by Lemma \ref{lemma:r-k-pseudo-maximum-identifier} (Equation \ref{eq:max-est-bias-c-leq-k} applied with $\varepsilon/10$).
        Let $p_{\max, I} \leq p_{\max}$ denote the maximum true bias in $I$.
        Then, we have
        \begin{equation*}
            \hat{p}_{\max, 0} \geq p_{\max} - \frac{\varepsilon}{5} \geq p_{\max, I} - \frac{\varepsilon}{5} > \hat{p}_{\max, I} - \frac{6}{5} \varepsilon,
        \end{equation*}
        so $S_0 \subset S$.
        By Lemma \ref{lemma:r-k-pseudo-maximum-identifier}, every $i \in S_0 \subset S$ is an $\varepsilon$ pseudo-maximum since
        \begin{equation*}
            p_i \geq \hat{p}_{\max, 0} - \frac{7 \varepsilon}{10} \geq p_{\max} - \frac{9 \varepsilon}{10}.
        \end{equation*}
        Furthermore, for any $i \in S \setminus S_0$, we have $i \in S \cap I$ so that
        \begin{equation*}
            p_i \geq \hat{p}_i - \varepsilon \geq \hat{p}_{\max, 0} - 2 \varepsilon > p_{\max} - \frac{16}{5} \varepsilon.
        \end{equation*}
        So any $i \in S$ is a $\frac{11}{5} \varepsilon$ pseudo-maximum, thus satisfying soundness.
        Completeness is immediate from Lemma \ref{lemma:r-k-pseudo-maximum-identifier} since $C \leq T$ and $S_0 \subset S$.

        \item {\bf Case 2: $C > T$}.
        We additionally condition on the event that $I$ contains at least $K$ $\varepsilon$-pseudo-maximum coins. 
        Following a similar argument to Lemma \ref{lemma:num-coins-in-bucket}, if we include each coin in $I$ with probability $\frac{2 K}{T}$.
        Then, $\E{|I \cap C|} \geq C \frac{2 K}{T} > 2 K$.
        Furthermore, applying a standard Chernoff bound,
        \begin{equation*}
            \Pr \left( |I \cap C| < K \right) < \exp \left( - \frac{K}{6} \right) < \frac{\delta}{4}
        \end{equation*}
        by our choice of $K$ in Line \ref{line:r-pseudo-maximum-identifier:set-k}.
        
        From Lemma \ref{lemma:r-k-pseudo-maximum-identifier} we have $|\hat{p}_{\max, 0} - p_i| \leq \frac{7}{10} \varepsilon$ for any $i \in S$.
        Therefore $S_0 \subset S$ implies that for all $i \in S$,
        \begin{equation*}
            p_i \geq \hat{p}_{\max, 0} - \frac{7}{10} \varepsilon \geq \hat{p}_{\max, I} - \frac{37}{10} \varepsilon \geq p_{\max, I} - \frac{47}{10} \varepsilon \geq p_{\max} - \frac{57}{10} \varepsilon.
        \end{equation*}
        Thus, every $i \in S_0$ is a $6 \varepsilon$ pseudo-maximum.
        Following similar arguments as above, we have every $i \in I \cap S$ is a $4 \varepsilon$ pseudo-maximum, as
        \begin{equation*}
            p_i \geq \hat{p}_i - \varepsilon \geq \hat{p}_{\max, I} - 2 \varepsilon \geq p_{\max, I} - 3 \varepsilon \geq p_{\max} - 4 \varepsilon.
        \end{equation*}
        Thus, soundness holds.

        For completeness, we observe that for any $i \in I \cap F$,
        \begin{align*}
            \hat{p}_i &\geq p_i - \varepsilon \geq p_{\max} - 2 \varepsilon \geq p_{\max, I} - 2 \varepsilon \geq \hat{p}_{\max, I} - 3 \varepsilon \\
            \hat{p}_i &\geq p_i - \varepsilon \geq p_{\max} - 2 \varepsilon \geq p_{\max, 0} - 2 \varepsilon \geq \hat{p}_{\max, 0} - 3 \varepsilon
        \end{align*}
        so that $I \cap F \subset S$ and $|S| \geq |I \cap F| \geq K$.
    \end{enumerate}

    \proofparagraph{Replicability and Sample Complexity}
    Replicability follows from a union bound over the replicability of $\rKPseudoMaximumIdentifier$ (Lemma \ref{lemma:r-k-pseudo-maximum-identifier}) and $\rAdaptiveStatQ$ (Theorem \ref{thm:r-adaptive-stat-q}).
    
    By Lemma \ref{lemma:r-k-pseudo-maximum-identifier}, the expected sample complexity of $\rKPseudoMaximumIdentifier$ is at most
    \begin{equation*}
        \bigO{\frac{TN}{\rho \varepsilon^2} \log^{4} \frac{N}{\delta}}.
    \end{equation*}
    Then, by Theorem \ref{thm:r-adaptive-stat-q} and $|I| \leq \frac{2 K N}{T}$, the expected sample complexity over all iterations of $\rAdaptiveCoinTester$ is
    \begin{equation*}
        \bigO{\frac{K^2 N^2}{T^2 \rho \varepsilon^2} \log^3 \frac{N}{\delta}},
    \end{equation*}
    where the logarithmic term is cubed to consider the case $K < 6 \log(4/\delta)$.
    Summing the two terms and applying our choice of $T$, we obtain the desired overall sample complexity
    \begin{equation*}
        \bigO{\frac{N^{4/3} K^{2/3}}{\rho \varepsilon^2} \log^{4} \frac{N}{\delta}}.
    \end{equation*}
\end{proof}

\subsection{Approximate Replicability and Sample Complexity}
In the previous section we considered relaxing in some sense the correctness constraint of the problem, requiring it only return $K$ maximally biased coins. In this section, we will instead relax the requirement of replicability, and only require that the set of output coins across two runs match \textit{approximately}.


\approxreplicability*

Observe that $(\rho, 0)$-replicability is exactly equivalent to $\rho$-replicability.
Our definition of approximate replicability requires that with high probability, there are few disagreements between the output sets of $\innerAlg$ given two samples $S_p, S_p'$, whereas the strict definition requires that there are no disagreements between the output sets. We give an approximately replicable algorithm for the $N$-Coin problem with two major improvements in sample and computational efficiency. First, as the approximation parameter $R$ increases, the sample complexity decreases as $\frac{N^2}{R}$, smoothly interpolating between the best known replicable and optimal non-replicable dependence. Second, even for $R=O(1)$, we show approximate replicability can be achieved with only \textit{logarithmic} overhead in $\rho^{-1}$.


\begin{theorem}[\Cref{thm:r-approx-n-coin-alg}]
    \label{thm:r-approx-n-coin-alg-formal}
    Algorithm \ref{alg:r-approx-n-coin-tester} is a $(\rho, R)$-replicable algorithm solving the $N$-coin problem (Definition \ref{def:n-coin-problem}) with expected sample complexity $\bigO{\frac{q_0 N^2 \log(N/\delta) \log(1/\rho)}{(q_0 - p_0)^2 R}}$.
\end{theorem}

We note also that the worst case sample complexity can be bounded via Markov's inequality as
\begin{equation*}
    \bigO{\frac{q_0 N^2}{(q_0 - p_0)^2 R \rho} \log \frac{N}{\delta} \log \frac{1}{\rho}}.
\end{equation*}
Later in this section, we also give a lower bound.

\begin{restatable}{theorem}{ThmApproxReplicabilityLB}
    \label{thm:r-approx-n-coin-lower-adaptive}
    Let $p_0 < q_0$ with $p_0, q_0 \in \left( \frac{1}{4}, \frac{3}{4} \right)$.
    Let $1 \leq R \leq N$.
    Any non-adaptive $(1/20, R)$-replicable algorithm for the $N$-coin problem requires sample complexity $\bigOm{\frac{N^2}{(q_0 - p_0)^2 R \log^3 N}}$.
\end{restatable}



We begin with a high level overview of the algorithm.

\paragraph{Algorithm Overview}
Following the replicable algorithm for the $N$-coin problem, our algorithm for approximate replicability simply invokes $N$ parallel instances of the $\rAdaptiveCoinTester$ single coin testing algorithm.
However, instead of requiring $\frac{\rho}{N}$ replicability of each instance, we only require $\frac{R}{N}$ replicability.
Therefore, in expectation there are at most $R$ coins where the output sets of the algorithm $\rApproxMultiCoinTester$ given two samples disagrees.
Using standard concentration bounds, since the number of discrepancies is the sum of independent Bernoulli random variables, we can show that the number of discrepancies does not exceed its expectation (significantly) with high probability, where we incur only a polylogarithmic factor in $\frac{1}{\rho}$.

We now present the algorithm and proof.
Let $\multiCoinTester$ denote the standard non-replicable algorithm solving the $N$-coin problem with sample complexity $\bigO{\frac{N}{(q_0 - p_0)^2} \log \frac{N}{\delta}}$.

\IncMargin{1em}
\begin{algorithm}

\SetKwInOut{Input}{Input}\SetKwInOut{Output}{Output}\SetKwInOut{Parameters}{Parameters}
\Input{Sample access to $N$ Bernoulli Distributions $\distribution_i$ with parameter $p_i \in [0, 1]$. Bias thresholds $p_0 < q_0$ with $\varepsilon = q_0 - p_0$.}
\Parameters{$R$-coin replicability and $\delta$-accuracy}
\Output{$S$ such that $i \in S$ if $p_i \geq q_0$, and $i \notin S$ if $p_i \leq p_0$.}

$\delta \gets \min(\delta, \rho/4)$

$\mathrm{SampleLimit} \gets \frac{100 q_0 N^2}{(q_0 - p_0)^2 R \rho} \log \frac{N}{\delta} \log \frac{1}{\rho}$

\If{Total number of samples collected exceeds $\mathrm{SampleLimit}$}{
    \Return $\multiCoinTester (N, p_0, q_0, \delta/2)$
}

$S \gets \emptyset$

\For{$i \in [N]$}{
    $\rho' \gets \frac{R}{2 N \left( 1 + 2 \log \frac{4}{\rho} \right)}$
    
    \If{$\rAdaptiveCoinTester \left(\distribution_i, p_0, q_0, \rho', \frac{\delta}{2N}\right) = \accept$}{
        $S \gets S \cup \set{i}$
    }
}

\Return $S$

\caption{$\rApproxMultiCoinTester(N, p_0, q_0, R, \delta)$} 
\label{alg:r-approx-n-coin-tester}

\end{algorithm}
\DecMargin{1em}

\begin{proof}[Proof of Theorem \ref{thm:r-approx-n-coin-alg-formal}]
    We prove that Algorithm $\rApproxMultiCoinTester$ is correct, approximate replicable, and has the stated sample complexity.
    Throughout the proof, we assume no sampling errors occur, as is shown in \Cref{claim:sample-error-bound}.
    We obtain correctness by the union bound, over $N$ invocations of $\rAdaptiveCoinTester$ and one invocation of $\multiCoinTester$.

    \proofparagraph{Sample Complexity}
    We now bound the expected sample complexity of Algorithm $\rApproxMultiCoinTester$.
    The expected sample complexity of any given invocation of $\rAdaptiveCoinTester$ is,
    \begin{equation*}
        \bigO{\frac{q_0}{(q_0 - p_0)^2 \rho'} \log \frac{N}{\delta}} = \bigO{\frac{q_0 N}{(q_0 - p_0)^2 R} \log \frac{N}{\delta} \log \frac{1}{\rho}}
    \end{equation*}
    so that over $N$ invocations, the expected sample complexity is,
    \begin{equation*}
        \bigO{\frac{q_0 N^2}{(q_0 - p_0)^2 R} \log \frac{N}{\delta} \log \frac{1}{\rho}}
    \end{equation*}
    Now, by Markov's inequality, the probability that the total sample complexity taken across all invocations exceeds 
    \begin{equation*}
        \bigO{\frac{q_0 N^2}{(q_0 - p_0)^2 R \rho} \log \frac{N}{\delta} \log \frac{1}{\rho}}
    \end{equation*}
    is at most $\rho$, which yields the expected sample complexity as $\multiCoinTester$ requires fewer samples than the expected sample complexity of $N$ invocations of $\rAdaptiveCoinTester$.
    Furthermore, note that this gives a bound on the worst case sample complexity as the sample complexity of $\multiCoinTester$ does not exceed the sample limit.

    \proofparagraph{Replicability}
    We claim that Algorithm $\rApproxMultiCoinTester$ is $(\rho, R)$-replicable. 
    In particular, we compute the expected number of replicated experiments.
    Consider one execution of the algorithm, with samples $S_p$.
    By Markov's inequality, for a sufficiently large constant $C$, the probability that the sample limit is reached is at most $\frac{\rho}{4}$. 
    Therefore, we may assume that the sample limit is not reached.
    Conditioned on this, we observe that $\rAdaptiveCoinTester$ is replicable whenever the sampled threshold $r_i$ is not within $\rho'$ of the true bias $p_i$ (and no sampling errors occur).
    In particular, let $R_i$ be a random variable denoting whether coin $i$ is replicable, so that $\set{R_i}$ is a set of $N$ independent Bernoulli random variables with parameter at most $2 \rho'$.
    Therefore,
    \begin{equation*}
        \E{\sum R_i} = N \rho' = \frac{R}{1 + 2 \log \frac{4}{\rho}}
    \end{equation*}
    Following standard concentration bounds,
    \begin{equation*}
        \Pr \left( \sum R_i > R \right) =  \Pr \left( \sum R_i > \left(1 + 2 \log \frac{4}{\rho} \right) \E{\sum R_i} \right) < \exp \left( - R \log \frac{4}{\rho} \right) \leq \frac{\rho}{4}
    \end{equation*}
    
    Therefore, with probability at most $\frac{\rho}{2}$, any execution of the algorithm has at most $R$ coins where the true bias is more than $\rho'$ from the random threshold $r_i$.
    Union bounding over two executions of the algorithm give us $(\rho, R)$-replicability.
\end{proof}

Next, we prove a lower bound against non-adaptive algorithms, using \Cref{thm:n-coin-lower-const-delta}.

\ThmApproxReplicabilityLB*

\paragraph{Proof Overview}
To obtain our lower bound, we show that given a $(\rho, R)$-replicable algorithm $\mathcal{A}$ on $m$ samples, it is possible to construct a $\Theta(1)$-replicable algorithm for the $\frac{N}{R}$-Coin Problem using roughly $\Theta(m)$ samples. 
This allows us to appeal to \Cref{thm:n-coin-lower-const-delta} (lower bound for the standard replicable $N$-Coin Problem) to prove the result.
At a high level, a $(\rho, R)$-replicable algorithm disagrees on at most $R$ coins over two independent samples with probability at least $1 - \rho$.
Thus, we can design a $\rho$-replicable algorithm for the $\frac{N}{2 R}$-coin problem by hiding the $\frac{N}{2 R}$ coins that we care about being replicable on among $N$ coins.
Note that this creates $2 R$ instances of the $\frac{N}{2 R}$-coin problem. 
Thus, on average, we expect fewer than one of the $R$ disagreements to lie in the $\frac{N}{2R}$ coins that we care about, thus obtaining a replicable algorithm for the $\frac{N}{2R}$-coin problem.

In order to successfully hide the coins we want to be replicable, we need to ensure that all input distributions are indistinguishable from one another.
Since we have an explicit hard input distribution in \Cref{thm:n-coin-lower-const-delta} (drawing all biases uniformly $p_i \in (p_0, q_0$), we can take a hard instance of the $\frac{N}{2R}$-coin problem and extend it to a hard instance of the $N$-coin problem by adding coins with biases drawn uniformly from $(p_0, q_0)$.

Specifically, we will design a $(1/20)$-replicable algorithm for input distributions where $p_i \in (p_0, q_0)$ is sampled independently for all $i$ as described in the lower bound of \Cref{thm:n-coin-lower-const-delta}.
Suppose there is a non-adaptive $(1/20, R)$-replicable algorithm $\innerAlg$ for the $N$-coin problem with sample complexity $m$.
We construct a non-adaptive $\frac{1}{20}$-replicable algorithm $\outerAlg$ for the $\frac{N}{20R}$-coin problem given inputs from this distribution.

In particular, we draw $20R - 1$ input distributions independently from same hard instance (all biases randomly from $(p_0, q_0)$).
Combined with the input distribution given to $\outerAlg$, the input distribution is drawn from the hard distribution on $N$ coins.
Whenever $\innerAlg$ requests a sample from a coin in the input distribution given to $\outerAlg$, we draw a sample.
Whenever $\innerAlg$ requests a sample from any other coin, we can simulate a sample as we know the bias of the coin.
Furthermore, we randomly permute the indices of the coins so that $\innerAlg$ cannot distinguish which coins in its input distribution are in the input distribution of $\outerAlg$.
Therefore, the sample complexity of $\outerAlg$ is roughly $\frac{m}{20R}$, since $\innerAlg$ solves $20R$ instances of $\outerAlg$.

Upon receiving the outputs of $\innerAlg$, we output the subset of outputs relevant to the input distribution of $\outerAlg$.
With high probability, over two samples the outputs of $\innerAlg$ disagree in at most $R$ places.
On expectation, there are $\frac{1}{20}$ disagreements on the outputs relevant to $\outerAlg$.
In particular, this shows that $\outerAlg$ is $\frac{1}{20}$-replicable.
Since Theorem \ref{thm:n-coin-lower-const-delta} states that $\outerAlg$ has sample complexity at least $\frac{N^2}{R^2 (q_0 - p_0)^2}$, we have shown that $\innerAlg$ has sample complexity at least $\frac{N^2}{R (q_0 - p_0)^2}$. 

We now give the proof of Theorem \ref{thm:r-approx-n-coin-lower-adaptive}.

\begin{proof}
    Let $\delta < \frac{1}{20}$.
    As before, we restrict our adversary to construct a hard instance against the at least $N/2$ coins with sample complexity $m_i \leq \frac{2 m}{N}$.
    Suppose for contradiction there is a $(1/20, R)$-replicable non-adaptive algorithm $\innerAlg$ solving the $N$-coin problem with sample complexity $\littleO{\frac{N^2}{(q_0 - p_0)^2 R \log^3 N}}$. 

    \subsubsection*{Step 1: Constructing a Hard Input Distribution}

    We show our lower bound for a specific distribution of problem instances.
    In particular, following the lower bound of Theorem \ref{thm:n-coin-lower-const-delta}, consider the following distribution of instances.
    The adversary will independently for all $i \in [N]$ choose a coin with bias $p_i \in [p_0, q_0]$ uniformly at random.
    Let $\hypotheses$ denote this distribution of input instances.
    As we have shown in Theorem \ref{thm:n-coin-lower-const-delta}, any $\frac{1}{10}$-replicable algorithm solving the $N$-coin problem with the input distribution drawn from $\hypotheses(N)$ has sample complexity at least
    \begin{equation*}
        \bigOm{\frac{N^2}{(q_0 - p_0)^2 \rho^2 \log^{3} N}} = \bigOm{\frac{N^2}{(q_0 - p_0)^2 \log^{3} N}}.
    \end{equation*}

    \subsubsection*{Step 2: Constructing a $\rho$-replicable Algorithm}
    
    Consider the following $\frac{1}{10}$-replicable algorithm $\outerAlg$ for the $\frac{N}{20 R}$ coin problem.
    Following Theorem \ref{thm:n-coin-lower-const-delta}, we have that the sample complexity of $\outerAlg$ must be at least
    \begin{equation*}
        \bigOm{\frac{N^2}{(q_0 - p_0)^2 R^2 \log^3 N}}.
    \end{equation*}
    
    \IncMargin{1em}
    \begin{algorithm}
    
    \SetKwInOut{Input}{Input}\SetKwInOut{Output}{Output}
    \SetKwInOut{Parameters}{Parameters}
    \Input{Sample access to $\frac{N}{20R}$ Bernoulli Distributions $\distribution_i$ with parameter $p_i \in [0, 1]$. Bias thresholds $p_0 < q_0$ with $\varepsilon = q_0 - p_0$. $\rho$-replicable Algorithm $\innerAlg$.}
    \Parameters{$\frac{1}{10}$ replicability and $\delta$ accuracy}
    \Output{$S$ such that $i \in S$ if $p_i \geq q_0$, and $i \notin S$ if $p_i \leq p_0$.}
    
    $h^{(1)}, \dotsc, h^{(20R - 1)}$ are $20R - 1$ instances drawn independently from $\hypotheses\left( \frac{N}{20R} \right)$

    $h^{(20R)}$ is the vector of biases given by $\distribution_{1}, \dotsc, \distribution_{N/20R}$

    $\pi$ is a uniformly random permutation on $20R$ elements

    $h \gets \set{h^{(\pi(i))}}_{i = 1}^{20R}$ is a vector of $N$ biases

    $\innerAlg$ gets samples access to $N$ Bernoulli distributions where the biases are given by $\distribution_i \sim \BernD{h_i}$.

    \Return $\innerAlg\left( N, R, p_0, q_0, \frac{1}{20}, \delta \right)[\pi(20R)]$, the outputs of $\innerAlg$ on $h^{(20R)}$. 
    
    \caption{$\outerAlg\left(\frac{N}{20R}, p_0, q_0, \frac{1}{10}, \delta\right)$} 
    \label{alg:r-approx-n-coin-tester-lb}
    
    \end{algorithm}
    \DecMargin{1em}

    We now prove that $\outerAlg$ is a $\frac{1}{10}$-replicable algorithm with efficient sample complexity, therefore deriving a contradiction.

    \subsubsection*{Step 3: Obtaining a Contradiction for Theorem \ref{thm:n-coin-lower-const-delta}}

    First, we show the sample complexity of $\outerAlg$.
    Whenever $\innerAlg$ requests a sample from $h^{(j)}$ for $j < 20R$, we can generate a sample from $h^{(j)}$ using the internal randomness of $\outerAlg$, since the bias vector $h^{(j)}$ is sampled by $\outerAlg$ and therefore any sample from $h^{(j)}$ can be simulated using the random coins of $\outerAlg$.
    Whenever $\innerAlg$ requests a sample from $h^{(20R)}$, we draw a sample from the corresponding distribution $\distribution_{i}$.
    The sample complexity of $\innerAlg$ is by assumption
    \begin{equation*}
        \littleO{\frac{N^2}{(q_0 - p_0)^2 R \log^3 N}}.
    \end{equation*}
    Then, since each coin has sample complexity $m_i \leq \frac{2 m}{N}$, the sample complexity of $\outerAlg$ is at most
    \begin{equation*}
        \littleO{\frac{N^2}{(q_0 - p_0)^2 R^2 \log^3 N}},
    \end{equation*}
    since each coin has at most twice the average number of samples $\littleO{\frac{N}{(q_0 - p_0)^2 R \log^3 N}}$ and each $h^{(j)}$ distribution consists of $\frac{N}{20R}$ coins.
    Therefore, it remains to prove that $\outerAlg$ is $\frac{1}{10}$-replicable algorithm solving the $\frac{N}{20R}$-coin problem to obtain a contradiction to Theorem \ref{thm:n-coin-lower-const-delta}.

    We begin by showing correctness.
    With probability at least $1 - \delta$, $\innerAlg$ accepts all coins with bias $p_i \geq q_0$ and rejects all coins with bias $p_i \leq p_0$.
    Since $h^{(20R)}$ is a subset of the all $N$ coins, $\outerAlg$ also accepts all coins with bias $p_i \geq q_0$ and rejects all coins with bias $p_i \leq p_0$, whenever $\innerAlg$ does.

    It remains to show that $\outerAlg$ is $\rho$-replicable. 
    Let $h^{(20R)}$ be some distribution on $\frac{N}{20R}$ coins and $S, S'$ be two samples drawn from this distribution.
    Then, over both executions of the algorithm $\outerAlg$, since random strings are shared, the same distribution $h$ is given to $\innerAlg$.
    Since $\innerAlg$ is $\left(\frac{1}{20}, R \right)$-replicable, then with probability at least $\frac{19}{20}$, the two output strings $\innerAlg(S; r), \innerAlg(S'; r)$ differ in at most $R$ indices. 
    Let $D^{(i)}$ be a binary random variable indicating whether one of these differing indices correspond to the input distribution $h^{(i)}$.
    Again, since $h^{(i)}$ are drawn i.i.d. from $\hypotheses$ and $\pi$ is a uniformly random permutation, the binary random variables $D^{(i)}$ are identically distributed.
    In particular, conditioned on the event that $\innerAlg(S; r), \innerAlg(S'; r)$ differ in at most $R$ indices, we have
    \begin{equation*}
        D = \sum D^{(i)} \leq R.
    \end{equation*}
    Thus,
    \begin{equation*}
        \E{D^{(20R)}} \leq \frac{\E{D}}{20R} \leq \frac{1}{20}.
    \end{equation*}
    By Markov's inequality,
    \begin{equation*}
        \Pr(D^{(20R)} > 1) \leq \frac{1}{20},
    \end{equation*}
    so that $\outerAlg$ is $\frac{1}{10}$-replicable after applying the union bound on the event that the output strings of $\innerAlg$ differ in more than $R$ positions.
\end{proof}

\section{Conclusion}
In this section, we give some additional motivation for why replicable hypothesis testing (and more generally replicable statistics) is important, as well as raising some interesting open questions left open by our work.

\subsection{Replicable Hypothesis Testing}
Our work is primarily motivated by the lack of replicability in science \cite{baker2016}.
As a concrete example
we examine the $p$-value method, which is a widely used algorithmic framework for conducting null hypothesis significance testing in a wide range of scientific disciplines \cite{thiese2016p}.
In spite of its widespread applications, it is a testing procedure with well documented replicability issues \cite{nuzzo2014statistical, halsey2015fickle}.
Specifically, suppose we have a null hypothesis modeled by the distribution $\distribution_{null}$.
After collecting a sample set $S$ from the true data distribution $\distribution$, we compute a test statistic $\phi(S)$ --- for example a $Z$-score --- and then the corresponding $p$-value
\begin{equation*}
    p(S) = \Pr_{S' \sim \distribution_{null}} (\phi(S') \geq \phi(S))
\end{equation*}
where we can $\reject$ the null hypothesis if $p(S) < p_0$ (typically $p_0 = 0.05$ is a common value in the scientific literature).
We say that the experiment has \emph{power} $q_0$ for some distribution $\distribution'$ if the probability of rejecting a null hypothesis when data is drawn from $\distribution'$ is at least $q_0$.
When $S$ is drawn from the null distribution, a simple calculation yields that $p(S)$ will always be uniformly distributed. Consequently, the power of the experiment for $\distribution_{null}$ is at most $p_0$.
We say a result is \emph{statistically significant} if we reject the null hypothesis.


Unfortunately, random chance alone can cause a single experiment to reject a null hypothesis (causing a false positive). 
In fact, over a sufficiently large number of experiments, we \emph{expect} to reject a large number of null hypotheses.
Decreasing the $p$-value used for significance might decrease the number of published studies where the null hypothesis is true, but would not necessarily increase \emph{replicability} or remove bias. 
As a toy example, say in our experiment, we are taking $100$ samples of a boolean value, fixing a $p$-value of $.01$ to determine significance, and say that each event either has probability $.5$ (the null hypothesis) or $.6$. 
To get a significant result, we need to have the event occur 63 times.
For the nulls, this happens $1\%$ of the time.  
For the non-nulls, this happens $31\%$ of the time.  
However, over a set of $20,000$ experiments split evenly between null and non-null hypotheses, we expect that $100$ of the null experiments produce false positives (and be more likely to be published than the $3,100$ non-null experiments).
While we could use alternative statistical standards, such as error bars or false discovery rate, any deterministic standard based on these criteria would have the same issue and not be guaranteed to replicate (see \cite{impagliazzo2022reproducibility, CMY23}).
In all cases, it seems that the decision of whether some work is significant is a combination of truth (difference between the actual situation and the null hypothesis), and luck (whether the particular data exaggerates or minimizes this difference).

Naturally, we would like the results of all of these experiments to be \emph{simultaneously replicable}.
To this end, replicability requires that with high probability (over both the internal randomness of the algorithm and the data samples) that the result can be replicated.
In this sense, the decision of significance is a function of the random choices and underlying distribution, but not the particular data used.
In this way, the ``luck'' is factored out at the start, leaving only the ``truth'' to determine the outcome.

Furthermore, our replicable algorithms afford a certain degree of protection against malicious actors.
Consider a dishonest scientist who manipulates the experimental procedure or data in order to lower the $p$-value and therefore increase chances of publication.
Guarantees for replicability then imply that another team of scientists can efficiently verify whether the original study followed the claimed procedure, as with high probability, any replicating team should obtain the same conclusion by following the published experimental procedure.

A perhaps more insidious source of error is a dishonest scientist who adversarially chooses some internal randomness $r$ for the algorithm in order to claim statistical significance of their experiment, otherwise following the experimental procedure as claimed.
In this case, it is possible that a null hypothesis that would not have been rejected by choosing a random threshold honestly is in fact rejected. 
However, for each of our algorithms, regardless of the choice of internal randomness $r$, the algorithm is a (non-replicable) deterministic algorithm that solves the $(p_0, q_0)$-hypothesis testing problem.
Thus, it is not possible to fabricate a random string $r$ that allows the scientist to reject a completely incorrect result.
That is, an adversarial choice of random string can only help the scientist publish a result whose correctness is ambiguous, maintaining the desired correctness properties of the experimental procedure.

\subsection{Future Work}

In this paper, we've looked at how to make replicable experiments that simulate a given experiment.  We treat the given experiment as a black box, and assume that scientists are willing and able to repeat the experiment independently several times.  This might be useful in data-rich situations, where scientists can take a large amount of data and partition it randomly into batches and analyse each batch using the same methods. However, in many other scientific situations, this is costly or even impossible.  Data can be very difficult to obtain, analysis tools might involve equipment with limited uses, data from different sources can be differently distributed, and different experiments might involve highly correlated or even identical data.  In future work, we would like to look at different ways that might make replicable hypothesis testing more useful and more robust. 

One possible way to get around the lower bounds we showed is to incorporate replicability earlier in the experimental design process. 
Instead of trying to estimate the power of a given experiment to distinguish a distribution from the null, we might more directly look at the distances between distributions being implicitly tested by the experiment, or specific types of statistics being computed in the experiment.
For example, in the setting of differentially private hypothesis testing, there has been work in constructing differentially private algorithms for linear regression and the $\chi^2$ statistic \cite{DBLP:conf/icml/RogersVLG16, DBLP:conf/nips/AlabiV22, DBLP:conf/icml/KazanSGB23}.
We believe that finding replicable ways of computing such functions is an interesting direction for future research in replicable hypothesis testing.

In our work we have assumed access to i.i.d.\ samples from the underlying distribution $\distribution$.
It would be very interesting to see if algorithms can be replicable if dependence is introduced in the samples.
Our current notion of replicability also only guarantees replicability if different experiments are getting identically distributed data.  
This is not always the case in the scientific literature, so we would like to look at extending the guarantees of identical outcomes to ``close'' distributions of data for two runs, where close is defined via some suitable distance metric between distributions. 

Finally, there are several interesting technical open questions that remain from our work:

\begin{enumerate}
    \item Is there a tiling of $\R^{N}$ where each cell has radius $1$ and surface area to volume ratio $O(N)$ with an efficient membership oracle? 
    As a first step, we ask whether there is a tiling with an efficient membership oracle whose sets have radius $1$ and surface area to volume ratio $O(N^{3/2 - c})$ for any $c > 0$.

    \item While we settle the sample complexity of non-adaptive algorithms for the $N$-Coin problem, we do not resolve the sample complexity of adaptive algorithms.
    Are there adaptive algorithms for the $N$-Coin problem with expected sample complexity $O(N^{2 - c})$ for any $c > 0$?
    In fact, it is not even known if there are algorithms with sample complexity $o(N^2)$.
\end{enumerate}

\section*{Acknowledgements}

We would like to thank Jelena Bradic, Byron Chin, Nicholas Genise, Rex Lei, Shachar Lovett, Toniann Pitassi, Mark Schultz, and Saket Shah for many helpful discussions and thoughtful comments.

\newpage
\bibliographystyle{alpha}
\bibliography{references}

\newpage
\appendix

\section{Stronger Lower Bounds for the Replicable \texorpdfstring{$N$}{N}-Coin Problem}
\label{app:replicable-lb-isoperimetry}

We give two stronger lower bounds mean estimation and the $N$-Coin problem, the first removing the $\delta \leq \rho$ requirement, and the second directly analyzing the $\ell_\infty$ setting for any value of $\varepsilon$.

\subsection{Lower Bounds for Learning Biases}
\label{app:replicable-l-2-lb}

We begin with the lower bound for the $\ell_{2}$ Learning $N$-Coin Problem.

\lTwoLearnNCoinLB*

Above using the assumption $\delta \leq \rho$, we argued that any replicable algorithm induces a partition of the cube $\cube = \left[\frac{1}{4}, \frac{3}{4}\right]^{N}$ such that each part is entirely contained in a ball of radius $2 \varepsilon$.
Removing this assumption, we still show in \Cref{lemma:learn-f-p-l-properties} that the replicable algorithm induces partitions of $\cube$ such that at most a $\delta$-fraction of the points in the partition are not contained in a ball of radius $\varepsilon$.
In particular, a replicable algorithm induces partitions with small diameter, as long as some degree of error is allowed in the partition.

\begin{proof}
    As in \Cref{thm:replicable-implies-isoperimetry}, assume without loss of generality all fixed sample sizes $m_i \leq \frac{2m}{N}$, where $m$ is the total sample complexity of algorithm $\innerAlg$.
    
    Again, define $\cube = \left[ \frac{1}{4}, \frac{3}{4} \right]^{N}$ and define the adversary to sample a bias vector $p \in \cube$ uniformly.
    Fix a good random string $r$ (\Cref{def:good-random-string}).
    
    As before, apply \Cref{lemma:replicable-lipschitz} and define $R = \bigTh{\sqrt{N/m}}$ such that for any $p, q \in \cube$ satisfying $\norm{p - q}{2} \leq R$ and $\hat{p}$,
    \begin{equation*}
        |\Pr_{S_p}(\innerAlg(S_p; r) = \hat{p}) - \Pr_{S_q}(\innerAlg(S_q; r) = \hat{p})| < \frac{1}{15}.
    \end{equation*}
    Note that we set $R = \bigTh{\sqrt{N/m}}$ instead of $R = \bigTh{1/\sqrt{m}}$ since we are now in the coordinate sampling model.
    However, the proof of \Cref{lemma:replicable-lipschitz} holds identically under the assumptions $m_i \leq \frac{2m}{N}$ for all coins $i$.
    
    Then, we use $\innerAlg(; r)$ to define partitions of the cube $\cube$ as in \Cref{def:f-p-hat} for all $\ell \in [0, R]$.

    Whenever $p \in \bigcup_{\hat{p}} \boundary F_{\hat{p}, \ell}$, $\innerAlg(; r)$ is not replicable given samples from $p$ with probability at least $\frac{3}{8}$ since $p \not\in F_{\hat{p}}$ (otherwise $p \in F_{\hat{p}} = \interior(F_{\hat{p}}) \subset \interior(F_{\hat{p}, \ell})$ is not on the boundary). 
    Then
    \begin{equation*}
        \frac{3}{8} \vol \left( \bigcup_{\ell} \bigcup_{\hat{p}} \boundary F_{\hat{p}, \ell} \cap \cube \right) \leq 3 \rho \cdot \vol(\cube).
    \end{equation*}
    Furthermore, the sets $\boundary F_{\hat{p}, \ell}$ are disjoint for all distinct $\hat{p}, \ell$ by applying the following lemma to the set $F_{\hat{p}}$.

    \begin{restatable}{lemma}{fplProperties}
        \label{lemma:learn-f-p-l-properties}
        Let $F_{\hat{p}, \ell} = F_{\hat{p}} + B_{\ell}$ be the collection of subsets indexed by $\hat{p}$ with non-empty $F_{\hat{p}}$.
        Then, the following properties hold:
        \begin{enumerate}
            \item (Disjoint) $\closure \left(F_{\hat{p_1}, \ell} \cap \cube\right) \cap \closure \left(F_{\hat{p_2}, \ell} \cap \cube\right) = \emptyset$ for all $\hat{p_1} \neq \hat{p_2}$.
            \label{item:learn-f-p-l-properties:disjoint}
    
            \item (Large Total Volume) 
            \begin{equation*}
                \vol \left( \bigcup_{\hat{p}} F_{\hat{p}, \ell} \cap \cube \right) = \sum_{\hat{p}} \vol \left( F_{\hat{p}, \ell} \cap \cube \right) \geq (1 - 8 \rho) \vol(\cube).
            \end{equation*}
            \label{item:learn-f-p-l-properties:total-volume}
    
            \item (Small Total Error) 
            \begin{equation*}
                \frac{\vol \left( \bigcup_{\hat{p}} (F_{\hat{p}, \ell} \cap \cube) \setminus B_{\varepsilon}(\hat{p}) \right)}{\vol(\cube)} = \frac{\sum_{\hat{p}} \vol \left( (F_{\hat{p}, \ell} \cap \cube) \setminus B_{\varepsilon}(\hat{p}) \right)}{\vol(\cube)} \leq 5 \delta.
            \end{equation*}
            \label{item:learn-f-p-l-properties:total-error}
    
            \item There are finitely many non-empty $F_{\hat{p}}$ and the number of non-empty $F_{\hat{p}} \cap \cube$ is at least,
            \begin{equation*}
                \frac{(1 - 8 \rho - 5 \delta) \vol(\cube)}{\varepsilon^{N} \vol(B_1)}.
            \end{equation*}
            \label{item:learn-f-p-l-properties:finite-bound}
        \end{enumerate}
    \end{restatable}

    Thus, to lower bound the volume of the left hand side above, it suffices to lower bound the surface area $\boundary F_{\hat{p}, \ell} \cap \cube$.
    \Cref{lemma:learn-f-p-l-properties} shows that the sets satisfy certain properties that force the partition to have large surface area.
    Using these properties, in Lemma \ref{lemma:learn-f-p-l-total-area} we argue that the sets $F_{\hat{p}, \ell}$ have large surface area.

    \begin{restatable}{lemma}{fplTotalArea}
        \label{lemma:learn-f-p-l-total-area}
        Let $\cube = [1/4, 3/4]^{N}$ and $F_{\hat{p}, \ell}$ be a collection of sets satisfying the properties of \Cref{lemma:learn-f-p-l-properties}.
        Then
        \begin{equation*}
            \area \left( \bigcup_{\hat{p}} \boundary F_{\hat{p}, \ell} \cap \cube \right) \geq \left( \frac{1}{4 \varepsilon} - 4 \right) N \vol(\cube).
        \end{equation*}
    \end{restatable}
    
    Since $\boundary F_{\hat{p}, \ell}$ are disjoint for distinct $\hat{p}, \ell$ by \Cref{lemma:learn-f-p-l-properties},
    \begin{align*}
        \vol\left( \bigcup_{\ell} \bigcup_{\hat{p}} \boundary F_{\hat{p}, \ell} \cap \cube \right) &= \int_{0}^{R} \area \left( \bigcup_{\hat{p}} \boundary F_{\hat{p}, \ell} \cap \cube \right) d\ell \\
        &= \bigOm{\frac{R N}{\varepsilon} \vol(\cube)} \\
        &= \bigOm{\frac{N^{3/2}}{\varepsilon \sqrt{m}} \vol(\cube)}
    \end{align*}
    by plugging in $R = \bigTh{\sqrt{N/m}}$ and $\varepsilon < \frac{1}{17}$.
    In particular, whenever $r$ is a good random string, the probability a uniformly chosen $p$ falls into $\boundary F_{\hat{p}, \ell}$ for some $\hat{p}, \ell$ is at least $\bigOm{\frac{N^{3/2}}{\varepsilon \sqrt{m}}}$.
    Since $\innerAlg(; r)$ is replicable with probability at least $1 - 3 \rho$,
    \begin{equation*}
        \frac{N^{3/2}}{\varepsilon \sqrt{m}} = O(\rho)
    \end{equation*}
    or equivalently $m = \bigOm{\frac{N^3}{\varepsilon^2 \rho^2}}$.
    This concludes the proof of Theorem \ref{thm:l-2-learning-n-coin-lb}.
\end{proof}

\paragraph{Surface Area of Partition}
First, we show that the $F_{\hat{p}, \ell}$ sets satisfy certain useful properties.

\fplProperties*

\begin{proof}
    Note that for every $p \in \closure(F_{\hat{p}, \ell})$, $\Pr(\innerAlg(S_{p}; r) = \hat{p}) \geq \frac{3}{5}$, so that $\closure(F_{\hat{p}, \ell})$ are disjoint for distinct $\hat{p}$, thus proving Property \ref{item:learn-f-p-l-properties:disjoint}. 

    We now prove Property \ref{item:learn-f-p-l-properties:total-volume}.
    Suppose $\vol \left( \bigcup_{\hat{p}} F_{\hat{p}} \cap \cube \right) < (1 - 8 \rho) \vol(\cube)$.
    Our adversary chooses $p \in \cube$ uniformly at random, so that $p \not\in \bigcup_{\hat{p}} F_{\hat{p}} \cap \cube$ with probability at least $8 \rho$.
    By Lemma \ref{lemma:complement-not-replicable}, if $p \not\in \bigcup_{\hat{p}} F_{\hat{p}}$, $\innerAlg(; r)$ is not replicable on samples from $p$ with probability at least $\frac{3}{8}$.
    However, since $r$ is a good random string and satisfies replicability, this implies $3 \rho < 3 \rho$, a contradiction.
    The property then follows for all $\ell \geq 0$ from observing $F_{\hat{p}} \subset F_{\hat{p}, \ell}$.

    We now verify property \ref{item:learn-f-p-l-properties:total-error}.
    Suppose $p \in \bigcup_{\hat{p}} F_{\hat{p}, \ell} \setminus B_{\varepsilon}(\hat{p})$.
    Then, $\Pr_{S_p} \left( \innerAlg(S_p; r) = \hat{p} \right) \geq \frac{3}{5}$ but $p \not\in B_{\varepsilon}(\hat{p})$ implies
    \begin{equation*}
        \Pr_{S_p} \left( p \not\in B_{\varepsilon}(\hat{p}) \right) \geq \frac{3}{5}.
    \end{equation*}
    Since $p \in \cube$ is uniformly chosen,
    \begin{align*}
        3 \delta &\geq \Pr_{p \sim \cube, S_p} \left( p \not\in B_{\varepsilon}(\hat{p}) \right) \\
        &\geq \frac{3}{5} \Pr_{p} \left(p \in \bigcup_{\hat{p}} F_{\hat{p}} \setminus B_{\varepsilon}(\hat{p}) \right),
    \end{align*}
    or equivalently,
    \begin{align*}
        \frac{\vol \left( \bigcup_{\hat{p}} (F_{\hat{p}} \cap \cube) \setminus B_{\varepsilon}(\hat{p}) \right)}{\vol(\cube)} = \Pr_{p} \left(p \in \bigcup_{\hat{p}} F_{\hat{p}} \setminus B_{\varepsilon}(\hat{p}) \right) &\leq 5 \delta.
    \end{align*}

    We now prove Property \ref{item:learn-f-p-l-properties:finite-bound}. 
    To show finitely many $F_{\hat{p}}$ are non-empty, note that every non-empty $F_{\hat{p}}$ is separated by at least $d_0 \sqrt{N/m}$ for some constant $d_0$.
    Note that
    \begin{equation*}
        F_{\hat{p}} \cap \cube = \left( F_{\hat{p}} \cap \cube \cap B_{\varepsilon}(\hat{p}) \right) \cup \left( F_{\hat{p}} \cap \cube \setminus B_{\varepsilon}(\hat{p}) \right)
    \end{equation*}
    and by Property \ref{item:learn-f-p-l-properties:disjoint} these are disjoint for distinct $\hat{p}$.
    We note that since $F_{\hat{p}}$ are finite and disjoint, the sums of the volumes is the volume of the unions.
    Combining Properties \ref{item:learn-f-p-l-properties:total-volume} and \ref{item:learn-f-p-l-properties:total-error}, we have
    \begin{align*}
        \sum_{\hat{p}} \vol \left( F_{\hat{p}} \cap \cube \cap B_{\varepsilon}(\hat{p}) \right) &= \sum_{\hat{p}} \vol \left( F_{\hat{p}} \cap \cube \right) - \vol \left( F_{\hat{p}} \cap \cube \setminus B_{\varepsilon}(\hat{p}) \right) \\
        &\geq \left( 1 - 8 \rho - 5 \delta \right) \vol(\cube).
    \end{align*}
    Let $X$ denote the number of $\hat{p}$ with non-empty $F_{\hat{p}, \ell}$ or equivalently non-empty $F_{\hat{p}}$ so that
    \begin{equation*}
        \sum_{\hat{p}} \vol \left( F_{\hat{p}} \cap \cube \cap B_{\varepsilon}(\hat{p}) \right) \leq X \vol \left(  B_{\varepsilon}(\hat{p}) \right) = X \varepsilon^{N} \vol(B_1).
    \end{equation*}
    Combining, we have
    \begin{equation*}
        X \geq \frac{\left( 1 - 8 \rho - 5 \delta \right) \vol(\cube)}{\varepsilon^{N} \vol(B_1)}.
    \end{equation*}
\end{proof}

Then, we show that all sets satisfying the desired properties have non-negligible surface area.
We require the following lemma regarding minimizing multi-variable functions.

\begin{lemma}
    \label{lemma:kkt-minimize}
    Let $C_f, C_1, C_2$ be constants.
    Let $f, g_1, g_2: \R^{d} \mapsto \R$ be defined:
    \begin{align*}
        f(x) &= C_f \sum_{i = 1}^{d} x_i^{N - 1} \\
        g_1(x) &= \sum_{i = 1}^{d} x_i^{N} - C_1 \\
        g_2(x) &= C_2 - \sum_{i = 1}^{d} x_i^{N}.
    \end{align*}
    Then, if $x^*$ minimizes $f$ subject to constraints $g_1, g_2 \leq 0$, $x_{1}^* = x_{2}^* = \dotsc = x_{d}^*$.
\end{lemma}

\begin{proof}
    Note $x^*$ is a local optimum, so it must satisfy the Karush-Kuhn-Tucker conditions.
    In particular, for all $i$, there exist multipliers $\mu_1, \mu_2$ such that
    \begin{align*}
        0 &= C_f (N - 1) (x_i^*)^{N - 2} + N (x_i^*)^{N - 1} (\mu_1 - \mu_2) \\
        x_i^* &= \frac{C_f (N - 1)}{N (\mu_1 - \mu_2)}.
    \end{align*}
\end{proof}

\fplTotalArea*

\begin{proof}
    For ease of exposition, we drop the $\ell$ and write $F_{\hat{p}}$.
    Note that since we only use the properties required by \Cref{lemma:learn-f-p-l-properties}, the proof also goes through for all $F_{\hat{p}, \ell}$.
    
    For each $\hat{p}$, let $r_{\hat{p}}$ be the radius such that $r_{\hat{p}}^{N} \vol(B_1) = \vol(B_{r_{\hat{p}}}) = \vol(F_{\hat{p}} \cap \cube)$.
    Then, the isoperimetric inequality (\Cref{lemma:isoperimetric-inequality}) implies that
    \begin{align*}
        \area(\boundary (F_{\hat{p}} \cap \cube)) &\geq N \vol(F_{\hat{p}} \cap \cube)^{(N - 1)/N} \vol(B_{1})^{1/N} \\
        &= N \vol(F_{\hat{p}} \cap \cube)^{(N - 1)/N} \left( \frac{1}{r_{\hat{p}}^{N}} \vol(B_{r_{\hat{p}}}) \right)^{1/N} \\
        &= \frac{N}{r_{\hat{p}}} \vol(F_{\hat{p}} \cap \cube) \\
        &= N r_{\hat{p}}^{N - 1} \vol(B_1).
    \end{align*}
    Note that there are only finitely many non-zero $r_{\hat{p}}$ by Property \ref{item:learn-f-p-l-properties:finite-bound}.
    In particular, by summing over the finitely many $\hat{p}$ with non-zero $r_{\hat{p}}$, we obtain
    \begin{equation}
        \label{eq:r-hat-p-minimize}
        \sum_{\hat{p}} \area(\boundary (F_{\hat{p}} \cap \cube)) \geq N \vol(B_1) \sum_{\hat{p}} r_{\hat{p}}^{N - 1}.
    \end{equation}
    
    In order to lower bound this quantity, we show that $r_{
    \hat{p}}$ must satisfy certain constraints.
    In particular, by \Cref{lemma:learn-f-p-l-properties} and Property \ref{item:learn-f-p-l-properties:total-error},
    \begin{align*}
        5 \delta \vol(\cube) &\geq \vol \left( \bigcup_{\hat{p}} (F_{\hat{p}} \cap \cube) \setminus B_{\varepsilon}(\hat{p}) \right) \\
        &= \sum_{\hat{p}} \vol \left( (F_{\hat{p}} \cap \cube) \setminus B_{\varepsilon}(\hat{p}) \right) \\
        &\geq \sum_{\hat{p}} \vol \left( F_{\hat{p}} \cap \cube \right) - \vol \left( B_{\varepsilon}(\hat{p}) \right) \\
        &= \vol(B_1) \sum_{\hat{p}} \left( r_{\hat{p}}^{N} - \varepsilon^{N} \right).
    \end{align*}
    Rearranging, we obtain the constraint
    \begin{equation}
        \label{eq:r-hat-p-constraint-correct}
        \sum_{\hat{p}} r_{\hat{p}}^{N} \leq \frac{5 \delta \vol(\cube)}{\vol(B_1)} + \sum_{\hat{p}} \varepsilon^{N}.
    \end{equation}
    From \Cref{lemma:learn-f-p-l-properties} and Property \ref{item:learn-f-p-l-properties:total-volume} we also obtain the constraint
    \begin{equation}
        \label{eq:r-hat-p-constraint-replicable}
        \vol(B_1) \sum_{\hat{p}} r_{\hat{p}}^{N} = \sum_{\hat{p}} \vol \left( B_{r_{\hat{p}}} \right) = \sum_{\hat{p}} \vol \left( F_{\hat{p}} \cap \cube \right) \geq (1 - 8 \rho) \vol(\cube).
    \end{equation}

    Under the constraint of \Cref{eq:r-hat-p-constraint-correct} and \Cref{eq:r-hat-p-constraint-replicable}, we observe that \Cref{eq:r-hat-p-minimize} is minimized when $r_{\hat{p}}$ are all equal by Lemma \ref{lemma:kkt-minimize}. 
    Thus, to lower bound \Cref{eq:r-hat-p-minimize}, we set all $r_{\hat{p}} = r_0$ for the optimizer $r_0$.
    Then
    \begin{align*}
        N \vol(B_1) \sum_{\hat{p}} r_{\hat{p}}^{N - 1} &\geq N \vol(B_1) \sum_{\hat{p}} r_{0}^{N - 1} \\
        &= \frac{N}{r_0} \vol(B_1) \sum_{\hat{p}} r_{0}^{N} \\
        &\geq \frac{N}{r_0} (1 - 8 \rho) \vol(\cube),
    \end{align*}
    where the last inequality uses the constraint of \Cref{eq:r-hat-p-constraint-replicable}.
    
    Let $X$ denote the number of $\hat{p}$ with non-empty $F_{\hat{p}}$.
    Then, since $r_0$ satisfies the constraint from \Cref{eq:r-hat-p-constraint-correct},
    \begin{equation*}
        r_{0}^{N} \leq \frac{5 \delta \vol(\cube)}{ \vol(B_1) X} + \varepsilon^{N}.
    \end{equation*}
    Now, apply Property \ref{item:learn-f-p-l-properties:finite-bound} so that
    \begin{align*}
        X &\geq \frac{\left( 1 - 8 \rho - 5 \delta \right) \vol(\cube)}{\varepsilon^{N} \vol(B_1)} \\
        \frac{5 \delta \vol(\cube)}{\vol(B_1) X} &\leq \frac{5 \delta}{\left( 1 - 8 \rho - 5 \delta \right)} \varepsilon^{N}.
    \end{align*}
    By our assumption $\rho, \delta < \frac{1}{16}$, we can upper bound $r_{0}^{N} \leq 2 \varepsilon^{N}$.
    
    Thus, we upper bound $r_0 \leq 2^{1/N} \varepsilon \leq 2 \varepsilon$.
    Combining all of the above, we obtain the result
    \begin{equation*}
        \sum_{\hat{p}} \area(\boundary (F_{\hat{p}} \cap \cube)) \geq N \vol(B_1) \sum_{\hat{p}} r_{\hat{p}}^{N - 1} \geq \frac{N}{r_0} \frac{\vol(\cube)}{2} \geq \frac{N}{4 \varepsilon} \vol(\cube),
    \end{equation*}
    where we used \Cref{eq:r-hat-p-minimize} and our above bound on $r_0$.

    Now, we apply Lemma \ref{lemma:boundary-w/o-cube} and use Property \ref{item:learn-f-p-l-properties:disjoint} to observe that
    \begin{align*}
        \sum_{\hat{p}} \area(\boundary F_{\hat{p}} \cap \cube) &\geq \sum_{\hat{p}} \area(\boundary (F_{\hat{p}} \cap \cube)) - \area(F_{\hat{p}} \cap \boundary \cube) \\
        &\geq \frac{N}{2 \varepsilon} \vol(\cube) - \area(\boundary \cube) \\
        &= \frac{N}{4 \varepsilon} \vol(\cube) - 4 N \vol(\cube) \\
        &= \left( \frac{1}{4 \varepsilon} - 4 \right) N \vol(\cube),
    \end{align*}
    where $\area(\boundary \cube) = 2 N \frac{1}{2^{N - 1}} = 4 N \vol(\cube)$.
\end{proof}

\subsection{A Direct Lower Bound for the \texorpdfstring{$N$}{N}-Coin Problem}
\label{app:replicable-l-inf-lb-reflection}

We prove a sample complexity lower bound for non-adaptive $\rho$-replicable algorithms solving the $N$-coin problem. 

\begin{theorem}
    \label{thm:n-coin-lower-const-prob}
    Let $p_0 < q_0$ with $p_0, q_0 \in \left( \frac{1}{4}, \frac{3}{4} \right)$.
    Let $\rho < \frac{1}{10}$ and $\delta < \frac{1}{40 \sqrt{N}}$.
    Any non-adaptive $\rho$-replicable algorithm $\innerAlg$ solving the $(p_0, q_0)$ $N$-Coin Problem requires at least $\bigOm{\frac{N^2}{\rho^2 (q_0 - p_0)^2}}$ coordinate samples.
\end{theorem}

Since we can amplify correctness by taking the majority output, we obtain the lower bound for general $\delta$, losing only log factors.

\NCoinLBLogN*

\begin{proof}
    Let $\varepsilon = q_0 - p_0$.
    Suppose for contradiction there is an non-adaptive $\rho$-replicable algorithm $\innerAlg$ with sample complexity $\littleO{\frac{N^2}{\rho^2 \varepsilon^2 \log^{3} N}}$.
    We construct the folowing non-adaptive $\rho$-replicable algorithm $\outerAlg$.

    \paragraph{Non-Adaptive Algorithm $\outerAlg$}
    We run $\innerAlg$ on $T$ independent samples $S_1, \dotsc, S_T$ (and independent random strings) with replicability parameter $\frac{\rho}{T}$.
    Let $\hat{o}_{t}$ be the respective outputs of $\innerAlg$ on samples $S_{t}$.
    Then, we $\accept$ coin $i \in [N]$ if and only if at least half of the outputs $\hat{o}_t$ output $\accept$ for the $i$-th coin.

    First, observe that $\outerAlg$ is replicable, as across two executions of $\outerAlg$, each output $\hat{o}_{t}$ is not replicable with probability at most $\frac{\rho}{T}$.
    By the union bound, all outputs are replicable with probability at least $1 - \rho$.
    Since our final output is a deterministic function (i.e. majority) of the outputs $\hat{o}_{t}$, the final output is replicable.

    We next argue correctness.
    Suppose $p_i \leq p_0$.
    Then, for each $t$, $\hat{o}_{t}$ accepts $i$ with probability at most $\delta < \frac{1}{3}$.
    Since each execution is independent, it follows from standard concentration bounds that the probability that at least half of the vectors $\hat{o}_{t}$ accept the $i$-th coin is at most
    \begin{equation*}
        \Pr\left(\sum_{t = 1}^{T} \hat{o}_{t, i} > \frac{T}{2}\right) < \exp(-T/3).
    \end{equation*}
    A similar argument holds if $p_i \geq q_0$.
    Union bounding over all $i \in [N]$, the probability any coin is decided incorrectly is at most
    \begin{equation*}
        N \exp(- T/3) < \frac{1}{40 \sqrt{N}},
    \end{equation*}
    as long as $T \geq 3 \log 40 N^{3/2} \geq \frac{9}{2} \log 40 N$.
    In particular, the sample complexity of $\outerAlg'$ is
    \begin{equation*}
        T \; \littleO{\frac{N^2 T^2}{\rho^2 \varepsilon^2 \log^{3} N}} = \littleO{\frac{N^2 T^3}{\rho^2 \varepsilon^2 \log^{3} N}} = \littleO{\frac{N^2}{\rho^2 \varepsilon^2}},
    \end{equation*}
    contradicting Theorem \ref{thm:n-coin-lower-const-delta}.
\end{proof}


Thus, it suffices to prove \Cref{thm:n-coin-lower-const-prob}.
For our overview, assume $\varepsilon:=|q_0 - p_0|$ is a constant.
As before, we argue that a replicable algorithm partitions the hypercube $\cube = [p_0, q_0]^{N}$ such that each set in the partition represents the inputs where the algorithm outputs a certain outcome vector $\hat{o} \in \set{\accept, \reject}^{N}$ with high probability.
Since most inputs must have a canonical outcome, this partition must cover at least a $1 - \rho$ fraction of the input space $\cube$. 
Due to its large volume, we apply the isoperimetric inequality and argue that this partition must have large surface area.
Concretely, we show that the surface area of any such partition is at least $\sqrt{N} \vol(\cube)$.
In particular, on the boundary of each part, we bound the probability of the canonical outcome away from $1$, and therefore argue that whenever an input distribution is selected near the boundary the algorithm fails to be replicable on this input.
Using our mutual information bound, we can argue in that algorithms with low sample complexity cannot effectively distinguish input distributions that are similar to one another.
Concretely, if $\norm{p - q}{2} \leq \sqrt{N/m}$, the probability of any outcome can only change by a constant, where $m$ is the total coordinate samples taken by the algorithm.
As a result, the algorithm fails to be replicable on any point within $\sqrt{N/m}$ of the boundary, and such a point is selected with probability $N / \sqrt{m} \leq \rho$, thus obtaining the desired lower bound.

We define the set of distributions that agree or disagree with a certain outcome \newline
$\hat{o} \in \set{\accept, \reject}^{N}$. 
Note that we only define the notion of agreement on the boundary of the cube, since this is the only region where we enforce correctness.

\begin{definition}
    \label{def:bias-outcome-agreement}
    Given an outcome $\hat{o} \in \set{\reject, \accept}^{N}$, a bias vector $p \in \boundary \cube$ \emph{disagrees} with $\hat{o}$ if there is any coin $i \in I(p)$ where $\hat{o}$ is incorrect.
    That is, there is $i$ where $p_i = p_0$ and $\hat{o}_i = \accept$ or $p_i = q_0$ and $\hat{o}_i = \reject$.
    A vector $p \in \boundary \cube$ \emph{agrees} with $\hat{o}$ if it does not disagree with $\hat{o}$.

    Given an outcome $\hat{o}$, let $Y_{\hat{o}} \subset \boundary \cube$ 
    denote the set of input bias vectors that agree with $\hat{o}$. 
\end{definition}

\begin{proof}
    As in \Cref{thm:replicable-implies-isoperimetry}, we assume that the algorithm takes a fixed sample size $m_i \leq \frac{2m}{N}$ for each coordinate, where $m$ is the total coordinate sample complexity.
    Note that in the $N$-coin problem, the algorithm is only required to give correct output for a certain coin if its bias is either $p_0$ or $q_0$.
    In this case we will define a slightly different notion of good random strings, 
    where we impose correctness constraints only on the boundary $\boundary \cube$, rather than the whole cube $\cube = [p_0, q_0]^{N}$.
    \begin{definition}
        \label{def:l-inf-good-random-string}
        A random string $r$ is good if it satisfies the following two conditions:
        \begin{enumerate}
            \item (Boundary Correctness)
            \begin{equation*}
                \Pr_{p \sim \boundary \cube, S_p} \left(p \not\in Y_{\innerAlg(S_p; r)}\right) \leq 3 \delta .
            \end{equation*}
            \item (Global Replicability)
            \begin{equation*}
                \Pr_{p \sim \cube, S_p, S_p'} \left( \innerAlg(S_p; r) \neq \innerAlg(S_p'; r) \right) \leq 3 \rho .
            \end{equation*}
        \end{enumerate}
    \end{definition}
    A similar argument applying Markov's inequality as in \Cref{lemma:non-adapt-learning-n-coin-good} shows that at least $\frac{1}{3}$ fraction of random strings are good.

    \begin{restatable}{lemma}{manyGoodRTest}
        \label{lemma:non-adapt-n-coin-good}
        If $r$ is selected uniformly at random, $r$ is good with probability at least $\frac{1}{3}$.
    \end{restatable}

    Fix a good string $r$.
    Following our previous proofs, apply \Cref{lemma:replicable-lipschitz} and define $R = \bigTh{\sqrt{N/m}}$\footnote{Note that there is now an extra factor of $N$ in the bound since we are working with the coordinate sample complexity.} such that for any $p, q \in \cube$ satisfying $\norm{p - q}{2} \leq R$ and a fixed outcome $\hat{o} \in \set{\accept, \reject}^{N}$,

    \begin{equation*}
        |\Pr_{S_p}(\innerAlg(S_p; r) = \hat{o}) - \Pr_{S_q}(\innerAlg(S_q; r) = \hat{o})| < \frac{1}{15}.
    \end{equation*}

    Then, we use $\innerAlg(; r)$ to define partitions of the cube $\cube$ for all $\ell \in [0, R]$.
    Consider the polynomial
    \begin{equation}
        \label{eq:outcome-prob-polynomial-o}
        h_{\hat{o}}(p) = \sum_{(j_1, j_2, \dotsc, j_{N})} f_{(j_1, j_2, \dotsc, j_{N})} \prod_{i = 1}^{N}  \binom{m_i}{j_i} p_i^{j_i} (1 - p_i)^{m_i - j_i},
    \end{equation}
    where $f_{(j_1, j_2, \dotsc, j_{N})}$ is the proportion of samples with $j_i$ heads on the $i$-th coin on which $\innerAlg(; r)$ outputs $\hat{o}$.
    Since every sample with $j_i$ heads on the $i$-th coin is equally likely, observe that for $p \in \cube$,
    \begin{equation*}
        h_{\hat{o}}(p) = \Pr(\innerAlg(S_p; r) = \hat{o}).
    \end{equation*}
    We then partition $\cube$ into sets of distributions for each canonical outcome $\hat{o}$.

    \begin{definition}
        \label{def:f-outcome}
        Fix random string $r$ and $\hat{o} \in \set{\accept, \reject}^{N}$.
        Define
        \begin{equation*}
            F_{\hat{o}}(r) = \bigset{p \in \R^N \given h_{\hat{o}}(p) \geq \frac{3}{4}},
        \end{equation*}
        where $h_{\hat{o}}$ is defined as in \Cref{eq:outcome-prob-polynomial-o}.
        
        For any $\ell > 0$, define $F_{\hat{o}, \ell}(r) = F_{\hat{o}}(r) + \closure(B_{\ell})$.
        When the random string $r$ is clear, we omit $r$ and write $F_{\hat{o}}, F_{\hat{o}, \ell}$.
    \end{definition}

    Note that by applying \Cref{lemma:complement-not-replicable}, whenever $p \in \bigcup_{\hat{o}} \boundary F_{\hat{o}, \ell} \cap \cube$, $\innerAlg(; r)$ is not replicable given samples from $p$ with probability at least $\frac{3}{8}$ since $p \not\in F_{\hat{o}}$ (otherwise $p \in F_{\hat{o}} = \interior(F_{\hat{o}}) \subset \interior(F_{\hat{o}, \ell})$ is not on the boundary). 
    Then
    \begin{equation*}
        \frac{3}{8} \vol \left( \bigcup_{\ell} \bigcup_{\hat{o}} \boundary F_{\hat{o}, \ell} \cap \cube \right) \leq 3 \rho \cdot \vol(\cube).
    \end{equation*}
    Furthermore, the sets $\boundary F_{\hat{o}, \ell} \cap \cube$ are disjoint for all distinct $\hat{o}, \ell$ by applying the following lemma.

    \begin{restatable}{lemma}{folProperties}
        \label{lemma:f-o-l-properties}
        Let $F_{\hat{o}, \ell}$ be the collection of subsets defined in \Cref{def:f-outcome} for some $\ell \in [0, R]$.
        Then, the following properties hold,
        \begin{enumerate}
            \item (Disjoint) $\left(F_{\hat{o}_1, \ell} \cap \cube\right) \cap \left(F_{\hat{o}_2, \ell} \cap \cube\right) = \emptyset$ for all $\hat{o}_1 \neq \hat{o}_2$.
            \label{item:f-o-l-properties:disjoint}

            \item (Large Total Volume) 
            \begin{equation*}
                \vol \left( \bigcup_{\hat{o}} F_{\hat{o}, \ell} \cap \cube \right) = \sum_{\hat{o}} \vol \left( F_{\hat{o}, \ell} \cap \cube\right) \geq (1 - 8 \rho) \cdot \vol(\cube).
            \end{equation*}
            \label{item:f-o-l-properties:total-volume}

            \item (Small Total Error) 
            \begin{equation*}
                \area \left(\bigcup_{\hat{o}} F_{\hat{o}, \ell} \cap (\boundary \cube \setminus Y_{\hat{o}}) \right) = \sum_{\hat{o}} \area \left(F_{\hat{o}, \ell} \cap (\boundary \cube \setminus Y_{\hat{o}}) \right) \leq \frac{10 \delta N}{\varepsilon} \cdot \vol(\cube),
            \end{equation*}
            \label{item:f-o-l-properties:total-error}
            where $\eps = \abs{q_0 - p_0}$. 

            \item (Orthant Contained) $F_{\hat{o}, \ell} \cap \cube = p_{\hat{o}} + D_{\hat{o}, \ell}$ where $p_{\hat{o}}$ is the vector with entries defined as
            \begin{equation*}
                p_{\hat{o}}[i] = \begin{cases}
                    p_0 & \hat{o}_i = \reject \\
                    q_0 & \hat{o}_i = \accept
                \end{cases}
            \end{equation*}
            and $D_{\hat{o}, \ell}$ lies in one orthant of $\R^{N}$.
            \label{item:f-o-l-properties:orthant}

            \item (Closed) $F_{\hat{o}, \ell}$ are closed for all $\hat{o}, \ell$.
            \label{item:f-o-l-properties:closed}
        \end{enumerate}
    \end{restatable}

    Thus, to lower bound the volume of the left hand side above, it suffices to lower bound the surface area of $\boundary F_{\hat{o}, \ell} \cap \cube$.
    \Cref{lemma:f-o-l-properties} shows that the sets satisfy certain properties that force the partition to have large surface area.
    Using these properties, in Lemma \ref{lemma:f-o-l-total-area} we argue that the sets $F_{\hat{o}, \ell}$ have large surface area.

    \begin{restatable}{lemma}{folTotalArea}
        \label{lemma:f-o-l-total-area}
        Let $\cube = [p_0, q_0]^{N}$ and $F_{\hat{o}, \ell}$ be a collection of sets satisfying the properties of Lemma \ref{lemma:f-o-l-properties}.
        Then
        \begin{equation*}
            \area \left( \bigcup_{\hat{o}} \boundary F_{\hat{o}, \ell} \cap \cube \right) \geq \left( \frac{\sqrt{N}}{2 \varepsilon} - \frac{10 \delta N}{\varepsilon} \right) \vol(\cube).
        \end{equation*}
    \end{restatable}

    Assume that $\delta < \frac{1}{40 \sqrt{N}}$.
    Then we have
    \begin{equation*}
        \area \left( \bigcup_{\hat{o}} \boundary F_{\hat{o}, \ell} \right) \geq \frac{\sqrt{N}}{4 \varepsilon} \cdot \vol(\cube).
    \end{equation*}
    Since $\boundary F_{\hat{o}, \ell}$ are disjoint for distinct $\hat{o}, \ell$ by \Cref{lemma:f-o-l-properties}, 
    \begin{align*}
        \vol\left( \bigcup_{\ell} \bigcup_{\hat{o}} \boundary F_{\hat{o}, \ell} \cap \cube \right) &= \int_{0}^{R} \area \left( \bigcup_{\hat{o}} \boundary F_{\hat{o}, \ell} \cap \cube \right) d\ell \\
        &= \bigOm{\frac{R \sqrt{N}}{\varepsilon} \vol(\cube)} \\
        &= \bigOm{\frac{N}{\varepsilon \sqrt{m}} \vol(\cube)}
    \end{align*}
    by plugging in $R = \bigTh{\sqrt{N/m}}$.
    In particular, whenever $r$ is a good random string, the probability a uniformly chosen $p$ falls into $\boundary F_{\hat{o}, \ell}$ for some $\hat{o}, \ell$ is at least $\bigOm{\frac{N}{\varepsilon \sqrt{m}}}$.
    Since $\innerAlg(; r)$ is replicable with probability at least $1 - 3 \rho$,
    \begin{equation*}
        \frac{N}{\varepsilon \sqrt{m}} = O(\rho)
    \end{equation*}
    or equivalently $m = \bigOm{\frac{N^2}{\varepsilon^2 \rho^2}}$.
    This concludes the proof of \Cref{thm:n-coin-lower-const-prob}.
\end{proof}

We next prove the required lemmas.

\paragraph{Preliminaries and Notation}

Consider the boundary of the hypercube $\cube = [p_0, q_0]^{N}$.
For any fixed point on this boundary $p = (p_1, p_2, \dotsc, p_{N})$, there is some subset $I(p) \subset [N]$ of indices where $p_i \in \set{p_0, q_0}$ for $i \in I(p)$.
For every point on the boundary, $|I(p)| \geq 1$.
Thus, we can partition $\boundary C = \bigcup_{i = 1}^{N} \cube_{(i)}$ where
\begin{equation*}
    \cube_{(i)} = \set{p \in \cube \given |I(p)| = i}.
\end{equation*}

For each outcome $\hat{o}$, let 
$\eps = \abs{q_0 - p_0}$ and
$\cube_{\hat{o}} = \prod_{i = 1}^{N} [0, \sigma_i \varepsilon]$ with $\sigma_i = +1$ if $\hat{o}_i = \reject$ and $\sigma_i = \snew{-1}$ if $\hat{o}_i = \accept$. In other words, we have $\cube = p_{\hat{o}} + \cube_{\hat{o}}$ for all $\hat o$.
Similarly, define $\cube_{\hat{o}, (i)}$ to be the orthant-contained set such that $\cube_{(i)} = p_{\hat{o}} + \cube_{\hat{o}, (i)}$.
Recall that $Y_{\hat{o}}$ is the set of boundary points agreeing with outcome $\hat{o}$.
We also define $Z_{\hat{o}}$ to be the orthant-contained set such that $Y_{\hat{o}} = p_{\hat{o}} + Z_{\hat{o}}$.

Let $D_{\hat{o}, \ell}$ be the set such that $F_{\hat{o}, \ell} \cap \cube = p_{\hat{o}} + D_{\hat{o}, \ell}$.
Note that $D_{\hat{o}, \ell}$ is orthant contained.
We then define the set $G_{\hat{o}, \ell}$ to be the set $p_{\hat{o}} + R(D_{\hat{o}, \ell})$ where for any set $S$ contained in one orthant of $\R^{N}$, $R(S)$ is the set $S$ reflected in all $2^{N}$ quadrants.

Formally, for a fixed vector $\sigma \in \set{\pm 1}^{N}$, define
\begin{equation*}
    R(S; \sigma) = \set{\sigma * x \given x \in S},
\end{equation*}
where $*$ denotes component-wise multiplication.
In particular,
\begin{equation*}
    R(S) = \bigcup_{\sigma} R(S; \sigma) = \set{\sigma * x \given \sigma \in \set{\pm 1}^{N}, x \in S}.
\end{equation*}

\paragraph{Fixing a Good Random String}

\manyGoodRTest*

\begin{proof}
    Note an error occurs whenever $p$ disagrees with $\innerAlg(S_p; r)$.
    By the correctness requirement of $\innerAlg$, given a sample $S_p$ drawn from the distribution parameterized by $p \in \boundary \cube$, 
    \begin{equation*}
        \Ep_{r} \left[ \Pr_{S_p} \left(p \not\in Y_{\innerAlg(S_p; p_r)}\right) \right] = \Pr_{S_p, r} \left(p \not\in Y_{\innerAlg(S_p; p_r)}\right) < \delta
    \end{equation*}
    
    In particular, since this inequality holds for every single $p$ on the boundary, if we also select $p$ on the boundary randomly,
    \begin{equation*}
        \Ep_{p \sim \boundary \cube, r} \left[ \Pr_{S_p} \left(p \not\in Y_{\innerAlg(S_p; p_r)}\right) \right] = \Ep_{r} \left[ \Pr_{p \sim \boundary \cube, S_p} \left(p \not\in Y_{\innerAlg(S_p; p_r)}\right) \right] < \delta.
    \end{equation*}
    By Markov's Inequality,
    \begin{equation*}
        \Pr_{r} \left( \Pr_{p \sim \boundary \cube, S_p} \left(p \not\in Y_{\innerAlg(S_p; p_r)}\right) > 3 \delta \right) < \frac{1}{3}.
    \end{equation*}
    
    Next, since $\innerAlg$ is replicable, for any $p \in [p_0, q_0]^{N}$, consider two independent samples $S_p, S_p'$ drawn from the distribution parameterized by $p$,
    \begin{equation*}
        \Ep_{r} \left[ \Pr(\innerAlg(S_p; r) \neq \innerAlg(S_p'; r) \right] = \Pr_{r, S_p, S_p'} \left( \innerAlg(S_p; r) \neq \innerAlg(S_p'; r) \right) < \rho.
    \end{equation*}
    As before, this bound also holds if $p \in [p_0, q_0]^{N}$ is chosen uniformly at random.
    Again by applying Markov's Inequality,
    \begin{equation*}
        \Pr_{r} \left( \Pr_{p, S_p, S_p'} \left( \innerAlg(S_p; r) \neq \innerAlg(S_p'; r) \right) > 3 \rho \right) < \frac{1}{3}.
    \end{equation*}
    By the union bound, a randomly chosen $r$ satisfies the following conditions simultaneously with probability at least $\frac{1}{3}$.
\end{proof}

In the remainder of the proof, we condition on fixing some good random string and consider the deterministic algorithm $\innerAlg(; r)$.

\paragraph{Surface Area of Partition}

\folProperties*

\begin{proof}[Proof of \Cref{lemma:f-o-l-properties}]
    Note that for every $p \in F_{\hat{o}}$, $\Pr(\innerAlg(S_{p}; r) = \hat{o}) \geq \frac{3}{4}$.
    By our choice of $\ell \leq R$, for any $q \in \closure(F_{\hat{o}, \ell} \cap \cube)$, $\norm{q - p}{2} \leq \ell$ for some $p \in F_{\hat{o}}$, so that $\Pr(\innerAlg(S_{q}; r) = \hat{o}) \geq \frac{3}{5}$.
    We have proved Property \ref{item:f-o-l-properties:disjoint}. 

    Towards Property \ref{item:f-o-l-properties:total-volume} it suffices to show $F_{\hat{o}}$ satisfies the volume constraint since $F_{\hat{o}} \subset F_{\hat{o}, \ell}$.
    Suppose $\vol \left( \bigcup_{\hat{o}} F_{\hat{o}} \cap \cube \right) < (1 - 8 \rho) \cdot \vol(\cube)$.
    Note that for any $p \not\in \bigcup_{\hat{o}} F_{\hat{o}} \cap \cube$, $\innerAlg(; r)$ is not replicable with probability at least $\frac{3}{8}$.
    Then, such a $p$ is chosen with probability at least $8 \rho$, so that the algorithm $\innerAlg(;r)$ is not replicable with probability at least $\frac{3}{8} 8 \rho$. 
    On the other hand, the algorithm $\innerAlg(;r)$ should be non-replicable with probability at most $3 \rho$
    since $r$ is a good random string, a contradiction.

    Towards Property \ref{item:f-o-l-properties:total-error}, since $F_{\hat{o}, \ell} \cap \cube$ are disjoint, we may write
    \begin{equation*}
        \area \left(\bigcup_{\hat{o}} F_{\hat{o}, \ell} \cap (\boundary \cube \setminus Y_{\hat{o}}) \right) = \sum_{\hat{o}} \area \left(F_{\hat{o}, \ell} \cap (\boundary \cube \setminus Y_{\hat{o}}) \right).
    \end{equation*}
    Suppose $p \in F_{\hat{o}, \ell} \cap (\boundary \cube \setminus Y_{\hat{o}})$ for some $\hat{o}$.
    Then, $\Pr_{S_p} \left( \innerAlg(S_p; r) = \hat{o} \right) \geq \frac{3}{5}$ but $p \in \boundary \cube \setminus Y_{\hat{o}}$ implies
    \begin{equation*}
        \Pr_{S_p} \left( p \not\in Y_{\innerAlg(S_p; r)} \right) \geq \frac{3}{5}.
    \end{equation*}
    Since $p \in \cube$ is uniformly chosen, if we condition on $p \in \boundary \cube$, we have
    \begin{align*}
        3 \delta &\geq \Pr_{p \sim \boundary \cube, S_p} \left( p \not\in Y_{\innerAlg(S_p; r)} \right) \\
        &\geq \frac{3}{5} \frac{\area \left(\bigcup_{\hat{o}} F_{\hat{o}, \ell} \cap (\boundary \cube \setminus Y_{\hat{o}}) \right)}{\area(\boundary \cube)},
    \end{align*}
    or equivalently,
    \begin{align*}
        \area \left(\bigcup_{\hat{o}} F_{\hat{o}, \ell} \cap (\boundary \cube \setminus Y_{\hat{o}}) \right) &\leq 5 \delta \area(\boundary \cube) = \frac{10 \delta N}{\varepsilon} \cdot \vol(\cube).
    \end{align*}

    Towards Property \ref{item:f-o-l-properties:orthant}, observe that $\cube = p_{\hat{o}} + \cube_{\hat{o}}$ and $F_{\hat{o}, \ell} \cap \cube \subset p_{\hat{o}} + \cube_{\hat{o}}$ so that $F_{\hat{o}, \ell} \cap \cube = p_{\hat{o}} + D_{\hat{o}, \ell}$ where $D_{\hat{o}, \ell} \subset \cube_{\hat{o}}$ is orthant contained.

    Towards Property \ref{item:f-o-l-properties:closed}, note that $F_{\hat{o}}$ is closed as the preimage of the closed interval $[3/4, \infty)$ under a continuous map $h_{\hat{o}}$.
    We claim $F_{\hat{o}} + \closure(B_{\ell})$ is closed for all $\ell$, which follows as the Minkowski sum of a closed set with a compact set is closed.
\end{proof}

We now show all sets satisfying the desired properties have non-trivial surface area.
We first bound the surface area of $G_{\hat{o}, \ell} = p_{\hat{o}} + R(D_{\hat{o}, \ell})$.
Observe that $G_{\hat{o}, \ell}$ is the set $F_{\hat{o}, \ell} \cap \cube$ reflected in all $2^{N}$ directions.
Since $R(D_{\hat{o}, \ell})$ is the union of $2^N$ copies of $D_{\hat{o}, \ell}$ where the intersection has volume zero (as a subset of $\boundary \cube$), we have $\vol(G_{\hat{o}, \ell}) = 2^{N} \vol(F_{\hat{o}, \ell} \cap \cube)$.
Then, we apply the isoperimetric inequality (Lemma \ref{lemma:isoperimetric-inequality}) to obtain
\begin{equation}
    \label{eq:area-G-o}
    \area(G_{\hat{o}, \ell}) \geq 2 \sqrt{N} \vol(G_{\hat{o}, \ell})^{(N - 1)/N} \geq \frac{\sqrt{N}}{\varepsilon} \vol(G_{\hat{o}, \ell}),
\end{equation}
where we have used $\vol(G_{\hat{o}, \ell}) = 2^{N} \vol(F_{\hat{o}, \ell} \cap \cube) \leq (2\varepsilon)^{N}$ so that $\vol(G_{\hat{o}, \ell})^{(N - 1)/N}$ is at most a factor of $2 \varepsilon$ less than $\vol(G_{\hat{o}, \ell})$.
We can then proceed to bound the surface area of $\boundary F_{\hat{o}, \ell} \cap \cube$.
We will prove our result assuming the following lemmas.

\begin{restatable}{lemma}{surfaceAreaG}
    \label{lemma:surface-area-G}
    For all $\hat{o}$, $\area(\boundary G_{\hat{o}})$ is equal to
    \begin{equation*}
        2^{N} \left( \area\left(\boundary G_{\hat{o}} \cap \interior(\cube)\right) + \area\left(\boundary G_{\hat{o}} \cap (\boundary \cube \setminus Y_{\hat{o}}) \right) \right) + 2^{N - 1} \area\left(\boundary G_{\hat{o}} \cap Y_{\hat{o}} \right).
    \end{equation*}
\end{restatable}

\begin{restatable}{lemma}{interiorEquivalence}
    \label{lemma:interior-C-F-G-equivalence}
    For all $\hat{o}$,
    \begin{equation*}
        \boundary F_{\hat{o}} \cap \interior(\cube) = \boundary G_{\hat{o}} \cap \interior(\cube).
    \end{equation*}
\end{restatable}

\begin{restatable}{lemma}{agreementSubset}
    \label{lemma:agreement-G-subset-F}
    For all $\hat{o}$,
    \begin{equation*}
        \boundary G_{\hat{o}} \cap Y_{\hat{o}} \subseteq \boundary F_{\hat{o}} \cap Y_{\hat{o}}.
    \end{equation*}
\end{restatable}

\begin{restatable}{lemma}{disagreementSubset}
    \label{lemma:disagreement-G-subset-F}
    For all $\hat{o}$,
    \begin{equation*}
        \area\left(\boundary G_{\hat{o}} \cap (\boundary \cube \setminus Y_{\hat{o}})\right) \leq \area\left(F_{\hat{o}} \cap (\boundary \cube \setminus Y_{\hat{o}})\right).
    \end{equation*}
\end{restatable}

\begin{lemma}
    \label{lemma:f-o-total-area}
    Let $\cube = [p_0, q_0]^{N}$ and $F_{\hat{o}, \ell}$ be a collection of sets satisfying the properties of Lemma \ref{lemma:f-o-l-properties}.
    Then
    \begin{equation*}
        \area \left( \bigcup_{\hat{o}} \boundary F_{\hat{o}, \ell} \cap \cube \right) \geq \left( \frac{\sqrt{N}}{2 \varepsilon} - \frac{10 \delta N}{\varepsilon} \right) \vol(\cube).
    \end{equation*}
\end{lemma}

\begin{proof}
    For notational simplicity, we prove the statement for $F_{\hat{o}}$, noting that an identical proof follows for $F_{\hat{o}, \ell}$.
    Fix an outcome vector $\hat{o}$.
    
    We start by relating the surface area of $G_{\hat{o}}$ and $F_{\hat{o}}$.
    To do so, we similarly partition the boundary of $G_{\hat{o}} = p_{\hat{o}} + R(D_{\hat{o}})$.
    Then, Lemma \ref{lemma:surface-area-G} implies
    \begin{align*}
        \area(\boundary G_{\hat{o}}) &= 2^{N} \left( \area\left(\boundary G_{\hat{o}} \cap \interior(\cube)\right) + \area\left(\boundary G_{\hat{o}} \cap (\boundary \cube \setminus Y_{\hat{o}}) \right) \right) + 2^{N - 1} \area\left(\boundary G_{\hat{o}} \cap Y_{\hat{o}} \right) \\
        &\leq 2^{N} \left( \area\left(\boundary F_{\hat{o}} \cap \interior(\cube)\right) + \area\left(F_{\hat{o}} \cap (\boundary \cube \setminus Y_{\hat{o}}) \right) \right) + 2^{N - 1} \area\left(\boundary F_{\hat{o}} \cap Y_{\hat{o}} \right) \\
        &\leq 2^{N} \left( \area\left(\boundary F_{\hat{o}} \cap \cube \right) + \area\left(F_{\hat{o}} \cap (\boundary \cube \setminus Y_{\hat{o}}) \right) \right),
    \end{align*}
    where we applied Lemmas \ref{lemma:interior-C-F-G-equivalence}, \ref{lemma:agreement-G-subset-F} and \ref{lemma:disagreement-G-subset-F} and $\interior(\cube), Y_{\hat{o}}$ are disjoint subsets of $\cube$.
    Then, summing over all $\hat{o}$,
    \begin{align*}
        \sum_{\hat{o}} \area\left(\boundary F_{\hat{o}} \cap \cube \right) &\geq \sum_{\hat{o}} \frac{\area(\boundary G_{\hat{o}})}{2^{N}} - \area\left(F_{\hat{o}} \cap (\boundary \cube \setminus Y_{\hat{o}}) \right),
    \end{align*}
    where we have used all $F_{\hat{o}} \cap \cube$ are closed and disjoint, so all $\boundary F_{\hat{o}} \cap \cube$ are disjoint.
    Then, we use \Cref{eq:area-G-o} and Properties \ref{item:f-o-l-properties:total-volume} and \ref{item:f-o-l-properties:total-error}, so that
    \begin{align*}
        \sum_{\hat{o}} \area\left(\boundary F_{\hat{o}} \cap \cube \right) &\geq \sum_{\hat{o}} \frac{\sqrt{N}}{2^{N} \varepsilon} \vol(G_{\hat{o}}) - \area\left(F_{\hat{o}} \cap (\boundary \cube \setminus Y_{\hat{o}}) \right) \\
        &= \sum_{\hat{o}} \frac{\sqrt{N}}{\varepsilon} \vol(F_{\hat{o}} \cap \cube) - \area\left(F_{\hat{o}} \cap (\boundary \cube \setminus Y_{\hat{o}}) \right) \\
        &\geq \frac{1}{2} \sqrt{N} \varepsilon^{N - 1} - 10 \delta N \varepsilon^{N - 1} \\
        &= \left( \frac{\sqrt{N}}{2 \varepsilon} - \frac{10 \delta N}{\varepsilon} \right) \vol(\cube).
    \end{align*}
\end{proof}

We start by describing how often points are counted in the reflected sets.

\begin{lemma}
    \label{lemma:num-R-sets-containment}
    Let $S \subset \R^{N}$ be contained in some orthant of $\R^{N}$.
    For all $x \in R(S)$ where $x$ has $N - \supp(x)$ zero coordinates, $x \in R(S; \sigma)$ for at most $2^{N - \supp(x)}$ distinct $\sigma \in \set{\pm 1}^{N}$ 

    In particular, for all $\hat{o}$, $0 \leq i \leq N$ and $x \in R\left(\cube_{\hat{o}, (i)}\right)$, $x \in R\left(\cube_{\hat{o}}; \sigma \right)$ for $2^i$ distinct $\sigma \in \set{\pm 1}^{N}$.
\end{lemma}

\begin{proof}
    First, observe that $\sigma * x$ does not change the number of zero's in $x$.
    Suppose $x \in R(S)$ and $x \in R(S; \sigma)$ for $2^{N - \supp(x)} + 1$ distinct $\sigma$.
    Then, $x = \sigma_i * y_i$ for $1 \leq i \leq 2^{N - \supp(x)} + 1$ distinct $\sigma_i$ with $y_i \in S$.
    Consider $\sigma_1$.
    There are at least $N - \supp(x) + 1$ coordinates where $\sigma_1 \neq \sigma_j$ for some $j$.
    However, this implies that $y_1$ (and therefore $x$) must be zero in all these coordinates, a contradiction since $x$ has exactly $N - \supp(x)$ zero coordinates.

    Finally, $x \in R\left(\cube_{(i)}\right)$ implies $x$ has $i$ zero coordinates.
    In particular, $x \in R\left(\cube_{\hat{o}, (i)}; \sigma \right)$ for at most $2^{i}$ distinct $\sigma$ by applying our above result.
    We can argue that $x \in R\left(\cube_{\hat{o}, (i)}; \sigma \right)$ for exactly $2^{i}$ distinct $\sigma$ by fixing the coordinates of $\sigma$ where $x \neq 0$ and varying the $i$ coordinates where $x$ is exactly zero, noting that this will not change the value of $\sigma * y$ where $x = \sigma * y$.
    To conclude the proof, we observe that $x$ has $i$ zero coordinates so that $x \in R\left(\cube_{\hat{o}, (i)}; \sigma \right)$ if and only if $R\left(\cube_{\hat{o}}; \sigma \right)$.
\end{proof}

The following statements relate the boundaries of $G_{\hat{o}}, F_{\hat{o}}$.

\surfaceAreaG*

\begin{proof}
    Since $G_{\hat{o}} = p_{\hat{o}} + R(F_{\hat{o}}) \subset p_{\hat{o}} + R(\cube_{\hat{o}})$, we write $R(\cube_{\hat{o}})$ as
    \begin{align*}
        R(\cube_{\hat{o}}) &= R(\interior(\cube_{\hat{o}})) \cup R(\boundary \cube_{\hat{o}}) \\
        &= R(\interior(\cube_{\hat{o}})) \cup R\left(\cube_{\hat{o}, (1)}\right) \cup \bigcup_{i = 2}^{N} R\left(\cube_{\hat{o}, (i)}\right).
    \end{align*}
    Note that $\cube_{\hat{o}, (i)}$ are disjoint, and since the reflection $R$ does not change the number of coordinates equal to zero, the reflected sets remain disjoint.
    
    We start with $R(\interior(\cube_{\hat{o}}); \sigma)$.
    By Lemma \ref{lemma:num-R-sets-containment}, all reflections are disjoint, i.e.\ $R(\interior(\cube_{\hat{o}}); \sigma_1) \cap R(\interior(\cube_{\hat{o}}); \sigma_2) = \emptyset$ for all $\sigma_1 \neq \sigma_2$.
    Thus, 
    \begin{align*}
        \area(\boundary G_{\hat{o}} \cap (p_{\hat{o}} + R(\interior(\cube_{\hat{o}})))) &= \sum_{\sigma} \area(\boundary G_{\hat{o}} \cap (p_{\hat{o}} + R(\interior(\cube_{\hat{o}}); \sigma))) \\
        &= 2^{N} \area(\boundary G_{\hat{o}} \cap (p_{\hat{o}} + \interior(\cube_{\hat{o}})))  \\
        &= 2^{N} \area(\boundary G_{\hat{o}} \cap \interior(\cube)).
    \end{align*}
    
    Next, we consider $R\left(\cube_{\hat{o}, (1)}\right)$.
    We have previously split $\cube_{(1)}$ into $\cube_{(1)} \cap Y_{\hat{o}}$ and $\cube_{(1)} \setminus Y_{\hat{o}}$.
    Recall, that $Y_{\hat{o}} = p_{\hat{o}} + Z_{\hat{o}}$.
    Then, we can write $\cube_{(1)} \cap Y_{\hat{o}} = p_{\hat{o}} + \left( \cube_{\hat{o}, (1)} \cap Z_{\hat{o}} \right)$ since we are simply translating the sets by $p_{\hat{o}}$.
    Observe that $\cube_{\hat{o}, (1)}$ is exactly the set of $x \in \cube_{\hat{o}}$ such that $N - 1$ coordinates have absolute value in $(0, \varepsilon)$ and one coordinate has absolute value in $\set{0, \varepsilon}$.
    Furthermore, $\cube_{\hat{o}, (1)} \cap Z_{\hat{o}}$ is exactly the set where this one coordinate has absolute value $0$, while $\cube_{\hat{o}, (1)} \setminus Z_{\hat{o}}$ is exactly the set where this one coordinate has absolute value $\varepsilon$.
    
    Now, $R\left( \cube_{\hat{o}, (1)} \cap Z_{\hat{o}} \right) = \bigcup_{\sigma} R\left( \cube_{\hat{o}, (1)} \cap Z_{\hat{o}}; \sigma \right)$.
    By Lemma \ref{lemma:num-R-sets-containment}, $x \in R\left(\cube_{\hat{o}, (1)} \cap Z_{\hat{o}}\right)$ is contained in at most two distinct $R\left(\cube_{\hat{o}, (1)} \cap Z_{\hat{o}}; \sigma \right)$.
    Thus, each $x \in R\left(\cube_{\hat{o}, (1)} \cap Z_{\hat{o}}\right)$ is in at most two distinct $R\left(\cube_{\hat{o}, (1)}; \sigma \right)$.
    In fact, we observe each $x$ is in exactly two such sets (determined by fixing all coordinates of $\sigma$ except where $x = 0$).
    Thus,
    \begin{align*}
        \area\left(\boundary G_{\hat{o}} \cap \left(p_{\hat{o}} + R\left( \cube_{\hat{o}, (1)} \cap Z_{\hat{o}}\right)\right)\right) &= \frac{1}{2} \sum_{\sigma} \area\left(\boundary G_{\hat{o}} \cap \left(p_{\hat{o}} + R\left(\cube_{\hat{o}, (1)} \cap Z_{\hat{o}}; \sigma \right)\right)\right) \\
        &= 2^{N - 1} \area\left(\boundary G_{\hat{o}} \cap \left(p_{\hat{o}} + \left(\cube_{\hat{o}, (1)} \cap Z_{\hat{o}}\right)\right)\right)  \\
        &= 2^{N - 1} \area(\boundary G_{\hat{o}} \cap \left(\cube_{(1)} \cap Y_{\hat{o}}\right)).
    \end{align*}

    By a similar line of reasoning, $R\left( \cube_{\hat{o}, {(1)}} \setminus Z_{\hat{o}} \right) = \bigcup_{\sigma} R\left( \cube_{\hat{o}, (1)} \setminus Z_{\hat{o}}; \sigma \right)$, and furthermore all sets $R\left(\cube_{\hat{o}, (1)} \setminus Z_{\hat{o}}; \sigma \right)$ are disjoint by Lemma \ref{lemma:num-R-sets-containment} since all coordinates have positive absolute value. 
    This implies that
    \begin{align*}
        \area\left(\boundary G_{\hat{o}} \cap \left(p_{\hat{o}} + R\left(\cube_{\hat{o}, (1)} \setminus Z_{\hat{o}}\right)\right)\right) &= \sum_{\sigma} \area\left(\boundary G_{\hat{o}} \cap \left(p_{\hat{o}} + R\left(\cube_{\hat{o}, (1)} \setminus Z_{\hat{o}}; \sigma \right)\right)\right) \\
        &= 2^{N} \area\left(\boundary G_{\hat{o}} \cap \left(p_{\hat{o}} + \left(\cube_{\hat{o}, (1)} \setminus Z_{\hat{o}}\right)\right)\right)  \\
        &= 2^{N} \area(\boundary G_{\hat{o}} \cap \left(\cube_{(1)} \setminus Y_{\hat{o}}\right)).
    \end{align*}

    Finally, for all $i > 1$, since all $\cube_{\hat{o}, (i)}$ are at most $(N - 2)$-dimensional they have no $(N - 1)$-dimensional volume, so they cannot contribute to $\area(\boundary G_{\hat{o}})$.
    
    In summary, we have that
    \begin{align*}
        \area\left(\boundary G_{\hat{o}} \cap \left(p_{\hat{o}} + R\left(\cube_{\hat{o}, (1)} \cap Z_{\hat{o}}\right)\right)\right) &= 
        2^{N - 1} \area(\boundary G_{\hat{o}} \cap \left(\cube_{(1)} \cap Y_{\hat{o}}\right)) \\
        &= 2^{N - 1} \area\left(\boundary G_{\hat{o}} \cap Y_{\hat{o}}\right) \, ,
    \end{align*}
    where the equality follows from the fact that $Y_{\hat o}$ is supported on $\boundary \cube = \bigcup_{i=1}^N \cube_{(i)}$ but only subsets of $\cube_{(1)}$ can have non-trivial $(N-1)$-dimension volume,
    and
    \begin{align*}
        \area\left(\boundary G_{\hat{o}} \cap \left(p_{\hat{o}} + R\left(\cube_{\hat{o}, (1)} \setminus Z_{\hat{o}}\right)\right)\right) &= 2^{N} \area\left(\boundary G_{\hat{o}} \cap \left(\cube_{(1)} \setminus Y_{\hat{o}}\right)\right) \\
        &= 2^{N} \area\left(\boundary G_{\hat{o}} \cap \left(\boundary \cube \setminus Y_{\hat{o}}\right)\right).
    \end{align*}
    Summing up the contributions gives the desired bound.
\end{proof}

Our next lemmas relate the surface area of $G_{\hat{o}}$ in $\cube$ to the surface area of $F_{\hat{o}}$ in $\cube$.

\interiorEquivalence*

\begin{proof}
    Note $G_{\hat{o}} = p_{\hat{o}} + R(D_{\hat{o}})$ and $\interior(\cube) = \cube_{(0)} = p_{\hat{o}} + \cube_{\hat{o}, (0)}$.
    Then, if $x \in \interior(\cube)$, by Lemma \ref{lemma:num-R-sets-containment}, we have $x \in p_{\hat{o}} + R(\cube_{\hat{o}, (0)}, \sigma)$ only if $\sigma = (1, \dotsc, 1)$.
    In particular, $F_{\hat{o}} \cap \interior(\cube) = G_{\hat{o}} \cap \interior(\cube)$.
    We further claim that $\boundary F_{\hat{o}} \cap \interior(\cube) = \boundary G_{\hat{o}} \cap \interior(\cube)$.
    In particular, it suffices to show that
    $\boundary F_{ \hat o }  \cap \interior(\cube) \setminus \boundary G_{\hat o} = \emptyset$, and that
    $\boundary G_{\hat o} \cap \interior(\cube) \setminus \boundary F_{\hat o} = \emptyset$.

    Suppose $x \in \interior(\cube) \setminus \boundary G_{\hat{o}}$.
    Then, $x \in \interior(\cube) \cap \interior(G_{\hat{o}})$ or $\interior(\cube) \cap \exterior(G_{\hat{o}})$.
    In the former case, there is a neighborhood around $x$ in $G_{\hat{o}} \cap \interior(\cube) = F_{\hat{o}} \cap \interior(\cube) \subset F_{\hat{o}}$.
    Therefore $x \in \interior(F_{\hat{o}})$ so $x \not \in \boundary F_{\hat{o}}$.
    In the latter case we have $x \in \interior(\cube) \cap \exterior(F_{\hat{o}}) \subseteq \exterior(F_{\hat{o}})$, implying that $x \not \in \boundary F_{\hat o}$.
    Hence, we have that $\boundary F_{ \hat o }  \cap \interior(\cube) \setminus \boundary G_{\hat o} = \emptyset$.

    Suppose $x \in \interior(\cube) \setminus \boundary F_{\hat{o}}$.
    Again, we have either 
    $x \in \interior(\cube) \cap  \interior \lp( F_{\hat{o}} \rp)$
    or $x \in \interior(\cube) \cap \exterior \lp( F_{\hat{o}} \rp)$
    If $x \in \interior(\cube) \cap \interior(F_{\hat{o}})$ then $x \in \interior(G_{\hat{o}})$ as above, 
    implying that $x \not \in \boundary( G_{\hat o} )$.
    In the latter case, 
    we have $x \in \exterior(G_{\hat o})$ as above, also implying that $x \not \in \boundary(G_{\hat o})$.
\end{proof}

\agreementSubset*

\begin{proof}
    Let $x \in \boundary G_{\hat{o}} \cap Y_{\hat{o}}$.
    Since $x \in Y_{\hat{o}}$, let $i$ denote the unique coordinate where $x[i] \in \set{p_0, q_0}$.
    Without loss of generality, assume $x[i] = p_0$.
    
    Suppose $x \not\in \boundary F_{\hat{o}}$.
    If $x \in \interior(F_{\hat{o}})$, then there is a neighborhood $B_{r}(x)$ around $x$ in $F_{\hat{o}}$.
    In particular, $B_{r}(x) \cap \cube = \set{y \in B_r(x) \given y[i] \geq p_0}$.
    However, $B_{r}(x) \subset G_{\hat{o}}$, since for $y \in B_r(x)$ with $y[i] < p_0$, then $y \in p_{\hat{o}} + R\left(D_{\hat{o}}, \sigma\right) \subset G_{\hat{o}}$ where $\sigma$ has $-1$ only in the $i$-th coordinate.
    This shows that $x \not \in \boundary(G_{\hat o})$.

    If $x \in \exterior(F_{\hat{o}})$, then there is a neighborhood $x \in B_r(x) \subset \R^N \setminus F_{\hat{o}}$.
    By a similar argument, $B_r(x) \subset \R^N \setminus G_{\hat{o}}$,
    showing that $x \not \in \boundary G_{\hat o}$.
    This concludes the proof.
\end{proof}

\disagreementSubset*

\begin{proof}
    Since $\boundary \cube = \bigcup_{i = 1}^{N} \cube_{(i)}$ and $\area\left(\cube_{(i)}\right) = 0$ for all $i \geq 2$, it suffices to show
    \begin{equation*}
        \boundary G_{\hat{o}} \cap \left(\cube_{(1)} \setminus Y_{\hat{o}}\right) \subseteq F_{\hat{o}} \cap \left(\cube_{(1)} \setminus Y_{\hat{o}}\right),
    \end{equation*}
    or equivalently after translation by $p_{\hat{o}}$,
    \begin{equation*}
        \boundary R(D_{\hat{o}}) \cap \left(\cube_{\hat{o}, (1)} \setminus Z_{\hat{o}}\right) \subseteq  D_{\hat{o}} \cap \left(\cube_{\hat{o}, (1)} \setminus Z_{\hat{o}}\right).
    \end{equation*}
    Again, we observe that $x \in \cube_{\hat{o}, (1)} \setminus Z_{\hat{o}}$ implies that every coordinate is non-zero. 
    By Lemma \ref{lemma:num-R-sets-containment}, $x \in R(\cube_{\hat{o}}, \sigma)$ for a unique $\sigma = (1, 1, \dotsc, 1)$.
    Then, if $x \notin D_{\hat{o}}$, since $D_{\hat{o}}$ is closed, there is a neighborhood around $x$ disjoint from $D_{\hat{o}}$.
    We claim this neighborhood is also disjoint from $R(D_{\hat{o}})$.
    In particular
    \begin{equation*}
        B_r(x) = (B_r(x) \cap \cube_{\hat{o}}) \cup (B_r(x) \setminus \cube_{\hat{o}}),
    \end{equation*}
    where $r$ is small enough so that no point in $B_r(x)$ has any zero coordinates.
    We first consider $B_r(x) \cap \cube_{\hat{o}}$.
    As before, any point $y \in B_r(x) \cap \cube_{\hat{o}}$ is in $R(\cube_{\hat{o}}, \sigma)$ for unique $\sigma = (1, 1, \dotsc, 1)$.
    Therefore, if $y \not\in D_{\hat{o}}$, $y \not\in R(D_{\hat{o}})$.
    Finally, any point in $y \in R(D_{\hat{o}})$ satisfies $\norm{y}{\infty} \leq \varepsilon$, which is not true of $B_r(x) \setminus \cube_{\hat{o}}$.
\end{proof}

\section{Omitted Proofs}
\label{app:omitted-proofs}

\subsection{A Simple Hypothesis Testing Algorithm}
\label{app:missing-proofs-hypothesis-testing}

We provide for completeness the proof of Theorem \ref{thm:r-stat-hypothesis-testing}.

\rstathypothesistest*

\IncMargin{1em}
\begin{algorithm}
\SetKwInOut{Input}{Input}\SetKwInOut{Output}{Output}\SetKwInOut{Parameters}{Parameters}
\Input{Sample access to distribution $\distribution$ on $[0, 1]$. Probability thresholds $p_0 < q_0$ with $\varepsilon = q_0 - p_0$.}
\Parameters{$\rho$ replicability and $\delta$ accuracy}
\Output{$\failtoreject$ if $\distribution$ is uniform and $\reject$ if $\Pr_{X \sim \distribution}(X < p_0) \geq q_0$}

\BlankLine

$\tau \gets \frac{q_0 - p_0}{3}$

$\hat{p} \gets \mathrm{rSTAT}_{\tau, \rho, \delta}(\distribution, \phi(p) = \chi(p \leq p_0))$

\If{$\hat{p} \leq \frac{p_0 + q_0}{2}$}{
    \Return $\failtoreject$
}
\Else{
    \Return $\reject$
}

\caption{$\simpleHypothesisTester(\distribution, p_0, q_0, \rho, \delta)$} 
\label{alg:r-hypothesis-testing}

\end{algorithm}
\DecMargin{1em}

As a sub-routine, we require the statistical query oracle of \cite{impagliazzo2022reproducibility}.

\begin{lemma}
    \label{lemma:ilps-stat-q}
    Let $\tau, \delta, \rho > 0$ and $\rho > 4 \delta$.
    Then, $\mathrm{rSTAT}_{\tau, \rho, \delta}$ is a $\rho$-replicable statistical query oracle with sample complexity
    \begin{equation*}
        \bigO{\frac{1}{\tau^2 \rho^2} \log \frac{1}{\delta}} = \bigtO{\frac{1}{\tau^2 \rho^2}}.
    \end{equation*}
\end{lemma}

We now prove our theorem.

\begin{proof}[Proof of Theorem \ref{thm:r-stat-hypothesis-testing}]
    Suppose $\distribution$ is uniform.
    Then, $\E{\phi(p)} = p_0$.
    In particular, $\hat{p} \leq p_0 + \tau$ with probability at least $1 - \delta$, as desired.
    Otherwise, suppose $\distribution$ is at most $p_0$ with probability at least $q_0$.
    Then, $\E{\phi(p)} \geq q_0$ so that $\hat{p} \geq q_0 - \tau$ with probability at least $1 - \delta$.
    Replicability and sample complexity follow directly from Lemma \ref{lemma:ilps-stat-q}.
\end{proof}

\subsection{Testing and Learning Biases}
\label{app:testing-learning-equivalence}

\testBinarySearchLearn*

\begin{proof}[Proof of \Cref{lemma:l-inf-tester-implies-learner}]
    We construct an algorithm that proceeds by simultaneous binary search.
    Divide $[0,1]$ into $\ceil{\frac{2}{\varepsilon}}$ sub-intervals of length $\frac{\varepsilon}{4} \leq \varepsilon' \leq \frac{\varepsilon}{2}$. For every coin $i$, we keep a set of `alive' intervals, starting with the full set.
    At each stage, the algorithm selects for each coin $i$ the median interval of length $\varepsilon$ labelled $(a_i, b_i)$ and hopes to determine whether coin $i$ has bias at most $a_i$ or at least $b_i$.
    If there are two median intervals, choose one arbitrarily.
    After determining which side of the interval the true bias lies on, we remove the current interval and all intervals on the wrong side from the set of ``alive" intervals (so that the number of ``alive" intervals strictly decreases). 
    This procedure terminates when there is at most one interval.
    If there is exactly one interval remaining, return the midpoint of the interval as an estimate for the bias of the coin.
    Otherwise if there are no intervals remaining, note that this implies there were two intervals in the previous iteration, we return the common endpoint of the two intervals.
    Finally, the binary search must terminate after $O(\log(1/\varepsilon))$ iterations. 
    We will ensure that with high probability, each execution of the $N$-Coin problem is correct, and therefore the biases are learned up to error at most the width of two intervals, or $\varepsilon$.

    It remains to describe how to execute one iteration. 
    Fix a coin $i$ and let $(a_i, b_i)$ be the median interval (with ties broken arbitrarily).
    Note that the $(a_i,b_i)$ may be different for each coordinate, so we cannot apply the $N$-coin problem directly to samples from the distribution. 
    Instead, we shift the biases by reflipping coins based on the observed samples, so that it suffices to test the interval $\left(\frac{1}{2} - \frac{\varepsilon'}{4}, \frac{1}{2} + \frac{\varepsilon'}{4} \right)$.
    In each iteration, after performing the appropriate bias shift (a different one for each coin), we execute the $N$-Coin algorithm for the on the coins with shifted biases.
    We show that solving the $N$-Coin problem on the shifted biases is equivalent to deciding which side of the interval the true bias of the coin lies on. 

    We now describe the bias shift.
    Suppose $(a, b)$ is the interval in the current iteration for coin $i$.
    Without loss of generality, assume $a \leq \frac{1 - \varepsilon'}{2}$ (otherwise we can flip the results of the coin flips).
    We transform the bias of the given coin as follows.
    If we sample $H$ from coin $i$, keep the sample as $H$ with probability $h$.
    If we sample $T$ from coin $i$, flip to $H$ with probability $t$.
    Fix an integer $k \geq 1$.
    We choose $h, t$ such that
    \begin{align*}
        \frac{1}{2} - \frac{\varepsilon'}{4} &= a h + (1 - a) t \\
        \frac{1}{2} + \frac{\varepsilon'}{4} &= b h + (1 - b) t.
    \end{align*}
    Solving,
    \begin{align*}
        \frac{(2 - \varepsilon') b}{4} &= a b h + (1 - a) b t \\
        \frac{(2 + \varepsilon') a}{4} &= a b h + a (1 - b) t \\
        \frac{b - a}{2} - \varepsilon' \frac{a + b}{4} &= ((1 - a) b - a (1 - b)) t = \varepsilon t \\
        t &= \frac{1}{2} - \frac{(a + b)}{4},
    \end{align*}
    where we have used $b - a = \varepsilon'$.
    Then
    \begin{align*}
        h &= \frac{1}{b} \left( \frac{1}{2} + \frac{\varepsilon'}{4} - (1 - b) \left(\frac{1}{2} - \frac{a + b}{4} \right) \right) \\
        &= \frac{1}{b} \left( b - \frac{b(a + b)}{4} \right) \\
        &= 1 - \frac{a + b}{4}.
    \end{align*}
    Of course, we require that $h, t \in (0, 1)$.
    Given the above two equations, we note that $a + b < 2$ so that there exist $h,  t \in (0, 1)$, as desired.
    
    Thus, any coin with bias at most $a$ has bias at most $\frac{1}{2} - \frac{\varepsilon'}{4}$ after the transformation, and any coin with bias at least $b$ has bias at least $\frac{1}{2} + \frac{\varepsilon'}{4}$ after the transformation.
    In particular, since the tester correctly decides if the transformed coin has bias at most $\frac{1}{2} - \frac{\varepsilon'}{4}$ or at least $\frac{1}{2} + \frac{\varepsilon'}{4}$, then the tester correctly determines if the original coin has bias at most $a$ or at least $b$.

    In each iteration, for some sufficiently large constant $C$ we will execute a $\frac{\rho}{C \log (1/\varepsilon)}$-replicable algorithm for the $\left( \frac{1}{2} - \frac{\varepsilon}{16}, \frac{1}{2} - \frac{\varepsilon}{16} \right)$ $N$-Coin problem (note $\frac{\varepsilon}{16} < \frac{\varepsilon'}{4}$) with error $\frac{\delta}{C \log (1/\varepsilon)}$, which has sample complexity 
    \begin{equation*}
        f\left( N, \frac{1}{2} - \frac{\varepsilon}{16}, \frac{1}{2} + \frac{\varepsilon}{16}, \frac{\rho}{C \log(1/\varepsilon)}, \frac{\delta}{C \log(1/\varepsilon)} \right).
    \end{equation*}

    Overall correctness and replicability follows from union bounding over the $O(\log 1/\varepsilon)$ above executions, so that each execution of the $N$-Coin algorithm is correct and replicable.
    To conclude, we note that the sample complexity is
    \begin{equation*}
        \bigO{\log \frac{1}{\varepsilon} \cdot f\left( N, \frac{1}{2} - \varepsilon, \frac{1}{2} + \varepsilon, \frac{\rho}{\log(1/\varepsilon)}, \frac{\delta}{\log(1/\varepsilon)} \right)}.
    \end{equation*}.
\end{proof}

\subsection{From Vector Samples to Arbitrary Non-Adaptive Algorithms}
\label{sub-sec:vector-to-non-adaptive}

\productPartition*

\begin{proof}[Proof of \Cref{prop:product-partition}]
    Let $V$ denote the subspace spanned by the first $N/2$ coordinates.
    Let $\partition_1, \partition_2$ be partitions of hypercubes $\cube_1$ (reps. $\cube_2$) be the $N/2$-dimensional hypercube $\left[ \frac{1}{2} - \boxwidth \varepsilon, \frac{1}{2} + \boxwidth \varepsilon \right]^{N/2}$ with $0$ on the last $N/2$ (resp. first $N/2$) coordinates.
    Then $\partition' = \partition \oplus \partition = \partition_1 + \partition_2$.
    Note $\cube_1 \subset V, \cube_2 \subset V^{\perp}$.
    Let $\cube = \left[ \frac{1}{2} - \boxwidth \varepsilon, \frac{1}{2} + \boxwidth \varepsilon \right]^{N}$ so $\cube = \cube_1 + \cube_2$.
    
    Toward Property \ref{item:approximate-cube-partition:disjoint}, suppose $x \in (S_1 + T_1) \cap (S_2 + T_2)$ so that,
    \begin{equation*}
        s_1 + t_1 = s_2 + t_2
    \end{equation*}
    for $s_i \in S_i$ and $t_i \in T_i$.
    Then, for any vector $v \in V$, $v \cdot (s_i + t_i) = v \cdot s_i$.
    Since $s_i \in V$, we have $s_1 = s_2$.
    Thus, $S_1 = S_2$ since $s_i \in S_i$ and the $S_i$ are disjoint.
    Similarly, $T_1 = T_2$.

    Toward Property \ref{item:approximate-tiling:non-zero-volume}, observe $\vol(S + T) = \volF{N/2}(S) \volF{N/2}(T) > 0$.

    Toward Property \ref{item:approximate-cube-partition:semialgebraic}, we claim that $S + T$ is semialgebraic by \Cref{lemma:minkowski-sum-semialgebraic}.

    Toward Property \ref{item:approximate-cube-partition:approx-volume}, we show that the product partition has large volume in the $N$-dimensional cube.
    Since the partition $\partition$ is finite and $S + T$ are disjoint,
    \begin{align*}
        \vol\left( \bigcup_{S, T} S + T \right) &= \sum_{S, T} \vol(S + T) \\
        &= \sum_{S, T} \volF{N/2}(S) \volF{N/2}(T) \\
        &= \left( \sum_{S} \volF{N/2} (S) \right) \left( \sum_{T} \volF{N/2} (T) \right) \\
        &\geq (1 - \rho) \vol(\cube_{1}) (1 - \rho) \vol(\cube_{2}) \\
        &\geq (1 - 2 \rho) \vol(\cube).
    \end{align*}

    Toward Property \ref{item:approximate-cube-partition:radius}, we equip $S_{v} + T_{u}$ with the label $v + u$.
    Then, for any $x \in S_{v} + T_{u}$, where $x = s + t$ with $s \in S_{v}$ and $t \in T_{u}$,
    \begin{equation*}
        \norm{(s + t) - (v + u)}{2} \leq \norm{s - v}{2} + \norm{t - u}{2} \leq 2 \varepsilon.
    \end{equation*}

    Toward Property \ref{item:approximate-cube-partition:surface-area}, we first show
    \begin{align*}
        \boundary(S + T) \cap \cube = \boundary(S + T) \cap \cube = ((\boundary S + T) \cup (S + \boundary T)) \cap \cube = ((\boundary S + T) \cap \cube) \cup ((S + \boundary T) \cap \cube).
    \end{align*}
    Since $\area((\boundary S + T) \cap (S + \boundary T)) = 0$, we have
    \begin{align*}
        \area(\boundary(S + T) \cap \cube) &= \area((\boundary S + T) \cap \cube) + \area((S + \boundary T) \cap \cube) \\
        &= \area((\boundary S \cap \cube_1) + (T \cap \cube_2)) + \area((S \cap \cube_1) + (\boundary T \cap \cube_2)) \\
        &= \volF{N/2 - 1}(\boundary S \cap \cube_1) \volF{N/2}(T \cap \cube_2) + \volF{N/2}(S \cap \cube_1) \volF{N/2 - 1}(\boundary T \cap \cube_2),
    \end{align*}
    using \Cref{lemma:orthogonal-minkowki-sum-intersect} in the second inequality.
    In the third inequality, we note that the volume of the sum of two orthogonal sets is the product of the volumes.
    In particular, by writing $\vol(\cube) = \volF{N/2}(\cube_1) \volF{N/2}(\cube_2)$, we factorize the sums
    \begin{align*}
        \sum_{S, T} \frac{\area(\boundary(S + T) \cap \cube)}{\vol(\cube)} &= \sum_{S, T} \frac{\volF{N/2 - 1}(\boundary S \cap \cube_1) \volF{N/2}(T \cap \cube_2)}{\vol(\cube)} \\
        &+ \sum_{S, T} \frac{\volF{N/2}(S \cap \cube_1) \volF{N/2 - 1}(\boundary T \cap \cube_2)}{\vol(\cube)} \\
        &= \left( \sum_{S} \frac{\volF{N/2 - 1}(\boundary S \cap \cube_1)}{\volF{N/2}(\cube_1)} \right) \left( \sum_{T} \frac{\volF{N/2}(T \cap \cube_2)}{\volF{N/2}(\cube_2)} \right) \\
        &+ \left( \sum_{S} \frac{\volF{N/2}(S \cap \cube_1)}{\volF{N/2}(\cube_1)} \right) \left( \sum_{T} \frac{\volF{N/2 - 1}(\boundary T \cap \cube_2)}{\volF{N/2}(\cube_2)} \right) \\
        &\leq \left( \sum_{S} \frac{\volF{N/2 - 1}(\boundary S \cap \cube_1)}{\volF{N/2}(\cube_1)} \right) + \left( \sum_{T} \frac{\volF{N/2 - 1}(\boundary T \cap \cube_2)}{\volF{N/2}(\cube_2)} \right) \\
        &\leq 2 A.
    \end{align*}

    Finally, note that the membership oracles can easily be queried on each half of the coordinates, thus obtaining a membership oracle for the product partition $\partition'$.
\end{proof}

\begin{lemma}
    \label{lemma:orthogonal-minkowki-sum-intersect}
    Let $V, \cube_1, \cube_2$ be defined as in the proof of \Cref{prop:product-partition}.
    For $A \subset V$ and $B \subset V^{\perp}$, 
    \begin{equation*}
        (A + B) \cap \cube = A \cap \cube_1 + B \cap \cube_2.
    \end{equation*}
\end{lemma}

\begin{proof}
    Suppose $z \in (A + B) \cap \cube$. 
    Then, $z = z_1 + z_2$ with $z_1 \in A, z_2 \in B$.
    Furthermore, $z_1$ (resp. $z_2$) is zero on the last $N/2$ (resp. first $N/2$) coordinates.
    Then, since $z \in \cube$ we have $z_1 \in \cube_1, z_2 \in \cube_2$ so that $z_1 \in A \cap \cube_1$ and $z_2 \in B \cap \cube_2$.

    Now, suppose $z \in A \cap \cube_1 + B \cap \cube_2$ so that $z = z_1, z_2$ with $z_1 \in A \cap \cube_1, z_2 \in B \cap \cube_2$.
    Then, since $z_1$ (resp. $z_2$) is zero on the last $N/2$ (resp. first $N/2$) coordinates, $z \in \cube$.
    We conclude by observing that $z \in A + B$.
\end{proof}


\end{document}